\documentclass{article}
\usepackage{kr}

\usepackage{times}
\usepackage{soul}
\usepackage{url}
\usepackage[hidelinks]{hyperref}
\usepackage[utf8]{inputenc}
\usepackage[small]{caption}
\usepackage{graphicx}
\usepackage{amsmath}
\usepackage{amsthm}
\usepackage{booktabs}
\usepackage{algorithm}
\usepackage{algorithmic}
\newtheorem{example}{Example}
\newtheorem{theorem}{Theorem}
\newtheorem{proposition}{Proposition}
\newtheorem{lemma}{Lemma}
\newtheorem{claim}{Claim}

\newtheorem{definition}{Definition}
\newtheorem{remark}{Remark}

\usepackage{thmtools}
\usepackage{thm-restate}
\usepackage{amsfonts}
\usepackage{xspace}
\usepackage{pifont}
\usepackage{color}
\usepackage{amsmath, upgreek, stmaryrd}
\usepackage{amssymb}
\usepackage{misc}   
\usepackage{xargs}
\usepackage[pdftex,dvipsnames]{xcolor}
\usepackage{graphicx}
\usepackage{mathtools}
\usepackage[capitalise]{cleveref}
\usepackage{enumitem}

\usepackage{compat-quentin}

\newlist{properties}{enumerate}{1}
\setlist[properties,1]{label=(\alph*)}

\creflabelformat{propertiesi}{#2#1#3}

\crefname{propertiesi}{Property}{Properties}

\creflabelformat{observationi}{#2#1#3}

\crefname{observationi}{Observation}{Observations}

\title{Querying Circumscribed Description Logic Knowledge Bases}

\author{%
Carsten Lutz$^1$\and
Quentin Mani\`ere$^{1, 2}$\and
Robin Nolte$^3$\\
\affiliations
$^1$Department of Computer Science, Leipzig University, Germany \\
$^2$Center for Scalable Data Analytics and Artificial Intelligence (ScaDS.AI), Dresden/Leipzig, Germany \\
$^3$University of Bremen, Digital Media Lab, Germany \\
\emails
\{carsten.lutz,
quentin.maniere\}@uni-leipzig.de,
nolte@uni-bremen.de
}

\begin{document}

\maketitle


\begin{abstract}
  Circumscription is one of the main approaches for defining
  non-monotonic description logics (DLs) and the decidability and
  complexity of traditional reasoning tasks, such as satisfiability of
  circumscribed DL knowledge bases (KBs), are well understood. For
  evaluating conjunctive queries (CQs) and unions thereof (UCQs), in
  contrast, not even decidability has been established.  In this
  paper, we prove decidability of (U)CQ evaluation on circumscribed DL
  KBs and obtain a rather complete picture of both the combined
  complexity and the data complexity for DLs ranging from \ALCHIO via
  \EL to various versions of DL-Lite. We also study the much simpler
  atomic queries (AQs).
\end{abstract}

\section{Introduction}
\label{sec-introduction}

While standard description logics (DLs), such as those underlying the OWL 2 ontology language, do not support non-monotonic reasoning, it is generally acknowledged that extending DLs with non-monotonic features is very useful. Concrete examples of
applications include ontological modeling in the biomedical domain \cite{DBLP:conf/psb/Rector04,DBLP:journals/ijmms/StevensAWSDHR07} and the formulation of access control policies \cite{DBLP:conf/dagstuhl/BonattiS03}. Circumscription is one of the traditional AI approaches to non-monotonicity, and it provides an important way to define non-monotonic DLs. In contrast to other approaches, such as default rules, it does not require the adoption of strong syntactic restrictions to preserve decidability. DLs with circumscription are closely related to several other approaches to non-monotonic DLs, in particular to DLs with defeasible inclusions and typicality operators \cite{DBLP:journals/jair/BonattiFS11,DBLP:journals/jair/CasiniS13,DBLP:journals/ai/GiordanoGOP13,DBLP:journals/ai/BonattiFPS15,DBLP:journals/ijar/PenselT18}.

The main feature of circumscription is that selected predicate symbols can be minimized, that is, the extension of these predicates must be minimal regarding set inclusion. Other predicates may vary freely or be declared fixed.  In addition, a preference order can be declared on the minimized predicates.  In DLs, minimizing or fixing role names causes undecidability of reasoning, and consequently, only concept names may be minimized or fixed \cite{Bonatti2009}.  The traditional AI use of circumscription is to introduce and minimize abnormality predicates such as \mn{AbnormalBird}, which makes it possible to formulate defeasible implications such as `birds fly, unless they are abnormal, which shouldn't be assumed unless there is a reason to do so.' Circumscription is also closely related to the closure of predicates symbols as studied, for instance, in \cite{DBLP:conf/ijcai/LutzSW13,Ngo2016,DBLP:journals/lmcs/LutzSW19}; in fact, this observation goes back to \cite{DBLP:conf/adbt/Reiter77,DBLP:journals/ai/Lifschitz85}. While DLs usually assume open-world semantics and represent incomplete knowledge, such closed predicates are interpreted under a closed-world assumption, reflecting that complete knowledge is available regarding those predicates. Circumscription may then be viewed as a soft form of closing concept names: there are no other instances of a minimized concept name except the explicitly asserted ones unless we are forced to introduce (a minimal set of) additional instances to avoid inconsistency.

A primary application of DLs is ontology-mediated querying, where an ontology is used to enrich data with domain knowledge. Surprisingly, little is known about ontology-mediated querying with DLs that support circumscription. For the important conjunctive queries (CQs) and unions of CQs (UCQs), in fact, not even decidability has been established.  This paper aims to close this gap by studying the decidability and precise complexity of ontology-mediated querying for DLs with circumscription, both w.r.t.\ combined complexity and data complexity.  We consider the expressive DL $\ALCHIO$, the tractable (without circumscription) DL \EL, and several members
of tge DL-Lite family. 
These may be viewed as logical cores of the profiles OWL 2 DL, OWL 2~EL, and OWL~2 QL of the OWL~2 ontology language \cite{profiles}.

One of our main results is that UCQ evaluation is decidable in all these DLs when circumscription is added.  It is \TwoExpTime-complete in \ALCHIO w.r.t.\ combined complexity, and thus not harder than query evaluation without circumscription. W.r.t.\ data complexity, however, there is a significant increase from 
\coNP- to $\Uppi^\textrm{P}_2$-completeness.
For \EL, the combined and data complexity turns out to be identical to that of \ALCHIO. This improves lower bounds from \cite{DBLP:journals/jair/BonattiFS11}. 
All these lower bounds already hold for CQs. Remarkably, the $\Uppi^\textrm{P}_2$ lower bound for data complexity already holds when there is only a single minimized concept name (and thus no preference relation) and without fixed predicates.
The complexities for DL-Lite are lower, though still high.  A summary can be found in Table~\ref{tab-results-UCQ}.  Evaluation is `only' \coNP-complete w.r.t.\ data complexity for all considered versions of DL-Lite, except when role disjointness constraints are added (this case is not in the table). The combined complexity remains at \TwoExpTime when role inclusions are present and drops to \coNExpTime without them. The lower bounds already apply to very basic versions of DL-Lite that are positive
in the sense that they do not provide concept disjointness constraints, and the upper bounds apply to expressive versions that include all Boolean operators.

We also study the evaluation of the basic yet important atomic queries (AQs), conjunctive queries of the form $A(x)$ with $A$ a concept name. Also here, we obtain a rather complete picture of the complexity landscape. It is known from \cite{Bonatti2009} that AQ evaluation in \ALCHIO is $\coNExp^{\NPclass}$-complete w.r.t.\ combined complexity. 
We show that the lower bound holds already for \EL. Moreover, our $\Uppi^\textrm{P}_2$-lower bound for the data complexity of (U)CQ-evaluation in \EL mentioned above only requires an AQ, and thus AQ evaluation in both \ALCHIO and \EL are $\Uppi^\textrm{P}_2$-complete w.r.t.\ data complexity. For DL-Lite, the data complexity drops to \PTime in all considered versions, and the combined complexity ranges from \coNExpTime-complete to $\Uppi^\textrm{P}_2$-complete, depending on which Boolean operators are admitted. A summary can be found in Table~\ref{tab-results-instance}.

Proof details are in the appendix.


\newcommand{\stylingref}[1]{\textsuperscript{#1}}
\newcommand{\complete}{-c.}
\begin{table*}
	\centering
	\begin{tabular}{ccccc}
		& $\EL$, $\ALCHIO$ 
		& $\dllitecoreh$, $\dlliteboolh$ 
		& $\dllitebool$ 
		& $\dllitecore$, $\dllitehorn$
		\\ \midrule
		\it Combined complexity
		& \TwoExp\complete 
		\stylingref{(Thm.\;\ref{thm-combined-upper-alchi}, \ref{thm:combined-lower-el})}
		& \TwoExp\complete $^{(\dagger)}$ 
		\stylingref{(Thm.\;\ref{thm-combined-upper-alchi}, \ref{thm:combined-lower-dlliter})}
		& \coNExp\complete 
		\stylingref{(Thm.\;\ref{thm-combined-upper-dllitebool}, \ref{thm-combined-lower-bool-instance})}
		& \coNExp\complete $^{(\dagger)}$ 
		\stylingref{(Thm.\;\ref{thm-combined-upper-dllitebool}, \ref{thm-combined-lower-dllitepos})}
		\smallskip\\
		\it Data complexity
		& $\Uppi^\textrm{P}_2$\complete 
		\stylingref{(Thm.\;\ref{thm-data-upper-alchi}, \ref{thm-data-lower-el})} 
		& \coNP\complete 
		\stylingref{(Thm.\;\ref{thm-data-upper-hornh}, \ref{thm-data-lower-dllitepos})}
		& \coNP\complete 
		\stylingref{(Thm.\;\ref{thm-data-upper-hornh}, \ref{thm-data-lower-dllitepos})}
		& \coNP\complete 
		\stylingref{(Thm.\;\ref{thm-data-upper-hornh}, \ref{thm-data-lower-dllitepos})} 
		\\ \bottomrule
	\end{tabular}
	\caption{Complexity of (U)CQ evaluation on
		circumscribed KBs. $\quad$ $\cdot^{(\dagger)}$ indicates that
		UCQs are needed for lower bound.}
	\label{tab-results-UCQ}
\end{table*}

\medskip
\noindent
{\bf Related Work.} \ A foundational paper on description logics with circumscription is \cite{Bonatti2009}, which studies concept satisfiability and knowledge base consistency in the \ALCHIO family of DLs; these problems are interreducible with AQ evaluation in polynomial time (up to complementation). The same problems have been considered in \cite{DBLP:journals/jair/BonattiFS11} for \EL and DL-Lite and in \cite{DBLP:conf/birthday/BonattiFLSW14} for DLs without the finite model property, including a version of DL-Lite.  The recent \cite{Bonatti2021} is the only work we are aware of that considers ontology-mediated querying in the context of circumscription. It provides lower bounds for \EL and DL-Lite, which are both improved in the current paper, but no decidability results / upper bounds.  A relaxed version of circumscription that enjoys lower complexity has recently been studied in \cite{DBLP:conf/dlog/Stefano0S22}. We have already mentioned connections
to DLs with defeasible inclusions and typicality operators,
see above for references. A connection between circumscription and the
complexity class $\Uppi^\textrm{P}_2$ was first observed in \cite{DBLP:journals/tcs/EiterG93}, and this complexity shows up in our data
complexity results. Our proofs, however, are quite different.

\section{Preliminaries}
\label{sec-preliminaries}

Let \NC, \NR, and \NI be countably infinite sets of \emph{concept names}, \emph{role names}, and \emph{individual names}.
An \emph{inverse role} takes the form $r^-$ with $r$ a role name, and a \emph{role} is a role name or an inverse role. If $r=s^-$ is an inverse role, then $r^-$ denotes $s$.
An \emph{\ALCIO concept} $C$ is built according to the rule
$
C,D ::= \top \mid A \mid \{ a \} \mid \neg C \mid C \sqcap D \mid \exists r . D
$
where $A$ ranges over concept names, $a$ over individual names, and $r$ over roles. 
A concept of the form $\{a\}$ is called a \emph{nominal}. 
We write $\bot$ as abbreviation for $\neg \top$, $C \sqcup D$ for $\neg(\neg C \sqcap \neg D)$, and $\forall r . C$ for $\neg \exists r. \neg C$. 
An \emph{\ALCI concept} is a nominal-free \ALCIO concept.
An \emph{\EL concept} is an \ALCI concept that uses neither negation nor inverse roles.

An \emph{\ALCHIO TBox} is a finite set of \emph{concept inclusions
  (CIs)} $C \sqsubseteq D$, where $C,D$ are \ALCIO concepts, and
\emph{role inclusions (RIs)} $r \sqsubseteq s$, where $r,s$ are
roles. In an \emph{\EL TBox}, only \EL concepts may be used in CIs,
and RIs are disallowed. An \emph{ABox} is a finite set of
\emph{concept assertions} $A(a)$ and \emph{role assertions} $r(a,b)$
where $A$ is a concept name, $r$ a role name, and $a,b$ are individual
names. We use $\Ind(\Amc)$ to denote the set of individual names used
in \Amc. An \emph{\ALCHIO knowledge base (KB)} takes the form
$\Kmc=(\Tmc,\Amc)$ with \Tmc an \ALCHIO TBox and \Amc an ABox.
 {\ALCHI TBoxes and KBs} are defined analogously but may not use nominals.

The semantics is defined as usual in terms of interpretations
$\Imc=(\Delta^\Imc,\cdot^\Imc)$ with $\Delta^\Imc$ the (non-empty)
\emph{domain} and $\cdot^\Imc$ the \emph{interpretation function}, we
refer to \cite{Baader2017} for full details.  An interpretation
satisfies a CI $C \sqsubseteq D$ if $C^\Imc \subseteq D^\Imc$ and
likewise for RIs. It satisfies an assertion $A(a)$ if $a \in A^\Imc$
and $r(a,b)$ if $(a,b) \in r^\Imc$; we thus make the \emph{standard
  names assumption}.  An interpretation \Imc is a \emph{model} of a
TBox \Tmc, written $\Imc \models \Tmc$, if it satisfies all inclusions
in it. Models of ABoxes and KBs are defined likewise. For an
interpretation \Imc and $\Delta \subseteq \Delta^\Imc$, we use
$\Imc|_\Delta$ to denote the restriction of \Imc to subdomain $\Delta$.

A \emph{signature} is a set of concept and role names referred to as
\emph{symbols}. For any syntactic object $O$ such as a TBox or an
ABox, we use $\mn{sig}(O)$ to denote the symbols used in $O$ and
$|O|$ to denote the \emph{size} of $O$, meaning the encoding of $O$
as a word over a suitable alphabet.

We next introduce several more restricted DLs of the DL-Lite family. A
\emph{basic concept} is of the form $A$ or $\exists r . \top$.  A
\dllitecoreh TBox is a finite set of concept inclusions
$C \sqsubseteq D$, \emph{(concept) disjointness assertions}
$C \sqsubseteq \neg D$, and role inclusions $r\sqsubseteq s$  where $C,D$ are basic
concepts and  $r,s$
roles. 
We drop superscript $\cdot^\Hmc$ if no role inclusions are admitted,
	use subscript $\cdot_{\mn{horn}}$ to indicate that the concepts
	$C, D$ in concept inclusions may be conjunctions of basic concepts,
	and subscript $\cdot_{\mn{bool}}$ to indicate that $C,D$ may be
	Boolean combinations of basic concepts, that is, built from basic
	concepts using $\neg$, $\sqcap$, $\sqcup$.
%


A \emph{circumscription pattern}
is a tuple $\CP = (\prec , \Msf, \Fsf, \Vsf)$, where $\prec$ is a
strict partial order on $\Msf$ called the \emph{preference relation},
and $\Msf$, $\Fsf$ and $\Vsf$ are a partition of $\NC$. 
The elements
of $\Msf$, $\Fsf$ and $\Vsf$ are the \emph{minimized}, \emph{fixed}
and \emph{varying} concept names. 
Role names always vary to avoid
undecidability \cite{Bonatti2009}.
%
%
The preference relation $\prec$ on $\Msf$ induces a preference relation
$<_\CP$ on interpretations by setting $\Jmc <_\CP \Imc$ if the
following conditions hold:
\begin{enumerate}
	\item
	$\Delta^\Jmc = \Delta^\Imc$,
	\item
	for all $A \in \Fsf$, $A^\Jmc = A^\Imc$,
	\item
	for all $A \in \Msf$ with $A^\Jmc \not \subseteq A^\Imc$, there
        is a $B \in \Msf$,
	$B \prec A$, such that $B^\Jmc \subsetneq B^\Imc$,
	\item
	there exists an $A \in \Msf$ such that $A^\Jmc \subsetneq A^\Imc$
        and for all $B \in \Msf$, $B \prec A$ implies $B^\Jmc = B^\Imc$.
\end{enumerate}
A \emph{circumscribed KB (cKB)} takes the form $\Circ(\Kmc)$ where
\Kmc is a KB and \CP a circumscription pattern.
A model \Imc of \Kmc is a \emph{model} of $\Circ(\Kmc)$, denoted $\Imc
\models \Circ(\Kmc)$, if no $\Jmc <_\CP \Imc$ is a model of \Kmc.
A cKB $\Circ(\Kmc)$ is \emph{satisfiable} if
it has a model.
%

A \emph{conjunctive query (CQ)} takes the form
$q(\bar x) = \exists \bar y \, \varphi(\bar x, \bar y)$
where $\bar x$ and $\bar y$ are tuples of variables and $\varphi$
is a
conjunction of \emph{atoms} $A(z)$ and $r(z,z')$, with
\mbox{$A \in \NC$}, $r \in \NR$, and $z,z'$ variables from
$\bar x \cup \bar y$.
%
The variables in $\bar x$ are the \emph{answer variables},
and $\mn{var}(q)$ denotes $\bar x \cup \bar y$.
We take the liberty to view
$q$ as a set of atoms, writing, e.g., $\alpha \in
q$ to indicate that $\alpha$ is an atom in
$q$. We may also write $r^-(x,y) \in q$ in place of
\mbox{$r(y,x) \in q$}.
A CQ $q$ gives rise to an interpretation
$\Imc_q$ with $\Delta^{\Imc_q}=\mn{var}(q)$, $A^{\Imc_q} = \{ x \mid
A(x) \in q \}$, and $r^{\Imc_q}=\{(x,y) \mid r(x,y) \in
q\}$ for all $A \in \NC$ and $r \in \NR$.
A {\em union of conjunctive queries (UCQ)} $q(\bar
x)$ is a disjunction of CQs that all have the same answer variables
$\bar x$. The \emph{arity} of $q$ is the length of $\bar x$, and $q$ is \emph{Boolean} if it is of arity zero. An \emph{atomic query (AQ)}
is a CQ of the simple form $A(x)$, with $A$ a concept name.

%
%

%
With a homomorphism from a CQ $q$ to an
interpretation \Imc, we mean a homomorphism from $\Imc_q$ to \Imc
(defined as usual).
A tuple $\bar d \in (\Delta^\Imc)^{|\bar x|}$
is an \emph{answer} to a UCQ $q(\bar x)$ on an interpretation
\Imc, written $\Imc \models q(\bar d)$,
if there is a homomorphism $h$ from a CQ $p$ in $q$ to \Imc with
$h(\bar x)=\bar d$. 
A tuple $\bar a \in \mn{ind}(\Amc)$ is an \emph{answer} to $q$
on a cKB $\Circ(\Kmc)$ with $\Kmc=(\Tmc,\Amc)$, written
$\Circ(\Kmc) \models q(\bar a)$, if $\Imc \models q(\bar a)$ for all models
\Imc of $\Circ(\Kmc)$.
\begin{example}
  Consider a database about universities. The TBox contains domain
  knowledge such as
  $$
  \begin{array}{rcl}
    \mn{University} &\sqsubseteq& \mn{Organization} \\[1mm]
    \mn{Organization} &\sqsubseteq& \mn{Public} \sqcup \mn{Private}
    \\[1mm]
    \mn{Public} \sqcap \mn{Private} &\sqsubseteq&\bot.
  \end{array}
  $$
  Circumscription can be used to express defeasible inclusions.
  For example, from a European
  perspective, universities are usually public:
  $$
  \begin{array}{rcl}
     \mn{University} & \sqsubseteq & \mn{Public} \sqcup \mn{Ab}_U 
  \end{array}
  $$
  where $\mn{Ab}_U$ is a fresh `abnormality' concept name that is
  minimized. If the ABox contains
  $$
  \mn{University}(\mn{leipzigu}), \mn{University}(\mn{mit}),
  \mn{Private}(\mn{mit})
  $$
  and we pose the CQ $q(x)=\mn{Organization}(x) \wedge
  \mn{Public}(x)$, then the answer is \mn{leipzigu}.
  
  We may also use circumscription to implement a soft closed-world assumption, similar in spirit to soft constraints in constraint satisfaction. Assume that the ABox additionally contains a database of nonprofit corporations:
  $$
  \begin{array}{rl}
    \mn{NPC}(\mn{greenpeace}) & \mn{NPC}(\mn{wwf}) 
  \end{array}
  $$
  and that this database is \emph{essentially
    complete}, expressed by minimizing \mn{NPC}. If we
  also know that
  $$
  \begin{array}{@{}c@{}}
  \begin{array}{rcl}
     \mn{IvyLeagueU} & \sqsubseteq& \exists \mn{ownedBy} . (\mn{NPC}
                                      \sqcap \mn{Rich}) \\[1mm]
    \mn{DonationBased} &\sqsubseteq& \neg \mn{Rich}
  \end{array} ~\\[3mm]
  \begin{array}{rcl}
    \mn{DonationBased}(\mn{greenpeace}) & \mn{DonationBased}(\mn{wwf}) \\[1mm]
    \mn{IvyLeagueU}(\mn{harvard})
  \end{array}
  \end{array}
  $$
  then we are forced to infer that our list of NPCs was not actually complete as all explicitly known NPCs are not rich, but a rich NPC must exist. A strict closed-world assumption would instead result in an inconsistency. 
\end{example}
Let $\Lmc$ be a description logic such as \ALCHIO or \EL and let \Qmc
be a query language such as UCQ, CQ, or AQ.  With \emph{\Qmc
  evaluation on circumscribed \Lmc KBs}, we mean the problem to
decide, given an \Lmc cKB $\Circ(\Kmc)$ with $\Kmc = (\Tmc,\Amc)$, a
query $q(\bar x)$ from \Qmc, and a tuple
$\bar a \in \Ind(\Amc)^{|\bar x|}$, whether
$\Circ(\Kmc) \models q(\bar a)$. When studying the combined complexity
of this problem, all of $\Circ(\Kmc)$, $q$, and $\bar a$ are treated
as inputs. For data complexity, in contrast, we assume $q$, \Tmc, and
\CP to be fixed and thus of constant size.

%
%


\section{Between $\ALCHIO$ and \EL}
\label{section-alchi}

We study the complexity of query evaluation on circumscribed KBs for
DLs between \ALCHIO and \EL. In fact, we prove \TwoExpTime-completeness
in combined complexity and $\Uppi^\mathrm{P}_2$-completeness in data complexity
for all these DLs. 


\subsection{Fundamental Observations}
\label{subsect:fundamental}
\label{subsection-alchi-interlacing}

We start with some fundamental observations that underlie the subsequent
proofs, first observing a 
reduction from UCQ evaluation on circumscribed \ALCHIO KBs to UCQ
evaluation on circumscribed \ALCHI KBs. Note that a nominal may be
viewed as a (strictly) closed concept name with a single
instance. This reduction is a simple version of a reduction
from query evaluation with closed concept names to query evaluation on
cKBs in the proof of Theorem~\ref{thm:combined-lower-el}
below.
\begin{proposition}
  \label{prop:nonom}
  UCQ evaluation on circumscribed \ALCHIO KBs can be reduced in
  polynomial time to UCQ evaluation on circumscribed \ALCHI KBs.
\end{proposition}
\begin{proof}
  Let $\Circ(\Kmc)$ be an \ALCHIO cKB, with
  $\Kmc=(\Tmc,\Amc)$, and let $q(\bar x)$ be a UCQ. Let $N$ be the set
  of individual names $a$ such that the nominal $\{a\}$ is used in
  \Tmc. Introduce fresh concept names $A_a,B_a,D_a$ for every
  $a \in N$. We obtain $\Tmc'$ from \Tmc by replacing every $a \in N$
  with $A_a$ and adding the CI $A_a \sqcap \neg B_a \sqsubseteq D_a$,
  $\Amc'$ from \Amc by adding $A_a(a)$ and $B_a(a)$ for every
  $a \in N$, and $q'$ from $q$ by adding the disjunct
  $\exists y \, D_a(y)$ for every $a \in N$.\footnote{Strictly
    speaking, we need to adjust $\exists y \, D_a(y)$ so that it has
    the same answer variables as the other CQs in $q$. This is easy by
    adding to $\Tmc'$ a CI $\top \sqsubseteq T$ for a fresh concept
    name $T$ and extending $\exists y \,
    D_a(y)$ with atom $T(x)$ for every answer variable $x$.} Set $\Kmc'=(\Tmc',\Amc')$. The
  circumscription pattern $\CP'$ is obtained from \CP by minimizing
  the concept names $B_a$ with higher preference than any other
  concept name (and with no preferences between them). We show in the
  appendix that $\Circ(\Kmc) \models q(\bar a)$ iff
  $\mn{Circ}_{\mn{CP}'}(\Kmc') \models q'(\bar a)$ for all
  $\bar a \in \mn{ind}(\Amc)^{|\bar x|}$.
\end{proof}
We are thus left with $\ALCHI$ KBs. We generally assume that TBoxes
are in \emph{normal form}, meaning that every concept inclusion in
$\tbox$ has one of the following shapes:
	\[
	\begin{array}{r@{\qquad}l@{\qquad}r}
		\axtop 
		&
		\axexistsright
		&
		\axexistsleft
		\smallskip\\
		\axand
		&
		\axnotright
		&
		\axnotleft
	\end{array}
	\]
	where $\cstyle{A, A_1, A_2, B}$ range over \NC and
        $r$ ranges over
        roles. 
	The set of concept names in $\tbox$ is denoted $\cnames(\tbox)$.
%
	\begin{restatable}{lemma}{lemnormalform}
          \label{lem:normalform}
          Every $\ALCHI$ TBox $\tbox$ can be transformed in linear
          time into an \ALCHI TBox $\tbox'$ in normal form such that
          for all cKBs $\Circ(\Tmc,\Amc)$, UCQs $q(\bar x)$ that do
          not use symbols from
          $\mn{sig}(\Tmc') \setminus \mn{sig}(\Tmc)$, and
          $\bar a \in \mn{ind}(\Amc)^{|\bar x|}$:
          $\Circ(\Tmc,\Amc) \models q(\bar a) \text{ iff }
          \Circ(\Tmc',\Amc) \models q(\bar a).$
	\end{restatable}

	
	
	
	
	
%
        \noindent Let $\Circ(\Kmc)$ be a cKB with $\Kmc=(\Tmc,\Amc)$.
        A \emph{type} is a set of concept names
        $t \subseteq \cnames(\tbox)$. For an interpretation \Imc and $d \in \Delta^\Imc$, 
        we define
$\typeinof{\I}{d} \coloneqq \{ \cstyle{A} \in \cnames(\tbox) \mid
d \in \cstyle{A}^\Imc \}.$
For a subset $\Delta \subseteq \Delta^\Imc$, we set
$\typeinof{\I}{\Delta} = \{ \typeinof{\I}{d} \mid d \in \Delta\}$.
We further write $\types(\I)$ for $\typeinof{\I}{\Delta^\Imc}$.
Finally, we set
$$\types(\tbox) \coloneqq\bigcup_{\Imc \text{ model of } \Tmc}
\types(\I).$$
%
%
%
%
%
For a role $r$, we write $t \rightsquigarrow_\rstyle{R} t'$ if
for all $\cstyle{A,B} \in \cnames$:
\begin{itemize}
	\item
	$\cstyle{B} \in t'$ and $\tbox \models \exists \rstyle{R}.\cstyle{B} \incl \cstyle{A}$
	implies $\cstyle{A} \in t$ and
	\item $\cstyle{B} \in t$ and $\tbox \models \exists \rstyle{R}^-.\cstyle{B} \incl \cstyle{A}$
	implies
	$\cstyle{A} \in t'$.
\end{itemize}
%
%
%
We next show how to identify a `core' part of a model $\I$
of~\Kmc. These cores will play an important role in dealing with
circumscription in our upper bound proofs.
\begin{definition}
	\label{def-core}
Let $\I$ be a model of \Kmc.
We use $\mn{TP}_{\mn{core}}(\Imc)$ to denote the set of
all types $t \in \mn{TP}(\Imc)$ such that
$$
|\{ d \in \Delta^\Imc \setminus \mn{ind}(\Amc) \mid \mn{tp}_\Imc(d)=t \}| < |\mn{TP}(\Tmc)|
$$
and set $\mn{TP}_{\overline{\mn{core}}}(\Imc)=\mn{TP}(\Imc) \setminus
\mn{TP}_{\mn{core}}(\Imc)$ and
$$\Delta^\Imc_\mn{core} = \{  d \in \Delta^\Imc \mid \mn{tp}_\Imc(d) \in \mn{TP}_{\mn{core}}(\Imc) \}.$$
\end{definition}%
So the core consists of all elements whose types are realized not too
often, except possibly in the ABox. A good way of thinking about cores
is that if a model \Imc of \Kmc is minimal w.r.t.\ $<_\CP$, then all
instances of minimized concept names are in the core. This is, however not strictly true since we may have $A \sqsubseteq B$ where $A$ is $\top$ or fixed, and $B$ is minimized.

The following crucial lemma provides a sufficient condition for
a model \Jmc of \Kmc to be minimal w.r.t.\ $<_\CP$,  relative
to a model \Imc of \Kmc that is known to be minimal w.r.t.~$<_\CP$.
\begin{restatable}[Core Lemma]{lemma}{lemmafive}
	\label{lem-modelOfK}
	\label{lem-lemma5}
	Let \Imc be a model of $\Circ(\Kmc)$ and let \Jmc be a model of \Kmc with
	$\Delta^\Imc_\mn{core} \subseteq \Delta^\Jmc$.
	If 
	\begin{enumerate}
		\item
		$\mn{tp}_\Imc(d)=\mn{tp}_\Jmc(d)$
		for all $d \in \Delta^\Imc_\mn{core}$ and
		\item
		$\mn{tp}_\Jmc(\Delta^\Jmc \setminus \Delta^\Imc_\mn{core}) = \mn{TP}_{\overline{\mn{core}}}(\Imc)$,
	\end{enumerate}
	then \Jmc is a model of $\Circ(\Kmc)$.
\end{restatable}

We give a sketch of the proof of (the contrapositive of)
Lemma~\ref{lem-modelOfK}.  Assume that \Jmc is not a model of
$\Circ(\Kmc)$.  Then there must be a model $\Jmc'$ of \Kmc with
$\Jmc' <_\CP \Jmc$ and to obtain a contradiction it suffices to
construct from $\Jmc'$ a model $\Imc'$ of \Kmc with
$\Imc' <_\CP \Imc$.  Note that \Imc and \Jmc may have domains of
different sizes.  The elements of $\Delta^\Imc_{\mn{core}}$ receive
the same type in $\Imc'$ as in $\Jmc'$.  For each non-core type $t$ in
$\Imc$, we consider the set $S_t$ of types in $\Jmc'$ of those
elements that have type $t$ in $\Jmc$.  Since $t$ is realized in \Imc at
least $|\mn{TP}(\Tmc)|$ many times, we have enough room to realize in
$\Imc'$ exactly the types from $S_t$ among those elements that had
type $t$ in \Imc, that is, within $(\type_{\I})^{-1}( t)$. It is
easy to see that $\I'$ is a model of $\kb$: it satisfies \Tmc as it
realizes the same types as $\Jmc'$ and it satisfies \Amc since
$\Imc'|_{\mn{ind}(\Amc)}=\Jmc'|_{\mn{ind}(\Amc)}$.  By construction
and since $\Jmc' <_\CP \Jmc$, it further satisfies $\I' <_\CP \I$.

We next use the core lemma to show that if
$\Circ(\Kmc) \not\models q(\bar a)$ for any CQ $q$ and
$\bar a \in \Ind(\Amc)^{|\bar x|}$, then this is witnessed by a
countermodel \Imc that has a regular shape. Here and in what follows,
a \emph{countermodel} is a model \Imc of $\Circ(\Kmc)$ with
\mbox{$\Imc\not\models q(\bar a)$}. By regular shape, we mean that
there is a `base part' that contains the ABox, the core of \Imc, as
well as some additional elements; all other parts of \Imc are
tree-shaped with their root in the base part, and potentially with
edges that go back to the core (but not to other parts of the
base!). We next make this precise.

Let \Imc be a model of $\circkb$. Set
$\Omega = \{ rA \mid B \sqsubseteq \exists r . A \in \Tmc \}$ and fix
a function $f$ that chooses, for every $d \in \Delta^\Imc$ and
$rA \in \Omega$ with $d \in (\exists r . A)^\Imc$, an element
$f(d,rA) = e\in A^\Imc$ such that $(d,e) \in r^\Imc$. Further choose,
for every $t \in \noncoretypesof{\I}$, a representative $e_t \in \Delta^\Imc$
with $\mn{tp}_\Imc(e_t)=t$.
We inductively define the set $\Pmc$ of \emph{paths through \Imc} along with
a mapping $h$ assigning to each $p \in \Pmc$ an element of~$\Delta^\Imc$:
\begin{itemize}

\item each element $d$ of the set
$$\basedomainof{\I} := \indsof{\abox} \cup \coredomainof{\I} \cup \{
e_t \mid t \in \noncoretypesof{\I} \}$$
is a path in \Pmc and
  $h(d)=d$;

\item if $p \in \Pmc$ with $h(p)=d$ and $rA \in \Omega$ with
  $f(d,rA)$ defined and not from $\Delta^\Imc_{\mn{core}}$, then
  $p'=p rA$ is a path in \Pmc and $h(p')=f(d,rA)$.
  
\end{itemize}
For every role $r$, define
$$
\begin{array}{@{}r@{\;}c@{\;}l}
  R_r &=& \{ (a,b) \mid a,b \in \Ind(\Amc), r(a,b) \in \Amc
                \} \, \cup \\[1mm]
  && \{ (d,e) \mid d,e \in  \coredomainof{\I}, (d,e) \in r^\Imc \} \,
     \cup \\[1mm]
            && \{ (p,p') \mid p'=prA \in \Pmc 
               \} \, \cup\\[1mm]
               && \{ (p,e) \mid 
                  e=f(h(p),rA) \in \Delta^\Imc_{\mn{core}} 
                  \}. 
\end{array}
$$
Now the \emph{unraveling} of \Imc is defined by setting
$$
\begin{array}{r@{\;}c@{\;}l}
\Delta^{\Imc'} &=& \Pmc \qquad 
  A^{\Imc'} \ = \ \{ p \in \Pmc \mid h(p) \in A^\Imc \} \\[1mm]
  r^{\Imc'} &=& \displaystyle\bigcup_{\Tmc \models s
                \sqsubseteq r } ( R_s \cup \{ (e,d) \mid (d,e) \in R_{s^-} \})
\end{array}
$$
for all concept names $A$ and role names $r$. It is easy to verify
that $h$ is a homomorphism from $\Imc'$ to \Imc. This version of
unraveling is inspired by constructions from \cite{maniere:thesis} where
the resulting model is called an \emph{interlacing}.


%
%
%
%
\begin{restatable}{lemma}{leminterlacingiscountermodel}
  \label{lem_interlacing_is_countermodel}
  Let $q(\bar x)$ be a UCQ and $\bar a \in \Ind(\Amc)$. 
					If $\I$ is a countermodel
                                        against $\circkb \models
                                        q(\bar a)$, then so is its
                                        unraveling $\interlacing$.
\end{restatable}
For some of our upper bound proofs, it will be important that the
reference model \Imc in the core lemma is finite and sufficiently
small. The following lemma shows that we can always achieve this.
\begin{restatable}{lemma}{lemfinitemodelpropertywrtcore}
	\label{lem-finite-model-property-wrt-core}
	Let $\I$ be a model of $\circkb$.
	There exists a model $\Jmc$ of $\circkb$ such that $|\Delta^\Jmc| \leq\sizeof{\abox} + 2^{2\sizeof{\tbox}}$, $\I|_{\coredomainof{\I}} = \Jmc|_{\coredomainof{\Jmc}}$, and $\noncoretypesof{\I} = \noncoretypesof{\Jmc}$.	
\end{restatable}
In the proof of Lemma~\ref{lem-finite-model-property-wrt-core}, we
construct the desired model~$\Jmc$ by starting from
$\I|_{\indsof{\abox} \cup \coredomainof{\I}}$ and adding exactly
$m = \sizeof{\types(\tbox)}$ instances of each type from
$\noncoretypesof{\I}$.

\subsection{Combined complexity}
\label{subsection-alchi-combined}

\newcommand{\J}{\Jmc}

We show that in any DL between \EL and \ALCHIO,
the evaluation of CQs and UCQs on cKBs is \TwoExpTime-complete
w.r.t.\ combined complexity, starting with the upper bound.
%
%
%
\begin{theorem}
	\label{thm-combined-upper-alchi}
	UCQ evaluation on circumscribed \ALCHIO KBs is in \TwoExpTime w.r.t.\ combined complexity.
\end{theorem}
By Proposition~\ref{prop:nonom}, it suffices to consider \ALCHI.
Assume that we are given as an input an \ALCHI cKB
$\Circ(\Kmc)$ with $\Kmc=(\Tmc,\Amc)$, a UCQ $q(\bar x)$, and a tuple
$\bar a \in \Ind(\Amc)^{|\bar x|}$. We have to decide whether or not
there is a countermodel \Imc against $\Circ(\Kmc) \models q(\bar a)$.

Fix a set $\Delta$ of size $|\Ind(\Amc)|+2^{2|\Tmc|+1}$ such that
\mbox{$\Ind(\Amc) \subseteq\Delta$}. Note that $\Delta$ is sufficiently large so that
we may assume the base domain $\Delta^\Imc_{\mn{base}}$ of unraveled
interpretations to be a subset of $\Delta$. In an outer loop, our
algorithm iterates over all triples
$(\Imc_{\mn{base}}, \Delta_{\mn{core}},\typescandidate)$ such that the
following conditions are satisfied:
\begin{itemize}

\item $\Imc_{\mn{base}}$ is a model of \Amc with 
  $\Delta_{\mn{core}} \subseteq\Delta^{\Imc_{\mn{base}}} \subseteq \Delta$; 

  
\item $\typeinof{\basecandidate}{\domain{\basecandidate}
    \setminus \Delta_{\mn{core}}} = 
  \typescandidate$;
          
        \item
          $\typeinof{\basecandidate}{\Delta_{\mn{core}}} \cap
          \typescandidate = \emptyset$.

       
\end{itemize}
For each triple $(\Imc_{\mn{base}}, \Delta_{\mn{core}}, \typescandidate)$, we then check
whether the following additional conditions are satisfied:
\begin{description}

\item[\textnormal{(I)}] $\Imc_{\mn{base}}$ can be extended to a model \Imc of $\Tmc$
  such that
  \begin{enumerate}

  \item[(a)] $\Imc|_{\Delta^{\Imc_{\mn{base}}}} = \Imc_{\mn{base}}$,


  \item[(b)] $\mn{tp}_\Imc(\Delta^\Imc \setminus  \Delta^{\Imc_{\mn{base}}})
    \subseteq \typescandidate$, and

  \item[(c)] $\Imc \not\models q(\bar a)$;

  \end{enumerate}

\item[\textnormal{(II)}]  there exists a model $\Jmc$ of $\circkb$ such that
  \mbox{$\Jmc|_{\Delta^\Jmc_{\mn{core}}} = \Imc_{\mn{base}}|_{\Delta_{\mn{core}}}$} and $\noncoretypesof{\Jmc} = \typescandidate$.

\end{description}
We return `yes' if all triples fail the check and `no' otherwise.

If the checks succeed, then the model \Imc of \Kmc from Condition~(I)
is a countermodel against $\Circ(\Kmc) \models q(\bar a)$.  In
particular, we may apply Lemma~\ref{lem-lemma5}, using the model \Jmc
from Condition~(II) as the reference model, to show that \Imc is
minimal w.r.t.\ $<_{\mn{CP}}$. Conversely, any countermodel $\Imc_0$
against $\Circ(\Kmc) \models q(\bar a)$ can be unraveled into a
countermodel $\Imc$ from which we can read off a triple
$(\Imc_{\mn{base}}, \Delta_{\mn{core}},\typescandidate)$ in the
obvious way, and then $\Imc$ witnesses Condition~(I) and choosing
$\Jmc=\Imc$ witnesses Condition~(II).

Of course, we have to prove that Conditions~(I) and~(II) can be
verified in \TwoExpTime. This is easy for Condition~(II): by
Lemma~\ref{lem-finite-model-property-wrt-core}, it suffices to
consider models \Jmc of size at most
$\sizeof{\abox} + 2^{2\sizeof{\tbox}}$ and thus we can iterate
over all candidate interpretations \Jmc up to this size, check
whether \Jmc is a model of \Kmc with
\mbox{$\Jmc|_{\Delta^\Jmc_{\mn{core}}} =
  \basecandidate|_{\Delta_{\mn{core}}}$} and
$\noncoretypesof{\Jmc} = \typescandidate$, and then iterate over all
models $\Jmc'$ of \Kmc with $\Delta^{\Jmc'}=\Delta^\Jmc$ to check that
$\Jmc$ is minimal w.r.t.~$<_{\mn{CP}}$.

Condition~(I) requires more work. We use a mosaic approach, that is,
we attempt to assemble the interpretation \Imc from Condition~(I) by
combining small pieces called \emph{mosaics}.
%
 Each mosaic contains the
base part of \Imc and at most two additional elements
$e^*_1, e^*_2$. 
We trace
partial homomorphisms from CQs in $q$ through the mosaics, as follows.

%
%
%
Fix a triple
$(\Imc_{\mn{base}}, \Delta_{\mn{core}},\typescandidate)$. A
\emph{match triple} for an interpretation \Jmc takes the form
$(p,\widehat p,h)$ such that $p$ is a CQ in~$q$,
$\widehat p \subseteq p$, and $h$ is a partial map from
$\mn{var}(\widehat p)$ to $\Delta^\Jmc$ that is a homomorphism from
$\widehat p|_{\mn{dom}(h)}$ to~\Jmc. Intuitively, \Jmc is a mosaic and
the triple $(p,\widehat p,h)$ expresses that a homomorphism from
$\widehat p$ to \Imc exists, with the variables in $\mn{dom}(h)$ being
mapped to the current piece~\Jmc and the variables in
$\mn{var}(\widehat p) \setminus \mn{dom}(h)$ mapped to other
mosaics. A match triple is \emph{complete} if $\widehat p = p$ and
\emph{incomplete} otherwise.  To make \Imc a countermodel, we must
avoid complete match triples.  A \emph{specification} for \Jmc is a
set $S$ of match triples for \Jmc and we call $S$ \emph{saturated} if
the following conditions are satisfied:
\begin{itemize}

\item if $p$ is a CQ in $q$, $\widehat p \subseteq p$, and $h$ is
  a homomorphism from $\widehat p$ to $\Jmc$, then $(p,\widehat p,h) \in S$;

\item if $(p,\widehat p,h),(p,\widehat p',h') \in S$ and $h(x)=h'(x)$
  is defined for all
  $x \in \mn{var}(\widehat p) \cap \mn{var}(\widehat p')$, then
  $(p,\widehat p \cup \widehat p',h \cup h') \in S$.
  

  
\end{itemize}
\begin{definition}
  A \emph{mosaic} for $(\Imc_{\mn{base}}, \Delta_{\mn{core}}, \typescandidate)$ is a
  tuple $M=(\Jmc,S)$ where
  \begin{itemize}

  \item $\Jmc$ is an interpretation such that
    \begin{enumerate}
    \item 
    $\Delta^{\Imc_{\mn{base}}} \subseteq \Delta^\Jmc \subseteq
    \Delta^{\Imc_{\mn{base}}} \uplus \{ \frsymb, \gensymb \}$;

    \item 
      $\Jmc|_{\Delta^{\Imc_{\mn{base}}}} = \basecandidate$;

    \item 
      $\mn{tp}_\Jmc(e_i^*) \in \typescandidate$ if $e_i^* \in
      \Delta^\Jmc$, for $i \in
      \{1,2\}$;

      \item \Jmc satisfies all $\exists r . A \sqsubseteq B \in
        \Tmc$
        and all \mbox{$r \sqsubseteq s \in \Tmc$}.
      
    \end{enumerate}

  \item $S$ is a saturated specification for \Jmc that contains only
    incomplete match triples.

  \end{itemize}
  We use $\Jmc_M$ to refer to \Jmc and $S_M$ to refer
  to $S$.
\end{definition}
%
  Let \Mmc be a set of mosaics for
  $(\Imc_{\mn{base}}, \Delta_{\mn{core}}, \typescandidate)$. We say that $M \in \Mmc$ is
  \emph{good in} \Mmc if for every
  $e \in \Delta^{\Jmc_M}$
%
  and every
  $A \sqsubseteq \exists r . B \in \Tmc$ with $e \in (A \sqcap \neg
  \exists r . B)^{\Jmc_M}$, we
  find a mosaic $M' \in \Mmc$ such that the following conditions are
  satisfied:
  \begin{enumerate}

  \item $\mn{tp}_{\Jmc_M}(e)=\mn{tp}_{\Jmc_{M'}}(e)$;

  \item $e \in (\exists r . B)^{\Jmc_{M'}}$;

  \item if $(p,\widehat p,h') \in S_{M'}$, then $(p,\widehat p,h) \in S_{M}$ where
    $h$ is the restriction of $h'$ to range $\Delta^{\Imc_{\mn{base}}}
      \cup \{e \}$.



      
    
  \end{enumerate}
  If $M$ is not good in \Mmc, then it is \emph{bad}. To verify
  Cond.~(I), we start
  with the set of all mosaics for
  $(\Imc_{\mn{base}}, \Delta_{\mn{core}}, \typescandidate)$
  and repeatedly and exhaustively eliminate bad mosaics.

  \begin{restatable}{lemma}{lemmamosaic}
  	\label{lemma-mosaic}
    $\Imc_{\mn{base}}$ can be extended to a model \Imc of \Tmc that
    satisfies Conditions~(a) to~(c) iff at least one mosaic survives
    the elimination process.
\end{restatable}

We provide a matching lower bound for
Theorem~\ref{thm-combined-upper-alchi}.  It is rather strong as it
already applies to CQs, to \EL KBs, and uses a single minimized
concept name (and consequently no preferences) and no fixed concept
names.  It is proved by a subtle reduction from CQ evaluation on \EL
KBs with closed concept names, that is, with KBs $(\Tmc,\Amc)$
enriched with a set $\Sigma$ of closed concept names $A$ that have to
be interpreted as $A^\Imc = \{ a \mid A(a) \in \Amc \}$ in all models
\Imc.  This problem was proved to be \TwoExpTime-hard in
\cite{Ngo2016}.  The reduction also sheds some light on the connection
between circumscription and closing concept names. 
%
\begin{restatable}{theorem}{thmcombinedlowerel}
  \label{thm:combined-lower-el}
  CQ evaluation on circumscribed \EL KBs is \TwoExpTime-hard w.r.t.\ combined
        complexity.
        This holds even with a single minimized concept name and no
        fixed concept names.
\end{restatable}

\subsection{Data complexity}
\label{subsection-alchi-data}

We show that in any DL between \EL and \ALCHIO, the evaluation of CQs
and UCQs on cKBs is $\Uppi^\textrm{P}_2$-complete w.r.t.\ data
complexity. We start with the upper bound.
\begin{restatable}{theorem}{thmdataupperalchi}
\label{thm-data-upper-alchi}
UCQ evaluation on circumscribed \ALCHIO KBs is in $\Uppi^\textrm{P}_2$
w.r.t.\ data complexity.
\end{restatable}
\newcommand{\interleaving}{\interlacing}
We may again limit our attention to \ALCHI.  To prove
Theorem~\ref{thm-data-upper-alchi}, we show that if a countermodel
exists, then there is one of polynomial size (with TBox
and query of constant size).  Note that this is not a
consequence of Lemma~\ref{lem-finite-model-property-wrt-core} since
there is no guarantee that, if the model \Imc from that lemma is a
countermodel, then so is~\Jmc.  Once the size bound is in place, we
obtain the $\Uppi^\textrm{P}_2$ upper bound by a straightforward
guess-and-check procedure.
\begin{restatable}{lemma}{lemmaalchicountermodel}
  \label{lemma-alchi-countermodel}
  Let $\Circ(\Kmc)$ be an \ALCHI cKB,
  $\Kmc=(\Tmc,\Amc)$, $q(\bar x)$ a UCQ, and $\bar a \in
  \mn{ind}(\Amc)^{|\bar x|}$. If $\Circ(\Kmc) \not\models q(\bar a)$,
  then there exists a
  countermodel \Imc with 
  $|\Delta^\Imc| \leq (\sizeof{\abox} + 2^\sizeof{\tbox})^{\sizeof{\tbox}^\sizeof{q}}$.
\end{restatable}
To prove Lemma~\ref{lemma-alchi-countermodel}, let $\Circ(\Kmc)$ be an
\ALCHI cKB, $\Kmc=(\Tmc,\Amc)$, and let \Imc be a model of
$\Circ(\Kmc)$ with $\Imc \not\models q(\bar a)$.  We construct a small countermodel by starting from the unraveling
$\interlacingof{\I}$ of $\I$ and applying a quotient construction. The
latter is based on a suitable equivalence relation on $\Delta^{\Imc'}$
which we define next. Recall that $\Imc'$ consists of a base part
$\basedomainof{\I}$ and a tree-part with backedges to
$\basedomainof{\I}$.

For $n \geq 0$ and $d \in \Delta^{\Imc'} \setminus \basedomainof{\I}$,
we use $\Nmc_n(d)$ to denote the \emph{$n$-neighborhood} of $d$ in
${\Imc'}$ up to $\basedomainof{\I}$, that is, the set of all elements
$e \in \Delta^{\Imc'} \setminus \basedomainof{\I}$ such that the
undirected graph
$$G_{\Imc'}=(\Delta^{\Imc'},\{\{ d,e\} \mid (d,e) \in r^{\Imc'} \text{ for 
        some } r \in \NR \})$$
      contains a path $d_0,\dots,d_k$ with $d_0=d$,
      $d_0,\dots,d_{k-1} \notin \basedomainof{\I}$, and $d_k=e$,
      $0 \leq k \leq n$.

      Recall that the elements of $\Delta^{\Imc'}$ are paths through
      $\Imc$, sequences $p=d_0 r_1A_1 \cdots r_nA_n$ with
      $d_0 \in \basedomainof{\I}$ and \mbox{$r_iA_i \in \Omega$}; the
      \emph{length} of $p$, denoted by $|p|$, is $n$.  By definition
      of $\Imc'$ and of neighborhoods, for every $n \geq 0$ and
      $d \in \Delta^{\Imc'} \setminus \basedomainof{\I}$, there is a
      unique path $p_{d,n} \in \Nmc_n(d)$ that is a prefix of all
      paths in $\Nmc_n(d)$, that is, all
      $e \in \Nmc_n(d) \setminus
      \basedomainof{\I}$ take the form
      $p_{d, n} r_1A_1 \cdots r_kA_k$.

      For $n \geq 0$, the equivalence relation $\sim_n$ on
      $\Delta^{\Imc'}$ is defined by setting
      $d_1 \sim_n d_2$ 
      if
      $d_1=d_2 \in \basedomainof{\I}$ or
      $d_1,d_2 \notin \basedomainof{\I}$ and the following conditions
      are satisfied:
      \begin{enumerate}

      \item $d_1=p_{d_1,n}w$ and $d_2=p_{d_2,n}w$ for some $w \in \Omega^*$;

      \item for every $w=r_1A_1 \cdots r_kA_k \in \Omega^*$:
        $$
        p_{d_1,n}w \in \Nmc_n(d_1) \text{ iff }
        p_{d_2,n}w \in \Nmc_n(d_2)
        $$
        and if $p_{d_1,n}w \in \Nmc_n(d_1)$, then
        \begin{enumerate}


        \item
          $\mn{tp}_{\Imc'}(p_{d_1,n}w)=\mn{tp}_{\Imc'}(p_{d_2,n}w)$;
          
        \item 
          $(p_{d_1,n}w,e) \in r^{\Imc'}$ iff $(p_{d_2,n}w,e) \in
          r^{\Imc'}$
          for all roles $r$ and $e \in \basedomainof{\I}$.
        
        \end{enumerate}
        
      \item $|d_1| \equiv |d_2| \!\!\mod \maxmod$.
        
      \end{enumerate}
%
%
%
%
  %
    %
      %
        %
      %
      For an element $d \in \Delta^{\Imc'}$,
      we use $\equivclass{d}$ to denote the equivalence class of $d$ w.r.t.\ $\sim_\maxradius$.
      The \emph{quotient} $\Imc'/{\sim}_\maxradius$ of $\Imc'$ is the
      interpretation whose domain is the set of all equivalence
      classes of $\sim_\maxradius$ and  where
      $$
      \begin{array}{rcl}
        A^{\Imc'/{\sim}_\maxradius} &=& \{ \equivclass{d} \mid d \in A^{\Imc'} \} \\[1mm]
        r^{\Imc'/{\sim}_\maxradius} &=& \{ (\equivclass{d},\equivclass{e}) \mid (d,e) \in r^{\Imc'} \} 
      \end{array}
      $$
      for all concept names $A$ and role names $r$. 
      It can be verified that
      $|\Delta^{\Imc'/{\sim}_{\maxradius}}| \leq
 (\sizeof{\abox} + 2^\sizeof{\tbox})^{\sizeof{\tbox}^\sizeof{q}}$, as desired.
      %
      				\begin{restatable}{lemma}{lemquotient}
                                  \label{lem-quotient}
  Let $q(\bar x)$ be a UCQ and $\bar a \in \Ind(\Amc)^{\sizeof{\bar{x}}}$. 
					If $\I$ is a countermodel
                                        against $\Circ(\Kmc) \models
                                        q(\bar a)$, then so is $\interleavingof{\I}/{\sim_{\maxradius}}$.
\end{restatable}
It is in fact straightforward to show that
$\interleavingof{\I}/{\sim_{\maxradius}}$ is a model of
$\kb$. Minimality w.r.t.\ $<_\CP$ is proved using
Lemma~\ref{lem-lemma5}. The most subtle part of the proof of
Lemma~\ref{lem-quotient} is showing that
$\interleavingof{\I}/{\sim_{\maxradius}} \not\models q(\bar a)$. This
is done by exhibiting suitable `local' homomorphisms from
$\interleavingof{\I}/{\sim_{\maxradius}}$ back
to~$\interleavingof{\I}$ so that from any homomorphism from a CQ $p$
in $q$ to $\interleavingof{\I}/{\sim_{\maxradius}}$ with
$h(\bar x)=\bar a$, we obtain a homomorphism from $p$ to \Imc with the
same property by composition.
This finishes the proof of Lemmas~\ref{lem-quotient}
and~\ref{lemma-alchi-countermodel} and of
Theorem~\ref{thm-data-upper-alchi}.

We next provide a matching lower bound for Theorem~\ref{thm-data-upper-alchi}. It already applies to AQs and \EL KBs, and when there is a single minimized concept name and no fixed concept name.
It is proved by a subtle reduction from the validity of $\forall\exists\mn{3SAT}$ sentences. Several non-obvious technical tricks are needed to make the reduction work with a single minimized concept name.
Our result improves upon a known \coNP lower bound from \cite{Bonatti2021}.
\begin{restatable}{theorem}{thmdatalowerel}
  \label{thm-data-lower-el}
	AQ evaluation on circumscribed \EL KBs is $\Uppi^\textrm{P}_2$-hard.
	This holds even with a single minimized concept name and no fixed concept names. 
\end{restatable}

\section{$\dllite$}
\label{section-dllite}


We consider the DL-Lite family of DLs.
Without circumscription, these DLs enjoy low complexity of query evaluation, typically \NPclass-complete in combined complexity and in {\sc AC}$^0$ in data complexity (depending on the dialect).
With circumscription, the complexity tends to still be very high, though in some relevant cases it is lower than in \ALCHIO.


\subsection{Combined complexity}
\label{subsection-dllite-combined}

Our first result shows that, when role inclusions are admitted, nothing is gained from transitioning from \ALCHIO to DL-Lite.
  The proof is a variation of that of 
  \cref{thm:combined-lower-el}, but technically much 
  simpler. 
%
\begin{restatable}{theorem}{thmcombinedlowerdlliter}
  \label{thm:combined-lower-dlliter}
   UCQ evaluation on circumscribed \dllitecoreh KBs is \TwoExpTime-hard w.r.t.\ combined
        complexity.
        This holds even with a single minimized concept name, no fixed concept names, and no disjointness constraints.
\end{restatable}
We now move to $\dllitebool$ as a very expressive DL-Lite dialect
without role inclusions and observe that the combined complexity
decreases.
\begin{restatable}{theorem}{thmcombinedupperdllitebool}
  \label{thm-combined-upper-dllitebool}
  UCQ evaluation on circumscribed $\dllitebool$ KBs is in
  $\coNExpTime$ w.r.t.\ combined complexity.
\end{restatable}
To prove Theorem~\ref{thm-combined-upper-dllitebool}, we first observe
that we can refine the unraveling of countermodels $\I$ from
Section~\ref{subsection-alchi-interlacing} such that each element
outside of the base part $\basedomainof{\I}$ has at most one successor
per role.  This property allows us to simplify the notion of a
neighborhood in the quotient construction in
Section~\ref{subsection-alchi-data}. This, in turn, yields the
following result.  Note that in contrast to
 Lemma~\ref{lemma-alchi-countermodel}, there is no double exponential
 dependence on $|q|$.
\begin{restatable}{lemma}{lemmaboolcountermodel}
  \label{lemma-bool-countermodel}
  Let $\Circ(\Kmc)$ be a $\dllitebool$ cKB with
  $\Kmc=(\Tmc,\Amc)$, $q(\bar x)$ a UCQ, and $\bar a \in
  \mn{ind}(\Amc)^{|\bar x|}$. If $\Circ(\Kmc) \not\models q(\bar a)$,
  then there exists a
  countermodel \Imc with 
  $|\Delta^\Imc| \leq (\sizeof{\abox} + 2^\sizeof{\tbox})^{\sizeof{q}^2(\sizeof{\tbox}
    +1)}$.
\end{restatable}
Lemma~\ref{lemma-bool-countermodel} yields a $\coNExpTime^{\NPclass}$
upper bound by a straightforward guess-and-check procedure: guess a
model \Imc of \Kmc with $\Imc \not\models q(\bar a)$ and
$|\Delta^\Imc|$ bounded as in Lemma~\ref{lemma-bool-countermodel}, and
use an oracle to check that $\Imc$ is minimal w.r.t.\ $<_\CP$ where
the oracle decides, given a cKB $\Circ(\Kmc)$ and a model \Imc of
$\Circ(\Kmc)$, whether $\Imc$ is non-minimal w.r.t.\ $<_\CP$ by
guessing a model $\Jmc <_\CP \Imc$ of~\Kmc; such an approach was
used also in \cite{Bonatti2009}.

To obtain a \coNExpTime upper bound as desired, we replace the oracle
with a more efficient method to check whether a given model \Imc of
\Kmc is minimal w.r.t.\ $<_\CP$. The crucial observation is that
instead of guessing a model $\Jmc <_\CP \Imc$ of~\Kmc, it suffices to
consider certain interpretations $\Imc'$ of polynomial size, derived
from sub-interpretations of \Imc, and guess models $\Jmc' <_\CP \Imc'$
of~\Kmc. Intuitively, we decompose an expensive `global' test into
exponentially many inexpensive `local' tests. This even works in the
presence of role inclusions, that is, in $\dllitebool^{\Hmc}$, which
shall be useful in Section~\ref{section-atomic-queries-el}.

Let $\Circ(\Kmc)$ be a $\dllitebool^{\Hmc}$ cKB with
$\kb = (\tbox, \abox)$ and \Imc a model of~\Kmc. For each role $r$
used in \Tmc such that $r^\Imc \neq \emptyset$, we choose a witness
$w_r \in (\exists r^-)^\I$. Every $\Pmc \subseteq \domain{\I}$ gives
rise to an interpretation $\I_\Pmc$ as follows:
$$
\begin{array}{r@{\;}c@{\;}l}
        \Delta^{\I_\Pmc} &=  & \Pmc \cup \mn{ind}(\Amc) \cup \{ w_r
                                \mid r \text{ used in } \Tmc, r^\I \neq \emptyset \} \\[1mm]
	\cstyle{A}^{\I_\Pmc} &=  & \cstyle{A}^{\I} \cap \domain{\I_\Pmc}
	\\[1mm]
	\rstyle{r}^{\I_\Pmc} &=  & \rstyle{r}^{\I} \cap (\indsof{\abox} \times \indsof{\abox}) 
	\\[1mm]
	&& \cup \, \{ (e, w_s) \mid e \in (\exists s)^\I \cap \domain{\I_\Pmc}, \tbox \models s \incl r \}
	\\[1mm]
	&& \cup \,\{ (w_s, e) \mid e \in (\exists s)^\I \cap \domain{\I_\Pmc}, \tbox \models s \incl r^- \}.
\end{array}
$$
Note that $\Imc_\Pmc$ is derived from the subinterpretation
$\Imc|_{\Pmc \cup \mn{ind}(\Amc)}$ by `rerouting' some role edges to
elements~$w_r$. We have
$\sizeof{\I_\Pmc} \leq \sizeof{\abox} + \sizeof{\tbox} +
\sizeof{\Pmc}$ and will only consider sets \Pmc with
$\sizeof{\Pmc} \leq 2\sizeof{\tbox} + 1$. It is not difficult to show
that $\I_\Pmc$ is a model of $\kb$, for all
$\Pmc \subseteq \domain{\I}$. The next lemma characterizes the
(non)-minimality of \Imc in terms of the (non)-minimality of the
interpretations $\Imc_\Pmc$.
%
%
  \begin{restatable}{lemma}{lemmacriteriacombined}
	\label{lemma-criteria-combined}
	The following are equivalent:
	\begin{enumerate}
		\item There exists a model \Imc of \kb with $\J <_\CP \I$;
		\item There exist a $\Pmc \subseteq \domain{\I}$
                  with $\sizeof{\Pmc} \leq 2\sizeof{\tbox} + 1$ and a
                  family
                  $(\Jmc_{e})_{e \in \domain{\I}}$ of models of \Kmc such that
                  $\J_{ e} <_\CP \I_{\Pmc \cup \{ e\} }$ and
                  $\J_{e}|_{\domain{\I_\Pmc}} =
                  \J_{e'}|_{\domain{\I_\Pmc}}$ for all $e,e' \in \domain{\I}$.
	\end{enumerate} 
\end{restatable}
It should now be clear that we have established
Theorem~\ref{thm:combined-lower-dlliter}. After guessing the model
\Imc of \Kmc with $\Imc \not\models q(\bar a)$, we check the
complement of Point~2 of Lemma~\ref{lemma-criteria-combined} in a
brute-force way. More precisely, we first iterate over all
$\Pmc \subseteq \Delta^\Imc$ with
$\sizeof{\Pmc} \leq 2\sizeof{\tbox} + 1$, then over all
interpretations $\Jmc_0$ with $\Delta^{\Jmc_0} = \Delta^{\Imc_\Pmc}$
(as candidates for the common restriction of the models
$(\Jmc_{e})_{e \in \domain{\I}}$ to ${\domain{\I_\Pmc}}$), then over
all $e \in \Delta^\Imc$, and finally over all models $\Jmc_e$ of \Kmc
with $\Jmc_e|_{\Delta^{\Imc_\Pmc}}=\Jmc_0$, and test whether
$\Jmc_e <_\CP \Imc_{\Pmc \cup \{e\}}$. We accept if for every \Pmc and
$\Jmc_0$, there is an $e$ such that for all $\Jmc_e$ the final check
fails. Overall, we obtain a \coNExpTime algorithm.

We provide a matching lower bound that holds even for $\dllitecore$ cKBs.
It is proved by reduction from the complement of Succinct3COL, which
is known to be \NExpTime-complete \cite{Papadimitriou1986}.
%
\begin{restatable}{theorem}{thmcombinedlowerdllitepos}
	\label{thm-combined-lower-dllitepos}
	UCQ answering on circumscribed \dllitecore KBs is \coNExpTime-hard
        w.r.t.\ combined complexity. This holds even with a single
        minimized concept name, no fixed concept names, and no disjointness constraints.
      \end{restatable}

\subsection{Data complexity}
\label{subsection-dllite-data}

\newcommand{\refmap}{\mathsf{ref}}
\newcommand{\refof}[1]{\refmap(#1)}

We now consider data complexity, where the landscape is less diverse.
Indeed, we obtain \coNP-completeness for all DLs between
$\dllitebool^{\Hmc}$ and \dllitecore, and both for UCQs and CQs.
\begin{restatable}{theorem}{thmdataupperhornh}
	\label{thm-data-upper-hornh}
	UCQ evaluation on circumscribed $\dllitebool^{\Hmc}$ KBs is in \coNP w.r.t.\ data complexity.
\end{restatable}
To prove Theorem~\ref{thm-data-upper-hornh}, we again guess a model
\Imc of \Kmc with $\Imc \not\models q(\bar a)$ and $|\Delta^\Imc|$
bounded as in Theorem~\ref{lemma-bool-countermodel}, and then verify
that \Imc is minimal w.r.t.\ $<_\CP$.  For the latter, we introduce a
variation of Lemma~\ref{lemma-criteria-combined}.  The original
version of Lemma~\ref{lemma-criteria-combined} is not helpful because
its Point~2 involves deciding whether, given an interpretation
$\Imc_{\Pmc \cup \{e\}}$, there is a model
$\Jmc_e <_\CP \Imc_{\Pmc \cup \{e\}}$ of \Kmc, and given that
$\mn{ind}(\Amc) \subseteq \Imc_{\Pmc \cup \{e\}}$ there is no reason
to believe that this can be done in polynomial time in data
complexity.  We actually conjecture this problem to be \coNP-complete.
In the following, we vary the definition of the interpretations
$\Imc_\Pmc$ so that their size is independent of that of $\Amc$.

Let $\Circ(\Kmc)$ be a $\dllitebool^{\Hmc}$ cKB with
$\kb = (\tbox, \abox)$ and \Imc a model of~\Kmc. We assume \tbox to be
in normal form. For an element $e \in \domain{\I}$, we define its
\emph{ABox type} to be
\[
\atypeinof{\abox}{e} = \{ A \mid A \in \cnames, \kb \models A(e) \}.
\]
Note that \atypeinof{\abox}{e} actually needs not be a proper type as
defined in Section~\ref{section-alchi} due to the presence of
disjunction in $\dllitebool^\Hmc$. If $e \notin \indsof{\abox}$, which
is permitted, then $\atypeinof{\abox}{e}$ is simply the set of all
concepts names $A$ with $\Kmc \models \top \sqsubseteq A$.

For each pair $(t_1, t_2)$ such that $\atypeinof{\abox}{e} = t_1$ and
$\typeinof{\I}{e} = t_2$ for some $e \in \domain{\I}$, we choose such
an element $e_{t_1, t_2}$; we use $E$ to denote the set of all elements chosen in this way.
For each role $r$ used in \Tmc such that
$r^\Imc \neq \emptyset$, we choose an element
$w_r \in (\exists r^-)^\I$. Now define, for every $\Pmc \subseteq \domain{\I}$,
the interpretation $\Imc_\Pmc$ as follows:
\begin{align*}
  \Delta^{\Imc_\Pmc} = \; & \Pmc \cup E \cup \{w_r \mid r \text{ used in } \Tmc, r^\I \neq \emptyset \}
  \\
	\cstyle{A}^{\I_\Pmc} = \; & \cstyle{A}^{\I} \cap \domain{\I_\Pmc}
	\\
	\rstyle{r}^{\I_\Pmc} = \; 
	& 
	\{ (e, w_s) \mid e \in (\exists s)^\I \cap \domain{\I_\Pmc}, \tbox \models s \incl r \}
	\\
	& \cup \{ (w_s, e) \mid e \in (\exists s)^\I \cap \domain{\I_\Pmc}, \tbox \models s \incl r^- \}
\end{align*}
Note that $\sizeof{\I_\Pmc} \leq 4^\sizeof{\tbox} + \sizeof{\tbox} + \sizeof{\Pmc}$.
We also define an ABox
\[
\abox_\Pmc = \{ A(\istyle{a}) \mid a \in \indsof{\abox} \cap \domain{\I_\Pmc}, A \in \atypeinof{\abox}{\istyle{a}} \}
\]
and set $\kb_\Pmc = (\tbox, \abox_\Pmc)$.  The ABoxes $\Amc_\Pmc$ 
act as a decomposition of the ABox \Amc, similarly to how the
interpretations $\Imc_\Pmc$ act as a decomposition of the
interpretation \Imc.  Note that
$\bigcup_\Pmc \Amc_\Pmc$ does not contain role assertions.  This is
compensated by the use of $\atypeinof{\abox}{\istyle{a}}$ in the
definition of $\Amc_\Pmc$ and the fact that, since \Tmc is in normal
form, all relevant endpoints of role assertions are `visible' in the
ABox types.  Note that this fails in the case of $\EL$ where dropping
role assertions could result in CIs $\exists r.B \incl A$ to be left
unsatisfied, and such CIs are crucial for proving
$\Uppi^\textrm{P}_2$-hardness in that case, see the proof of
Theorem~\ref{thm-data-lower-el}.

The following is the announced variation of
Lemma~\ref{lemma-criteria-combined}. 
%
\begin{restatable}{lemma}{lemmacriteriadata}
	\label{lemma-criteria-data}
	The following are equivalent:
	\begin{enumerate}
		\item $\I \models \circkb$; 
		\item $\I_\Pmc \models \Circ(\kb_\Pmc)$ for all $\Pmc \subseteq \domain{\I}$ with $\sizeof{\Pmc} \leq 2\sizeof{\tbox} +1$.
	\end{enumerate}
\end{restatable}
%
In the ``$1 \Rightarrow 2$'' direction we
  use the witnesses $e_{t_1,t_2}$ to extend a potential model
  $\J' <_\CP \I_\Pmc$ of some
$\kb_\Pmc$
   to a model
  $\J <_\CP \I$ of $\kb$, obtaining a contradiction.
  In
  the ``$2 \Rightarrow 1$'' direction, for a potential model
  $\J <_\CP \I$ of $\kb$, we can find a $\Pmc$ with
  $\sizeof{\Pmc} \leq 2\sizeof{\tbox} +1$ so that we can construct
  from $\J$ a model $\J' <_\CP \Imc_\Pmc$ of 
$\kb_\Pmc$.

To establish Theorem~\ref{thm-data-upper-hornh}, it thus suffices to
argue that checking Point~2 of Lemma~\ref{lemma-criteria-data}
can be implemented in time polynomial in $|\Amc|$. We iterate over all
polynomially many
sets \Pmc with $\sizeof{\Pmc} \leq 2\sizeof{\tbox} +1$
(recall that $\tbox$ is fixed)
and check whether
$\I_\Pmc \models \Circ(\kb_\Pmc)$ in a brute force way. To verify
the minimality of $\I_\Pmc$, 
we iterate over all models \Jmc of $\kb_\Pmc$ with
$\Delta^\Jmc=\Delta^{\Imc_\Pmc}$ (of which there are only polynomially many,
thanks to the modified definition of
$\Delta^{\Imc_\Pmc}$), and make sure that
$\Jmc \not<_\CP \Imc_\Pmc$ for any such \Jmc. Overall, we 
obtain a \coNP algorithm.




The $\coNP$ upper bound turns out to be tight, even for
\dllitecore cKBs and CQs, and with very restricted circumscription
patterns. We reduce from 3-colorability. 
\begin{restatable}{theorem}{thmdatalowerdllitepos}
	\label{thm-data-lower-dllitepos}
	CQ evaluation on circumscribed \dllitecore KBs is \coNP-hard
        w.r.t.\
        data complexity.
        This holds even with a single minimized concept name, no fixed
        concept names, and no disjointness constraints.
\end{restatable}

\section{Atomic Queries}
\label{section-atomic-queries-el}

We study the evaluation of atomic queries on circumscribed KBs, which
is closely related to concept satisfiability w.r.t.\ such KBs. In
fact, the two problems are mutually reducible in polynomial time.
Our
results from this section, summarized in
Table~\ref{tab-results-instance}, can thus also be viewed as completing the
complexity landscape for concept satisfiability, first studied in
\cite{Bonatti2006,DBLP:journals/jair/BonattiFS11}.


\begin{table*}
	\centering
	\begin{tabular}{cccc}
		& $\EL$, $\ALCHIO$ 
		& $\dllitebool$, $\dlliteboolh$ 
		& $\dllitecore$, $\dllitehornh$ 
		\\ \midrule
		\it Combined complexity 
		& $\coNExp^{\NPclass}$\complete$^{(\dagger)}$ 
		\stylingref{(Thm.\;\ref{thm-combined-upper-alchio-aq}, \ref{thm-combined-lower-el-aq})}
		& $\coNExp$\complete
		\stylingref{(Thm.\;\ref{thm-combined-upper-boolh-instance}, \ref{thm-combined-lower-bool-instance})}
		& $\Uppi^\textrm{P}_2$\complete$^{(\ddagger)}$ 
		\stylingref{(Thm.\;\ref{thm-combined-upper-hornh-instance})}
		\smallskip\\
		\it Data complexity 
		& $\Uppi^\textup{P}_2$\complete
		\stylingref{(Thm.\;\ref{thm-data-upper-alchi}, \ref{thm-data-lower-el})}
		& in $\PTime$ 
		\stylingref{(Thm.\;\ref{thm-data-upper-boolh-instance})}
		& in $\PTime$ 
		\stylingref{(Thm.\;\ref{thm-data-upper-boolh-instance})}
		\\\bottomrule
	\end{tabular}
	\caption{Complexity of AQ evaluation on circumscribed KBs. $^{(\dagger)}$: completeness already known for $\ALC(\Imc\Omc)$. $^{(\ddagger)}$: hardness already known.}
	\label{tab-results-instance}
\end{table*}



\subsection{Between $\ALCHIO$ and $\EL$}
\label{subsection-instance-alchi}

Concept satisfiability w.r.t.\ circumscribed $\ALCIO$ KBs was proved to be $\coNExpTime^\NPclass$-complete in \cite{Bonatti2009}.  The proof of the upper bound can easily be extended to cover also role inclusions. We thus obtain:
\begin{theorem}\textnormal{\cite{Bonatti2009}}
  	\label{thm-combined-upper-alchio-aq}
        AQ evaluation on circumscribed \ALCHIO KBs is in
        $\coNExpTime^\NPclass$ w.r.t.\ combined complexity.  
\end{theorem}
An alternative way to obtain Theorem~\ref{thm-combined-upper-alchio-aq} for \ALCHI is to use Lemma~\ref{lemma-alchi-countermodel}, which yields the existence of single exponentially large countermodels in the special case where the query $q$ is of constant size (here $\sizeof{q} = 1$), and a straightforward guess-and-check procedure as sketched in Section~\ref{subsection-dllite-combined}.
This can be lifted to \ALCHIO by a variation of (the proof of)
Proposition~\ref{prop:nonom} tailored  to AQs.
%

We next prove a matching lower bound for \EL, improving on an \ExpTime lower bound from~\cite{DBLP:journals/jair/BonattiFS11}.
\begin{restatable}{theorem}{thmcombinedlowerelaq}
	\label{thm-combined-lower-el-aq}
  AQ evaluation on circumscribed \EL KBs is $\coNExpTime^\NPclass$-hard w.r.t.\ combined
  complexity.
  This holds even without
  fixed concept names and with an empty preference order.
\end{restatable}
This is proved by a reduction from AQ evaluation on \ALC cKBs, which is known to be $\coNExpTime^\NPclass$-hard \cite{Bonatti2009}.

Regarding data complexity,  it suffices to recall that
Theorem~\ref{thm-data-lower-el} applies even to AQs and
$\EL$ cKBs.

\subsection{$\dllite$}
\label{subsection-instance-dllite-combined}

Recall that in Section~\ref{subsection-dllite-combined}, we have
proved that UCQ evaluation over \dllitebool cKBs is in \coNExpTime
w.r.t.\ combined complexity.  We started with a guess-and-check
procedure that gives a $\coNExpTime^\NPclass$ upper bound,
relying on countermodels of single exponential size as per
Lemma~\ref{lemma-bool-countermodel}, and then improved to \coNExpTime
using Lemma~\ref{lemma-criteria-combined}.  Here, we use the same
algorithm.  The only difference in the correctness proof is that
Lemma~\ref{lemma-bool-countermodel} is replaced with
Lemma~\ref{lemma-alchi-countermodel} as the former does not support
role inclusions and the latter delivers a single exponential upper
bound for queries of constant size.
%
%
%
\begin{restatable}{theorem}{thmcombinedupperboolhinstance}
	\label{thm-combined-upper-boolh-instance}
	AQ evaluation on circumscribed $\dlliteboolh$ KBs is in $\coNExpTime$ w.r.t.\ combined complexity.
\end{restatable}

We match this upper bound even in the absence of role inclusions,
demonstrating that evaluating AQs and UCQs over \dllitebool cKBs is
equally difficult.

\begin{restatable}{theorem}{thmcombinedlowerboolinstance}
  \label{thm-combined-lower-bool-instance}
	AQ evaluation on circumscribed \dllitebool KBs is \coNExpTime-hard w.r.t.\ combined
	complexity.
\end{restatable}
The proof is by reduction from the complement of the
$\NExpTime$-complete problem \succinct\tcol\ and relies on fixed concept names.

The situation is more favorable if we restrict our attention to $\dllitehornh$ cKBs, which is still a very expressive dialect of the \dllite family.
For $\dllitehornh$, we prove that if there is a countermodel, then there is one of linear size, in drastic contrast to the \dllitebool case where the proof of Theorem~\ref{thm-combined-lower-bool-instance} shows that exponentially large countermodels cannot be avoided.
\begin{restatable}{theorem}{thmcombinedupperhornhinstance}
	\label{thm-combined-upper-hornh-instance}
	AQ evaluation on circumscribed $\dllitehornh$ KBs is in
        $\Uppi^\textup{P}_2$ w.r.t.\ combined complexity.
\end{restatable}



A matching lower bound for $\dllitecore$ cKBs can be found in \cite{DBLP:journals/jair/BonattiFS11}.

\label{subsection-instance-dllite-data}

%
%
%
%

We now move to data complexity, where we obtain tractability even for the maximally expressive $\dlliteboolh$ dialect of \dllite.
\begin{theorem} \label{thm-data-upper-boolh-instance} AQ evaluation on circumscribed $\dlliteboolh$ KBs is in $\PTime$ w.r.t.\ data complexity.  
\end{theorem}
We sketch the proof of Theorem~\ref{thm-data-upper-boolh-instance}.
Let $\circkb$ be a $\dlliteboolh$ cKB with $\kb = (\tbox, \abox)$, $A_0(x)$ an AQ, and $a_0 \in \mn{Ind}(\Amc)$. We 
construct an ABox $\Amc'$ whose size is independent of that of \Amc
and which can replace \Amc when deciding whether $a_0$ is an answer to
$A_0$.  The construction of $\Amc'$ may be viewed as a variation of
the constructions that we have used in the proofs of
Theorems~\ref{thm-combined-upper-dllitebool}
and~\ref{thm-data-upper-hornh} to avoid an \NPclass oracle.  In
particular, it reuses the notion of ABox types from the latter.
However, the construction employed here works directly with ABoxes
rather than with countermodels.

Let $\mn{TP}(\Amc)$ denote the set of all ABox types $t$ realized
in~\Amc, that is, all $t$ such that $\mn{tp}_\Amc(a)=t$ for some
$a \in \mn{Ind}(\Amc)$. For every $t \in \mn{TP}(\Amc)$, set
$$m_t = \min( \sizeof{ \{ a \in \indsof{\abox} \mid
  \atypeinof{\abox}{a} = t\} },~ 4^\sizeof{\tbox})$$ and choose
$m_t$ individuals $a_{t, 1}, \dots, a_{t, m_t} \in \mn{ind}(\Amc)$ such that
$\mn{tp}_\Amc(a_{t,i})=t$ for $1 \leq i \leq m_t$. We assume that the
individual $a_0$ of interest is among the chosen ones. Let $W$ be the
set of all chosen individuals, that is, 
$W = \{ a_{t, i} \mid t \in \mn{TP}(\Amc),~ 1 \leq i \leq 
m_t \}$. Now define an ABox
$$
\begin{array}{rcl}
\abox' &=& \{ A(\istyle{a}) \mid a \in W, A \in
\atypeinof{\abox}{\istyle{a}} \}
\end{array}
$$
and set $\kb' = (\tbox, \abox')$. Note that
$|\indsof{\abox'}| \leq 8^\sizeof{\tbox}$,
and thus the size of $\abox'$ depends only on the TBox \Tmc,
but not on the original ABox $\abox$. Also, note that,
just like the ABoxes $\Amc_\Pmc$ from the proof of
Theorem~\ref{thm-data-upper-hornh}, $\Amc'$ no longer contains role
assertions.
\begin{restatable}{lemma}{lemmadataindep}
	\label{lemma-data-indep}
	$\circkb \models A_0(a_0)\!$ iff $\Circ(\kb') \models A_0(a_0)$. 
\end{restatable}
It is important for the `only if' direction of
Lemma~\ref{lemma-data-indep} that we keep at least $4^{|\Tmc|}$
individuals of each ABox type~$t$ (if existent).  In fact, this allows
us to convert a model $\I'$ of $\Circ(\kb')$ into a model of \kb that
is minimal w.r.t.\ $<_\CP$, using arguments similar to those in the proof of
Lemma~\ref{lem-lemma5}. A crucial point is that if
$a \in \Ind(\Amc) \setminus \Ind(\Amc')$ and $t=\mn{tp}_\Amc(a)$,
then $m_t =4^{|\Tmc|}$ which implies that there is a
regular type $t'$ such that the combination $(t, t')$ is realized at
least $2^\sizeof{\tbox}$ many times in $\I'$ among the individuals
from $\Amc'$.

%
%
Lemma~\ref{lemma-data-indep} gives $\PTime$ membership as the size of $\abox'$ is bounded by a constant. We compute $\kb'$
in polynomial time and check whether
$\Circ(\kb') \models A_0(a_0)$, which can be decided in
$\TwoExpTime$ by Theorem~\ref{thm-combined-upper-alchi}, that is, in
constant time w.r.t.\ data complexity.  The correctness of this
procedure immediately follows from Lemma~\ref{lemma-data-indep}.

\subsection{Negative role inclusions}
\label{subsection-negative-ri}

DL-Lite is often defined to additionally include negative role
inclusions of the form $r \sqsubseteq \neg s$, with the obvious semantics.
It is known that these sometimes lead to increased computational complexity; see, for example, \cite{maniere:thesis}.
We close by observing that this is also the case for circumscription.
While querying circumscribed DL-Lite KBs (in all considered dialects) is \coNP-complete w.r.t.\ data complexity for (U)CQs and in \PTime for AQs, adding negative role inclusions results in a jump back to $\Uppi^\textrm{P}_2$.
We prove this by reduction from $\forall\exists\mn{3SAT}$.
Some ideas are shared with the proof of
Theorem~\ref{thm-data-lower-el}, but the general strategy of the
reduction is different.
\begin{restatable}{theorem}{thmdatalowercoreh}
  \label{thm-data-lower-coreh} AQ evaluation on
  circumscribed $\dllitecoreh$ KBs with negative role inclusions is $\Uppi^\textrm{P}_2$-hard.
  This holds even without fixed concept names and with a single negative role inclusion.
\end{restatable}

\section{Conclusion}
\label{sec-conclusion}

We have provided a rather complete picture of the complexity of query
evaluation on circumscribed KBs. Some cases, however, remain open. For
example, the lower bounds in combined complexity for UCQ evaluation on
\dllite cKBs given in Theorems~\ref{thm:combined-lower-dlliter}
and~\ref{thm-combined-lower-dllitepos} cannot be improved in an
obvious way to CQs, for which the complexity remains open. Also, the
lower bounds provided in
Theorems~\ref{thm-combined-lower-bool-instance},
\ref{thm-data-lower-coreh} and the $\Uppi^\textrm{P}_2$ one from
\cite{DBLP:journals/jair/BonattiFS11} rely on the preference relation
in circumscription patterns,
and it remains open whether the complexity decreases
when the preference relation is forced to be empty. Finally, it would
be interesting to study query evaluation under ontologies that are
sets of existential rules or formulated in the guarded (negation)
fragment of first-order logic, extended with circumscription. We
believe that these problems are still decidable.

\clearpage
%
%
 \section*{Acknowledgments}
 The research reported in this paper has been supported by the German Research Foundation DFG, as part
 of Collaborative Research Center (Sonderforschungsbereich) 1320 Project-ID 329551904 ``EASE - Everyday
 Activity Science and Engineering'', University of Bremen (\url{http://www.ease-crc.org/}). The research was
 conducted in subprojects ``P02 -- Ontologies with Abstraction'' and ``P05-N -- Principles of Metareasoning for Everyday Activities''.

 The authors acknowledge the financial support by the Federal Ministry of Education and Research of Germany 
 and by the Sächsische Staatsministerium für Wissenschaft Kultur und Tourismus in the program Center of Excellence for AI-research 
 ``Center for Scalable Data Analytics and Artificial Intelligence Dresden/Leipzig'', 
 project identification number: ScaDS.AI

This work is partly supported by BMBF (Federal Ministry of Education and Research)
in DAAD project 57616814 (\href{https://secai.org/}{SECAI, School of Embedded Composite AI})
as part of the program Konrad Zuse Schools of Excellence in Artificial Intelligence.

\bibliographystyle{kr}
\bibliography{main}

\cleardoublepage

\appendix 

\section{Proofs for Section~\ref{subsect:fundamental}}

To complete the proof of Proposition~\ref{prop:nonom}, it suffices
to show the following.
\begin{lemma}
  $\Circ(\Kmc) \models q(\bar a)$ iff 
  $\mn{Circ}_{\mn{CP}'}(\Kmc') \models q'(\bar a)$ for all 
  $\bar a \in \mn{ind}(\Amc)$. 
\end{lemma}
\begin{proof}
  First assume that $\Circ(\Kmc) \not\models q(\bar a)$. Then there is
  a model \Imc of $\Circ(\Kmc)$ such that $\Imc\not\models q(\bar a)$.
  Let $\Imc'$ be defined like \Imc except that
  $A_a^{\Imc'}=B_a^{\Imc'}= \{a \}$ and $D_a^{\Imc'}=\emptyset$ for
  all $a \in N$. It is readily checked that $\Imc'$ is a model of
    $\mn{Circ}_{\mn{CP}'}(\Kmc')$ and that $\Imc' \not \models q'(\bar
    a)$.

    Now assume that
    $\mn{Circ}_{\mn{CP}'}(\Kmc')\not\models q'(\bar a)$. Then there is
    a model $\Imc'$ of $\mn{Circ}_{\mn{CP}'}(\Kmc')$ with
    $\Imc' \not\models q'(\bar a)$. Since $\Imc'$ is a model of
    $\Amc'$, we have $a \in B_a^{\Imc'}$ for all $a \in N$. Moreover,
    if $\{a\} \subsetneq B_a^{\Imc'}$, then we can find a model
    $\Jmc'$ of $\mn{Circ}_{\mn{CP}'}(\Kmc')$ with
    $\Jmc' <_{\CP'} \Imc'$ by setting $B_a^{\Jmc'}=\{a\}$ and
    $D_a^{\Jmc'}=A_a^{\Imc'} \setminus \{ a \}$. This contradicts the
    minimality of $\Imc'$, and, consequently, we must have
    $B_a^{\Imc'} = \{a\}$. But then also $A_a^{\Imc'} = \{ a \}$ since
    $A(a) \in \Amc'$, $A_a \sqcap \neg B_a \sqsubseteq D_a \in \Tmc'$,
    and the disjunct $\exists y \, D_a(y)$ is not satisfied in
    $\Imc'$. It can be verified that, consequently, $\Imc'$ is also a
    model of $\Circ(\Kmc)$, and we are done.
\end{proof}

\lemnormalform*
We omit a proof of the above lemma, which is entirely standard. The same normal form was used for \ALCHI, e.g., in \cite{maniere:thesis}.

\lemmafive*
\begin{proof}
	Assume to the contrary that \Jmc is not a model of $\Circ(\Kmc)$.
	As \Jmc is a model of \Kmc by prerequisite, there is thus a model $\Jmc'$ of \Kmc with $\Jmc' <_\CP \Jmc$.	
	To derive a contradiction, we construct a model $\Imc'$ of \Kmc with $\Delta^{\Imc'} = \Delta^{\Imc}$ and show that $\Imc' <_\CP \Imc$.
	For each type $t \in \mn{TP}_{\overline{\mn{core}}}(\Imc)$,	let
	$$
	\begin{array}{rcl}
		D_t & = & \{ d \in \Delta^\Imc \setminus \mn{ind}(\Amc) \mid \mn{tp}_\Imc(d)=t
		\}   \\[1mm]
		S_t & = & \{ \mn{tp}_{\Jmc'}(d) \mid d \in \Delta^\Jmc: \mn{tp}_\Jmc(d)=t \}.
	\end{array}
	$$
        We have
        \begin{itemize}
        \item $|D_t| \geq |\mn{TP}(\Tmc)|$ by definition of
          $\noncoretypesof{\I}$ and
        \item  $|S_t| \leq |\mn{TP}(\Tmc)|$, by definition of $S_t$.
        \end{itemize}
	Moreover, the `$\supseteq$'	direction of Point~2 implies $S_t \neq \emptyset$.
	We can thus	find a surjective function $\pi_t: D_t \rightarrow S_t$.
	
	Let $\pi$ be the function that is obtained as the union of 
        the (domain disjoint) functions $\pi_t$, for all $t \in
        \mn{TP}_{\overline{\mn{core}}}(\Imc)$.
	By definition of $\pi$ and of $\Delta^\Imc_\mn{core}$, the domain of $\pi$ is 
        $$\Delta^\Imc \setminus 
	\big (\Delta^\Imc_\mn{core} \cup \mn{ind}(\Amc) \big ).$$
        Now consider the range of $\pi$.
        By Point~1 and the `$\subseteq$' direction of Point~2,
        $\mn{tp}_\Jmc(d) \in  \mn{TP}_{\overline{\mn{core}}}(\Imc)$
        iff
        $d \notin \Delta^\Imc_\mn{core}$. By definition, the range
        of $\pi$ is thus 
	$$\{ \mn{tp}_{\Jmc'}(d) \mid d \in \Delta^\Jmc
        \setminus \Delta^\Imc_\mn{core}\}.$$
        %

	Extend $\pi$ to domain $\Delta^\Imc$ by setting $\pi(d)=\mn{tp}_{\Jmc'}(d)$ for all
	$d \in \Delta^\Imc_\mn{core} \cup \mn{ind}(\Amc)$.
	It is easy to see that now, $\pi$ is a surjective function from $\Delta^\Imc$ to $\mn{TP}(\Jmc')$.
	We construct an interpretation $\Imc'$ as follows:
	\begin{align*}
		\Delta^{\Imc'} & =	\Delta^\Imc                                                                               \\
		\cstyle{A}^{\Imc'}      & =	\{ d \in
		\Delta^{\Imc'}\mid \cstyle{A} \in
		\pi(d) \}
		\\
		\rstyle{r}^{\Imc'}      & =	\{ (d,e) \in \Delta^{\Imc'}\times \Delta^{\Imc'} \mid \pi(d) \rightsquigarrow_\rstyle{r} \pi(e) \} 
	\end{align*}
        for all concept names $A$ and role names $r$.  Notice that
        $\typeinof{\I'}{d} = \pi(d)$, hence
        $\types(\I') \subseteq \types(\Jmc')$.  We first show that
        $\Imc'$ is a model of \Kmc by considering each possible shape
        of assertions and inclusion (recall that $\tbox$ is {in}
        normal form):
	\begin{itemize}
		
		\item{$\conceptassertion{A}{a}$.}
		Since $\Jmc'$ is a model of $\kb$, we have $\istyle{a} \in \cstyle{A}^{\Jmc'}$, hence $\cstyle{A} \in \typeinof{\Jmc'}{\istyle{a}}$.
		By definition of $\pi$, we have $\pi(\istyle{a}) = \typeinof{\Jmc'}{\istyle{a}}$, hence $\cstyle{A} \in \pi(\istyle{a})$.
		The definition of $\cstyle{A}^{\I'}$ yields $\istyle{a} \in \cstyle{A}^{\I'}$.
		
		\item{$\roleassertion{r}{a}{b}$.}
		Since $\Jmc'$ is a model of $\kb$, we have $(\istyle{a}, \istyle{b}) \in \rstyle{r}^{\Jmc'}$, hence $\typeinof{\Jmc'}{\istyle{a}} \legitsucc_\rstyle{r} \typeinof{\Jmc'}{\istyle{b}}$.
		By definition of $\pi$, we have $\pi(\istyle{a}) = \typeinof{\Jmc'}{\istyle{a}}$ and $\pi(\istyle{b}) = \typeinof{\Jmc'}{\istyle{b}}$.
		The definition of $\rstyle{r}^{\I'}$ yields $(\istyle{a}, \istyle{b}) \in \rstyle{r}^{\I'}$.
		
		\item
		Satisfaction of CIs of shape $\axtop$, $\axand$,
                $\axnotright$ and $\axnotleft$ immediately follows
                from $\types(\I') \subseteq \types(\Jmc')$ and $\Jmc'$
                being a model of \Tmc.
		
		\item{$\axexistsright$.}
		Let $d \in \cstyle{A}^{\I'}$.
		Then $\cstyle{A} \in \pi(d)$. 
		Since $\pi(d) \in \types(\Jmc')$, there exists $d' \in \domain{\Jmc'}$ such that $\typeinof{\Jmc'}{d'} = \pi(d)$.
		Since $\Jmc'$ is a model of $\tbox$, there exists $e' \in \cstyle{B}^{\Jmc'}$ such that $(d', e') \in \rstyle{R}^{\Jmc'}$.
		In particular, $\cstyle{B} \in \typeinof{\Jmc'}{e'}$ and $\typeinof{\Jmc'}{d'} \legitsucc_\rstyle{R} \typeinof{\Jmc'}{e'}$.
		It follows from $\pi$ being surjective that there exists $e \in \domain{\I'}$ with $\pi(e) = \typeinof{\Jmc'}{e'}$.
		Then $e \in \cstyle{B}^{\I'}$ and $(d, e) \in
                \rstyle{R}^{\I'}$ by definition of $\I'$, hence $d \in (\exists\rstyle{R}.\cstyle{B})^{\I'}$. 
		
		\item{$\axexistsleft$.}
		Let $d \in (\exists\rstyle{R}.\cstyle{B})^{\I'}$. 
		Then there exists $e \in \cstyle{B}^{\I'}$ such that $(d, e) \in \rstyle{R}^{\I'}$.
		By definition of $\rstyle{R}^{\I'}$, we have $\pi(d) \legitsucc_\rstyle{R} \pi(e)$.
		Since $\cstyle{B} \in \typeinof{\I'}{e} = \pi(e)$ and $\axexistsleft$, the definition of $\legitsucc_\rstyle{R}$ yields $\cstyle{A} \in \pi(d) = \typeinof{\I'}{d}$.
		Thus $d \in \cstyle{A}^{\I'}$.

		\item{$r \sqsubseteq s$.}
		Let $(d, e) \in \rstyle{r}^{\I'}$. Then $\pi(d) \legitsucc_\rstyle{r} \pi(e)$.
		Since $\tbox \models r \sqsubseteq s$, it is immediate that $\pi(d) \legitsucc_\rstyle{s} \pi(e)$.
		Therefore the definition of $\rstyle{s}^{\I'}$ yields $(d, e) \in \rstyle{s}^{\I'}$.
		
	\end{itemize}
	It remains to show that $\Imc' <_\CP \Imc$.
	We make use of the following claim.
	\begin{claim}
		$A^{\Imc'} \odot A^\Imc$ iff $A^{\Jmc'} \odot A^\Jmc$,
                for each concept name $A$ and $\odot \in \{\subseteq, \supseteq\}$.
	\end{claim}
	
		Let $A$ be a concept name.
		We prove the claim for $\odot =\; \subseteq$ only; for $\odot =\; \supseteq$, the arguments are similar.

                \smallskip
		``$\Rightarrow$''.
		Let $A^{\Imc'} \subseteq A^\Imc$ and $d \in
                A^{\Jmc'}$; we need to show that $d \in A^{\Jmc}$.
                First assume that $d \in \Delta_\mn{core}^\Imc$. Then
                $\mn{tp}_\Imc(d)=\mn{tp}_\Jmc(d)$ and
                $\mn{tp}_{\Imc'}(d) = \mn{tp}_{\Jmc'}(d)$. So
                $A^{\Imc'} \subseteq A^\Imc$ clearly implies
                $d \in A^\Jmc$, as required. Now assume that
                $d \notin \Delta_\mn{core}^\Imc$. Let
                $t=\mn{tp}_\Jmc(d)$. Then $\mn{tp}_{\Jmc'}(d) \in S_t$
                and consequently we find a $d' \in D_t$ with
                $\pi(d')=\mn{tp}_{\Jmc'}(d)$. This implies
                $\mn{tp}_{\Imc'}(d')=\mn{tp}_{\Jmc'}(d)$.  Since
                $d' \in D_t$, we have
                $\mn{tp}_\Imc(d)=t=\mn{tp}_\Jmc(d)$.
                So again,  $A^{\Imc'} \subseteq A^\Imc$ implies
                $d \in A^\Jmc$, as required.

                \smallskip
		``$\Leftarrow$''.
		Let $A^{\Jmc'} \subseteq A^\Jmc$ and $d \in
                A^{\Imc'}$; we need to show $d \in A^{\Imc}$.
                First assume that $d \in \Delta_\mn{core}^\Imc$. Then
                $\mn{tp}_\Imc(d)=\mn{tp}_\Jmc(d)$ and
                $\mn{tp}_{\Imc'}(d) = \mn{tp}_{\Jmc'}(d)$. So
                $A^{\Jmc'} \subseteq A^\Jmc$ clearly implies
                $d \in A^\Imc$, as required. Now assume that
                $d \notin \Delta_\mn{core}^\Imc$ and let
                $t'=\mn{tp}_{\Imc'}(d)$. By definition of $\Imc'$, we
                have $t'=\pi(d)$. By definition of $\pi$, there is a
                $t \in \mn{TP}_{\overline{\mn{core}}}(\Imc)$ such that
                $t'=\pi_t(d)$. Then $d \in D_t$ and $t' \in S_t$. The
                former yields $\mn{tp}_\Imc(d)=t$ and due to the
                latter, there is a $d'$ such that
                $\mn{tp}_\Jmc(d)=t=\mn{tp}_\Imc(d)$ and
                $\mn{tp}_{\Jmc'}(d)=t'=\mn{tp}_{\Imc'}(d)$.
                So again,  $A^{\Jmc'} \subseteq A^\Jmc$ implies
                $d \in A^\Imc$, as required.

                \medskip

        It is easy to see that since
	$\Imc' <_\CP \Imc$, the claim implies $\Jmc'<_\CP \Jmc$.
	We have derived a contradiction and conclude that $\Jmc$ is a
        model of $\Circ(\Kmc)$, as desired.
\end{proof}

\leminterlacingiscountermodel*
\begin{proof}
  Let $\I$ be a countermodel against $\circkb \models q(\bar a)$. It
  is clear that $\Imc' \not\models q[\bar a]$ since any homomorphism
  $g$ from $q$ to $\Imc'$ with $g(\bar x)=\bar a$ could be composed
  with the homomorphism $h$ from $\Imc'$ to \Imc to show that
  $\Imc \models q[\bar a]$. It is also straightforward to verify that
  \begin{itemize}
  \item[($*$)] 
    $\mn{tp}_{\Imc'}(d)=\mn{tp}_\Imc(h(d))$ for all $d \in \Delta^{\Imc'}$.
  \end{itemize}
  Using this and the
  definition of role extensions in $\Imc'$, it can be verified that
  $\Imc'$ is a model of \Kmc. 
  Let us consider each possible shape of assertions and axioms in our normal form:
  \begin{itemize}
  	
  	\item
  	Assertions from $\abox$ are clearly satisfied as $\I$ is a model of $\abox $ and $\I'$ preserves $\I|_{\basedomainof{\I}}$ (recall $\indsof{\abox} \subseteq \basedomainof{\I}$).
  	
  	\item
  	Satisfaction of CIs with shape $\axtop$, $\axand$,
  	$\axnotright$ and $\axnotleft$ follows from remark ($*$) on types.
  	
  	\item{$\axexistsright$.}
  	Let $d \in \cstyle{A}^{\I'}$.
  	Thus $h(d) \in \cstyle{A}^\I$, and since $\I$ models $\kb$, $d' = f(h(d), r.B)$ is defined.
  	If $d' \in \coredomainof{\I}$, then definition of $\I'$ yields $(d, d') \in r^{\I'}$ and $d' \in B^{\I'}$.
  	Otherwise, $d r.B$ is a path through $\I$, and, by definition of $\I'$, we have in this case $(d, d r.B) \in r^{\I'}$ and $d r.B \in B^{\I'}$.
  	In both cases, $d \in (\exists r.B)^{\I'}$.
  	
  	\item{$\axexistsleft$.}
  	Let $d \in (\exists\rstyle{R}.\cstyle{B})^{\I'}$.
  	That is, there exists $e \in \cstyle{B}^{\I'}$ such that $(d, e) \in \rstyle{R}^{\I'}$.
  	Therefore $h(e) \in \cstyle{B}^{\I}$ and $(h(d), h(e)) \in \rstyle{R}^{\I}$, ie $h(d) \in (\exists\rstyle{R}.\cstyle{B})^{\I}$.
  	From $\I$ being a model of $\kb$, it follows $h(d) \in A^\I$, thus $d \in A^{\I'}$.

  	\item Axioms with shape $r \sqsubseteq s$ are clearly satisfied from the definition of $s^{\I'}$.
  	
  \end{itemize}

  It remains to prove that $\interlacing$ is minimal w.r.t.\ $<_\CP$.
  We use Lemma~\ref{lem-lemma5} with \Imc as the reference model, and it
  suffices to show that the preconditions of that lemma are
  satisfied. Clearly,
  $\coredomainof{\I} \subseteq \domain{\interlacing}$. Moreover, $(*)$
  and the fact that $h(d)=d$ for all $d \in \coredomainof{\I}$ implies
  that Condition~1 of Lemma~\ref{lem-lemma5} is satisfied.  We next
  verify that Condition~2 is also satisfied. First, let
  $d \in \Delta^\Jmc \setminus \coredomainof{\I}$. We have to show
  that $\mn{tp}_\Jmc(d) \in \noncoretypesof{\I}$. If
  $d \in \basedomainof{\I}$, then $h(d)=d$. From
  $d \notin \coredomainof{\I}$, we get
  $\mn{tp}_\Imc(d) \in \noncoretypesof{\I}$ and it remains to apply
  ($*$).  If $d \notin \basedomainof{\I}$, then
  $h(d) \notin \coredomainof{\I}$ and again we may use ($*$).
  Conversely, let $t \in \noncoretypesof{\I}$. We have to show that
  there is a $d \in \Delta^\Jmc \setminus \coredomainof{\I}$ with
  $\mn{tp}_\Jmc(d)=t$. But that $d$ is $e_t$ since $h(e_t)=e_t$ and by
  ($*$).
\end{proof}

\lemfinitemodelpropertywrtcore*
\begin{proof}
	Let $\Kmc=(\Tmc,\Amc)$ and let $\I$ be a model of $\circkb$.
	Assume that $t_i \notin \Delta^\I$, for every $t \in
        \noncoretypesof{\I}$ and $1 \leq i \leq m$. Define $\Jmc$ by setting
	\[
	\begin{array}{r@{~}c@{~}l}
		\domain{\Jmc} & = & \indsof{\abox} \cup \coredomainof{\I} \cup \{ t_i \mid t \in \noncoretypesof{\I}, 1 \leq i \leq m \}
		\medskip \\
		\cstyle{A}^{\Jmc} & = & (\cstyle{A}^{\I} \cap \domain{\Jmc}) \cup \{ t_i \mid \cstyle{A} \in t, 1 \leq i \leq m \}
		\medskip \\
		\rstyle{r}^{\Jmc} & = & (\rstyle{r}^{\I} \cap (\domain{\Jmc} \times \domain{\Jmc}))  \hfill \text{(i)}
		\smallskip \\ & &
		\cup \, \{ (e, t_i) \mid 
		e \in \domain{\I}, \typeinof{\I}{e}
                                  \legitsucc_{\rstyle{r}} t, 1 \leq i
                                  \leq m \}  \hfill \text{(ii)}
		\smallskip \\ & &
		\cup \, \{ (t_i, e) \mid 
		e \in \domain{\I}, t \legitsucc_{\rstyle{r}}
                                  \typeinof{\I}{e}, 1 \leq i \leq m \}  \hfill \text{(iii)}
		\smallskip \\ & &
		\cup \, \{ (t_i, t'_j) \mid 
		t \legitsucc_{\rstyle{r}} t', 1 \leq i,j \leq m \} \hfill \text{(iv)}
	\end{array}
      \]
      for all concept names $A$ and role names $r$. It is easy to see
      that       $\typeinof{\Jmc}{e} = \typeinof{\I}{e}$
      for all $e \in \domain{\Imc} \cap \domain{\Jmc}$ and
      $\typeinof{\Jmc}{t_i} = t$ for all $t \in \types(\Tmc)$ and $1
      \leq i \leq m$.  This implies that
      $\I_{\mid \coredomainof{\I}} = \Jmc_{\mid \coredomainof{\Jmc}}$
      and $\noncoretypesof{\I} = \noncoretypesof{\Jmc}$ as desired.
      It also implies that Conditions 1 and 2 from Lemma~\ref{lem-lemma5} are satisfied.
      To show that \Jmc is a model of $\circkb$, it thus suffices to
      show
      that \Jmc is a model of $\kb$. It is clear by construction that
      \Jmc is a model of \Amc. For \Tmc, we consider all different
      forms of axioms:
	\begin{itemize}
		
		
		
		\item
		Satisfaction of CIs of the form$\axtop$, $\axand$,
                $\axnotright$ and $\axnotleft$  immediately follows
                from $\types(\Jmc) \subseteq \types(\I)$ and $\I$
                being a model of \Tmc. We distinguish two cases:
		
		\item{$\axexistsright$.}
		Let $d \in \cstyle{A}^{\Jmc}$. 
		\begin{itemize}
		\item If $d \in \domain{\I}$, then we have $d \in \cstyle{A}^{\I}$.
		Since $\I$ is a model of $\Tmc$, there exists $e \in
                \cstyle{B}^{\I}$ such that $(d, e) \in \rstyle{R}^{\I}$.
		If $e \in \domain{\I}$, then $(d, e) \in
                \rstyle{R}^{\Jmc}$ due to Case~(i) in the definition
                of~$r^\Jmc$.
		Otherwise $t := \typeinof{\I}{e} \in \noncoretypesof{\I}$.
		In particular, $\cstyle{B} \in t$ and $\typeinof{\I}{d} \legitsucc_\rstyle{R} t$.
		Therefore $t_1 \in \cstyle{B}^\Jmc$ and $(d, t_1) \in
                \rstyle{R}^\Jmc$ due to Case~(ii), proving $d \in (\exists\rstyle{R}.\cstyle{B})^{\Jmc}$.
		\item If $d \notin \domain{\I}$, then $d = t_i$ for some $t \in \noncoretypesof{\I}$ and $1 \leq i \leq m$.
		By definition of $\noncoretypesof{\I}$, there exists $d' \in \domain{\I}$ s.t. $\typeinof{\I}{d'} = t = \typeinof{\Jmc}{d}$.
		From $d \in \cstyle{A}^{\Jmc}$, we get $A \in t$.
		Therefore $d' \in \cstyle{A}^{\I}$, and since $\I$ is
                a model of $\tbox$, there exists $e \in
                \cstyle{B}^{\I}$ such that $(d', e) \in \rstyle{R}^{\I}$.
		In particular, $\cstyle{B} \in \typeinof{\I}{e}$ 
                and $t \legitsucc_\rstyle{R} \typeinof{\I}{e}$.
		If $e \in \domain{\Jmc}$, then $e \in B^\Jmc$ and
                Case~(iii) in the definition of $\rstyle{R}^{\Jmc}$
                ensures $(d, e) \in \rstyle{R}^\Jmc$, which gives $d
                \in (\exists\rstyle{R}.\cstyle{B})^{\Jmc}$.
                If $e \notin \domain{\Jmc}$, then $t'=\typeinof{\I}{e} \in
                \noncoretypesof{\I}$. Let $e'= t'_j$ for some $j$ with
                $1 \leq j \leq m$.
		Case~(iv) in the definition of $\rstyle{R}^{\Jmc}$ yields $(d, e') \in \rstyle{R}^\Jmc$, which again gives $d \in (\exists\rstyle{R}.\cstyle{B})^{\Jmc}$.
		\end{itemize}
		
		\item{$\axexistsleft$.}
		Let $d \in (\exists\rstyle{R}.\cstyle{B})^{\Jmc}$.
		Then there exists $e \in \cstyle{B}^{\Jmc}$ such that $(d, e) \in \rstyle{R}^{\Jmc}$.
		From each case in definition of $\rstyle{R}^{\Jmc}$, we easily get $\typeinof{\Jmc}{d} \legitsucc_\rstyle{R} \typeinof{\Jmc}{e}$.
		Since $e \in \cstyle{B}^{\Jmc}$ we have in particular $\cstyle{B} \in \typeinof{\Jmc}{e}$.
		Combined with $\tbox \models \axexistsleft$ and the definition of $\typeinof{\Jmc}{d} \legitsucc_\rstyle{R} \typeinof{\Jmc}{e}$, we obtain $\cstyle{A} \in \typeinof{\Jmc}{d}$, that is $d \in \cstyle{A}^{\Jmc}$.

		\item{$r \sqsubseteq s$.}
		Let $(d, e) \in \rstyle{r}^{\Jmc}$.
		If this is due to Case~(i) in the definition of
                $\rstyle{r}^{\Jmc}$, then
                $(d, e) \in \rstyle{s}^{\Jmc}$ since $\I$ is a model of~$\Tmc$.
		For all other three cases, it suffices to note that,
                due to $r \sqsubseteq s \in \tbox$, $t_1
                \legitsucc_\rstyle{r} t_2$ implies $t_1
                \legitsucc_\rstyle{s} t_2$ for all types $t_1, t_2$.
		
	\end{itemize}
	Finally, we remark that the size of $\domain{\Jmc}$ is bounded by $\sizeof{\indsof{\abox}} + \sizeof{\coredomainof{\I}} + m \cdot \sizeof{\noncoretypesof{\I}}$.
	From $\sizeof{\coredomainof{\I}} \leq m \cdot
        \sizeof{\coretypesof{\I}}$ and $\sizeof{\coretypesof{\I}} +
        \sizeof{\noncoretypesof{\I}} = \sizeof{\types(\I)} \leq m$ we obtain
	$\sizeof{\domain{\Jmc}} \leq \sizeof{\abox} + m^2$. 
	Recalling $m = \sizeof{\types(\tbox)} \leq 2^\sizeof{\tbox}$,
        we obtain $|\Delta^\Jmc| \leq\sizeof{\abox} +
        2^{2\sizeof{\tbox}}$ as required.
\end{proof}

\section{Proofs for Section~\ref{subsection-alchi-combined}}

\lemmamosaic*


\renewcommand{\countermodel}{{\I}}
\newcommand{\patchoice}[3][\pattern]{\mathsf{ch}^{#2}_{#1, #3}}
\newcommand{\patref}{\mathsf{ref}}
\newcommand{\target}{\mathsf{target}}
\newcommand{\patrefof}[1]{\patref(#1)}
\newcommand{\targetof}[1]{\target(#1)}
\newcommand{\relevantpat}{\candidatepattern^+}

To prove the ``$\Leftarrow$'' direction of Lemma~\ref{lemma-mosaic}, assume we are given 
a good mosaic $\candidatepattern$. 
We aim to extend $\interpof[0]$, hence $\basecandidate$, into a model \Imc of \Tmc that
satisfies Conditions~(a) to~(c).
For each good mosaic $\pattern$, each $e \in (A \sqcap \neg\exists r . B)^{\Jmc_M}$ for some $A$ such that $A \sqsubseteq \exists r . B \in \Tmc$,
we choose a good mosaic $M'$ such that:
\begin{enumerate}
	
	\item $\mn{tp}_{\Jmc_M}(e)=\mn{tp}_{\Jmc_{M'}}(e)$;
	
	\item $e \in (\exists r . B)^{\Jmc_{M'}}$;
	
	\item if $(p,\widehat p,h) \in S_M$, then $(p,\widehat p,h') \in S_{M'}$ where
	$h'$ is the restriction of $h$ to range $\Delta^{\Imc_{\mn{base}}}
	\cup \{e \}$.
\end{enumerate}
We use $\patchoice{\headRule}{e}$ to denote this chosen $M'$.
Then, starting from $\candidatepattern$, we build a tree-shaped set of words which witnesses the acceptance of $\candidatepattern$.

\begin{definition}
	The mosaic tree $\relevantpat$ is the smallest set of words such that:
	\begin{itemize}
		\item $(\candidatepattern, \emptyset, \emptyset) \in \relevantpat$;
		\item If $w \cdot (\pattern, \_, \_) \in \relevantpat$ and $e \in (A \sqcap \neg
		\exists r . B)^{\Jmc_M}$ for some $A$ such that $A \sqsubseteq \exists r . B \in \Tmc$, then $w \cdot (\pattern, \_, \_) \cdot (\patchoice{\headRule}{e}, \headRule, e) \in \relevantpat$.
	\end{itemize}
\end{definition}


It remains to `glue' together 
the interpretations $\interpof$ according to the structure of $\relevantpat$.
Since a mosaic $\pattern$ may occur more than once, 
we create a copy of $\interpof$ for each 
node in $\relevantpat$ of the form $w \cdot (\pattern, \_, \_)$.
We do not duplicate however elements from $\basecandidate$ as they precisely are those we want to reuse.
Hence only the two special elements $\frsymb$ and $\gensymb$ may be duplicated.
When a node $w \cdot (\pattern[1], r_1.B_1, e_1) \cdot (\pattern[2], r_2.B_2, e_2)$ is encountered, we merge element $e_2$ from $\pattern[2]$ with the already-existing $e_2$ from $\pattern[1]$, which is indicated with a superscript $w \cdot (\pattern[1], r_1.B_1, e_1)$ as follows $e_2^{w \cdot (\pattern[1], r_1.B_1, e_1)}$.
Other elements from $\domain{\interpof[2]} \setminus \domain{\basecandidate}$ are considered freshly introduced by this node, which is indicated with a matching superscript, that is: $e^{w \cdot (\pattern[1], r_1.B_1, e_1) \cdot (\pattern[2], r_2.B_2, e_2)}$.
Formally, the copying and merging of elements is achieved by the 
following duplicating functions, defined for each $w \cdot (\pattern, r.B, e) \in \relevantpat$.
\[
\begin{array}{r@{~}l}
	\caster[w \cdot (\pattern, r.B, e)] : \domain{\interpof} \rightarrow  \domain{\basecandidate} \cup \{ (e_i^*)^w, (e_i^*)^{w \cdot (\pattern, r.B, e)} \}
	\\
	d \mapsto 
	\left\{
	\begin{array}{ll}
		d & \text{if } d \in \domain{\basecandidate}
		\\
		d^w & \text{if } d = e \notin \domain{\basecandidate}
		\\
		d^{w \cdot (\pattern, r.B, e)} & \text{otherwise}
	\end{array}
	\right.
\end{array}
\]
The desired model $\countermodel $ 
can then be defined as:
\[
\countermodel = \bigcup_{\substack{ w \cdot (\pattern, r.B, e) \in \relevantpat }} \casterof[w \cdot (\pattern, r.B, e)]{\interpof},
\]
that is the domain (resp.\ the interpretation of each concept name and each role name) of $\I$ is the union across all $w \cdot (\pattern, r.B, e) \in \relevantpat$ of the image by $\caster[w \cdot (\pattern, r.B, e)]$ of the domain (resp.\ the interpretation of each concept name and each role name) of $\interpof$.

By definition, each $\caster[w \cdot (\pattern, r.B, e)]$ is a homomorphism from $\interpof$ to $\countermodel$.
Due to Condition~1 in the definition of good mosaics, if the range of two duplicating functions overlap, then the common element satisfies the same concepts, 
so concept membership in $\countermodel$ transfers back to $\interpof$.
More formally:

\begin{lemma}
	\label{lemma:concepts-in-patterns}
	For all $w \cdot (\pattern, r.B, e) \in \relevantpat$, all $d \in \domain{\interpof}$ and all $\cstyle{A} \in \cnames$, it holds that $\casterof[w \cdot (\pattern, r.B, e)]{d} \in \cstyle{A}^{\countermodel}$ iff $d \in \cstyle{A}^{\interpof}$.%
\end{lemma}%
\begin{proof}
	Let $u_1 = w_1 \cdot (\pattern[1], r_1.B_1, e_1) \in \relevantpat$ be a node from the mosaic tree, $d_1$ an element from $\domain{\interpof[1]}$ and $\cstyle{A} \in \cnames$.
	\newline ``$\Leftarrow$''.
	This is the easy direction.
	Assume $d_1 \in \cstyle{A}^{\interpof[1]}$.
	It follows from the definition of $\cstyle{A}^\I$ that $\casterof[u_1]{d_1} \in \cstyle{A}^{\countermodel}$.
	\newline ``$\Rightarrow$''.
	Assume $\casterof[u_1]{d_1} \in \cstyle{A}^{\countermodel}$.
	By definition of $\cstyle{A}^{\countermodel}$ there exists a node $u_2 = w_2 \cdot (\pattern[2], r_2.B_2, e_2)$ from the mosaic tree, and an element $d_2 \in \cstyle{A}^{\interpof[2]}$ such that $\casterof[u_1]{d_1} = \casterof[u_2]{d_2}$.
	We refer to this equality as $(*)$ and distinguish 5 cases.
	\begin{enumerate}
		\item $d_1 \in \domain{\basecandidate}$ or $d_2 \in \domain{\basecandidate}$. \newline
		$(*)$ yields $d_1 = d_2$. 
		Interpretation $\interpof[2]$ preserves $\basecandidate$, hence $d_2 \in \cstyle{A}^{\basecandidate}$. 
		Interpretation $\interpof[1]$ preserves $\basecandidate$, hence $d_1 \in \cstyle{A}^{\interpof[1]}$.
		
		\emph{
			In the remaining cases, we thus assume $d_1, d_2 \notin \domain{\basecandidate}$.
		}

		\item
		$d_1 = e_1$ and $d_2 = e_2$. \newline
		$(*)$ yields $\pattern[1] = \pattern[2]$ and $d_1 = d_2$.

		\item 
		$d_1 \neq e_1$ and $d_2 = e_2$. \newline
		$(*)$ yields $w_1 = w_2 \cdot (\pattern[2], r_2.B_2, e_2)$ and $d_1 = d_2$.
		In particular, $\pattern[1] = \patchoice[{\pattern[2]}]{r_2,B_2}{e_2}$.
		From Condition~1 in the definition of good mosaics, we obtain $d_1 \in \cstyle{A}^{\interpof[1]}$.
		
		\item $d_1 = e_1$ and $d_2 \neq e_2$. \newline
		The same arguments as for Case~3 but this time with $\pattern[2] = \patchoice[{\pattern[1]}]{r_1,B_1}{e_1}$.
		
		\item $d_1 \neq e_1 $ and $d_2 \neq e_2$. \newline
		$(*)$ yields the existence of $w \cdot (M, r.B, e)$ such that $w_1 = w_2 = w \cdot (M, r.B, e)$.
		It follows that $\pattern[1] = \patchoice[{\pattern[]}]{r,B}{e}$ and $\pattern[2] = \patchoice[{\pattern[]}]{r,B}{e}$.
		Using the same arguments as in Cases~3 and 4, we obtain $d_1 \in \cstyle{A}^{\interpof[1]}$. \qedhere
	\end{enumerate}
\end{proof}

With Lemma~\ref{lemma:concepts-in-patterns} in hand, {we are ready to show that 
	$\countermodel$ is a model of~$\kb$. }
\begin{lemma}
	\label{lemma:countermodel-is-model}
	$\countermodel$ is a model of $\kb$.
\end{lemma}

\begin{proof}
	We consider each possible shape of assertion and axiom in $\Kmc$ w.r.t.\ the normal form from Lemma~\ref{lem:normalform}.
	\begin{itemize}
		
		\item 
		Assertions $\conceptassertion{A}{a}$ and $\roleassertion{P}{a}{b}$ from $\abox$ are satisfied in $\basecandidate$, which is preserved in $\candidate$ (Condition~2 in the definition of a mosaic), hence are also satisfied in $\countermodel$.
		
		\item
		Axioms with shapes $\axiom{\top}{A}$, $\axiom{A_1 \sqcap A_2}{A}$, $\axnotright$ and $\axnotleft$ are satisfied from Lemma~\ref{lemma:concepts-in-patterns} and the fact every mosaic is only allowed to give proper types to its element.
		
		\item
		Axioms with shapes $\exists{r}.A_1 \incl A_2$ and ${p}\incl {r}$ are satisfied from Lemma~\ref{lemma:concepts-in-patterns} and Condition~4 in the definition of mosaics, precisely requiring those to hold.
		
		\item
		Axioms with shape ${A_1} \incl \exists r.A_2$ are those that really involve the mosaics.
		Let $d_1 \in \rolestyle{A}_1^{\countermodel}$. 
		By definition of $\rolestyle{A_1}^{\countermodel}$, there exist a node $u = w \cdot (\pattern, \_, \_) \in \relevantpat$ and an element $d_1' \in \domain{\interpof}$ such that $d_1' \in \rolestyle{A_1}^{\interpof}$ and $\casterof[u]{d_1'} = d_1$.
		If $d_1' \in (\exists r . A_2)^{\interpof}$, then we pick $d_2' \in A_2^{\interpof}$ and $(d_1', d_2') \in r^{\interpof}$.
		It then follows immediately that $\casterof[u]{d_2'} \in A_2^{\countermodel}$ and $(\casterof[u]{d_1'}, \casterof[u]{d_2'}) \in r^{\countermodel}$, ensuring $d_1 \in (\exists r.A_2)^\countermodel$. 
		Otherwise $d_1' \notin (\exists r . A_2)^{\interpof}$, thus we have $u_0 = u \cdot (\pattern[0], r.A_2, d_1') \in \relevantpat$ with $\pattern[0] = \patchoice{r.A_2}{d_1'}$.
		Notice $\casterof[u]{d_1'} = \casterof[u_0]{d_1'}$.
		From Condition~2 in the definition of good mosaics, we get $d_1' \in (\exists r . A_2)^{\interpof[0]}$ hence we pick $d_2' \in A_2^{\interpof[0]}$ and $(d_1', d_2') \in r^{\interpof[0]}$.
		It then follows immediately that $\casterof[u_0]{d_2'} \in A_2^{\countermodel}$ and $(\casterof[u_0]{d_1'}, \casterof[u_0]{d_2'}) \in r^{\countermodel}$, ensuring $d_1 \in (\exists r.A_2)^\countermodel$.

	\end{itemize}
\end{proof}

This proves $\basecandidate$ can indeed by extended into a model of $\tbox$.
Notice that Condition~(a) is clearly satisfied due to Condition~2 in the definition of mosaics.
Similarly, Condition~(b) is clearly satisfied from Condition~3 in the definition of mosaics joint with Lemma~\ref{lemma:concepts-in-patterns}.
It remains to verify Condition~(c), that is for all $p \in q$, there is no homomorphism from $p$ in $\countermodel$.

By contradiction, assume there exists a CQ $p \in q$ and a homomorphism $\match : p \rightarrow \countermodel$.
Notice that for each $v \in \variablesof{p}$ s.t.\ $\match(v) \notin \domain{\basecandidate}$, we thus have $w$ such that $\match(v) = d^w$ for some $d \in \{ \frsymb, \gensymb \}$.
To identify from which nodes the facts of $\match(p)$ come involving
$d^w$, we notice the following consequence of the definition of functions $\caster[w]$ for $w \in \relevantpat$.
\begin{remark}
	\label{remark-casters}
	Let $w \in \relevantpat$ and $d \in \{ \frsymb, \gensymb \}$.
	For all $w' \cdot (\pattern, r.B, e) \in \relevantpat$, we have $d^w \in \casterof[w' \cdot (\pattern, r.B, e)]{\domain{\interpof[]}}$ iff one of the two following conditions holds:
	\begin{itemize}
		\item $w = w' \cdot (\pattern, r.B, e)$;
		\item $w = w'$ and $d = e$.
	\end{itemize} 
\end{remark}
It follows that, for each $v \in \variablesof{p}$ s.t.\ $\match(v) \notin \domain{\basecandidate}$, the set of nodes $W_v = \{ w' \cdot (\pattern, r.B, e) \in \relevantpat \mid \match(v) \in \casterof[w' \cdot (\pattern, r.B, e)]{\domain{\interpof[]}} \}$ is finite.
On $\relevantpat$, we consider the order given by the prefix relation, that for all $w_1, w_2 \in \relevantpat$, we define $w_1 \leq w_2$ iff $w_1$ is a prefix of $w_2$.
We now denote $W_p$ the prefix closure of $\bigcup_{v} W_v$, where $v$ ranges over $v \in \variablesof{p}$ s.t.\ $\match(v) \notin \domain{\basecandidate}$.
In particular, $(\candidatepattern, \emptyset, \emptyset) \in W_p$.
For every $w \in \relevantpat$, we define the interpretation:
\[
\I_w = \bigcup_{\substack{ w' \cdot (\pattern, r.B, e) \in \relevantpat \\ w \leq w' \cdot (\pattern, r.B, e)}} \casterof[w' \cdot (\pattern, r.B, e)]{\interpof}.
\] 
We further set $p_w \subseteq p$ consisting of those atoms from $p$ that are mapped by $\match$ in $\I_w$, and $\match_w$ the subsequent homomorphism, that is the restriction of $\match$ that maps to $\I_w$, i.e.:
\[
\match_w = \match|_{\match^{-1}(\domain{\I_w})}.
\]
In particular $p_{(\candidatepattern, \emptyset, \emptyset)} = p$ and $\match_{(\candidatepattern, \emptyset, \emptyset)} = \match$.
We now claim the following:
\renewcommand{\subquery}{\widehat{p}}
\begin{lemma}
	\label{lemma:captured-matches}
	For all nodes $u = w \cdot (\pattern, r.B, e) \in W_p$, 
	we have $(p, p_u, \submatch') \in \patternspec$ 
	where $\submatch' = (\caster[u])^{-1} \circ \match_{u}|_{\domain{}}$ 
	with $\domain{} = \submatch_u^{-1}(\casterof[u]{\domain{\interpof}})$.
\end{lemma}
\begin{proof}
	We proceed by induction on elements of $W_p$, starting from its finitely many maximal elements w.r.t.\ $\leq$.
	
	{\bf Base case.}
	Assume $u = w \cdot (\pattern, r.B, e) \in W_p$ is maximal for $\leq$.
	that is for all $u' \in \relevantpat$ s.t.\ $u \leq u'$, we have $u' \notin W_p$.
	In particular, for all $v \in \variablesof{p}$, $\match(v) \in \casterof[u']{{\J_{M'}}}$, where $M'$ is the mosaic of node $u'$, implies $\match(v) \in \domain{\basecandidate}$.
	Therefore $\submatch_u^{-1}(\casterof[u]{\domain{\interpof}}) = \variablesof{p_u}$ and, since $\basecandidate$ is preserved in every mosaic (Condition~2 in the definition), it follows that $p_u$ fully embeds in $\casterof[u]{\interpof[]}$ via $\match_u$.
	From $\patternspec$ being saturated, we thus have $(p, p_u, (\caster[u])^{-1} \circ \match_u)$ which is the desired triple.
	
	{\bf Induction case.}
	Consider $u = w \cdot (\pattern, r.B, e) \in W_p$ and assume the property holds for all $u \in W_p$ with $u \leq u'$.
	We cover $p_u$ by the set of all $p_{u \cdot (M', r'.B', e')}$ s.t.\ $u \cdot (M', r'.B', e') \in W_p$ and by the remaining atoms $p_{= u}$ from $p_u$.
	Notice that this is not a partition of $p_u$ as those subqueries may overlap.
	In particular, two distinct subqueries $p_{u \cdot (M_1, r_1.B_1, e_1)}$ and $p_{u \cdot (M_2, r_2.B_2, e_2)}$ may only share those variables which are mapped to $\domain{\basecandidate}$ and,  if $e_1 = e_2$, on $e_1$ (that is also $e_2$ in this case).
	Respectively, $p_{u \cdot (M', r'.B', e')}$ and $p_{=u}$ may only share variables that are mapped either on $\domain{\basecandidate}$ or on $e'$.
	We cover as well the homomorphism $\match_u$ by all $\match_{u \cdot (M', r'.B', e')}$ s.t.\ $u \cdot (M', r'.B', e') \in W_p$ and by the homomorphism $\match_{= u}$ which is defined as the restriction of $\match_u$ to variables from $p_{=u}$.
	Notice that from $S_M$ being saturated, we obtain $(p, p_{=u}, (\caster[u])^{-1} \circ \match_{=u}) \in S_M$.
	We further apply the induction hypothesis on each $u \cdot (M', r'.B', e') \in W_p$, which provides corresponding triples $(p, p_{u \cdot (M', r'.B', e')}, \match')$ in each $S_{M'}$.
	From Condition~3 in the definition of good mosaics, if follows that each $(p, p_{u \cdot (M', r'.B', e')}, \match'')$ belongs to $S_M$, where $\match''$ denotes the restriction of $\match'$ to range $\domain{\basecandidate} \cup \{ e' \}$.
	It remains to form the union (in the sense given in the definition of a saturated mosaic) of all those subsequent triples, along with $(p, p_{=u}, (\caster[u])^{-1} \circ \match_{=u})$, to obtain the desired triple.
	Let us first focus on two triples $(p, p_{u \cdot (M_1, r_1.B_1, e_1)}, \match_1'')$ and $(p, p_{u \cdot (M_2, r_2.B_2, e_2)}, \match_2'')$ and consider $v \in \variablesof{p_{u \cdot (M_1, r_1.B_1, e_1)}} \cap \variablesof{p_{u \cdot (M_2, r_2.B_2, e_2)}}$.
	From our covering of $p_u$, we have seen that either $\match_u(v) \in \domain{\basecandidate}$, or, if $e_1 = e_2$, that $\match(v) = e_1$.
	In any case, restricting $\match_1'$ and $\match_2'$ to ${\domain{\basecandidate} \cup \{ e_1 \}}$ preserved $\match_1''$ and $\match_2''$ being defined on $v$, and, since issuing from the same $\match_u$, equal.
	Let us now move to a triple $(p, p_{u \cdot (M', r'.B', e')}, \match'')$ and the triple $(p, p_{=u}, (\caster[u])^{-1} \circ \match_{=u})$.
	Let $v \in \variablesof{p_{u \cdot (M', r'.B', e')}} \cap \variablesof{p_{=u}}$.
	From our covering of $p_u$, we know $\match_u(v) \in \domain{\basecandidate} \cup \{  e' \}$.
	Again restricting to ${\domain{\basecandidate} \cup \{ e' \}}$ preserved $\match''$ being defined on $v$ and, since issuing from the same $\match_u$, equal to $\match_{=u}$ on $v$.
	Therefore, we can form the union of all those triples, since they are 2-by-2 compatible, which provides the desired triple in $S_M$.
\end{proof}

This concludes the proof of the ``$\Leftarrow$'' direction of Lemma~\ref{lemma-mosaic}, as existence of such a homomorphism $p$ then yields a complete triple in $\candidatespec$, contradicting $\candidatepattern$ being a mosaic.

%
%
%
%
We now turn to the ``$\Rightarrow$'' direction.
Assume that $\Imc_{\mn{base}}$ can be extended to a model \Imc of \Tmc that
satisfies Conditions~(a) to~(c).
Up to introducing $\sizeof{\types(\tbox)}$ copies of each element from $\domain{\I} \setminus \domain{\basecandidate}$, we can safely assume $\basedomainof{\I} = \domain{\basecandidate}$.
Indeed, the resulting interpretations is still a model of $\kb$ and still satisfies Conditions~(a) to (c).

Set $\candidate := \basecandidate$ and $\candidatespec := \{ (p, \widehat{p}, h|_{\domain{\basecandidate}}) \mid p \in q, \widehat{p} \subseteq p, h : \widehat{p} \rightarrow \I \text{ is a homomorphism} \}$.
Condition~(c) ensures in particular that $\candidatespec$ only contains incomplete triples.
We denote $\candidatepattern := (\candidate, \candidatespec)$ this initial mosaic.

Starting from $\Mmc := \{ \candidatepattern \}$, we extend $\Mmc$ into a set of mosaics in which each mosaic $\pattern \in \Mmc$ is good in $\Mmc$
and \emph{realized in $\I$}, meaning that $\interpof$ homomorphically embeds into $\I$.
Notice $\candidatepattern$ is indeed realized in $\I$, but may not be good in $\{ \candidatepattern \}$.
To pursue the construction, given a mosaic $\pattern \in \Mmc$ being realized in $\I$ and some $e \in (A \sqcap \lnot \exists r.B)^{\interpof}$ with $A \incl \exists r.B \in \tbox$, 
we show how to extract from $\I$ another mosaic $\altpattern$ to succeed to $M$, and we add $\altpattern$ to $\Mmc$.
Since the number of mosaics is finite, the construction of $\Mmc$ terminates.
We further check that all mosaics in the resulting $\Mmc$ set are good in $\Mmc$ (and in particular $\candidatepattern$). 

To formalize the construction, we introduce along each mosaic $\pattern$ a function $\track_M$ being a homomorphism $\interpof \rightarrow \I$.
We also assume chosen, for every $\rstyle{r}.\cstyle{A} \in \Omega$, a function $\successor[\I]_{\rstyle{r}.\cstyle{A}}$
that maps every element $e \in (\exists{r.A})^\I$ to an element $e' \in \domain{\I}$ such that $(e, e') \in \rolestyle{R}^\I$ and $e' \in \cstyle{A}^\I$.

\begin{definition}
	\label{def:pattern-extraction}
	We define $\Mmc$ as the smallest set such that:
	\begin{itemize}
		\item $\candidatepattern \in \Mmc$ and we set $\track_{M_0} = \identity_{\candidate \rightarrow \I}$, where $\identity_{\candidate \rightarrow \I}$ denotes the identity function.
		\item For each $\pattern[1] \in \Mmc$ and for each $e_1 \in (A \sqcap \lnot \exists r.B)^{\interpof[1]}$ with $A \incl \exists r.B \in \tbox$, we denote $\track_1 :=\track_{M_1}$ and set $e_1' := \trackof{1}(e_1)$.
		Since $\trackof{1}$ is a homomorphism and $\I$ is a model of $\tbox$, 
		we obtain $e_1' \in (\exists{r.B})^{\I}$ and can set 
		$e_2' := \successor[\I]_{\rstyle{R}.\cstyle{B}}(e_1')$.
		If $e_2' \in \domain{\basecandidate}$, then we set $e_2 := e_2'$, otherwise we set $e_2$ to either $\frsymb$ or $\gensymb$ such that $e_1 \neq e_2$.
		We let $\trackof{2}$ be the function that maps elements of $\domain{\basecandidate}$ to themselves, $e_1$ to $e_1'$ and $e_2$ to $e_2'$.
		We now define a new mosaic $\pattern[2]$.
		Its domain is $\domain{\basecandidate} \cup \{ e_1, e_2 \}$.
		Its interpretation $\interpof[2]$ is given by:
		\[
		\begin{array}{rcl}
			\domain{\interpof[2]} & = & \domain{\basecandidate} \cup \{ e_1, e_2 \} \\
			\cstyle{C}^{\interpof[2]} & = & \cstyle{C}^{\candidate} \cup \{ e_k \mid e_k' \in \cstyle{C}^{\I}, k = 1, 2 \} \\
			\rolestyle{P}^{\interpof[2]} & = & \rolestyle{P}^{\candidate} \cup \{ (e_1, e_2) \mid \tbox \models \axiom{r}{p} \} \\
			& & \phantom{\rolestyle{P}^{\candidate}} \cup \{ (e_2, e_1) \mid \tbox \models \axiom{r^-}{p} \}
		\end{array}
		\]
		Its specification is $\patternspec[2] = \{ (p, \widehat{p}, \trackof{2}^{-1} \circ h|_{\domain{\basecandidate} \cup \{ e_1', e_2' \}}) \mid p \in q, \widehat{p} \subseteq p, h : \widehat{p} \rightarrow \I \text{ is a homomorphism} \}$.
		We now add $\pattern[2] \in \Mmc$, and if $\track_{\pattern[2]}$ is not already defined, then we set $\track_{\pattern[2]} := \track_2$.
	\end{itemize}
	
\end{definition}

Recalling that $\I$ is a 
model of $\Kmc$ it is then straightforward 
to verify that $\pattern[2]$, from the induction case, is a mosaic satisfying all conditions to succeed to $\pattern[1]$ for $e$ and $A \incl \exists r.B$, and 
that $\trackof{2}$ (thus $\trackof{M_2}$) is a homomorphism.
These properties are verified by the next two lemmas.

\begin{lemma}
	Each $\pattern[] \in \Mmc$ is a well-defined mosaic. Furthermore, $\track_M$ is a homomorphism from $\interpof$ to $\I$.
\end{lemma}

\begin{proof}
	The base case consisting of verifying that $\candidatepattern$ is a mosaic is trivial, and $\identity_{\candidate \rightarrow \I}$ is indeed a homomorphism.
	
	We move to the induction case.
	Consider $M_2$ and $\trackof{2}$ as obtained from some $\pattern[1] \in \Mmc$ and $e_1 \in (A \sqcap \lnot \exists r.B)^{\interpof[1]}$ with $A \incl \exists r.B \in \tbox$.
	By induction hypothesis, $\pattern[1]$ is a mosaic and $\track_1 := \trackof{M_1}$ is a homomorphism from $\interpof[1]$ to $\I$.
	
	We first verify that $\trackof{2}$ is a homomorphism:
	\begin{itemize}
		\item
		Let $u \in \cstyle{A}^{\interpof[2]}$.
		If $u \in \cstyle{A}^{\candidate}$, then in particular $u \in \domain{\basecandidate}$ hence $\trackof{2}(u) = u \in \cstyle{A}^{\candidate} \subseteq \cstyle{A}^{\I}$.
		Otherwise, $u = e_k$ for $k = 1$ or $k=2$ with $e_k' \in \cstyle{A}^{\I}$.
		In that case, notice $\trackof{2}(u) = e_k'$ which concludes.
		\item
		Let $(u, v) \in \rolestyle{P}^{\interpof[2]}$.
		If $(u, v) \in \rolestyle{P}^{\candidate}$, then in particular $u, v \in \domain{\basecandidate}$, hence $(\trackof{2}(u), \trackof{2}(v)) = (u, v) \in \rolestyle{P}^{\candidate} \subseteq \rolestyle{P}^{\I}$.
		Otherwise, if $(u, v) = (e_1, e_2)$ with $\tbox \models \axiom{r}{p}$, then notice that $(\trackof{2}(u), \trackof{2}(v)) = (e_1', e_2')$.
		Since $e_2'$ is the successor of $e_1'$ for $\rstyle{r}.\cstyle{B}$ in $\I$, and $\I$ models $\tbox$, we obtain $(e_1', e_2') \in \rolestyle{P}^{\I}$ as desired.
		Otherwise we have $(u, v) = (e_2, e_1)$ with $\tbox \models \axiom{r^-}{p}$, then notice that $(\trackof{2}(u), \trackof{2}(v)) = (e_1', e_1')$.
		Since $e_2'$ is the successor of $e_1'$ for $\rstyle{r}.\cstyle{B}$ in $\I$, and $\I$ models $\tbox$, we have $(e_2', e_1') \in \rolestyle{p}^{\I}$  as desired. 
	\end{itemize}
	
	We now verify that $\pattern[2]$ is a well-defined mosaic.
	\begin{itemize}
		\item
		Regarding $\interpof[2]$:
		\begin{enumerate}
			\item
			By definition, we have $\Delta^{\Imc_{\mn{base}}} \subseteq \domain{\interpof[2]} \subseteq
			\Delta^{\Imc_{\mn{base}}} \uplus \{ \frsymb, \gensymb \}$ as desired;
			\item 
			Also by definition, we have $\interpof[2]|_{\Delta^{\Imc_{\mn{base}}}} = \basecandidate$;
			\item 
			From interpretations of concepts being inherited from $\I$ that satisfies Condition~(b), it follows that indeed $\mn{tp}_{\interpof[2]}(e_i^*) \in \typescandidate$ if $e_i^* \in
			\Delta^{\interpof[2]}$, for $i \in
			\{1,2\}$;
			\item
			From $\I$, being a model of $\kb$, it follows that $\interpof[2]$ satisfies in particular all $\exists r . A \sqsubseteq B \in
			\Tmc$
			and all \mbox{$r \sqsubseteq s \in \Tmc$}.
		\end{enumerate}
		\item
		Regarding $\patternspec[2]$, it is the restriction of a saturated specification (up to $\trackof{2}$, which essentially rename elements), and hence is also saturated.
		From $\I$ satisfying Condition~(c), we obtain that $\patternspec[2]$ doesn't contain any complete triple. \qedhere
	\end{itemize}
\end{proof}

\begin{lemma}
	For all $M \in \Mmc$, the mosaic $M$ is good in $\Mmc$.
\end{lemma}

\begin{proof}
	Let $\pattern[1] \in \Mmc$ and $e_1 \in (A \sqcap \lnot \exists r.B)^{\interpof[1]}$ with $A \incl \exists r.B \in \tbox$.
	Consider $\pattern[2]$ as obtained from the induction case of Definition~\ref{def:pattern-extraction}.
	We have $\pattern[2] \in \Mmc$ and verify the three additional conditions:
	\begin{enumerate}
		\item
		Interpretations of concepts names being directly imported from $\I$, we immediately  have $\mn{tp}_{\interpof[1]}(e_1)=\mn{tp}_{\interpof[2]}(e_1)$;
		\item
		Since we set $e_2' := \successor[\I]_{\rstyle{R}.\cstyle{B}}(e_1')$ in the construction of $M_2$, it follows naturally that $e \in (\exists r. B)^{\interpof[2]}$;
		\item 
		This is immediate as any partial homomorphism from either $\patternspec[1]$ (resp.\ $\patternspec[2]$) is the restriction to $\domain{\interpof[1]}$ (resp. $\domain{\interpof[2]}$) of a complete homomorphism to $\I$, which can in turn be restricted to $\domain{\interpof[2]}$ (resp.\ $\domain{\interpof[1]}$).
		In fact, this proves: if $(p,\widehat p,h) \in \patternspec[2]$, then $(p,\widehat p,h') \in \patternspec[1]$ where
		$h'$ is the restriction of $h$ to range $\Delta^{\Imc_{\mn{base}}} \cup \{ e_1 \}$, but also ``the converse'', that is: if $(p,\widehat p,h) \in \patternspec[1]$, then $(p,\widehat p,h') \in \patternspec[2]$ where
		$h'$ is the restriction of $h$ to range $\Delta^{\Imc_{\mn{base}}} \cup \{e_1\}$. \qedhere
	\end{enumerate}
\end{proof}

\thmcombinedlowerel*


\begin{proof}
        We give a polynomial time reduction from (Boolean) CQ evaluation on \EL KBs
        with closed concept names, which is \TwoExpTime-hard \cite{Ngo2016}.
        A \emph{KB with closed concept names} takes the form $\Kmc_\Sigma$
        with \Kmc a KB and $\Sigma \subseteq \NC$ a set of closed
        concept names.
        We say that an interpretation \Imc \emph{respects}~$\Sigma$ if for all
        $A \in \Sigma$, $d \in A^\Imc$ implies $A(d) \in \Amc$.
        Moreover, \Imc is a \emph{model} of $\Kmc_\Sigma$ if \Imc is a
        model of \Kmc and \Imc respects $\Sigma$. Thus, the only
        instances of a closed concept name are those that are
        explicitly asserted in the ABox. Now the problem proved to be
        \TwoExpTime-hard in \cite{Ngo2016} is: given a KB with closed
        concept names $\Kmc_\Sigma$ and a Boolean CQ $q$, decide
        whether $\Kmc_\Sigma \models q$. We remark that the formalism
        in \cite{Ngo2016} also admits closed role names, but these are
        not used in the hardness proof.

        Let $\Kmc_\Sigma$ be an \EL KB with closed
        concept names, $\Kmc = (\Tmc,\Amc)$, and $q$ a Boolean CQ.
        Specializing to $\EL$ the normal form presented in Section~\ref{sec-preliminaries},
        we assume every concept inclusion in
        $\tbox$ to have one of the following shapes:
        \[
        \begin{array}{r@{\qquad}l@{\qquad}r@{\qquad}r}
        	\axtop 
        	&
        	\axexistsright
        	&
        	\axexistsleft
        	&
        	\axand
        \end{array}
        \]
        where $\cstyle{A, A_1, A_2, B}$ range over \NC and
        $r$ ranges over \NR.
        We construct a
        circumscribed \EL KB $\mn{Circ}_\CP(\Kmc')$ with
        $\Kmc'=(\Tmc',\Amc')$, and a CQ $q'$ such that
        $\Kmc_\Sigma \models q$ iff $\mn{Circ}_\CP(\Kmc') \models q'$.

        A natural first attempt is to use use $\Kmc'=\Kmc$ and
        \mbox{$q'=q$}, and to minimize the concept names in $\Sigma$
        while letting all other concept names vary. Apart from using
        more than a single minimized concept name, this does not work
        since the extension of closed concept names may be too large.
        As a simple example, consider the empty ABox and the TBox
        $\top \sqsubseteq \exists r . A$ with $A \in \Sigma$. This KB
        is unsatisfiable when $A$ is closed, but satisfiable when $A$
        is minimized. We thus use a more refined approach.


        Let \CP be the circumscription pattern that minimizes the
        fresh concept name $M$ and lets all other symbols vary.  In
        what follows, we assemble $\Tmc'$ and $\Amc'$ and define
        $q'$. For $\Amc'$, we start from \Amc and extend with
        additional assertions. $\Tmc'$ contains the CIs from \Tmc
        in a modified form, as well as additional CIs, see below for
        details.

        \smallskip

        We start with connecting the concept names in $\Sigma$ to the
        minimized concept name $M$: 
        \begin{align}
                A \sqsubseteq \M & \quad \text{for all } A \in \Sigma
                \label{eq:A_subseteq_M}
        \end{align}
        Much of the reduction is concerned with preventing
        non-asserted instances
        $d$ of closed concept names \mbox{$A \in \Sigma$}, that is, elements $d$
        that satisfy
        $A$ but are not an individual from $\mn{ind}(\Amc)$ such that
        \mbox{$A(d) \in
                        \Amc$}. There are three cases to be distinguished.

        We first consider the case that $d=a \in \mn{ind}(\Amc)$, but
        $A(d) \notin \Amc$. We mark such $a$ with
        a fresh concept name~$L$: 
        \begin{align}
                \overline A(a)                     & \quad \text{for all } a \in
                \Ind(\Amc) \text{ with } A(a) \not
                \in \Amc \label{eq:bar_A}                                        \\
                A \sqcap \overline A \sqsubseteq L & \quad \text{for all }
                A \in \Sigma. \label{eq:L_A_Bar_A}
        \end{align}
        where $\overline A$ is a fresh concept name for each $A \in \Sigma$.

        We next want to achieve that any model \Imc that satisfies  $a \in
                L^\Imc$ for some ABox individual $a$ also satisfies
        $\Imc \models q$ and is thus ruled out as a countermodel
        against the query being entailed. To
        this end, we add a copy of $q$ to $\Amc'$:
        \begin{align}
                A(x)   & \quad \text{for all } A(x) \in \varphi \label{eq:copy_query_concepts}    \\
                r(x,y) & \quad \text{for all } r(x,y) \in \varphi \label{eq:copy_query_relations}
        \end{align}
        assuming w.l.o.g.\ that $\mn{var}(q) \cap \Ind(\Amc)= \emptyset$.



        We want the copy of $q$ in $\Amc'$ to only become `active'
        when the
        concept name $L$ is made true at some individual from $\mn{ind}(\Amc)$;
        otherwise, it should
        be `dormant'. The individual names in $\mn{ind}(\Amc)$, in contrast,
        should always be active.    To achieve this, we
        include in $\Amc'$ the following assertions where $X$ is
        a fresh concept name indicating activeness and $u$ is a
        fresh role name:
        \begin{align}
                \X(a)  & \quad  \text{for all } a \in \Ind(\Amc),
                \label{eq:X_mark}                                          \\
                u(x,a) & \quad \text{for all } a \in \mn{Ind}(\Amc) \text{
                        and } x \in \mn{var}(q) \label{eq:copy_query_links}
        \end{align}
        We then define the CQ $q'$ to be
        $$q' = q \wedge
                \bigwedge_{x \in \mn{var}(q)} \X(x),$$
        and add to $\Tmc'$:
        \begin{align}
                \exists u.L \sqsubseteq \X. \label{eq:propagate_L_copy}
        \end{align}
        Elements generated by existential quantifiers on the right-hand
        side
        of CIs in \Tmc should of course also be active. Moreover,
        dormant elements should not trigger CIs from \Tmc as
        this might generate `unjustified' existential (thus active)
        elements.
                We deal with these issues by including in $\Tmc'$ the following relativized version of~\Tmc (recall \Tmc is in normal form):
                \begin{align}
                        X  \sqsubseteq A                               & \quad \text{for all }  \top  \sqsubseteq A  \in \Tmc \label{eq:T_Axioms_X_A}              \\
                        X \sqcap A  \sqsubseteq \exists r.(X \sqcap B) & \quad \text{for all } A \sqsubseteq \exists r.B \in \Tmc \label{eq:T_Axioms_X_exists_rB}  \\
                        X \sqcap \exists r.(X \sqcap B)  \sqsubseteq A & \quad \text{for all } \exists r.B \sqsubseteq A  \in \Tmc \label{eq:T_Axioms_exists_rB_X} \\
                        X \sqcap  A_1 \sqcap A_2  \sqsubseteq A        & \quad \text{for all } A_1 \sqcap A_2 \sqsubseteq A \in \Tmc \label{eq:T_Axioms_X_A_A_A}
                \end{align}
        Note that, when activating the copy of $q$ in $\Amc'$ by
        making $X$ true at all its individual names, then the
        CIs~(\ref{eq:T_Axioms_X_A})--(\ref{eq:T_Axioms_X_A_A_A}) apply also there and may make
        concept names $A$ with $A \in \Sigma$ true at elements
        $x \in \mn{var}(q)$. Because of~(\ref{eq:A_subseteq_M}), this
        may lead to models in which the copy of $q$ is activated to
        be incomparable to models in which it is not activated
        regarding
        $<_\CP$. This can be remedied by adding to $\Amc'$:
        \begin{align}
                M(x) & \quad \text{ for all } x \in \mn{var}(q).
                \label{eq:M_var_q}
        \end{align}
        The second case of non-asserted instances $d$ of closed concept
        names is that $d$ is not from $\mn{ind}(\Amc) \cup \mn{var}(q)$.
        The idea is that models with
        such instances $d$ should be non-minimal.
        We achieve this by including in $\Amc'$ a fresh individual
        $\tc$ that must satisfy~$\M$, but not necessarily anything
        else. Since any existential instance $d$ of a closed concept
        name must also satisfy~$M$, we can find a model that is
        preferred w.r.t.\ $<_\CP$ by using $\tc$ as a surrogate
        for~$d$, that is, making true at $\tc$ exactly the concept
        names true at $d$, rerouting all incoming and outgoing role
        edges from $d$ to~$\tc$, and then removing $d$ from the
        extension of all concept and role names.  We thus include in
        $\Amc'$:
        \begin{align}
                \M(\tc). \label{eq:M_t_c}
        \end{align}
        The third and last case of non-asserted instances $d$ of closed
        concept names is that $d$ is $t$ or an individual in the copy
        of $q$ in~$\Amc'$. 
        Since these individuals satisfy $M$, we may make true at them
        a closed concept name without producing a
        non-minimal model.  To fix this problem, it suffices to
        prevent the individuals in $\mn{var}(q) \cup \{t\}$ to be used
        as witnesses for existential restrictions in \Tmc, as this is
        the only possible reason for creating non-asserted instances of
        closed concept names of the described form.  It is, however,
        easy to identify such witnesses as they must be active and
        satisfy the concept name~$X$. We thus add
        \begin{align}
                \overline \X (a) & \quad \text{for all } a \in
                \mn{var}(q) \cup \{ t\}
                \label{eq:bar_X}
                \\
                u(x,a)           & \quad \text{for all } a \in
                \mn{var}(q) \cup \{ t \} \text{
                        and } x \in \mn{var}(q) \label{eq:copy_query_links_add}
                \\
                                 & \X \sqcap \overline \X \sqsubseteq L. \label{eq:X_Bar_X}
        \end{align}
        where $\overline X$ is a fresh concept name.
        This finishes the construction of $\Tmc'$ and $\Amc'$.
        \\[2mm]
        {\bf Claim.}
        $\Kmc_\Sigma \models q$ iff $\Circ(\Kmc') \models q'$.

        \smallskip
        \noindent ``$\Leftarrow$''.
        Let $\Kmc_\Sigma \not \models q$.
        Then, there is a model \Imc of $\Kmc_\Sigma$ such that $\Imc \not \models q$.
        W.l.o.g. assume that $\Delta^\Imc \cap (\mn{var}(q) \cup
                \{t\})= \emptyset$.  
        We construct an interpretation \Jmc as follows:
        \begin{align*}
                \Delta^\Jmc & \coloneqq \Delta^\Imc \cup \mn{var}(q)
                \cup \{ t \}                                           \\
                A^\Jmc      & \coloneqq A^\Imc \cup \{ x \mid A(x) \in
                \Amc'\}                                                \\
                r^\Jmc      & \coloneqq r^\Imc \cup \{ (x,y) \mid
                r(x,y) \in \Amc'\}                                     \\
                \X^\Jmc     & \coloneqq \Delta^\Imc                    \\
                \M^\Jmc     & \coloneqq \bigcup_{A \in \Sigma } A^\Jmc
                \cup \mn{var}(q) \cup \{ t \}
        \end{align*}
        for all concept names $A \notin \{X,M\}$ and role names $r$.
        From $\Imc \models \Kmc$ and the construction of $\X^\Jmc$, it easily follows that $\Jmc \models \Kmc'$.
        To show $\Jmc \models \Circ(\Kmc')$, we are left to prove that \Jmc is minimal.

        Assume to the contrary that there is a model $\Jmc'$ of
        $\Kmc'$ with $\Jmc' <_\CP \Jmc$.
        As $\M$ is the only minimized concept name, there must be some $a \in \M^\Jmc$ with $a \notin \M^{\Jmc'}$.
        We must have $a \notin \mn{var}(q) \cup \{t\}$ since $M(x) \in \Amc'$ for $x \in \mn{var}(q) \cup \{t\}$ and $\Jmc'$ is
        a model of $\Kmc'$. But then, the definition of $M^\Jmc$
        implies
        $a \in
                A^\Jmc$ with $A \in \Sigma$. By definition of $A^\Jmc$, in turn,
        $a \in A^\Imc$ or  $A(a) \in \Amc'$. In the former case,
        it follows from the fact that \Imc respects
        $\Sigma$ that $A(a) \in \Amc \subseteq \Amc'$, and thus
        $A(a) \in \Amc'$ in both cases. Now,
        CI~(\ref{eq:A_subseteq_M}) and $a \notin \M^{\Jmc'}$ imply
        that
        $\Jmc'$ is not a model of $\Kmc'$, a contradiction.

        We are left to show that $\Jmc \not \models q'$.
        Assume to the contrary that there is a homomorphism $h$ from $\Dmc_{q'}$ to \Jmc.
        We derive a contradiction to $\Imc \not \models q$ by showing
        that $h$ is also a homomorphism from $\Dmc_q$ to \Imc.
        First, note that $h$ is a function from $\mn{var}(q)$ to
        $\Delta^\Imc$ since $\mn{var}(q')=\mn{var}(q)$ and for each
        $x \in \mn{var}(q')$, we have $X(x) \in q'$ and thus $h(x)\in
                \X^\Jmc = \Delta^\Imc$.
        Finally, observe that \Imc is identical to the restriction of
        \Jmc
        to $\Delta^\Imc$ and the concept and role names that occur
        in $\Kmc'$ and $q'$.

        \medskip
        ``$\Rightarrow$''.
        Assume that $\Circ(\Kmc') \not \models q'$. Then there is a
        model \Imc of $\Circ(\Kmc')$ with $\Imc \not \models q'$.
        We show that w.l.o.g, we may assume that \Imc satisfies
        several additional conditions that shall prove to be
        convenient
        in what follows.
        \begin{claim} \label{claim:pruned_model}
                There is a model \Imc of $\Circ(\Kmc')$ with $\Imc \not \models q'$ that satisfies the following properties.
                \begin{properties}
                        \item $a \in A^\Imc$ implies $A(a) \in \Amc'$
                        for all $a \in \mn{var}(q) \cup \{ t \}$, \label{enum:prop_1}
                        \item $(d,e) \in r^\Imc$ implies $r(d,e) \in
                                \Amc'$ for all $d,e \in \Delta^\Imc$ with
                        $\{d,e\} \cap (\{\mn{var}(q) \} \cup \{t \}) \neq \emptyset$, and \label{enum:prop_3}
                        \item $(d,e) \in u^\Imc$ implies $u(d,e) \in \Amc'$ for all $d,e \in \Delta^\Imc$. \label{enum:prop_2}
                \end{properties}
        \end{claim}

        \noindent We may construct the desired \Imc from any model
        $\Imc'$ of $\Circ(\Kmc')$ with $\Imc' \not \models q'$, as follows:
        \begin{align*}
                \Delta^\Imc & \coloneqq \Delta^{\Imc'}                   \\
                A^\Imc      & \coloneqq A^{\Imc'} \setminus \{a \in
                \mn{var}(q) \cup \{t\} \mid A(a) \notin \Amc'\}          \\
                u^\Imc      & \coloneqq \{ (x,y) \mid u(x,y) \in \Amc'\} \\
                \begin{split}
                        r^\Imc      & \coloneqq r^{\Imc'} \setminus \{ (x,y) \mid r(x,y) \notin \Amc' \text{ and } \\
                        & \phantom{\coloneqq r^{\Imc'} \setminus \{ (x,y) \mid \;\,}  \{x,y\} \cap (\mn{var}(q) \cup \{t\})\neq \emptyset \}
                \end{split}
        \end{align*}
        for all concept names $A$ and role names $r \neq u$.
        Obviously, $\Imc \not \models q'$, $\Imc \models \Amc'$ and
        Points~\ref{enum:prop_1} to~\ref{enum:prop_3} hold by construction.
        To show that $\Imc \models \Tmc'$, note the following:
        \begin{itemize}
                \item \Imc obviously satisfies the CIs~(\ref{eq:A_subseteq_M}) by construction and the fact that $\Imc'$ is a model of $\Kmc'$.
                \item
                      \Imc satisfies
                      CIs~(\ref{eq:L_A_Bar_A}),~(\ref{eq:propagate_L_copy})~and~(\ref{eq:X_Bar_X}) as
                      $L^{\Imc'}$ must be empty---otherwise
                      Assertions~(\ref{eq:copy_query_links})~and~CIs~(\ref{eq:propagate_L_copy})
                      activate the copy of $q$ in $\Amc'$
                      and thus there is a trivial homomorphism from $q'$ to $\Imc'$, contradicting $\Imc' \not \models q'$.

                \item
                      For the CIs~(\ref{eq:T_Axioms_X_A})--(\ref{eq:T_Axioms_X_A_A_A}), observe that $\X^{\Imc'} \cap
                              (\mn{var}(q) \cup \{t\})= \emptyset$ following the same argument and
                      using
                      Assertions~(\ref{eq:bar_X}) and CI~(\ref{eq:X_Bar_X}).
                      Hence $$\X^{\Imc'} \cap A^{\Imc'} = \X^\Imc \cap A^\Imc$$ and
                      $$(\X^{\Imc'} \times \X^{\Imc'}) \cap r^{\Imc'} = (\X^\Imc \times
                              \X^\Imc) \cap r^\Imc$$ for all concept names $A$
                      and role names $r \neq u$.
                              Consequently, it is easy to see that for all CIs $\alpha \in \Tmc'$ we have $\Imc \models \alpha$ as $\Imc' \models \alpha$.

                \item
                      We are left to show that \Imc is minimal; however, this is easy to see as $\Imc'$ is minimal and $\Imc \leq_\CP \Imc'$.
        \end{itemize}
        This finishes the proof of Claim~\ref{claim:pruned_model}.

        \smallskip

        We aim to show that the interpretation $\Jmc$, obtained as the
        restriction of \Imc to domain ${\X^\Imc}$, is a model of
        $\Kmc_\Sigma$ with \mbox{$\Jmc \not \models q$}.
        We assume that the properties from
        Claim~\ref{claim:pruned_model}
        are satisfied, and Property~(a) yields $\Delta^\Jmc
                \cap (\mn{var}(q) \cup \{ t \}) = \emptyset$.
        Assertion~(\ref{eq:X_mark}) yields
        $\Ind(\Amc) \subseteq \Delta^\Jmc$ and since
        $\Amc \subseteq \Amc'$ and \Imc is a model of $\Amc'$,
        \Jmc must be a model of~\Amc.
        Since $\Jmc$ is the restriction of \Imc to $\X^\Imc$ and \Imc satisfies CIs~(\ref{eq:T_Axioms_X_A})--(\ref{eq:T_Axioms_X_A_A_A}), \Jmc must satisfy \Tmc as well.

        To see that $\Jmc \not \models q$, assume to the contrary that there is a homomorphism $h$ from $q$ to \Jmc.
        As \Jmc is the restriction of \Imc to $\X^\Imc$ and $q'$ is
        the extension of $q$ that adds atoms $\X(x)$ for all variables
        $x$, $h$ must also be a homomorphism from $q'$ to~\Imc. This contradicts $\Imc \not \models q'$.

        We are left to argue that \Jmc respects $\Sigma$.
        Assume to the contrary that there is some $A \in \Sigma$ and $d_0 \in A^\Jmc$ with \mbox{$A(d_0) \notin \Amc$}.
        We have $d_0 \in \X^\Imc$ by construction of \Jmc.
        First assume that $d_0 \in \Ind(\Amc)$. Then (\ref{eq:bar_A}),
        (\ref{eq:L_A_Bar_A}), (\ref{eq:copy_query_links}), and (\ref{eq:propagate_L_copy}) activate the
        copy of $q$ in $\Amc'$, giving a homomorphism from $q$ to \Imc, contradicting $\Imc \not \models q'$.

        The more laborious case is $d_0 \notin \Ind(\Amc)$. As we show
        in the following, we may then
        %
        construct a model $\Imc' <_\CP \Imc$ of~$\Kmc'$,
        contradicting the minimality of \Imc.

        Intuitively, in the construction of $\Imc'$ we replace $d_0$
        with $t$ by making true at $t$ all concept names that are true
        at $d_0$, making false all concept names at $d_0$, and
        ``redirecting'' all incoming and outgoing edges of $d_0$ to
        $\tc$.  This forces us to make $X$ true at $t$, and
        by~(\ref{eq:propagate_L_copy}) and (\ref{eq:bar_X}) to
        (\ref{eq:X_Bar_X}) we must activate the copy of $q$ in
        $\Amc'$. But then the CIs~(\ref{eq:T_Axioms_X_A})--(\ref{eq:T_Axioms_X_A_A_A}) apply to that
        copy. To satisfy them, we make all concept and role names
        true in the copy of $q$.  In detail:
        \begin{align*}
                \Delta^{\Imc'} & \coloneqq \Delta^\Imc                \\
                A^{\Imc'}      & \coloneqq (A^\Imc \setminus \{d_0\})
                \cup \{ \tc \mid d_0 \in A^\Imc \}
                \cup \mn{var}(q)                                      \\
                L^{\Imc'}      & \coloneqq \Delta^\Imc                \\
                \begin{split}
                        r^{\Imc'} & \coloneqq (r^\Imc \setminus \{ (d,e) \mid d_0 \in \{d,e\} \}) \\
                        & \phantom{\coloneqq} \cup \{ (e, \tc) \mid (e,d_0) \in r^\Imc  \text{ and } e \neq d_0 \}                              \\
                        & \phantom{\coloneqq} \cup \{ (\tc, e) \mid (d_0,e) \in r^\Imc \text{ and } e \neq d_0 \} \\
                        & \phantom{\coloneqq} \cup \{ (\tc, \tc) \mid (d_0,d_0) \in r^\Imc \} \\
                        & \phantom{\coloneqq} \cup (\mn{var}(q) \times \mn{var}(q))
                \end{split}
        \end{align*}
        for all concept names $A \neq L$ and role names $r$. 

        Because \Imc satisfies
        Assertions~(\ref{eq:M_var_q})~and~(\ref{eq:M_t_c}) and
        $d_0 \in \Delta^\Jmc$ cannot be from
        $\mn{var}(q) \cup \{ t \}$, we have
        $M^{\Imc'} = M^\Imc \setminus \{d_0\}$ by construction of
        $\Imc'$. Thus $\Imc' <_\CP \Imc$, as desired. It remains to
        prove that $\Imc'$ is a model of $\Kmc'$.
        We have  $\Imc'\models \Amc'$ because $\Imc \models \Amc'$ and
        the restrictions of \Imc and of $\Imc'$ to $\Ind(\Amc')$ are identical.

        To see that $\Imc'$ satisfies the CIs~(\ref{eq:A_subseteq_M}), let
        $x \in A^{\Imc'}$ with $A \in \Sigma$.
        If $x \in \mn{var}(q) \cup \{t\}$, then $x \in \M^{\Imc}$ due
        to Assertions~(\ref{eq:M_var_q})~and~(\ref{eq:M_t_c}),
        and
        thus $x \in \M^{\Imc'}$
        by construction of $\Imc'$.
        %
        Otherwise, we have $x \in A^\Imc$ by construction of
        $A^{\Imc'}$, yielding $x \in \M^\Imc$ via
        Assertions~(\ref{eq:A_subseteq_M}).
        We must have $x \neq d_0$ and thus $x \in \M^{\Imc'}$ by
        construction
        of $\Imc'$.
        $\Imc'$ obviously satisfies CIs~(\ref{eq:L_A_Bar_A})~and~(\ref{eq:X_Bar_X}) by
        $L^{\Imc'} = \Delta^\Imc$.

                For CIs~(\ref{eq:T_Axioms_X_A}), consider some $X \sqsubseteq A \in \Tmc'$ and let $x \in X^{\Imc'}$.
                We know that $x \neq d_0$ as $d_0 \notin X^{\Imc'}$ by construction.
                If $x \in \mn{var}(q)$, $x \in A^{\Imc'}$ is entailed by construction.
                If $x = t$, note that Point~\ref{enum:prop_1} of Claim~(\ref{claim:pruned_model}) and the construction of $\Imc'$ imply
                \begin{align}
                        t \in A^{\Imc'} \text{ iff } d_0 \in A^\Imc \tag{\dag} \label{eq:t_d_0_dag}
                \end{align}
                for all concept names $A \neq L$.
                Thus, $d_0 \in X^\Imc$.
                Together with the fact that \Imc as a model of $\Kmc'$ must satisfy $X \sqsubseteq A$, we have that $d_0 \in A^\Imc$.
                Ergo, $t \in A^{\Imc'}$ follows via (\ref{eq:t_d_0_dag}).
                Similarly, if $x \in \Delta^\Imc \setminus (\mn{var}(q) \cup \{d_0,t\})$, then $x \in X^\Imc$ holds by construction, and as \Imc satisfies $X \sqsubseteq A$, we have $x \in A^\Imc$.
                This yields $x \in A^{\Imc'}$ by construction.

        		For CIs~(\ref{eq:T_Axioms_X_exists_rB}), consider some $X \sqcap A  \sqsubseteq \exists r.(X \sqcap B) \in \Tmc'$ and let $x \in (X \sqcap A)^{\Imc'}$.
                Again, $x \neq d_0$.
                If $x \in \mn{var}(q)$, $x \in (\exists r.(X \sqcap B))^{\Imc'}$ follows directly as $(x,x) \in r^{\Imc'}$ and $x \in (X \sqcap B)^{\Imc'}$ hold by construction.
                If $x = t$, Property~(\ref{eq:t_d_0_dag}) implies $d_0 \in (X \sqcap A)^\Imc$.
                Together with the fact that \Imc as a model of $\Kmc'$ must satisfy $X \sqcap A  \sqsubseteq \exists r.(X \sqcap B)$, there must be some $y \in \Delta^\Imc$ such that $(d_0,y) \in r^\Imc$ and $y \in (X \sqcap B)^\Imc$.
                Observe that Point~\ref{enum:prop_3} of Claim~(\ref{claim:pruned_model}) and the construction of $\Imc'$ imply
                \begin{align*}
                        \big(t,f(y) \big ) \in r^{\Imc'} \text{ iff } & (d_0,y) \in r^\Imc \text{, where} \\
                        f(y) \coloneqq                                & \begin{cases}
                                                                                t & \text{if } y = d_0 \\
                                                                                y & \text{otherwise}
                                                                        \end{cases}
                \end{align*}
                for all role names $r$.
                We thus have $(t,f(y)) \in r^{\Imc'}$.
                If $y = d_0$, Property~(\ref{eq:t_d_0_dag}) furthermore implies $f(y) \in (X \sqcap B)^{\Imc'}$.
                Otherwise, $f(y) \in (X \sqcap B)^{\Imc'}$ is entailed by construction.
                In either case, we have $x \in (\exists r.(X \sqcap B))^{\Imc'}$.
                Finally, if $x \in \Delta^\Imc \setminus (\mn{var}(q) \cup \{t, d_0\})$, then $x \in (X \sqcap A)^\Imc$ follows via construction.
                As \Imc is a model of $\Kmc'$, \Imc must satisfy $X \sqcap A  \sqsubseteq \exists r.(X \sqcap B)$, i.e., there must be some $y \in (X \sqcap B)^\Imc$ such that $(x,y) \in r^\Imc$.
                If $y \neq d_0$, $y \in (X \sqcap B)^{\Imc'}$ and $(x,y) \in r^{\Imc'}$ follow by construction of $\Imc'$.
                Otherwise, we have $t \in (X \sqcap B)^{\Imc'}$ and $(x,t) \in r^{\Imc'}$ by construction.
                In either case, $x \in (\exists r.(X \sqcap B))^{\Imc'}$ holds.

                We omit the details for the cases of CIs~(\ref{eq:T_Axioms_exists_rB_X}) and (\ref{eq:T_Axioms_X_A_A_A}) as they are very similar to those for CIs~(\ref{eq:T_Axioms_X_exists_rB}) and (\ref{eq:T_Axioms_X_A}), respectively.

        For CIs~(\ref{eq:propagate_L_copy}), let $x \in (\exists u . L)^{\Imc'}$.
        Point~\ref{enum:prop_2} of Claim~(\ref{claim:pruned_model}) and the construction of $\Amc'$
        entail that $x \in \mn{var}(q)$.
        By construction of $\Imc'$, we get $x \in \X^{\Imc'}$ as desired.

\end{proof}

\section{Proofs for Section~\ref{subsection-alchi-data}}

In the main part of the paper, we have defined neighborhoods only for
the unraveling $\Imc'$ of the interpretation \Imc. In the proofs, we
also consider neighborhods in the quotient
$\J = \interleavingof{\I}/{\sim_{\maxradius}}$. We thus start with
a more general definition of neighborhoods.

Let \Imc be an interpretation and $\Delta \subseteq \Delta^{\Imc}$.
For $n \geq 0$ and $d \in \Delta^\Imc \setminus \Delta$, we use
$\Nmc_n^{{\Imc},\Delta}(d)$ to denote the \emph{$n$-neighborhood of
  $d$ in {\Imc} up to} $\Delta$, that is, the set of all elements
$e \in \Delta^{\Imc} \setminus \Delta$ such that the undirected
graph $G_{\Imc}$ associated with \Imc contains a path
$d_0,\dots,d_k$ with $0 \leq k \leq n$, $d_0=d$,
$d_0,\dots,d_{k-1} \notin \Delta$, and $d_k=e$.

So the neighborhoods $\Nmc_n(d)$ defined in the main body of
the paper are now called $\Nmc_n^{{\Imc'}, \basedomainof{\I}}(d)$.

\lemquotient*

				We prove that $\J =
                                \interleavingof{\I}/{\sim_{\maxradius}}$
                                is indeed a model of $\circkb$ with $\J
                                \not\models q(\bar a)$.
				The key to proving the latter is to exhibit suitable local homomorphisms from $\J$ back to~$\interleavingof{\I}$.
				The existence of a homomorphism $h$ from
                                a CQ
                                $p(\bar x)$  in $q$ to $\J$ with
                                $h(\bar x)=a$ would then contradict
                                the fact that $\interleavingof{\I}$ is
                                a countermodel of $q$.
				Indeed, such a homomorphism $h$ would map each connected component $C$ of $p$ into a $\sizeof{q}$-neighborhood $\neighof{\sizeof{q}}{\equivclass{c}}{\J}{\equivclass{\basedomainof{\I}}}$, for some $c \in \domain{\interleavingof{\I}}$ and where $\equivclass{\basedomainof{\I}}$ stands for the set $\{\equivclass{e} \mid e \in {\basedomainof{\I}}\}$.
				By exhibiting a homomorphism $\rho_{c} : \neighof{\sizeof{q}}{\equivclass{c}}{\J}{\equivclass{{\basedomainof{\I}}}} \rightarrow \neighof{\sizeof{q}}{{c}}{\interleavingof{\I}}{{{\basedomainof{\I}}}}$ such that $\rho_{c}^{-1}(\equivclass{{\basedomainof{\I}}}) \subseteq \equivclass{{\basedomainof{\I}}}$, we can find a homomorphism from $C$ 
				to $\interleavingof{\I}$.
				Taking the union of homomorphisms for
                                all of $p$'s connected components
                                gives 
				a homomorphism from the full $p$ to $\interleavingof{\I}$.
				
				Except for the use of Lemma~\ref{lem-lemma5}, concerned with circumscription, we roughly follow \cite{maniere:thesis}, with slight simplifications as we are not considering negative role inclusions.
				For all $n \geq 0$, all $c \in \domain{\interleaving}$, and all $d \in \neighof{n}{c}{\I'}{\basedomainof{\I}} \setminus \basedomainof{\I}$, we denote $w_{n, c}^{d}$ the (unique) word such that $d = {\neighroot{n}{c} w_{n, c}^{d}}$. 
				Let us first formulate two remarks concerning the constructed interpretation~$\J$.
				
				\begin{remark}
					\label{remark:concepts-equiv}
					Combining Conditions~1 and 2.(a) from the definition of $\sim_n$, gives: if $d_1 \sim_n d_2$, then $\typeinof{\interleavingof{\I}}{d_1} = \typeinof{\interleavingof{\I}}{d_2}$.
				\end{remark}
				\begin{remark}
					\label{simsrelations}
					If $d_1 \sim_n d_2$,
					then $d_1 \sim_m d_2$ for any $m \leq n$. 
				\end{remark}
				
				We now define homomorphisms $\rhod$, mentioned in the proof sketch, inductively on $n$ increasing from $0$ to $\sizeof{q}$.
				Starting from the element $\equivclass{c} \in \neighof{0}{\equivclass{c}}{\J}{\equivclass{{\basedomainof{\I}}}}$, we can naturally carry it back as $\rhodof{c} = c \in \neighof{0}{{c}}{\interleavingof{\I}}{{\basedomainof{\I}}}$.
				Assume now that we have defined $\rhodof{d}$ for some $\equivclass{d} \in \neighof{n}{\equivclass{c}}{\J}{\equivclass{{\basedomainof{\I}}}}$ and that we are moving further to an element $\equivclass{e} \in \neighof{n+1}{\equivclass{c}}{\J}{\equivclass{{\basedomainof{\I}}}}$ along an edge $(\equivclass{d}, \equivclass{e})$ in $\J$.
				In the case of $\equivclass{e} \notin \equivclass{{\basedomainof{\I}}}$, the following lemma produces an element $e'$ which is to $\rhodof{d}$ (that is $d'$ in the below statement), what $\equivclass{e}$ is to $\equivclass{d}$.
				We will thus set $\rhodof{e}$ to $e'$.
				The case disjunction that arises reflects $e'$ being either a successor of $d'$ in the corresponding tree-shaped structure of the unraveling $\I'$ (which is identified in the quotient $\J$ as $\sizeof{e} \equiv \sizeof{d} + 1 \mod \maxmod$), or the converse (identified as $\sizeof{d} \equiv \sizeof{e} + 1 \mod \maxmod$).
				
%
				
				\begin{lemma}
					\label{lemma:core}
					Given two elements $\equivclass{d}, \equivclass{e} \in \domain{\J} \setminus \equivclass{{\basedomainof{\I}}}$, if there exists a role $\rolestyle{P}$ such that $(\equivclass{d}, \equivclass{e}) \in \rolestyle{P}^\J$, then there exists a unique element $rB \in \Omega$ such that one of the two following conditions is satisfied:
					
					\begin{itemize}
						\item[\caseedgeplus.]
						$\sizeof{e} \equiv \sizeof{d} + 1 \mod \maxmod$,
						$w_{\maxradius, e}^{e} = w_{\sizeof{q}, d}^{d} rB$ 
						and $\tbox \models \axiom{r}{p}$.
						Furthermore, for all $d' \sim_n d$, the element $e' = d' rB$ belongs to $\domain{\interleavingof{\I}}$ and satisfies $e' \sim_{n - 1} e$.
						\item[\caseedgeminus.] \label{case:edge-minus}
						$\sizeof{d} \equiv \sizeof{e} + 1 \mod \maxmod$,
						$w_{\maxradius, d}^{d} = w_{\sizeof{q}, e}^{e} rB$
						and $\tbox \models \axiom{r^-}{p}$.
						Furthermore, for all $d' \sim_n d$, we have $e'$ such that $d' = e' rB$ and the prefix $e'$ satisfies $e' \sim_{n - 1} e$.
					\end{itemize}
				\end{lemma}
				
				\begin{proof}
					Notice the two conditions are mutually exclusive: $\sizeof{e} \equiv \sizeof{d} + 1 \mod \maxmod$ and $\sizeof{d} \equiv \sizeof{e} + 1 \mod \maxmod$ would imply $0 \equiv 2 \mod \maxmod$, which is impossible as $\maxmod > 2$. Furthermore, in each case $\rstyle{r}\cstyle{B}$ is defined as the last letter of the word $w_{\maxradius, e}^{e}$ (resp $w_{\maxradius, d}^{d}$), which is unique and does not depends on the choice of $e$ (resp $d$) nor on $\rolestyle{P}$.
					This proves the uniqueness.
					
					We now focus on the existence and the additional property.
					From the definition of $\rolestyle{P}^\J$, there exist $({d_0}, {e_0}) \in \rolestyle{P}^{\interleavingof{\I}}$ such that $\equivclass{d_0} = \equivclass{d}$ and $\equivclass{e_0} = \equivclass{e}$. Recall $\equivclass{d}, \equivclass{e} \notin \equivclass{{\basedomainof{\I}}}$, hence $d_0, e_0 \notin {\basedomainof{\I}}$. 
					{Therefore the definition of $\rolestyle{P}^{\interleavingof{\I}}$, yields two cases:}
					\begin{itemize}
						\item
						We have $e_0 = d_0 \rstyle{r}\cstyle{B}$ with $\tbox \models \axiom{r}{p}$.
						It follows that $\sizeof{e_0} = \sizeof{d_0} + 1 \mod \maxmod$
						and $w_{\maxradius, e_0}^{e_0} = w_{\maxradius - 1, d_0}^{d_0} \rstyle{r}\cstyle{B}$, immediately yielding the same properties for $d$ and $e$ as $(\equivclass{d_0}, \equivclass{e_0}) = (\equivclass{d}, \equivclass{e})$.
						
						Consider now $d'$ and $1 \leq n \leq \maxradius$ s.t.\ now $d' \sim_n d$.
						Transitivity gives $d' \sim_n d_0$, and we have in particular $w_{n, d'}^{d'} = w_{n, d_0}^{d_0}$.
						Recall that $e_0 = d_0 \rstyle{r}\cstyle{B}$, hence Condition~2 from the definition of $d' \sim_n d_0$ ensures $d' \rstyle{r}\cstyle{B}$ belongs to $\Delta^{\I'}$.
						Notice it is now sufficient to prove $d' \rstyle{r}\cstyle{B} \sim_{n-1} e_0$: that is because $\equivclass{e} = \equivclass{e_0}$, hence transitivity will conclude the proof. 
						It should be clear that $\neighselfpath{n-1}{d' \rstyle{r}\cstyle{B}} = \neighselfpath{n-1}{e_0}$ and $\sizeof{d' \rstyle{r}\cstyle{B}} \equiv \sizeof{e_0} \mod \maxmod$, hence Conditions~1 and 3 from the definition of $\sim_{n-1}$ are satisfied.
						To verify Condition~2, it suffices to remark that $e_0 = d_0 \rstyle{r}\cstyle{B}$ ensures that the characterization via paths of the $n$-neighborhood of $d_0$ fully determines the characterization of the $(n-1)$ neighborhood of $e_0$.
						But since the former coincides with the characterization of the $n$-neighborhood of $d'$ (using Condition~2 from the definition of $\equivclass{d_0} = \equivclass{d'}$), that fully decides the characterization of the $(n-1)$-neighborhood of $e'$, we are done.
						
						\item
						We have $d_0 = e_0 \rstyle{r}\cstyle{B}$ with $\tbox \models \axiom{r^-}{p}$.
						It follows that $\sizeof{d_0} \equiv \sizeof{e_0} + 1 \mod \maxmod$
						and $w_{\maxradius, d_0}^{d_0} = w_{\maxradius - 1, e_0}^{e_0} \rstyle{r}\cstyle{B}$, immediately yielding the same properties for $d$ and $e$ as $(\equivclass{d_0}, \equivclass{e_0}) = (\equivclass{d}, \equivclass{e})$.
						
						Let now $1 \leq n \leq \maxradius + 1$ be an integer and $d' \sim_n d$.
						Transitivity gives $d' \sim_n d_0$, and we have in particular $\neighselfpath{1}{d'} = \neighselfpath{1}{d_0} = \rstyle{r}\cstyle{B}$ (it is here important to have $n \geq 1$!). 
						That is $d'$ ends by $\rstyle{r}\cstyle{B}$, and therefore we can indeed have prefix $e'$ such that $d' = e' \rstyle{r}\cstyle{B}$. \qedhere
					\end{itemize}
				\end{proof}
				
				Notice the ``strength'' of the equivalence relation $\sim_n$ between $\equivclass{e}$ and $\rhodof{e}$ decreases as we move further in the neighbourhood of $\equivclass{c}$.
				However, since we start from $\rhodof{c} := c \sim_{\maxradius} c$ and explore a $\sizeof{q}$-neighbourhood, the index remains at least $1$.
				This is essential as $\sim_1$ notably encodes relations to elements of $\equivclass{{\basedomainof{\I}}}$ as required by Condition~2.(b) from the definition of $\sim_1$.

				It remains to abstract from 
				the particular choice of $\equivclass{d}$, which is likely not to be the only element of $\neighof{n}{\equivclass{c}}{\J}{\equivclass{{\basedomainof{\I}}}}$ connected to $\equivclass{e}$.
				Taking a closer look at Lemma~\ref{lemma:core}, we observe that $\rhodof{e}$, that is $e'$, is obtained either by adding a letter to $\rhodof{d}$, that is $d'$, or by removing the last letter of $\rhodof{d}$, and that these letters coincide with those in the suffixes of elements ${d}$ and ${e}$.
				Therefore, when moving from $\equivclass{c}$ to $\equivclass{e}$, each added letter appears already in the suffix $\neighselfpath{\sizeof{q} + 1}{e}$ of ${e}$ and, similarly, each removed letter appeared in the suffix $\neighselfpath{\sizeof{q} + 1}{c}$ of $c$.
				This is however not completely true as some steps along the way from $\equivclass{c}$ to $\equivclass{e}$ may cancel each other, \emph{e.g.} a letter is introduced at one step and removed on the next one.
				
				The challenge is therefore to decide the 
				numbers $l^+$ of \emph{actual} additions and $l^-$ of \emph{actual} removals to build $\rhodof{e}$ directly from $c$ and $\equivclass{e}$ (which determines $\neighselfpath{\sizeof{q} + 1}{e}$).
				We then drop the $l^-$ last letters from $c$, which intuitively corresponds to keeping its prefix $\neighroot{l^-}{c}$, and keep the $l^+$ last letters from $\neighselfpath{\sizeof{q} + 1}{e}$, which intuitively corresponds to $\neighselfpath{l^+}{e}$, and concatenate those two words.
				To this end, the next definition captures the relative difference of letters between $\equivclass{c}$ and $\equivclass{e}$, encoded in $\sizeof{c}$ and $\sizeof{e} \mod \maxmod$.
				
				\begin{definition}
					\label{definition:relative-depth}
					Let $\equivclass{c} \in \domain{\J}$ and $n \leq \sizeof{q}$. The \emph{relative depth} of $\equivclass{e} \in \neighof{n}{\equivclass{c}}{\J}{\equivclass{{\basedomainof{\I}}}}$ from $\equivclass{c}$ is the integer $\moddist{c}{e} \in [- n, n ]$ such that $\sizeof{e} \equiv \sizeof{c} + \moddist{c}{e} \mod \maxmod$.
				\end{definition}

				\begin{remark}
					\label{remark:depth-modulo-2}
					By induction on $n \leq \sizeof{q}$, it is straightforward to see that $\moddist{c}{e}$ is well defined. Uniqueness is ensured by $\moddist{c}{e} \leq n \leq \sizeof{q}$.
					A consequence of Lemma~\ref{lemma:core} is that for the smallest $n \leq \sizeof{q}$ such that $\equivclass{e} \in \neighof{n}{\equivclass{c}}{\J}{\equivclass{{\basedomainof{\I}}}}$ we have $\moddist{c}{e} = n \mod 2$.
				\end{remark}

				We can now identify how many additions and removals canceled each other.
				Indeed, if it takes exactly $n$ steps to reach $\equivclass{e}$ {from $\equivclass{c}$}, with relative difference of $\delta := \moddist{c}{e}$, then $n - \sizeof{\delta}$ is the number of steps that canceled each other, hence: 
				$\frac{n-\sizeof{\delta}}{2}$ canceled additions and $\frac{n-\sizeof{\delta}}{2}$ canceled removals.
				Therefore, the number $l^+$ of \emph{actual} additions is $\frac{n-\sizeof{\delta}}{2} + \delta$ if $\delta \geq 0$, or $\frac{n - \sizeof{\delta}}{2}$ if $\delta \leq 0$, that is in both cases $\frac{n + \delta}{2}$.
				Similarly we obtain $\frac{n - \delta}{2}$ for the number $l^-$ of \emph{actual} removals.
				Notice Remark~\ref{remark:depth-modulo-2} ensures both $l^+$ and $l^-$ are integers.
				
				The next theorem formalizes the above intuitions.
				Notably, in the non-trivial cases, $\qtoiof{n}{e}$ is obtained by 
				removing the $\frac{n - \delta}{2}$ last letters of ${c}$ and keeping the $\frac{n + \delta}{2}$ last letters from the suffix of ${e}$, which results in the word $\neighroot{\frac{n - \moddist{c}{e}}{2}}{c} \neighselfpath{\frac{n + \moddist{c}{e}}{2}}{e}$.
				It is then a technicality to verify these syntactical operations on words produces an element in the domain of $\interleavingof{\I}$ and that the defined $\qtoi{n}$ is a homomorphism as desired.
				
				\begin{lemma}
					\label{morphismneighbourhoods}
					For all $c \in \domain{\interleavingof{\I}}$ and all $n \leq \sizeof{q}$, the following mapping $\qtoiof{n}{e}$:
					\[
					\begin{array}{r@{\;}c@{\;}l}
					 \neighof{n}{\equivclass{c}}{\J}{\equivclass{{\basedomainof{\I}}}} & \rightarrow & \neighof{n}{c}{\interleavingof{\I}}{{\basedomainof{\I}}} 
					 \smallskip \\
					\equivclass{e} & \mapsto &
					\left\{ \begin{array}{ll}
						\qtoiof{n-1}{e} \qquad \text{if } \equivclass{e} \in \neighof{n-1}{\equivclass{c}}{\J}{\equivclass{{\basedomainof{\I}}}} 
						\smallskip \\
						e \qquad \qquad \qquad \text{if } \equivclass{e} \in \equivclass{{\basedomainof{\I}}} 
						\smallskip \\
						\neighroot{\frac{n - \moddist{c}{e}}{2}}{c} \neighselfpath{\frac{n + \moddist{c}{e}}{2}}{e} \quad \text{otherwise}
					\end{array} \right.
					\end{array}
					\]
					is a homomorphism satisfying $\qtoiof{n}{{e}} \sim_{\maxradius - n} {e}$ and $\qtoi{n}^{-1}(\equivclass{{\basedomainof{\I}}}) \subseteq \equivclass{{\basedomainof{\I}}}$.
				\end{lemma}
				
				\begin{proof}
					Let $c \in \domain{\interleavingof{\I}}$.
					We proceed by induction on $n \leq \sizeof{q}$ and prove along a technical statement.
					Property $\qtoiof{n}{e} \sim_{\maxradius - n} {e}$ already ensures $\neighselfpath{\maxradius - n}{\qtoiof{n}{e}} = \neighselfpath{\maxradius - n}{e}$; we reinforce this latter fact as follows.
					If $e \in \neighof{n}{\equivclass{c}}{\J}{\equivclass{{\basedomainof{\I}}}} \setminus \neighof{n-1}{\equivclass{c}}{\J}{\equivclass{{\basedomainof{\I}}}}$, then:
					\begin{equation}
						\tag{$*$}
						\label{property:suffix}
						\neighselfpath{\maxradius - \frac{n - \moddist{c}{e}}{2}}{\qtoiof{n}{e}} = \neighselfpath{\maxradius - \frac{n - \moddist{c}{e}}{2}}{e}
					\end{equation}
					It is indeed a stronger statement since $-n \leq \moddist{c}{e} \leq n$ leads to $0 \leq \frac{n - \moddist{c}{e}}{2} \leq n$, hence $\maxradius - n \leq \maxradius - \frac{n - \moddist{c}{e}}{2}$. Property~\ref{property:suffix} therefore provides a more precise information about the suffix of $\qtoi{n}{e}$.
					
					{\bf Base case: $n = 0$.}
					Let $\equivclass{e} \in \neighof{0}{\equivclass{c}}{\J}{\equivclass{{\basedomainof{\I}}}}$, hence $\equivclass{e} = \equivclass{c}$. If $\equivclass{c} \in \equivclass{{\basedomainof{\I}}}$, then $\qtoi{0}{e} = e = c$. Otherwise we have $\moddist{c}{e} = 0$, hence $\qtoi{0}{e} = \neighroot{0}{c} \neighselfpath{0}{c} = c$. In both cases $\qtoi{0}{e} = c$, and it is straightforward that all the desired properties hold. {In particular, agreeing that $\neighof{-1}{\equivclass{c}}{\J}{\equivclass{{\basedomainof{\I}}}}$ can reasonably be set to $\emptyset$, our technical statement holds.}
					
					{\bf Induction case.} 
					Assume the statement holds for $0 \leq n - 1 < \sizeof{q}$.
					Let $\equivclass{e} \in \neighof{n}{\equivclass{c}}{\J}{\equivclass{{\basedomainof{\I}}}}$.
					If $\equivclass{e} \in \neighof{n-1}{\equivclass{c}}{\J}{\equivclass{{\basedomainof{\I}}}}$, then the induction hypothesis applies directly on $\equivclass{e}$ and provides (stronger versions of) the desired properties.
					Otherwise, we have by definition of neighbourhoods an element $\equivclass{d} \in \neighof{n-1}{\equivclass{c}}{\J}{\equivclass{{\basedomainof{\I}}}}$, not belonging to $\equivclass{{\basedomainof{\I}}}$ {nor to $\neighof{n-2}{\equivclass{c}}{\J}{\equivclass{{\basedomainof{\I}}}}$}, and a role $\rolestyle{P} \in \posroles$ such that $(\equivclass{d}, \equivclass{e}) \in \rolestyle{P}^\J$.
					We apply the induction hypothesis on $\equivclass{d}$, which gives $\qtoiof{n-1}{d} = \neighroot{\frac{n - 1 - \moddist{c}{d}}{2}}{d} \neighselfpath{\frac{n - 1 + \moddist{c}{d}}{2}}{d}$ since $\equivclass{d} \notin \equivclass{{\basedomainof{\I}}}$.
					We further distinguish between $\equivclass{e} \in {\basedomainof{\I}}$ and $\equivclass{e} \notin {\basedomainof{\I}}$, the latter subcase yielding two subcases by applying Lemma~\ref{lemma:core} and distinguishing between Cases \caseedgeplus and \caseedgeminus.
					We have therefore three cases to treat.
					
					{\bf $\equivclass{e} \in \equivclass{{\basedomainof{\I}}}$.}
					We have $\qtoiof{n}{e} = e$ and the only non-trivial property to prove is that $e \in \neighof{n}{c}{\interleavingof{\I}}{{\basedomainof{\I}}}$.
					Recall the induction hypothesis ensures in particular $\qtoiof{n-1}{d} \sim_1 d$. 
					Condition~2.(b) from the definition of $\sim_1$ ensures $(\qtoiof{n-1}{d}, e) \in \rolestyle{P}^{\interleavingof{\I}}$, which provides the desired property.
					
					{\bf \caseedgeplus.}
					Case \caseedgeplus ensures $\sizeof{e} = \sizeof{d} + 1 \mod \maxmod$, hence $\moddist{c}{e} = \moddist{c}{d} + 1$, and $w_{\maxradius, e}^{e} = w_{\maxradius - 1, d}^{d} \rstyle{r}\cstyle{B}$.
					Therefore, our element $\qtoiof{n}{e}$ of interest simplifies as:
					\[
					\begin{array}{r@{\;}c@{\;}l}
						\qtoiof{n}{e} & = &
						\neighroot{\frac{n - \moddist{c}{e}}{2}}{c} w_{\frac{n + \moddist{c}{e}}{2}, e}^e 
						\smallskip \\
						& = & \neighroot{\frac{n - (\moddist{c}{d} + 1)}{2}}{c} w_{\frac{n + (\moddist{c}{d} + 1)}{2}, e}^e 
						\smallskip \\
						& = & \neighroot{\frac{(n-1) - \moddist{c}{d}}{2}}{c} w_{\frac{(n-1) + \moddist{c}{d}}{2} + 1, e}^e 
						\smallskip \\
						& = & \neighroot{\frac{(n-1) - \moddist{c}{d}}{2}}{c} w_{\frac{(n-1) + \moddist{c}{d}}{2}, d}^{d} \rstyle{r}\cstyle{B} 
						\smallskip \\
						& = & \qtoiof{n-1}{d} \rstyle{r}\cstyle{B},
					\end{array}
					\]
					which is well-defined and satisfies $\qtoiof{n}{e} \sim_{\maxradius - n} e$ from Lemma~\ref{lemma:core}.
					Recalling that the induction hypothesis gives $\qtoiof{n-1}{d} \in \neighof{n-1}{c}{\interleavingof{\I}}{{\basedomainof{\I}}}$, it follows that $\qtoiof{n}{e} \in \neighof{n}{c}{\interleavingof{\I}}{{\basedomainof{\I}}}$.
					Furthermore, notice that $\equivclass{e}$ and $\equivclass{d}$ satisfy all conditions of our additional statement. 
					Since in Case~\caseedgeplus we have $\tbox \models \axiom{r}{p}$, reusing 
					$\qtoiof{n}{e} = \qtoiof{n-1}{d} \rstyle{r}\cstyle{B}$ immediately yields $(\qtoiof{n-1}{d}, \qtoiof{n}{e}) \in \rolestyle{P}^{\interleavingof{\I}}$.
					
					Checking that Property~\ref{property:suffix} holds is now a technicality, and recall that since $d \in \neighof{n-1}{\equivclass{c}}{\J}{\equivclass{{\basedomainof{\I}}}} \setminus \neighof{n-2}{\equivclass{c}}{\J}{\equivclass{{\basedomainof{\I}}}}$, we can apply it to $d$ by induction hypothesis. We hence have:
					\[
					\begin{array}{r@{\hspace*{-50pt}}l}
						w_{\maxradius - \frac{n - \moddist{c}{e}}{2}, \qtoiof{n}{e}}^{\qtoiof{n}{e}}
						\smallskip \\
						& = w_{\maxradius - \frac{n - \moddist{c}{e}}{2} - 1, \qtoiof{n-1}{d}}^{\qtoiof{n-1}{d}} \rstyle{r}\cstyle{B}
						\smallskip \\
						& = w_{\maxradius - \frac{(n - 1) + 1 - (\moddist{c}{d} + 1)}{2} - 1, \qtoiof{n-1}{d}}^{\qtoiof{n-1}{d}} \rstyle{r}\cstyle{B}
						\smallskip \\
						& = w_{\maxradius - \frac{(n - 1) - \moddist{c}{d}}{2} - 1, \qtoiof{n-1}{d}}^{\qtoiof{n-1}{d}} \rstyle{r}\cstyle{B}
						\smallskip \\
						& = w_{\maxradius - \frac{(n - 1) - \moddist{c}{d}}{2} - 1, d}^{d} \rstyle{r}\cstyle{B}
						\smallskip \\
						& = w_{\maxradius - \frac{n - \moddist{c}{e}}{2}, e}^e.
					\end{array}
					\]
					
					{\bf \caseedgeminus.}
					Case \caseedgeminus ensures $\sizeof{e} = \sizeof{d} - 1 \mod \maxmod$, hence $\moddist{c}{e} = \moddist{c}{d} - 1$, and $w_{\maxradius, d}^{d} = w_{\maxradius - 1, e}^{e} \rstyle{r}\cstyle{B}$.
					By induction hypothesis, element $\qtoiof{n-1}{d} = \neighroot{\frac{(n-1) - \moddist{c}{d}}{2}}{d} w_{\frac{(n - 1) + \moddist{c}{d}}{2}, d}^{d}$ is well-defined.
					Notice Property~\ref{property:suffix} on $d$ (which, again can be applied as $d \in \neighof{n-1}{\equivclass{c}}{\J}{\equivclass{{\basedomainof{\I}}}} \setminus \neighof{n-2}{\equivclass{c}}{\J}{\equivclass{{\basedomainof{\I}}}}$) gives more precise information on the suffix of $\qtoiof{n-1}{d}$ than the definition of $\qtoiof{n-1}{d}$, because $n \leq \maxradius$ leads to $\frac{(n - 1) + \moddist{c}{d}}{2} + 1 \leq \maxradius - \frac{(n - 1) - \moddist{c}{d}}{2}$. 
					Therefore, $w_{\frac{(n - 1) + \moddist{c}{d}}{2} + 1, d}^{d}$ is itself a suffix of $w_{\maxradius - \frac{(n-1) - \moddist{c}{d}}{2}, d}^{d}$, which equals $\neighselfpath{\maxradius - \frac{(n -1) - \moddist{c}{d}}{2}}{\qtoiof{n-1}{d}}$.
					Hence we obtain: 
					\[
					\begin{array}{r@{\;}c@{\;}l}
						\qtoiof{n-1}{d}
						& = & \neighroot{\frac{(n-1) - \moddist{c}{d}}{2} + 1}{d} w_{\frac{(n - 1) + \moddist{c}{d}}{2} + 1, d}^{d} 
						\smallskip \\
						& = & \neighroot{\frac{n - \moddist{c}{e}}{2}}{d} w_{\frac{n + \moddist{c}{e}}{2} + 1, d}^{d} 
						\smallskip \\
						& = &\neighroot{\frac{n - \moddist{c}{e}}{2}}{d} w_{\frac{n + \moddist{c}{e}}{2}, e}^{e} \rstyle{r}\cstyle{B} 
						\smallskip \\
						& = & \qtoiof{n}{e} \rstyle{r}\cstyle{B}
					\end{array}
					\]
					Lemma~\ref{lemma:core} now ensures $\qtoiof{n}{e} \sim_{\maxradius - n} e$.
					Recalling that the induction hypothesis gives $\qtoiof{n-1}{d} \in \neighof{n-1}{c}{\interleavingof{\I}}{{\basedomainof{\I}}}$, it follows that $\qtoiof{n}{e} \in \neighof{n}{c}{\interleavingof{\I}}{{\basedomainof{\I}}}$.
					Furthermore, notice that $\equivclass{e}$ and $\equivclass{d}$ satisfy all conditions of our additional statement. Since in Case~\caseedgeminus we have $\tbox \models \axiom{r^-}{p}$, reusing 
					$\qtoiof{n-1}{d} = \qtoiof{n}{e} \rstyle{r}\cstyle{B}$ immediately yields $(\qtoiof{n-1}{d}, \qtoiof{n}{e}) \in \rolestyle{P}^{\interleavingof{\I}}$.

					Again, we check Property~\ref{property:suffix} holds:
					\[
					\begin{array}{r@{\hspace*{-50pt}}l}
						w_{\maxradius - \frac{n - \moddist{c}{e}}{2}, \qtoiof{n}{e}}^{\qtoiof{n}{e}} \rstyle{r}\cstyle{B}
						\smallskip \\
						& = w_{\maxradius - \frac{n - \moddist{c}{e}}{2} + 1, \qtoiof{n-1}{d}}^{\qtoiof{n-1}{d}} 
						\smallskip \\
						& = w_{\maxradius - \frac{(n - 1) + 1 - (\moddist{c}{d} - 1)}{2} + 1, \qtoiof{n-1}{d}}^{\qtoiof{n-1}{d}}
						\smallskip \\
						& = w_{\maxradius - \frac{(n - 1) - \moddist{c}{d}}{2}, \qtoiof{n-1}{d}}^{\qtoiof{n-1}{d}}
						\smallskip \\
						& = w_{\maxradius - \frac{(n - 1) - \moddist{c}{d}}{2}, d}^{d}
						\smallskip \\
						& = w_{\maxradius - \frac{n - \moddist{c}{e}}{2}, e}^e \rstyle{r}\cstyle{B}
					\end{array}
					\]
					
					We now verify that $\qtoi{n}$ is a homomorphism.
					\begin{itemize}
						\item
						Let $\equivclass{u} \in \cstyle{A}^\J \cap \neighof{n}{\equivclass{c}}{\J}{\equivclass{{\basedomainof{\I}}}}$. By definition of $\cstyle{A}^\J$, we have $e \in \cstyle{A}^{\interleavingof{\I}}$. Since $n \leq \sizeof{q}$ we have $\qtoiof{n}{u} \sim_1 e$, hence applying Remark~\ref{remark:concepts-equiv} we obtain $\qtoiof{n}{u} \in \cstyle{A}^{\interleavingof{\I}}$.
						\item
						Let $(\equivclass{u}, \equivclass{v}) \in \rolestyle{R}^\J \cap (\neighof{n}{\equivclass{c}}{\J}{\equivclass{{\basedomainof{\I}}}} \times \neighof{n}{\equivclass{c}}{\J}{\equivclass{{\basedomainof{\I}}}})$.
						If $\equivclass{u} \in \equivclass{{\basedomainof{\I}}}$ or $\equivclass{v} \in \equivclass{{\basedomainof{\I}}}$, then Condition~2.(b) from the definition of $\sim_1$ applies on $\qtoiof{n}{u}$ or on $\qtoiof{n}{v}$ (recall $\qtoiof{n}{u} \sim_1 u$ and $\qtoiof{n}{v} \sim_1 v$) and gives $(\qtoiof{n}{u}, \qtoiof{n}{v}) \in \rolestyle{R}^\J$.
						Otherwise $\equivclass{u} \notin \equivclass{{\basedomainof{\I}}}$ and $\equivclass{v} \notin \equivclass{{\basedomainof{\I}}}$. 
						Let $n_1, n_2$ be the minimum integers such that $\equivclass{u} \in \neighof{n_1}{\equivclass{c}}{\J}{\equivclass{{\basedomainof{\I}}}}$ and $\equivclass{v} \in \neighof{n_2}{\equivclass{c}}{\J}{\equivclass{{\basedomainof{\I}}}}$. 
						Since $(\equivclass{u}, \equivclass{v}) \in \rolestyle{R}^\J$, we have $n_1 - n_2 \in \{ -1, 0, 1 \}$.
						Definitions of $\moddist{c}{u}$ and $\moddist{c}{v}$ lead to $\sizeof{u} - \sizeof{v} \equiv \moddist{c}{u} - \moddist{c}{v} \mod \maxmod$.
						Lemma~\ref{lemma:core} gives $\sizeof{u} \equiv \sizeof{v} \pm 1 \mod \maxmod$.
						Recall $\moddist{c}{u}, \moddist{d}{v} \in [-\sizeof{q}, \sizeof{q}]$, hence $-2\sizeof{q} - 1 \leq \moddist{c}{u} - \moddist{c}{u} \mp 1 \leq 2\sizeof{q} + 1$.
						Since $\moddist{c}{u} - \moddist{d}{v} \mp 1 \equiv 0 \mod \maxmod$ and $2 \sizeof{q} + 1 < 2 \sizeof{q} + 3$, we must have $\moddist{c}{u} - \moddist{c}{v} = \pm 1$.
						Joint to Remark~\ref{remark:depth-modulo-2}, it excludes the case $n_1 - n_2 = 0$. 
						We are hence left with $n_1 = n_2 \pm 1$. 
						Applying our additional property with $k := \max(n_1, n_2)$ gives $(\qtoiof{n}{u}, \qtoiof{n}{v}) \in \rolestyle{R}^{\interleavingof{\I}}$.
						
					\end{itemize}
					Finally, $\qtoi{n}^{-1}({{\basedomainof{\I}}}) \subseteq \equivclass{{\basedomainof{\I}}}$ is a straightforward consequence of $\qtoiof{n}{u} \sim_1 {u}$ (and again, recall elements from ${\basedomainof{\I}}$ are alone in their equivalent class!).
				\end{proof}

				Let us now complete the proof of Lemma~\ref{lem-quotient} with Lemma~\ref{morphismneighbourhoods} in hand.
				
				\begin{proof}[Proof of Lemma~\ref{lem-quotient}]
					We first prove that $\J$ is a model of $\kb$ by considering each possible shape of assertions and axioms:
					\begin{itemize}
						\item
						{\bf $\conceptassertion{A}{a}$.}
						Since $\interleavingof{\I}$ is a model, we have $\istyle{a} \in \cstyle{A}^{\interleavingof{\I}}$. Therefore, the definition of $\cstyle{A}^{\J}$ gives $\equivclass{\istyle{a}} = \istyle{a} \in \cstyle{A}^{\J}$.
						
						\item
						{\bf $\roleassertion{P}{a}{b}$.}
						Since $\interleavingof{\I}$ is a model, we have $(\istyle{a}, \istyle{b}) \in \rolestyle{P}^{\interleavingof{\I}}$. Therefore, the definition of $\rolestyle{P}^{\J}$ gives $(\equivclass{\istyle{a}}, \equivclass{\istyle{b}}) = (\istyle{a}, \istyle{b}) \in \rolestyle{P}^{\J}$.
						
						\item
						{\bf $\axiom{\top}{A}$.}
						Let $u \in \top^{\J} = \domain{\J}$. By definition of $\domain{\J}$, there exists $u_0 \in \domain{\interleavingof{\I}}$ such that $\equivclass{u_0} = u$. Since $u_0 \in \top^{\interleavingof{\I}}$ and $\interleavingof{\I}$ is a model, it ensures $u_0 \in \cstyle{A}^{\interleavingof{\I}}$. Therefore the definition of $\cstyle{A}^{\J}$ gives $u = \equivclass{u_0} \in \cstyle{A}^{\J}$.
						
						\item
						{\bf $\axiom{A_1 \sqcap A_2}{A}$.}
						Let $u \in (\cstyle{A}_1 \sqcap \cstyle{A}_2)^{\J}$. By definition of $\cstyle{A}_1^{\J}$ and $\cstyle{A}_2^{\J}$, there exists $u_1 \in \cstyle{A}_1^{\interleavingof{\I}}$ and $u_2 \in \cstyle{A}_2^{\interleavingof{\I}}$ with $\equivclass{u_1} = \equivclass{u_2} = u$.
						Remark~\ref{remark:concepts-equiv} ensures $u_1$ and $u_2$ satisfy the same concepts, that is in particular $u_1 \in (\cstyle{A}_1 \sqcap \cstyle{A}_2)^{\interleavingof{\I}}$.
						Since $\interleavingof{\I}$ is a model, it ensures $u_1 \in \cstyle{A}^{\interleavingof{\I}}$, yielding by definition of $\cstyle{A}^\J$ that $u = \equivclass{u_1} \in \cstyle{A}^\J$.
						
						\item
						{\bf $\axiom{A_1}{\existsrole{R}.{A_2}}$.}
						Let $u \in \cstyle{A}_1^{\J}$. By definition of $\cstyle{A_1}^{\J}$ there exists $u_0 \in \cstyle{A_1}^{\interleavingof{\I}}$ with $\equivclass{u_0} = u$.
						Since $\interleavingof{\I}$ is a model, it ensures there exists $v_0 \in \cstyle{A_2}^{\interleavingof{\I}}$ with $(u_0, v_0) \in \rolestyle{R}^{\interleavingof{\I}}$. By definition of $\cstyle{A_2}^\J$ and $\rolestyle{R}^\J$, the element $v := \equivclass{v_0}$ satisfies both $v \in \cstyle{A_2}^{\J}$ and $(u, v) \in \rolestyle{R}^{\J}$, that is $u \in (\existsrole{R}.{A_2})^{\J}$.
						
						\item
						{\bf $\axiom{\existsrole{R}.{A_1}}{A_2}$.}
						Let $u \in (\existsrole{R}.{A_1})^{\J}$, that is, there exists $v \in \cstyle{A}_1^{\J}$ with $(u, v) \in \rolestyle{R}^{\J}$.
						By definition of $\cstyle{A}_1^\J$ and $\rolestyle{R}^\J$, there exist $(u_0, v_0) \in \rolestyle{R}^{\interleavingof{\I}}$ and $v_1 \in \cstyle{A}_1^{\interleavingof{\I}}$ such that $\equivclass{u_0} = u$ and $\equivclass{v_0} = \equivclass{v_1} = v$.
						Remark~\ref{remark:concepts-equiv} ensures $v_0$ and $v_1$ satisfy the same concepts,  in particular $u_0 \in (\existsrole{R}.{A_1})^{\interleavingof{\I}}$.
						Since $\interleavingof{\I}$ is a model, this ensures $u_0 \in \cstyle{A}_2^{\interleavingof{\I}}$, yielding by definition of $\cstyle{A}_2^\J$ that $u = \equivclass{u_0} \in \cstyle{A}_2^\J$.
						
						\item
						{\bf \axnotright.}
						By contradiction, assume $u \in \cstyle{A}^{\J} \cap \cstyle{B}^{\J}$.
						By definition there exists $v \in \cstyle{A}^{\interlacingof{\I}}$ and $w \in \cstyle{B}^{\interlacingof{\I}}$ with $\equivclass{v} = \equivclass{w} = u$.
						Remark~\ref{remark:concepts-equiv} ensures $v$ and $w$ satisfy the same concepts, contradicting $\interlacingof{\I}$ being a model.
						
						\item
						{\bf \axnotleft.}
						Let $u \in \cstyle{\lnot B}^{\J}$.
						By definition of $\domain{\J}$, there exists $v \in \interlacingof{\I}$ such that $\equivclass{v} = u$.
						Since $u \notin \cstyle{B}^{\J}$, we have ${v} \notin \cstyle{B}^{\interlacingof{\I}}$.
						Hence $\interlacingof{\I}$ being a model gives ${v} \in \cstyle{A}^{\interlacingof{\I}}$, yielding by definition $u = \equivclass{v} \in \cstyle{A}^{\J}$.
						
						\item
						{\bf $\axiom{p}{r}$.}
						Let $(u, v) \in \rolestyle{P}^{\J}$.
						By definition of $\rolestyle{P}^{\J}$, there exists $(u_0, v_0) \in \rolestyle{P}^{\interleavingof{\I}}$ such that $\equivclass{u_0} = u$ and $\equivclass{v_0} = v$.
						Since $\interleavingof{\I}$ is a model, it ensures $(u_0, v_0) \in \rolestyle{R}^{\interleavingof{\I}}$, hence $(\equivclass{u_0}, \equivclass{v_0}) = (u, v) \in \rolestyle{R}^{\J}$ by definition of $\rolestyle{R}^\J$.
					\end{itemize}
				
					It remains to prove that $\J$ is minimal w.r.t.\ $<_\CP$.
					We use Lemma~\ref{lem-lemma5} with \Imc as reference model, and it
					suffices to show that the preconditions of that lemma are
					satisfied. This is in fact
                                        easy based on
                                        Remark~\ref{remark:concepts-equiv},
                                        following
                                        exactly the same lines as the
                                        final part of the proof of Lemma~\ref{lem_interlacing_is_countermodel}.
					
					We now prove $\J \not\models q(\bar{a})$.
					By contradiction, assume we have $p \in q$ and a homomorphism $\match : p \rightarrow \J$ s.t.\ $p(\bar{x}) = \bar{a}$.
					Consider the set of variables $\tv_\match := \{ v \mid v \in \ty, \match(v) \notin \equivclass{{\basedomainof{\I}}} \}$.
					Let $\mathcal{C}$ denote the set of connected components of $\tv_\match$ in $p_{\mid \tv_\match}$ (that is the query obtained by keeping only those atoms containing variables from $\tv_\match$).
					For each connected component $C \in \mathcal{C}$, choose a reference variable $v_C \in C$.
					Since $\match$ is a homomorphism and $\sizeof{C} \leq \sizeof{q}$, every variable $v \in C$ satisfies $\match(v) \in \neighof{\sizeof{q}}{\match(v_C)}{\J}{\equivclass{{\basedomainof{\I}}}}$. Let $d_C \in \domain{\interleavingof{\I}}$ denote one's favorite representative for the class of $\match(v_C)$ (that is $\equivclass{d_C} = \match(v_C)$).
					From Lemma~\ref{morphismneighbourhoods}, we have a homomorphism $\rho_C : \neighof{\sizeof{q}}{\match(v_C)}{\J}{\equivclass{{\basedomainof{\I}}}} \rightarrow \neighof{\sizeof{q}}{d_C}{\interleavingof{\I}}{{{\basedomainof{\I}}}}$.
					Using these $\rho_C$, one per $C \in \mathcal{C}$, we define:
					$$
					\begin{array}{rcl}
						\match' : \tx \cup \ty & \rightarrow & \domain{\interleavingof{\I}} \\
						v & \mapsto & \left\{
						\begin{array}{ll}
							\rho_C(\match(v)) & \text{ if } v \in C, C \in \mathcal{C} \\
							e & \text{ if } \match(v) = \equivclass{e} \in \equivclass{{\basedomainof{\I}}}
						\end{array}  \right.
					\end{array}
					$$
					Since each $\rho_{C}$ is a homomorphism (again Lemma~\ref{morphismneighbourhoods}), we can check the overall $\match'$ is also a homomorphism:
					\begin{itemize}
						\item
						Consider $\cstyle{A}(v) \in q$. If $v \in C$ for some $C \in \mathcal{C}$, then $\rho_C$ being a homomorphism gives $\match'(v) \in \cstyle{A}^{\interleavingof{\I}}$. Otherwise $\match(v) = \equivclass{e} \in \equivclass{{\basedomainof{\I}}}$, but since $\match$ is a homomorphism we have $\match(v) \in \cstyle{A}^\J$. Since $\equivclass{e} = \{ e \}$ and by definition of $\cstyle{A}^\J$, it ensures $e \in \cstyle{A}^{\interleavingof{\I}}$, that is $\match'(v) \in \cstyle{A}^{\interleavingof{\I}}$.
						\item 
						Consider $\rolestyle{R}(u, v) \in q$. 
						\begin{itemize}
							\item
							If both $\match(u), \match(v) \notin \equivclass{{\basedomainof{\I}}}$, then we can find $C \in \mathcal{C}$ such that $u, v \in C$, and then we use $\rho_C$ being a homomorphism.
							\item
							If both $\match(u), \match(v) \in \equivclass{{\basedomainof{\I}}}$, then the definition of $\rolestyle{R}^\J$ provides $(u_0, v_0) \in \rolestyle{R}^{\interleavingof{\I}}$ with $\equivclass{u_0} = \match(u) \in \equivclass{{\basedomainof{\I}}}$ and $\equivclass{v_0} = \match(v) \in \equivclass{{\basedomainof{\I}}}$. Hence $\equivclass{u_0} = \{ u_0 \}$ and $\equivclass{v_0} = \{ v_0 \}$, which gives $(\match'(u), \match'(v)) \in \rolestyle{R}^{\interleavingof{\I}}$.
							\item
							If $\match(u) \notin \equivclass{{\basedomainof{\I}}}$ and $\match(v) \in \equivclass{{\basedomainof{\I}}}$, then we have $\match'(u) = \rho_C(\match(u))$ for some $C \in \mathcal{C}$. Lemma~\ref{morphismneighbourhoods} ensures $\match'(u) \sim_1 \match(u)$, and since $\match$ is a homomorphism, we also have $(\match(u), \match(v)) \in \rolestyle{R}^\J$. Therefore we can apply Condition~2.(b) and we obtain $(\match'(u), \match'(v)) \in \rolestyle{R}^{\interleavingof{\I}}$.
						\end{itemize}
					\end{itemize}
					In particular, $\match'$ is a homomorphism s.t.\ $\match'(\bar{x}) = \bar{a}$, hence the desired contradiction with $\interleavingof{\I}$ being a countermodel.
					
				\end{proof}

\thmdataupperalchi*
\begin{proof}
  Assume that we are given a circumscribed $\ALCHI$ KB $\Circ(\Kmc)$ with
  $\Kmc=(\Tmc,\Amc)$ and a UCQ $q$.  We describe a $\Sigma^p_2$ procedure to decide whether $\Circ(\Kmc) \not\models
  q(\bar a)$.
  
  We first guess an interpretation \Imc with $|\Delta^\Imc| \leq |\Amc|+(2^{|\Tmc|+2}+1)^{3|q|}$.
  We next check in polynomial time that \Imc is a model of \Kmc
  and that \Imc is minimal w.r.t.\ $<_{\mn{CP}}$ by co-guessing a
  model \Jmc of \Kmc with $\Jmc <_{\mn{CP}} \Imc$. If one of the
  checks fails, we reject. Otherwise, for each CQ $p$ in $q$, we verify whether there is a homomorphism from $p$ to \Imc. This is done brute-force, in time $O(|\Delta^{\Imc}|^{|q|})$.
  We accept if there is no such homomorphism and reject otherwise.
  This procedure is correct due to Lemma~\ref{lem-quotient}.
\end{proof}

\thmdatalowerel*
%
%
\begin{proof}
	\renewcommand{\cstyle}[1]{{\mathsf{#1}}}
	\renewcommand{\rstyle}[1]{{\mathsf{#1}}}

	\newcommand{\cvar}{\cstyle{Var}} 
	\newcommand{\cmin}{\cstyle{Min}}
	\newcommand{\cpos}{\cstyle{Pos}}
	\newcommand{\cneg}{\cstyle{Neg}}
	\newcommand{\ctrue}{\cstyle{True}}
	\newcommand{\cfalse}{\cstyle{False}}
	\newcommand{\cfirst}{\cstyle{First}}
	\newcommand{\clast}{\cstyle{Last}}
	\newcommand{\cposneg}{\cstyle{\cpos \sqcap \cneg}}
	\newcommand{\cbad}{\cstyle{Bad}}
	\newcommand{\cgoal}{\cstyle{Goal}}
	\newcommand{\creserved}{\cstyle{Reserved}}
	\newcommand{\csavetrue}{\cstyle{SaveTrue}}
	\newcommand{\csavefalse}{\cstyle{SaveFalse}}
	
	\newcommand{\rval}{\rstyle{eval}}
	\newcommand{\rsave}{\rstyle{freeze}}
	\newcommand{\rsavetrue}{\rsave_\ctrue}
	\newcommand{\rsavefalse}{\rsave_\cfalse}
	\newcommand{\rclause}{\rstyle{next}}
	\newcommand{\rpos}{\rstyle{pos}}
	\newcommand{\rneg}{\rstyle{neg}}
	\newcommand{\raux}{\rstyle{checks}}
	
	\newcommand{\iclause}{\istyle{c}}
	\newcommand{\ivarx}{\istyle{x}}
	\newcommand{\ivary}{\istyle{y}}
	\newcommand{\itrue}{\istyle{t}}
	\newcommand{\ifalse}{\istyle{f}}
	\newcommand{\isave}{{x}}
	\newcommand{\iaux}{a}

	We give a polynomial time reduction from
        $\forall\exists\mn{3SAT}$ \cite{stockmeyer76}, that is, the
        problem to decide whether a given $\forall\exists$-3CNF
        sentence is true. Thus let 
        $\forall \bar x \exists \bar y \, \vp$ be such a sentence where
        $\bar x=x_1 \cdots x_m$, $\bar y = y_1 \cdots y_n$, and
        $\varphi=\bigwedge^\ell_{i=1} \bigvee_{j=1}^3 L_{ij}$ with
        $L_{ij} = v$ or $L_{ij} = \lnot{v}$ for some
        $v \in \{ x_1, \dots x_m, y_1, \dots, y_n \}$.  We construct a
        circumscribed \EL KB $\circkb$ and an atomic query
        $\query$ such that $\circkb \models \query$ iff
        $\forall \bar x \exists \bar y \, \varphi$ is true.

        The circumscription pattern minimizes the concept name $\cmin$
        and lets all other concept names vary, and the instance query
        is $q=\cgoal(c_0)$.
	We now describe how to construct the KB $\kb=(\Tmc,\Amc)$, not
        strictly
        separating \Tmc from \Amc. We  first represent $\varphi$ in
        \Amc using the following assertions:
\begin{align}
	\cvar(v)
	& \quad \text{for all } v \in \tx \cup \ty
	\label{eq:vars} \\
	\rpos_j(\iclause_i, v) 
	& \quad \text{if } L_{ij} = v
	\label{eq:pos} \\
	\rneg_j(\iclause_i, v) 
	& \quad \text{if } L_{ij} = \lnot v 
	\label{eq:neg} \\
	\cfirst(\iclause_0), \clast(\iclause_\ell) \\
	\rclause(\iclause_{i}, \iclause_{i+1}) 
	& \quad 0 \leq i < \ell
\end{align}
	We next require each variable to be connected via the role
        name $\rval$ to two instances of $\cmin$, a ``positive'' one and a ``negative'' one:
\begin{align}
	\cvar & \sqsubseteq \, \exists \rval.(\cmin \sqcap \cpos)
	\label{eq:variables_have_pos_value} \\
	\cvar & \sqsubseteq \, \exists \rval.(\cmin \sqcap \cneg).
	\label{eq:variables_have_neg_value}
\end{align}
As potential witnesses for the above existential restrictions, we provide 
several instances of $\cmin$, a positive and a negative one for each truth value:
\begin{align}
	\cmin(\itrue_+) 	& & \cmin(\itrue_-) 	& &\cmin(\ifalse_+) 	& & \cmin(\ifalse_-) \\
	\ctrue(\itrue_+) 	& & \ctrue(\itrue_-) 	& &
                                                            \cfalse(\ifalse_+)
                                                                                & & \cfalse(\ifalse_-)  	\label{eq:truth_values_null} \\
	\cpos(\itrue_+) 	& & \cneg(\itrue_-) 	& & \cpos(\ifalse_+) 	& &\cneg(\ifalse_-)
	\label{eq:truth_values}
\end{align}
	If such an instance is used as a witness, then the respective
        variable
        is assigned the corresponding truth value:
\begin{align}
	\exists \rval.\ctrue & \sqsubseteq \, \ctrue
	\label{eq:variables_import_true_value} \\
	\exists \rval.\cfalse & \sqsubseteq \, \cfalse
	\label{eq:variables_import_false_value}
\end{align}
We next introduce a mechanism for ``freezing'' the  valuation of the
$\tx$ variables in the sense that this valuation must be identical in
all models that are smaller w.r.t.\ `$<_\CP$'. This reflects the fact
that the $\tx$ variables are chosen prior to the $\ty$ variables.
Freezing is achieved by the following: 
%
\begin{align}
	\rsavetrue(\isave_{\ctrue}, x)
	& \quad \text{for all } x \in \tx 
	\label{eq:x_vars} \\
\rsavefalse(\isave_{\cfalse}, x) 
	& \quad \text{for all } x \in \tx 
	\label{eq:x_varstwo} \\
	%
	\exists \rsavetrue.\ctrue & \sqsubseteq \, \cmin
	\label{eq:save_true_forward} \\
	\exists \rsavefalse.\cfalse & \sqsubseteq \, \cmin
	\label{eq:save_false_forward}
\end{align}
Note that if a variable $x \in \bar x$ satisfies $\ctrue$, then
$\isave_{\ctrue}$ must satisfy $\cmin$, and likewise for $\cfalse$.
Since $\cmin$ is minimized, any model that is smaller w.r.t.\
`$<_\CP$' must thus have the same valuation for the variables in
$\bar x$.

There is a problem with this approach, though. On the one hand,
$\cmin$ acts as an incentive to use $t_+$, $t_-$, $f_+$, and $f_-$ as
witnesses for~(\ref{eq:variables_have_pos_value})
and~(\ref{eq:variables_have_neg_value}). On the other hand, freezing
implies that if a variable $x$ does so, then it triggers an instance
of $\cmin$ on $\isave_{\ctrue}$ or $\isave_{\cfalse}$.  There can
hence be models in which fresh instances $e^+$ of $\cmin \sqcap \cpos$
and $e^-$ of $\cmin \sqcap \cneg$ are created instead, leading to $x$
satisfying neither $\ctrue$ nor $\cfalse$.  Since $\cmin$ is
minimized, however, it must then hold that $e^+ = e^-$, \emph{i.e.}\
$x$ is connected to an instance of $\cpos \sqcap \cneg$.  Such models
are easily detected due to our use of positive and negative instances
of $\cmin$. We make sure that problematic models satisfy the query
and are thus ruled out as a countermodel:
%
%
\begin{align}
	\raux(\iclause_0, v)
  & \quad \text{for all } v \in \tx \cup \ty
    	\label{eq:misslabel}
	\\
	\exists \raux.( \exists \rval.(\cpos \sqcap \cneg) ) & \sqsubseteq \, \cgoal
	\label{eq:pos_neg_badness}
\end{align}
Now that we can assume that each variable is given a proper truth value, we
encode the evaluation of $\varphi$.  The truth value of each clause is
represented by concept name $\ctrue$ or $\cfalse$ holding on the
corresponding individual $\iclause_{i}$.  Let
$X_i = \{\exists \rpos_i . \cfalse, \exists \rneg_i . \ctrue \}$ and
$X = X_1 \times X_2 \times X_3$, and add:
\begin{align}
	\exists \rpos_i.\ctrue & \sqsubseteq \, \ctrue 
		\quad \text{for } 1 \leq i \leq 3
	\label{eq:clause_true_pos} \\
	\exists \rneg_i.\cfalse & \sqsubseteq \, \ctrue 
		\quad \text{for } 1 \leq i \leq 3 
	\label{eq:clause_true_neg} \\
	C_1 \sqcap C_2 \sqcap C_3 &\sqsubseteq \, \cfalse 
		\quad \text{for all } (C_1,C_2,C_3) \in X \label{eq:clause_false}
\end{align}
We next determine the truth value of $\varphi$ by walking along the
$\rclause$-chain of clauses in \Amc. If the formula evaluates to true,
then we make $\cgoal$ true at $c_0$ and thus satisfy the query.  If it
evaluates to false, we make $\cmin$ true at $\iclause_0$. This
reflects the existential quantification over the variables $\bar y$:
when they exist, we prefer valuations that make $\varphi$ true since
making $\varphi$ false gives us an additional instance of the
minimized concept name. Formally:
\begin{align}
	\exists \rclause.( \clast \sqcap \ctrue ) & \sqsubseteq \, \clast 
	\label{eq:gather_true} \\
	\cfirst \sqcap \clast & \sqsubseteq \, \cgoal 
	\label{eq:true_gives_goal} \\
	\exists \rclause.\cfalse & \sqsubseteq \, \cfalse 
	\label{eq:gather_false} \\
	\cfalse \sqcap \cfirst & \sqsubseteq \, \cmin
	\label{eq:false_gives_min}
\end{align}
It remains to solve another technical issue. As explained above,
we make 
$\cmin$ true not only at the individuals
$\itrue_+, \itrue_-, \ifalse_+, \ifalse_-$, but also at some
individuals $\isave_{\ctrue}, \isave_{\cfalse}$ and possibly at
$\iclause_0$.  However, we do not want these individuals to be used as
witnesses for ~(\ref{eq:variables_have_pos_value})
and~(\ref{eq:variables_have_neg_value}) as then there could be
variables that do not receive a truth value.
We can again detect this and then make $\cmin$ true at a
fresh individual $a$:
%
\begin{align}
	\raux(\iaux, \isave_{\ctrue}), \raux(\iaux, \isave_{\cfalse}) & \quad \text{for all } x \in \tx
	\label{eq:reserved_save} \\
	\raux(\iaux, \iclause_0) 
	\label{eq:reserved_clause} \\
	\exists \raux.\cpos & \sqsubseteq \, \cmin
	\label{eq:pos_gathers} \\
	\exists \raux.\cneg & \sqsubseteq \, \cmin.
	\label{eq:neg_gathers}
\end{align}
We now set $\kb = (\tbox, \abox)$ with $\tbox$ and $\abox$ as defined
above.  It is easily verified that $\tbox$ does not depend on the
instance of $\forall\exists\mn{3SAT}$.  It remains to
prove the following claim:
\\[2mm]
{\bf Claim.}
$\circkb \models \cgoal(\iclause_0)$ iff $\forall \tx \, \exists \ty \, \varphi(\tx, \ty)$.
\\[2mm]
        To prepare for the proof of the claim, we first observe that
        every valuation $V$ for $\bar x \cup \bar y$ gives rise
        to a corresponding model $\Imc_V$ of \Kmc.
        We use domain 
        $
          \Delta^{\Imc_V} = \Ind(\Amc) 
        $
        and set
        $$
        \begin{array}{rcl}
          A^{\Imc_V} &=& \{ a \mid A(a) \in \Amc \} \\[1mm]
          r^{\Imc_V} &=& \{ (a,b) \mid r(a,b) \in \Amc \}
        \end{array}
        $$
        for all concept names $A \in \{ \cvar, \cfirst, \cpos,\cneg \}$
        and all 
        role names $$r \in \{ \rpos_i, \rneg_i, \rclause, \rsavetrue,
        \rsavefalse, \raux \}.$$
        We interpret $\rval$ according to $V$, that is
        $$
        \begin{array}{r@{\;}c@{\;}l}
          \rval^{\Imc_V} &=& 
          \{ (v,\itrue_+), (v,\itrue_-) \mid v \in
                             \bar x \cup \bar y \text{ and } V(v)=1\} \,\cup\\[1mm]
                            && \{ (v,\ifalse_+), (v,\ifalse_-) \mid v \in
                             \bar x \cup \bar y \text{ and } V(v)=0\}. 
        \end{array}
        $$
        The valuation $V$ also assigns a truth value to each clause,
        let $V(c_i)$ denote the value of the $i^{th}$ clause. Put
        $$
        \begin{array}{rcl}
          \ctrue^{\Imc_V} &=&\{ \itrue_+, \itrue_- \} \, \cup \\[1mm]&& \{  v \mid v \in
                             \bar x \cup \bar y \text{ and } V(v)=1 \}
                              \, \cup \\[1mm]
          &&\{ c_i \mid 1 \leq i \leq \ell \text{ and } V(c_i)=1 \} \\[1mm]
          \cfalse^{\Imc_V} &=&\{ \ifalse_+, \ifalse_- \} \, \cup \\[1mm] &&\{ v \mid v \in
                             \bar x \cup \bar y \text{ and } V(v)=0 \}                              \, \cup \\[1mm]
          &&\{ c_i \mid 1 \leq i \leq \ell \text{ and } V(c_i)=1 \}.
        \end{array}
        $$
        Let $q$ be largest such that valuation $V$ makes the last
        $q$ of the $\ell$ clauses in $\varphi$ true. Put
        $$
          \clast^{\Imc_V} = \{ c_\ell \} \cup \{
          c_{\ell-1},\dots,c_{\ell-q+1}\}. 
          $$
          It remains to interpret $\cmin$ and $\cgoal$. As an
          abbreviation, set 
          $$
          \begin{array}{rcl}
             M &=& \{ \itrue_+, \itrue_-, \ifalse_+,
                                 \ifalse_-,  \} \, \cup \\[1mm]
            && \{ \isave_{\ctrue} \mid x \in \bar x \text{ and }
               V(x)=1 \} \\[1mm]
            && \{ \isave_{\cfalse} \mid x \in \bar x \text{ and }
               V(x)=0 \}.
          \end{array}
          $$
          Now if $V$ satisfies all clauses in $\varphi$, then put
          $$
          \begin{array}{rcl}
              \cmin^{\Imc_V} &=& M \\[1mm]
            \cgoal^{\Imc_V} &=& \{ c_0 \}
          \end{array}
          $$
          and otherwise put
          $$
          \begin{array}{rcl}
              \cmin^{\Imc_V} &=& M \cup \{ c_0 \}\\[1mm]
            \cgoal^{\Imc_V} &=& \emptyset.
          \end{array}
          $$
          It is straightforward to verify that $\Imc_V$ is indeed a
          model of \Kmc.  Moreover, $\Imc_V \models q$ if and only if
          $V$ satisfies all clauses in $\varphi$. Now for the actual
          proof of the claim

        \smallskip
        
	``$\Rightarrow$''.  Assume that
        $\forall \bar x \exists \bar y \, \varphi$ is false, and thus
        $\exists \bar x \forall \bar y \, \lnot \varphi$ is true and
        there is a valuation $V_{\bar x}: \bar x \rightarrow \{0,1\}$
        such that $\forall \bar y \, \lnot \varphi'$ holds where
        $\varphi'$ is obtained from $\varphi$ by replacing every
        variable $x \in \bar x$ with the truth constant
        $V_{\bar x}(x)$. Consider any extension $V$ of $V_{\bar x}$ to
        the variables in $\bar y$. Then the interpretation
        $\Imc_V$ is a model of \Kmc with
        $\Imc_V \not\models q$, to show that
        \mbox{$\circkb \not\models q$} it remains to prove
        that there is no model $\Jmc <_\CP \I$ of \Kmc.

        Assume to the contrary that there is such a $\Jmc$. Then
          $$
          \begin{array}{rcl}
             \cmin^\Jmc \subsetneq \cmin^{\Imc_v} &=& \{ \itrue_+, \itrue_-, \ifalse_+,
                                 \ifalse_-, c_0  \} \, \cup \\[1mm]
            && \{ \isave_{\ctrue} \mid x \in \bar x \text{ and }
               V(x)=1 \} \\[1mm]
            && \{ \isave_{\cfalse} \mid x \in \bar x \text{ and }
               V(x)=0 \}.
          \end{array}
          $$
          \Jmc being a model of \Amc implies
          $\{ \itrue_+, \itrue_-, \ifalse_+, \ifalse_-\}
          \subseteq \cmin^\Jmc$.
          
          Since $\J$ satisfies~(\ref{eq:vars}),
          (\ref{eq:variables_have_pos_value}) and
          (\ref{eq:variables_have_neg_value}), we can choose for each
          $v \in \tx \cup \ty$ an element
          $v^+ \in \cpos^\J \cap \cmin^\J$ such that
          $(v, v^+) \in \rval^\J$. Since \Jmc
          satisfies~(\ref{eq:reserved_save}) to~(\ref{eq:neg_gathers})
          and $a \notin \cmin^\Jmc$, we must have
          $v^+ \in \{ \itrue_+, \itrue_-, \ifalse_+, \ifalse_-\}$.
                    From (\ref{eq:variables_import_true_value}) 
          and~(\ref{eq:variables_import_false_value}), it follows that 
          $v \in \ctrue^\J \cup \cfalse^\J$. 
          Let $V_\Jmc$ be the valuation defined by setting, for all
          $v \in \bar x \cup \bar y$:
          $$
          V_\Jmc(v) = \left \{
            \begin{array}{ll}
            1 & \text{ if } v\in \ctrue^\Jmc \\[1mm]
            0 &\text{ otherwise.}
            \end{array}
            \right.
          $$
          From (\ref{eq:x_vars})
          to~(\ref{eq:save_false_forward}), we
          further obtain that
          $$
          \begin{array}{rcl}
             \cmin^\Jmc &\supseteq& \{ \isave_{\ctrue} \mid x \in \bar x \text{ and }
               V_\Jmc(x)=1 \} \\[1mm]
            && \{ \isave_{\cfalse} \mid x \in \bar x \text{ and }
               V_\Jmc(x)=0 \}.
          \end{array}
          $$
          which implies that
          $\cmin^{\Imc_V} \setminus \cmin^\Jmc = \{ c_0 \}$ and,
          through $\Jmc <_\CP \I$, that $V$ and $V_\Jmc$ agree on
          $\bar x$. From $\forall \bar y \, \lnot \varphi'$, it
          follows that $V_\Jmc \not\models \varphi$ and thus the
          soundness of the evaluation of $\varphi$ implemented by~(\ref{eq:clause_true_pos})
          to~(\ref{eq:false_gives_min}) yields
          $\iclause_0 \in \cmin^\J$, a contradiction.

\smallskip
        
``$\Leftarrow$''.  Assume that
$\forall \tx \, \exists \ty \, \varphi(\tx, \ty)$ is true and let $\I$ be
a model of $\circkb$. We have to show that $\Imc \models q$.

Since $\I$ satisfies~(\ref{eq:vars}),
(\ref{eq:variables_have_pos_value}) and
(\ref{eq:variables_have_neg_value}), we can choose for every
$v \in \tx \cup \ty$ two elements $v^+ \in \cpos^\I \cap \cmin^\I$ and
$v^- \in \cneg^\I \cap \cmin^\I$ such that
$(v, v^+), (v, v^-) \in \rval^\I$. If $v^+ = v^-$ for some $v$, then
by~(\ref{eq:misslabel}) and~(\ref{eq:pos_neg_badness}) we have
$c_0 \in \cgoal^\Imc$, thus $\Imc \models q$ and we are done.
For the remaining proof, we may thus assume $v^+ \neq v^-$
for all $v \in \tx \cup \ty$.

We next argue that for all $v \in \bar x \cup \bar y$,
\begin{equation*}
\{ v^+, v^- \} \cap \{ \itrue_+, \itrue_-, \ifalse_+, \ifalse_-\}
\neq \emptyset 
\tag{$\ast$}
\end{equation*}
Assume to the contrary that ($*$) is violated by some
\mbox{$v_0 \in \bar x \cup \bar y$}, that is,
$v_0^+, v_0^- \notin \{ \itrue_+, \itrue_-, \ifalse_+, \ifalse_-\}$.
Then
we can find a model $\Jmc <_\CP \Imc$ of \Kmc, contradicting the
minimality of \Imc.  We distinguish two cases:
\begin{enumerate}
  \item $a \in \cmin^\Imc$. 

    We construct \Jmc by choosing $\Delta^\Jmc = \Ind(\Amc)$ and
    interpreting the concept names $\cvar$ and $\cfirst$ as well as
    the role names
    $\rpos_i, \rneg_i, \rclause, \rsavetrue, \rsavefalse, \raux$ as in
    the initially defined interpretations $\Imc_V$---note that this is
    independent of $V$. We further set
        $$
        \begin{array}{rcl}
          \rval^{\Jmc} &=& 
          \{ (v,a) \mid v \in
                           \bar x \cup \bar y \} \\[1mm]
          \cpos^\Jmc &=& \{ \itrue_+,\ifalse_+,a \} \\[1mm]
          \cneg^\Jmc &=& \{ \itrue_-,\ifalse_-,a \} \\[1mm]
          \ctrue^{\Jmc} &=&\{ \itrue_+, \itrue_- \} \\[1mm]
          \cfalse^{\Jmc} &=&\{ \ifalse_+, \ifalse_- \} \\[1mm]
          \clast^{\Jmc} &=& \{ c_\ell \} \\[1mm] 
           \cmin^\Jmc  &=& \{ \itrue_+, \itrue_-, \ifalse_+,
                                 \ifalse_-,a  \} \\[1mm]
            \cgoal^{\Jmc} &=& \emptyset.
          \end{array}
          $$
          It can be verified that \Jmc is a model of \Kmc. Moreover,
          at least one of the (distinct!) $v_0^+, v_0^-$ is not in $\cmin^\Jmc$ and
          thus $\cmin^\Jmc \subsetneq \cmin^\Imc$ and 
          $\Jmc <_\CP \Imc$.

\item $a \notin \cmin^\Imc$. 

  We construct \Jmc exactly as in Case~1, but with $a$ replaced by
  $v^+_0$
  in $\cpos^\Jmc$, $\cneg^\Jmc$, and $\cmin^\Jmc$.
  

  
\end{enumerate}
We have thus established $(*)$.

\smallskip It follows from $(*)$ together with
(\ref{eq:truth_values_null}), (\ref{eq:variables_import_true_value}),
and~(\ref{eq:variables_import_false_value}) that
$v \in \ctrue^\Jmc \cup \cfalse^\Jmc$ for all
$v \in \bar x \cup \bar y$ and thus we find a valuation $V$ such that
$V(v)=1$ implies $v \in \ctrue^\Jmc$ and $V(v)=0$ implies
$v \in \cfalse^\Jmc$ for all $v$. We show below that
$V \models \varphi$.  Then the soundness of the evaluation of
$\varphi$ implemented by~(\ref{eq:clause_true_pos})
to~(\ref{eq:false_gives_min}) yields $c_0 \in \cgoal^\Imc$ and thus
$\Imc \models q$ as desired.

Assume to the contrary of what remains to be shown that
$V \not\models \varphi$. Soundness of evaluation ensures
$\iclause_0 \in \cmin^\I$.  Since
$\forall \tx \, \exists \ty \, \varphi(\tx, \ty)$ is true, we find a
valuation $V'$ that agrees with $V$ on the variables in $\bar x$ and
that satisfies $\varphi$. Now consider the model $\Imc_{V'}$ of
\Kmc. It satisfies 
$$
          \begin{array}{rcl}
             \cmin^{\Imc_{V'}} &=& \{ \itrue_+, \itrue_-, \ifalse_+,
                                 \ifalse_-,  \} \, \cup \\[1mm]
            && \{ \isave_{\ctrue} \mid x \in \bar x \text{ and }
               V(x)=1 \} \\[1mm]
            && \{ \isave_{\cfalse} \mid x \in \bar x \text{ and }
               V(x)=0 \}.
          \end{array}
$$
and it is easy to verify that $\cmin^{\Imc_{V'}} \subseteq
\cmin^\Imc$. But then $\Imc_{V'} < \Imc_V$ since $\iclause_0 \in \cmin^\I$.
\end{proof}

\section{Proofs for Section~\ref{subsection-dllite-combined}}

\thmcombinedlowerdlliter*
\begin{proof}
	%
  We reduce from (Boolean) UCQ evaluation on \dlliter KBs with
  closed concept names, which is defined in the expected way
   and known to be \TwoExpTime-hard
   \cite{Ngo2016}.
   Let $\Kmc_\Sigma$ be a \dlliter KB with closed concept names,
   $\Kmc = (\Tmc,\Amc)$, and let $q$ be a Boolean UCQ.  We construct a
   circumscribed \dlliter KB $\mn{Circ}_\CP(\Kmc')$, with
   $\Kmc'=(\Tmc',\Amc')$, and a CQ $q'$ such that
   $\Kmc_\Sigma \models q$ iff $\mn{Circ}_\CP(\Kmc') \models q'$.

   As
   in the proof of \cref{thm:combined-lower-el}, we minimize only a
   single concept name $M$ and include \Amc in $\Amc'$. To construct
   $\Tmc'$, we start from \Tmc and extend with additional concept
   inclusions. For $q'$, we start from $q$ and extend with
   additional disjuncts. We will actually use individuals
   $a \in \mn{Ind}(\Amc)$ as constants in~$q'$. These can be
   eliminated by introducing a fresh concept name $A_a$, extending
   \Amc with $A_a(a)$, and replacing in each CQ in $q'$ the constant
   $a$ with a fresh variable $x_a$ and adding the atom $A_a(x_a)$.

   As in the proof of \cref{thm:combined-lower-el}, we include in
   $\Tmc'$ the CI
   \begin{align}
		A \sqsubseteq \M & \quad \text{for all } A \in \Sigma
   \end{align}
   and then have to rule out non-asserted instances of closed concept
   names. To rule out non-asserted instances inside of $\Ind(\Amc)$,
   we add to $q'$ the disjunct $A(a)$ for all $A \in \Sigma$ and
   $a \in \Ind(\Amc)$ with $A(a) \notin \Amc$.

   To rule out instances of closed concept names outside of
   $\Ind(\Amc)$, add
	\begin{align}
		\M(\tc)
	\end{align}
	where $\tc$ is a fresh individual. As in the proof of
        \cref{thm:combined-lower-el}, this guarantees that models with
        instances of closed concept names outside of $\Ind(\Amc)$ are
        not minimal. To ensure that no closed concept names are made
        true at \tc, we extend $q'$ with the disjunct $r(x,t)$ for every
        role name $r$ used in \Tmc.
                \\[2mm]
        {\bf Claim.} 
                $\Kmc_\Sigma \models q$ iff $\Circ(\Kmc') \models q'$.
\\[2mm]
        The proof is rather straightforward; we omit the details.
%
%
\end{proof}

We now work towards a proof of
Theorem~\ref{thm-combined-upper-dllitebool}, first establishing the
following.
\lemmaboolcountermodel*

\newcommand{\pathsbool}{\Pmc_{\mathsf{bool}}}
\newcommand{\boolunrav}{\I'_\mathsf{bool}}
Refining the unraveling begins by noticing the set $\Omega$ now only contains elements with shape $r\top$ as $\dllitebool$ TBoxes are considered.
We drop the $\top$ concept for simplicity.
Recall that we assume chosen a representative $e_t \in \domain{\I}$ for each non-core type $t \in \noncoretypesof{\I}$.
We define the set $\pathsbool$ of {$\mathsf{bool}$-paths through \Imc} along with
a mapping $h$ that assigns to each $p \in \pathsbool$ an element of~$\Delta^\Imc$:
\begin{itemize}
	
	\item each element $d$ of the set
	$$\basedomainof{\I} := \indsof{\abox} \cup \coredomainof{\I} \cup \{
	e_t \mid t \in \noncoretypesof{\I} \}.$$
	is a path in $\pathsbool$ and
	$h(d)=d$;
	
	\item if $p \in \pathsbool$ with $h(p)=d$ and $r \in \Omega$ such that:
	\begin{enumerate}
		\item[(a)] $f(d,r)$ is defined and not from $\Delta^\Imc_{\mn{core}}$ and
		\item[(b)] $p$ does not end by $r^-$, which we denote $\mn{tail}(p) \neq r^-$.
	\end{enumerate} 
	then
	$p'=p r$ is a path in $\pathsbool$ and $h(p')=f(d,r)$.
	
\end{itemize}

For every role $r$, define
$$
\begin{array}{@{}r@{\;}c@{\;}l}
	R_r &=& \{ (a,b) \mid a,b \in \Ind(\Amc), \Kmc \models r(a,b)
	\} \, \cup \\[1mm]
	&& \{ (d,e) \mid d,e \in  \coredomainof{\I}, (d,e) \in r^\Imc \} \,
	\cup \\[1mm]
	&& \{ (p,p') \mid p'=pr 
	\} \, \cup\\[1mm]
	&& \{ (p,e) \mid 
	e=f(h(p),r) \in \Delta^\Imc_{\mn{core}}, \mn{tail}(p) \neq r^- 
	\}. 
\end{array}
$$
Now the $\mathsf{bool}$-{unraveling} of \Imc is defined by setting
$$
\begin{array}{r@{\;}c@{\;}l}
	\Delta^{\Imc'} &=& \pathsbool \\[1mm]
	A^{\Imc'} &=& \{ p \in \pathsbool \mid h(p) \in A^\Imc \} \\[1mm]
	r^{\Imc'} &=& \displaystyle R_r \cup \{ (e,d) \mid (d,e) \in R_{r^-} \}
\end{array}
$$
for all concept names $A$ and role names $r$. It is easy to verify
that $h$ is a homomorphism from $\Imc'$ to \Imc.
It is a technicality to verify that $\boolunrav$ satisfies the same properties as the previous unraveling $\interlacingof{\I}$, that is: if $\I$ is a countermodel for $\query$ over $\circkb$, then so is $\boolunrav$. The key additional property satisfied by $\boolunrav$ is the following:

\begin{lemma}
	\label{lemma-unique-succ}
	For any $\rolestyle{r} \in \rnames$ and $d_1 \in \domain{\boolunrav} \setminus \basedomainof{\I}$, there is at most one element $d_2 \in \domain{\boolunrav}$ such that $(d_1, d_2) \in \rolestyle{R}^{\boolunrav}$.
\end{lemma}
\begin{proof}
	Unfolding the definition of $\rolestyle{R}^{\boolunrav}$ and recalling $d_1 \notin \basedomainof{\I}$, we obtain that if $d_1 \in \domain{\boolunrav} \setminus \basedomainof{\I}$ and $(d_1, d_2) \in \rolestyle{R}^{\boolunrav}$, then either:
	\begin{itemize}
		\item $(d_1, d_2) \in R_r$ and either $d_2 = d_1 r$ (Case~1) or $d_2 = f(h(d_1), r) \in \basedomainof{\I}$ and $\mn{tail}(d_1) \neq r^-$ (Case~2).
		
		\item $(d_2, d_1) \in R_{r^-}$ and we have $d_1 = d_2 r^-$ (Case~3).
	\end{itemize}
	To verify uniqueness, let us assume there exist $d_2$ and $d_2'$ s.t.\ $(d_1, d_2'), (d_1, d_2') \in \rolestyle{R}^{\boolunrav}$.
	We want to prove $d_2 = d_2'$.
	From the small analysis above, there are $3 \times 3$ cases to consider:
	\begin{enumerate}
		\item if $d_2 = d_1 r$, then:
		\begin{enumerate}
			\item[1.] if $d_2' = d_1 r$, we obtain directly $d_2' = d_2$.
			\item[2.] if $d_2' = f(h(d_1), r) \in \basedomainof{\I}$ and $\mn{tail}(d_1) \neq r^-$, then it contradicts $d_1 r \in \pathsbool$ (Condition~(a)).
			\item[3.] if $d_1 = d_2' r^-$, then $d_2 = d_2' r^- r$, which contradicts $d_2 \in \pathsbool$ (Condition~(b)).
		\end{enumerate}
		\item if $d_2 = f(h(d_1), r) \in \basedomainof{\I}$ and $\mn{tail}(d_1) \neq r^-$, then:
		\begin{enumerate}
			\item[1.] if $d_2' = d_1 r$, then same argument as 1.2. 
			\item[2.] if $d_2' = f(h(d_1), r) \in \basedomainof{\I}$ and $\mn{tail}(d_1) \neq r^-$, then $d_2' = d_2$ from $h$ and $f$ being functions.
			\item[3.] if $d_1 = d_2' r^-$, then it contradicts $\mn{tail}(d_1) \neq r^-$.
		\end{enumerate}
			\item if $d_1 = d_2 r^-$, then:
		\begin{enumerate}
			\item[1.] if $d_2' = d_1 r$, then same argument as 1.3. 
			\item[2.] if $d_2' = f(h(d_1), r) \in \basedomainof{\I}$ and $\mn{tail}(d_1) \neq r^-$, then same argument as 2.3.
			\item[3.] if $d_1 = d_2' r^-$, we obtain directly $d_2' = d_2$.
		\end{enumerate}
	\end{enumerate}
\end{proof}

Repeated applications of Lemma~\ref{lemma-unique-succ} ensure that each partial homomorphism of a CQ $p \in q$ in the non-$\basedomainof{\I}$ part of the bool-unraveling can be completed uniquely in a maximal such partial homomorphism (see further Lemma~\ref{lemma-poly-neighborhoods}).
This motivates a refined notion of the neighborhood of an element $d$, restricting the (usual) neighborhood to those elements $e$ that can be reached by a homomorphism of some connected subquery involving both $d$ and $e$.
Formally, for $n \geq 0$ and
$\Delta \subseteq \Delta^\Imc$, we use $\coreneighof{n}{d}{\I}{\domain{}}$
to denote the \emph{$n$-bool-neighborhood of $d$ in \Imc up to}
$\Delta$, that is, the set of all elements $e \in \neighof{n}{d}{\I}{\domain{}}$
such that there exists a connected subquery $p' \subseteq p$ for some $p \in q$ and a homomorphism $\pi : p' \rightarrow \neighof{n}{d}{\I}{\domain{}}$ s.t.\ $d, e \in \pi(\variablesof{p'})$.



The following polynomial bound on their size in the bool-unraveling is the central property allowing bool-neighborhoods to improve our construction.


\begin{lemma}
	\label{lemma-poly-neighborhoods}
	Let $\I$ be a model of $\kb$ and $\boolunrav$ its bool-unraveling.
	Consider $d \in \domain{\boolunrav} \setminus \basedomainof{\I}$, then $\sizeof{\coreneighof{n}{d}{\boolunrav}{\basedomainof{\I}}} \leq \sizeof{q}^2 ( \sizeof{\tbox} + 1)$.
\end{lemma}
\begin{proof}
	\newcommand{\termsof}[1]{\variablesof{#1}}
	Let $c \in \domain{\boolunrav} \setminus \domain{\basecandidate}$
	in a first step, we prove the number of elements in $\coreneighof{n}{c}{\boolunrav}{\domain{\basecandidate}} \setminus \domain{\basecandidate}$ is at most $\sizeof{q}^2$.
	In a second step, we notice each element $e \in \coreneighof{n}{c}{\boolunrav}{\domain{\basecandidate}} \cap \domain{\basecandidate}$ must be connected to an element $d \in \coreneighof{n}{c}{\boolunrav}{\domain{\basecandidate}} \setminus \domain{\basecandidate}$ by construction of (usual) neighborhoods.
	However, from Lemma~\ref{lemma-unique-succ}, each such element $d$ is connected to at most $\sizeof{\tbox}$ elements.
	From the first step it follows there are at most $\sizeof{q}^2 \cdot \sizeof{\tbox}$ elements in $e \in \coreneighof{n}{c}{\boolunrav}{\domain{\basecandidate}} \cap \domain{\basecandidate}$, hence the claimed bound.
	
	It remains to prove the first step.
	We start by proving that if the connected subquery $p' \subseteq p$ for some $p \in q$ and the variable $v_0$ that shall map on $c$ are fixed, then all homomorphisms $p' \rightarrow { (\coreneighof{n}{c}{\boolunrav}{\basedomainof{\I}} \,\setminus\, \basedomainof{\I})}$ mapping $v_0$ on $c$ are equal.
	Consider two such homomorphisms $\match_1$ and $\match_2$.
	We proceed by induction on the variables $v$ of $p'$ being connected.
	For $v = v_0$, we have $\match_1(v_0) = \match_2(v_0)$ by definition.
	For a further term $v$, we use the induction hypothesis, that is, the existence of an atom $\rolestyle{R}(v', v) \in p'$ (or the other way around) such that $\match_1(v') = \match_2(v')$.
	Recall $\match_1$ and $\match_2$ are homomorphisms of $p'$ in $\coreneighof{n}{c}{\boolunrav}{\basedomainof{\I}} \,\setminus\, \basedomainof{\I}$, in particular $\match_1(v'), \match_2(v') \notin \basedomainof{\I}$, hence we can apply Lemma~\ref{lemma-unique-succ}, yielding $\match_1(v) = \match_2(v)$.
	
	This proves that, for a fixed $v_0 \in \termsof{p}$, each connected subquery $p' \subseteq p$ admitting a homomorphism in $\coreneighof{n}{c}{\boolunrav}{\basedomainof{\I}} \,\setminus\, \basedomainof{\I}$ defines at most $\sizeof{p}$ new neighbors, but also that if $p' \subseteq p'' \subseteq p$ are two such subqueries, then the neighbors defined by $p'$ are subsumed by those defined by $p''$ (the restriction to the variables of $p'$ of {the} unique homomorphism of $p''$ mapping $v_0$ on $c$ must coincide with {the} unique homomorphism of $p'$ mapping $v_0$ on $c$).
	Still, for a fixed $v_0$, consider now two connected subqueries $p_1, p_2 \subseteq p$, each admitting a (unique) homomorphism $\match_1$ resp. $\match_2$, to $\coreneighof{n}{c}{\boolunrav}{\basedomainof{\I}} \,\setminus\, \basedomainof{\I}$ mapping $v_0$ to $c$, and each maximal, w.r.t. the inclusion, for this property.
	By the previous property, we know $\match_1$ and $\match_2$ coincide on $\termsof{p_1} \cap \termsof{p_2}$.
	Therefore, $p_1 \cup p_2$ admits a homomorphism to $\coreneighof{n}{c}{\boolunrav}{\basedomainof{\I}} \,\setminus\, \basedomainof{\I}$ mapping $v_0$ to $c$, being $\match_1 \cup \match_2$.
	However, since $p_1$ and $p_2$ are assumed maximal for this property, we must have $p_1 = p_2$.
	
	Therefore, for a fixed $v_0 \in \termsof{p}$, there is a unique maximal connected subquery $p_{\textsf{max}} \subseteq p$ admitting a homomorphism in $\coreneighof{n}{c}{\boolunrav}{\basedomainof{\I}} \,\setminus\, \basedomainof{\I}$ and mapping $v_0$ to $c$.
	As previously seen, the neighbors defined by $p_\textsf{max}$ subsume those defined by other subqueries of $p$, and since the homomorphism for $p_{\textsf{max}}$ is unique, it defines at most $\sizeof{p}$ neighbors.
	This holds for each possible choices of term $v_0$, hence a total number of possible neighbors issued by $p$ bounded by $\sizeof{p}^2$.
	Iterating over each $p \in q$, we hence obtain the claimed bound of at most $\sizeof{q}^2$ elements in $\coreneighof{n}{c}{\boolunrav}{\basedomainof{\I}} \,\setminus\, \basedomainof{\I}$.
\end{proof}

Using bool-neighborhoods in the quotient construction presented in
Section~\ref{subsection-alchi-data}, the number of equivalence classes
drops, which yieldings Lemma~\ref{lemma-bool-countermodel}.

\begin{restatable}{lemma}{lemmaportionaremodels}
	\label{lemma-portion-are-models}
	For all $\Pmc \subseteq \domain{\I}$, $\I_\Pmc$ is a model of $\kb$.
\end{restatable}
%
\begin{proof}
	Let $\Pmc\subseteq \domain{\I}$, we have:
	\begin{itemize}
		\item $\I_\Pmc$ models $\abox$ as it inherits interpretations of concepts and roles on $\indsof{\abox}$ from $\I$, which is a model of $\abox$.
		\item Axioms in $\tbox$ with shape $\axtop$, $\axand$, $\axnotright$ or $\axnotleft$ are satisfied since $\I_\Pmc$ inherits interpretations of concept names from $\I$, which is a model of $\tbox$.
		\item Axioms in $\tbox$ with shape $A \incl \exists r$ are witnessed with an $r$-edge pointing to $w_r$.
		\item Axioms in $\tbox$ with shape $\exists r \incl A$ are satisfied since every element in $\I_\Pmc$ having some $r$-edge already has one in $\I$ (and interpretations of concept names are preserved).
		\item Role inclusions $r \incl s$ are satisfied on $\indsof{\abox} \times \indsof{\abox}$ from $\I$ being a model of $\tbox$ in the first place, otherwise directly from the definition of $s^{\Imc_\Pmc}$.
	\end{itemize}
	This proves $\I_\Pmc \models \kb$ as desired.
\end{proof}

\lemmacriteriacombined*
To piece together the $\J_{e}$ and reconstruct an eventual $\J$, we first prove the following property:

\begin{lemma}
	\label{lemma-portions-union}
	Let $\Pmc_1, \Pmc_2 \subseteq \domain{\I}$ and $\J_1$, $\J_2$ two models of $\kb$ s.t.\ $\J_1 <_\CP \I_{\Pmc_1}$ and $\J_2 <_\CP \I_{\Pmc_2}$.
	If $\J_1 |_{\domain{\J_2}} = \J_2 |_{\domain{\J_1}}$, then $\J_1 \cup \J_2 <_\CP \I_{\Pmc_1 \cup \Pmc_2}$.
\end{lemma}

\begin{proof}
	Assume $\J_1 |_{\domain{\J_2}} = \J_2 |_{\domain{\J_1}}$.
	It is easily verified that $\J_1 \cup \J_2$ is also a model of $\kb$ since $\J_1$ and $\J_2$ agree on their shared domain.
	We now check that all four conditions from the definition of $<_\CP$ are satisfied:
	\begin{enumerate}
		\item 
		From $\J_1 <_\CP \I_{\Pmc_1}$ we have $\domain{\J_1} = \domain{\I_{\Pmc_1}}$ and similarly for $\J_2$ we have $\domain{\J_2} = \domain{\I_{\Pmc_2}}$.
		It follows from its definition that $\domain{\I_{\Pmc_1 \cup \Pmc_2}} = \domain{\I_{\Pmc_1}} \cup \domain{\I_{\Pmc_2}}$.
		Therefore we have as desired $\domain{\J_1 \cup \J_2} = \domain{\I_{\Pmc_1 \cup \Pmc_2}}$.
		\item 
		Let $A \in \Fsf$.
		From $\J_1 <_\CP \I_{\Pmc_1}$ we have $A^{\J_1} = A^{\I_{\Pmc_1}}$ and similarly for $\J_2$ we have $A^{\J_2} = A^{\I_{\Pmc_2}}$.
		It follows from its definition that $A^{\I_{\Pmc_1 \cup \Pmc_2}} = A^{\I_{\Pmc_1}} \cup A^{\I_{\Pmc_2}}$.
		Therefore we have as desired $A^{\J_1 \cup \J_2} = A^{\I_{\Pmc_1 \cup \Pmc_2}}$.
		\item 
		Let $A \in \Msf$ such that $A^{\J_1 \cup \J_2} \not\subseteq A^{\I_{\Pmc_1 \cup \Pmc_2}}$.
		We may assume w.l.o.g.\ that $A$ is minimal w.r.t.\ $\prec$ for this property.
		We thus have $e \in A^{\J_1 \cup \J_2} \setminus A^{\I_{\Pmc_1 \cup \Pmc_2}}$.
		We either have $e \in \domain{\J_1}$ or $e \in \domain{\J_2}$.
		We treat the case $e \in \domain{\J_1}$, the arguments for $e \in \domain{\J_2}$ are similar.
		From $\J_1 <_\CP \I_{\Pmc_1}$, there exists a concept $B_1 \prec A$ s.t.\ $B_1^{\J_1} \subsetneq B_1^{\I_{\Pmc_1}}$.
		We chose such a $B_1$ that is minimal for this property, that is s.t.\ for all $B \prec B_1$, we have $B^{\J_1} = B^{\I_{\Pmc_1}}$ (recall $A$ is minimal for its property so that such a $B$ must verify $B^{\J_1 \cup \J_2} \subseteq B^{\I_{\Pmc_1 \cup \Pmc_2}}$ hence $B^{\J_1} \subseteq B^{\I_{\Pmc_1}}$).
		If $B_1^{\J_2} \subseteq B_1^{\I_{\Pmc_2}}$, then we are done as we obtain $B_1^{\J_1 \cup \J_2} \subsetneq B_1^{\I_{\Pmc_1 \cup \Pmc_2}}$.
		Otherwise, we have $B_1^{\J_2} \not\subseteq B_1^{\I_{\Pmc_2}}$ and from $\J_2 <_\CP \I_{\Pmc_2}$, there exists a concept $B_2 \prec B_1$ s.t.\ $B_2^{\J_2} \subsetneq B_2^{\I_{\Pmc_2}}$.
		Since $B_1$ has been chosen minimal, we have in particular $B_2^{\J_1} = B_2^{\I_{\Pmc_1}}$, which yields $B_2^{\J_1 \cup \J_2} \subsetneq B_2^{\I_{\Pmc_1 \cup \Pmc_2}}$.
		\item
		From $\J_1 <_\CP \I_{\Pmc_1}$, there exists a concept $A_1 \in \Msf$ s.t.\ $A_1^{\J_1} \subsetneq A_1^{\I_{\Pmc_1}}$ and for all $B \prec A_1$, we have $B^{\J_1} = B^{\I_{\Pmc_1}}$.
		If $A_1^{\J_2} \subsetneq A_2^{\Pmc_2}$ and for all $B \prec A_1$, $B^{\J_2} = B^{\Pmc_2}$, then we obtain directly $A_1^{\J_1 \cup \J_2} \subsetneq A_1^{\I_{\Pmc_1 \cup \Pmc_2}}$ and for all $B \prec A_1$, $B^{\J_1 \cup \J_2} = B^{\I_{\Pmc_1 \cup \Pmc_2}}$ and we are done.
		Otherwise:
		\begin{itemize}
			\item If $A_1^{\J_2} \not\subseteq A_1^{\I_{\Pmc_2}}$, then from $\J_2 <_\CP \I_{\Pmc_2}$, there exists a concept $A_2 \prec A_1$ s.t.\ $A_2^{\J_2} \subsetneq A_2^{\I_{\Pmc_2}}$.
			Consider a minimal such $A_2$, that is s.t.\ for all $A \prec A_2$, we have $A^{\J_2} = A^{\I_{\Pmc_2}}$ or $A^{\J_2} \not\subseteq A^{\I_{\Pmc_2}}$. 
			Notice the second option cannot happen otherwise from $\J_2 <_\CP \I_{\Pmc_2}$ we could obtain $A_3 \prec A_2$ s.t.\ $A_2^{\J_2} \subsetneq A_2^{\I_{\Pmc_2}}$, contradicting the minimality of $A_2$.
			Therefore $A_2$ being minimal yields that for all $A \prec A_2$, we have $A^{\J_2} = A^{\I_{\Pmc_2}}$.
			Now, from the minimality of $A_1$, we extend this to $\J_1 \cup \J_2$ and obtain $A_2^{\J_1 \cup \J_2} \subsetneq A_2^{\I_{\Pmc_1 \cup \Pmc_2}}$ and for all $B \prec A_2$, $B^{\J_1 \cup \J_2} = B^{\I_{\Pmc_1 \cup \Pmc_2}}$.
			\item If there exists $A_2 \prec A_1$ s.t.\ $A_2^{\J_2} \neq A_2^{\I_{\Pmc_2}}$.
			Consider a minimal such $A_2$, that is s.t.\ for all $A \prec A_2$, we have $A^{\J_2} = A^{\I_{\Pmc_2}}$.
			Notice we must have $A_2^{\J_2} \subseteq A_2^{\I_{\Pmc_2}}$, otherwise from $\J_2 <_\CP \I_{\Pmc_2}$ we could obtain $A_3 \prec A_2$ s.t.\ $A_2^{\J_2} \subsetneq A_2^{\I_{\Pmc_2}}$, contradicting that for all $A \prec A_2$, we have $A^{\J_2} = A^{\I_{\Pmc_2}}$.
			Therefore, since $A_2^{\J_2} \neq A_2^{\I_{\Pmc_2}}$, it must be that $A_2^{\J_2} \subsetneq A_2^{\I_{\Pmc_2}}$.
			Now, from the minimality of $A_1$, we extend this to $\J_1 \cup \J_2$ and obtain $A_2^{\J_1 \cup \J_2} \subsetneq A_2^{\I_{\Pmc_1 \cup \Pmc_2}}$ and for all $B \prec A_2$, $B^{\J_1 \cup \J_2} = B^{\I_{\Pmc_1 \cup \Pmc_2}}$.
		\end{itemize}
	\end{enumerate}
	Overall, this proves $\J_1 \cup \J_2 <_\CP \I_{\Pmc_1 \cup \Pmc_2}$ as desired.
\end{proof}

We are now ready to properly prove Lemma~\ref{lemma-criteria-combined}.

\begin{proof}[Proof of Lemma~\ref{lemma-criteria-combined}]
	``$1 \Rightarrow 2$''. 
	Assume there exists $\J \models \kb$ s.t.\ $\J <_\CP \I$.
	Based on $\J$, we build a subset $\Pmc \subseteq \domain{\J}$ containing:
	\begin{itemize}
		\item for each role $r$ s.t.\ $r^\J \neq \emptyset$, an element $w_r' \in (\exists r^-)^\J$;
		\item for each $A \in \Msf$ s.t.\ $A^\J \not\subseteq A^\I$, an element $e_A \in B^\J \setminus B^\I$ for some $B \prec A$ (Condition~3 in the definition of $\J <_\CP \I$ ensures existence of such $B$ and $e_A$);
		\item an element $e_\Msf \in A^\J \setminus A^\I$ for some $A \in \Msf$ s.t.\ $A^\J \subsetneq A^\I$ and for all $B \prec A$, $B^\J = B^\I$
		(Condition~4 in the definition of $\J <_\CP \I$ ensures existence of such $A$ and $e_\Msf$).
	\end{itemize}
	Notice $\Pmc$ has size at most $2\sizeof{\tbox} +1$.
	We now build $\J_e$ as:
	\begin{align*}
		\domain{\J_e} = \; & \domain{\I_{\Pmc \cup \{ e\}}}
		\\
		\cstyle{A}^{\J_e} = \; & \cstyle{A}^{\J} \cap \domain{\J_e}
		\\
		\rstyle{r}^{\J_e} = \; & \rstyle{r}^{\J} \cap (\indsof{\abox} \times \indsof{\abox}) 
		\\
		& \cup \{ (e, w_s') \mid e \in (\exists s)^\J, \tbox \models s \incl r \}
		\\
		& \cup \{ (w_s', e) \mid e \in (\exists s)^\J, \tbox \models s \incl r^- \}
	\end{align*}
	It is easily verified that $\J_e$ is a model of $\kb$, and we now prove $\J_e <_\CP \I_{\Pmc \cup \{ e\}}$.
	We check that all four conditions from the definition of $<_\CP$ are satisfied:
	\begin{enumerate}
		\item 
		By definition, we have $\domain{\J_e} = \domain{\I_{\Pmc \cup \{ e \}}}$.
		\item 
		Let $A \in \Fsf$. 
		Definitions of $\I_{\Pmc \cup \{ e \}}$ and $\J_e$ ensure $A^\I \cap {\domain{\I_{\Pmc \cup \{ e \}}}} = A^{\I_{\Pmc \cup \{ e \}}}$ and $A^\J \cap {\domain{\I_{\Pmc \cup \{ e \}}}} = A^{\J_e}$.
		From $\J <_\CP \I$, we get $A^{\I} = A^{\J}$, which thus yields $A^{\I_{\Pmc \cup \{ e \}}} = A^{\J_e}$.
		\item 
		Let $A \in \Msf$ such that $A^{\J_e} \not\subseteq A^{\I_{\Pmc \cup \{ e \}}}$.
		Therefore $A^{\J} \not\subseteq A^{\I}$, and recall we kept in $\Pmc$ an element $e_A \in B^\J \setminus B^\I$ for some $B \prec A$ to belong to $\Pmc$.
		Joint with $B^{\J_e} \subseteq B^{\I_{\Pmc \cup \{ e \}}}$ being trivial, $e_A$ additionally witnesses that $B^{\J_e} \subsetneq B^{\I_{\Pmc \cup \{ e \}}}$.
		\item
		Recall we kept in $\Pmc$ an element $e_\Msf \in A^\J \setminus A^\I$ for some $A \in \Msf$ s.t.\ $A^\J \subsetneq A^\I$ and for all $B \prec A$, $B^\J = B^\I$.
		It gives immediately that $A^{\J_e} \subsetneq A^{\I_{\Pmc \cup \{ e \}}}$ and for all $B \prec A$, $B^{\J_e} = B^{\I_{\Pmc \cup \{ e \}}}$.
	\end{enumerate}
	This proves $\J_e <_\CP \I_{\Pmc \cup \{ e \}}$.
	It is straightforward from the definition of each $\J_e$ that for all $e, e'$ in $\domain{\I}$, we have $\J_{e}|_{\domain{\I_\Pmc}} = \J_{e'}|_{\domain{\I_\Pmc}}$, concluding the proof of ``$1 \Rightarrow 2$''. 
	
	``$2 \Rightarrow 1$''. 
	Assume there exist $\Pmc \subseteq \domain{\I}$
	with $\sizeof{\Pmc} \leq 2\sizeof{\tbox} + 1$ and a
	family
	$(\Jmc_{e})_{e \in \domain{\I}}$ of models of \Kmc such that
	$\J_{ e} <_\CP \I_{\Pmc \cup \{ e\} }$ and
	$\J_{e}|_{\domain{\I_\Pmc}} =
	\J_{e'}|_{\domain{\I_\Pmc}}$ for all $e,e' \in \domain{\I}$.
	Consider such a $\Pmc \subseteq \domain{\I}$, and family of models $\J_{e}$ for each $e \in \domain{\I}$.
	We build:
	\[
	\J = \bigcup_{e \in \domain{\I}} \J_{e}
	\]
	It is clear $\domain{\J} = \domain{\I}$ and since all $\J_{e}$ agree on the shared domain $\domain{\I_\Pmc}$, we apply Lemma~\ref{lemma-portions-union} to obtain $\J <_\CP \I$.
\end{proof}

\thmcombinedupperdllitebool*
\begin{proof}
	We exhibit a $\NExpTime$ procedure to decide the complement of our problem, that is, existence of a countermodel for UCQ $q$ over $\dllitebool$ cKB $\circkb$.
	Our procedure starts by guessing an interpretation $\I$ whose domain has size at most $\abox + 2^{\sizeof{\tbox} \sizeof{q}}$.
	This can be done in exponentially many non-deterministic steps.
	It further checks whether $\I$ is a model of $\kb$ that does not entail $q$ and rejects otherwise.
	This can essentially be done naively in $\sizeof{\domain{\I}}^{\sizeof{\tbox}\sizeof{q}}$ deterministic steps, that is still simply exponential w.r.t.\ our input.
	The procedure finally checks whether $\I$ complies with $\CP$ by
	iterating over each $\Pmc$ with $\sizeof{\Pmc} \leq 2 \sizeof{\tbox} + 1$ and over each model $\J_\Pmc$ of $\kb$ s.t.\ $\J_\Pmc < \I_\Pmc$.
	If, for a choice of $\Pmc$ and $\J_\Pmc$, we can find a $\J_{\Pmc, e}$ for each $e \in \domain{\I} \setminus \domain{\I_\Pmc}$ s.t.\ $\J_{\Pmc, e} \models \kb$ with $\J_{\Pmc, e} <_\CP \I_{\Pmc \cup \{ e\} }$ and $\J_{\Pmc, e}|_{\domain{\J_\Pmc}} = \J_\Pmc$, then our procedure rejects.
	Otherwise, that is all choices of $\Pmc$ and $\J_\Pmc$ led to the existence of a $e \in \domain{\I} \setminus \domain{\I_\Pmc}$ without a fitting $\J_{\Pmc, e}$, then it accepts.
	Notice that: iterating over such $\Pmc$ can be done in $\sizeof{\domain{\I}}^{2 \sizeof{\tbox} + 1}$ iterations since we require $\sizeof{\Pmc} \leq 2 \sizeof{\tbox} + 1$.
	Computing $\I_\Pmc$ follows directly from the choice of $\Pmc$.
	Since each $\I_\Pmc$ has polynomial size, iterating over each $\J_\Pmc$ can be done naively with exponentially many steps.
	Further iterating over each $e \in \domain{\I} \setminus \domain{\I_\Pmc}$ can be done in $\sizeof{\domain{\I}}$ steps, which is at most exponential by construction.
	Since each $\I_{\Pmc \cup \{ e\} }$ has polynomial size, deciding the existence of $\J_{\Pmc, e}$ can be done naively in exponentially many steps.
	Overall, the procedure uses an exponential number of non-deterministic steps at the beginning and further performs several checks using exponentially many more deterministic steps.
	
	It remains to prove an accepting run exists iff there is a countermodel for $q$ over $\circkb$.
	If there exists an accepting run, then the corresponding guessed interpretation $\I$ is a model of $\kb$ that does not entail $q$, and it must also comply with $\CP$. Otherwise, the ``$1 \Rightarrow 2$'' direction from Lemma~\ref{lemma-criteria-combined} ensures the procedure would have rejected it.
	Conversely, if a countermodel exists, Lemma~\ref{lemma-bool-countermodel} ensures the existence of a countermodel $\I$ whose domain has exponential size.
	This $\I$ can be guessed by the procedure that checks whether it is indeed a model of $\kb$ not entailing $q$; hence, do not reject it at first.
	The ``$2 \Rightarrow 1$'' direction from Lemma~\ref{lemma-criteria-combined} further ensures $\I$ also passes the remaining check performed by the procedure; otherwise, it would contradict $\I$ being a model of $\circkb$.
\end{proof}

\thmcombinedlowerdllitepos*
\begin{proof}
	
	The proof proceeds by reduction from the complement of the $\NExpTime$-complete \succinct\tcol\ problem. 
	An instance of \succinct\tcol\ consists of a Boolean circuit $C$ with $2n$ input gates.
	The graph $G_C$ encoded by $C$ has $2^n$ vertices, identified by binary encodings on $n$ bits.
	Two vertices $u$ and $v$, with respective binary encodings $u_1 \dots u_n$ and $v_1 \dots v_n$, are adjacent in $G_C$ iff $C$ returns True when given as input $u_1 \dots u_n$ on its first $n$ gates and  $v_1 \dots v_n$ on the second half.
	The problem of deciding if $G_C$ is 3-colorable has been proven to be $\NExpTime$-complete in \cite{Papadimitriou1986}.
	
	Let $C$ be a Boolean circuit with $2n$ input gates.
	Let $G = (V,E)$ be the corresponding graph with $V = \{v_1, v_2, \cdots, v_{2^n}\}$.
	We denote $\bar i$ the binary encoding on $n$ bits of vertex $v_i$.
	We also identify $\mn{t}$ (True) and $\mn{f}$ (False) with their usual binary valuation $1$ and $0$, respectively.
%
	We construct a circumscribed \dllitecore KB $\Circ(\Kmc)$ with $\Kmc = (\Tmc,\Amc)$ and a UCQ $q$ that encode a given problem instance.
	With the aim to restrict colors and truth values to those defined in the ABox (see axioms (\ref{ax:min_value_col})), we minimize exactly the concept name $\mn{Min}$ and let others vary freely.
	Our reduction starts by representing the colors (\ref{ax:colors_def}) and truth values (\ref{ax:values_def}) in the ABox.
	The role name $\mn{neq}$ makes the inequality relations within them explicit; see axioms (\ref{ax:neq_colors}) and (\ref{ax:neq_values}).
	\begin{align}
		\mn{Col}(c)                      & \text{ for all } c \in \{\mn{r},\mn{g},\mn{b}\}         \label{ax:colors_def}                             \\
		\mn{neq}(c,c')                & \text{ for all } c,c' \in \{\mn{r},\mn{g},\mn{b}\}  \text{ with } c \neq c'  \label{ax:neq_colors}                            \\
		\mn{Val}_b(\mn{val}_b)           & \text{ for all } b \in \{\mn{f},\mn{t}\}      \label{ax:values_def}                        \\
		\mn{neq}(\mn{val}_b, \mn{val}_{b'}) & \text{ for all } b,b' \in \{\mn{f},\mn{t}\} \text{ with } b \neq b'    \label{ax:neq_values} \\
		\mn{Min}(x)                      & \text{ for all } x \in \{\mn{r},\mn{g},\mn{b},\mn{val}_\mn{t},\mn{val}_\mn{f}\} \label{ax:min_value_col}
	\end{align}

	We aim countermodels to contain $2^{2n}$ elements that each encode a pair $(v_k,v_l) \in V^2$.
	Towards this goal, we lay the foundation to branch a binary tree of depth $n$ from a root $a_\mn{tree}$, which will be later enforced via the constructed UCQ further down.
	The concept names $(\mn{Index}_i)_{i = 1,\cdots,2n}$ encode the tree levels, which fork via the role names $\mn{next}_{i,\mn{t}}$ and $\mn{next}_{i,\mn{f}}$, encoding whether the $i$-th bit in $\bar k \cdot \bar l$ should be set to $\mn{t}$ or $\mn{f}$, respectively.
	The desired elements will therefore be (a subset of) the extension $\mn{Index}_{2n}$.  
	A role name $(\mn{hBit}_j)_{j = 1, \cdots, 2n}$ will later set the actual bit values.
	Note that axioms (\ref{ax:index_hBit}) require any $i$-th tree level to set the bits $1$ to $i$ so that the UCQ can later ensure the truth values to coincide with the choice via $\mn{next}$.
	Formally:
	\begin{align}
		\mathclap{\mn{Index}_0(a_\mn{tree})} \label{ax:tree_root}                                                                                                                                         \\
		\mn{Index}_{i-1}            & \sqsubseteq \exists \mn{next}_{i,b}               \label{ax:tree_branch}                                                                                                \\
		\exists \mn{next}^-_{i,b} & \sqsubseteq  \mn{Index}_{i}                                                 \label{ax:next_index}                                                                     \\
		\mn{Index}_i            & \sqsubseteq \exists \mn{hBit}_{j}                                                                   &  & \text{ for all } j \in \{1, \cdots, i\}  \label{ax:index_hBit} \\
		\exists \mn{hBit}_j^-   & \sqsubseteq \mn{Min}                                                                                &  & \text{ for all } j \in \{1, \cdots, 2n\} \label{ax:hBit_min}
	\end{align}
	for all $i \in \{1, \cdots, 2n\}$ and $b \in \{\mn{t}, \mn{f}\}$.

	We encode the color assignments of $(v_k,v_l)$ only at the $2n$-th level via the role name $\mn{hCol}$ for $v_k$ and $\mn{hCol}'$ for $v_l$:
	\begin{align}
		\mn{Index}_{2n}     & \sqsubseteq \exists \mn{hCol} & \mn{Index}_{2n}        & \sqsubseteq \exists \mn{hCol}' \label{ax:hCol_index} \\
		\exists \mn{hCol}^- & \sqsubseteq \mn{Min}          & \exists {\mn{hCol}'}^- & \sqsubseteq \mn{Min} \label{ax:hCol_min}
	\end{align}

	Finally, the computation of a gate $g$ in $C$ is enforced for each pair by an outgoing edge $\mn{gate}_g$ at the $2n$-th level.
	Add:
	\begin{align}
		\mn{Index}_{2n}     & \sqsubseteq \exists \mn{gate}_g \label{ax:index_gate}\\
		\exists \mn{gate}_g & \sqsubseteq \mn{Min} \label{ax:gate_min}
	\end{align}
	for all gates $g$ in $C$.

	Let the constructed UCQ $q$ be the disjunction over all subsequent queries.
	We first make sure that no color is used as a truth value (CQs (\ref{cq:hasBit_not_to_color}) and (\ref{cq:gate_not_to_color})) or vice-versa (CQs (\ref{cq:hasCol_not_to_val}) and (\ref{cq:hasCol2_not_to_val})):
	\begin{align}
		\exists x,y \, \mn{hBit}_j(x,y) \wedge \mn{Col}(y)  & \text{ for all } j \in \{1, \cdots, 2n\}     \label{cq:hasBit_not_to_color} \\
		\exists x,y \, \mn{gate}_g(x,y) \wedge \mn{Col}(y)  & \text{ for all gates } g \text{ in } G'   \label{cq:gate_not_to_color}      \\
		\exists x,y \, \mn{hCol}(x,y) \wedge \mn{Val}_b(y)  & \text{ for all } b \in \{\mn{t}, \mn{f}\} \label{cq:hasCol_not_to_val}      \\
		\exists x,y \, \mn{hCol}'(x,y) \wedge \mn{Val}_b(y) & \text{ for all } b \in \{\mn{t}, \mn{f}\} \label{cq:hasCol2_not_to_val}
	\end{align}

	We next simultaneously enforce that a) the structure constructed by axioms (\ref{ax:tree_root})--(\ref{ax:next_index}) really is a tree, and b) that the valuation required by axioms (\ref{ax:index_hBit}) and (\ref{ax:hBit_min}) is consistent with the branching of the tree.
	We achieve this via CQs (\ref{cq:next_correct_additional_bit}), which makes countermodels assign the bits by $\mn{hBit}$ as dictated by $\mn{next}$, and via CQs (\ref{cq:next_same_bits}), which makes countermodels propagate the bit assignment of a node downwards the tree to all successors:
	\begin{align}
		\begin{split}
			\exists x,y,z_1,z_2 \, & \mn{next}_{i-1,b}(x,y) \wedge \mn{hBit}_j(y,z_1) \\
			& \wedge \mn{neq}(z_1,z_2) \wedge \mn{Val}_b(z_2)
		\end{split} \label{cq:next_correct_additional_bit} \\
		\begin{split}
			\exists x,y,z_1,z_2 \, & \mn{next}_{i-1,b}(x,y) \wedge \mn{hBit}_j(x,z_1) \\
			& \wedge \mn{hBit}_j(x,z_2)  \wedge \mn{neq}(z_1,z_2)
		\end{split} \label{cq:next_same_bits}
	\end{align}
	for all $i \in \{1, \cdots, 2n\}$, $j \in \{1, \cdots, i\}$ and $b \in \{\mn{t}, \mn{f}\}$.
	Observe that the above indeed enforces a tree.

	We go on to construct CQs (\ref{cq:v_color_consistent}) and (\ref{cq:v_u_color_consistent}), which prohibit any two encoded pairs of vertices $(v,u)$ and $(v,u')$ or $(v,u)$ and $(u,u')$ to assign different colors to $v$ or $u$, respectively:
	\begin{align}
		\begin{split}
			\exists & x,y,c,d,z_1, \cdots, z_n \, \mn{hCol}(x,c) \wedge \mn{hCol}(y,d) \\
			& \wedge \mn{neq}(c,d) \wedge \bigwedge_{1 \leq j \leq n} \big(\mn{hBit}_j(y, z_j) \wedge \mn{hBit}_j(x,z_j) \big)
		\end{split} \label{cq:v_color_consistent} \\
		\begin{split}
			\exists & x,y,c,d,z_1, \cdots, z_n \, \mn{hCol}(x,c) \wedge \mn{hCol}'(y,d) \\
			& \wedge \mn{neq}(c,d) \wedge \bigwedge_{1 \leq j \leq n} \big(\mn{hBit}_j(y, z_j) \wedge \mn{hBit}'_{j+n}(x,z_j) \big)
		\end{split} \label{cq:v_u_color_consistent}
	\end{align}

	To ensure that the computation of the boolean circuit is consistent, we construct CQs restricting $\mn{gate}_g$ to follow the logical operator of $g$.
	For this, we assume w.l.o.g.\ that $C$ only contains unary NOT-gates, binary AND- and OR-gates, and $2n$ nullary INPUT-gates (each encoding a bit in $\bar k \cdot \bar l$).
	Given an INPUT-gate $g$ from $C$, we denote the index of the encoded bit from $\bar k \cdot \bar l$ with $i_g$.
	Construct, for each INPUT-gate $g$ from $C$, the CQ as given next:
	\begin{align}
		\begin{split}
			\exists x,y,z \, \mn{gate}_g(x,y) \wedge \mn{hBit}_{i_g}(x,z) \wedge \mn{neq}(y,z) \label{cq:input_gates}
		\end{split}
	\end{align}
	For all NOT-gates $g$ with their parent $g'$, construct
	\begin{align}
		\begin{split}
			\exists x,y \, \mn{gate}_g(x,y) \wedge \mn{gate}_{g'}(x,y) \label{cq:not_gate}
		\end{split}
	\end{align}
	For all AND-gates $g$ with parents $g_1,g_2$, construct
	\begin{align}
		\begin{split}
			\exists x,y,z \, & \mn{gate}_g(x,y) \wedge \mn{Val}_\mn{t}(y) \wedge \mn{gate}_{g_1}(x,z) \\
			& \wedge \mn{Val}_\mn{f}(z)
		\end{split} \\
		\begin{split}
			\exists x,y,z \, & \mn{gate}_g(x,y) \wedge \mn{Val}_\mn{t}(y) \wedge \mn{gate}_{g_2}(x,z) \\
			& \wedge \mn{Val}_\mn{f}(z)
		\end{split} \\
		\begin{split}
			\exists x,y,z \, & \mn{gate}_g(x,y) \wedge \mn{Val}_\mn{f}(y) \wedge \mn{gate}_{g_1}(x,z)  \\
			& \wedge \mn{Val}_\mn{t}(z) \wedge \mn{gate}_{g_2}(x,z)
		\end{split}
	\end{align}
	For all OR-gates $g$ with parents $g_1,g_2$, construct
	\begin{align}
		\begin{split}
			\exists x,y,z \, & \mn{gate}_g(x,y) \wedge \mn{Val}_\mn{f}(y) \wedge \mn{gate}_{g_1}(x,z) \\
			& \wedge \mn{Val}_\mn{t}(z)
		\end{split} \\
		\begin{split}
			\exists x,y,z \, & \mn{gate}_g(x,y) \wedge \mn{Val}_\mn{f}(y) \wedge \mn{gate}_{g_2}(x,z) \\
			& \wedge \mn{Val}_\mn{t}(z)
		\end{split} \\
		\begin{split}
			\exists x,y,z \, & \mn{gate}_g(x,y) \wedge \mn{Val}_\mn{t}(y) \wedge \mn{gate}_{g_1}(x,z)  \\
			& \wedge \mn{Val}_\mn{f}(z) \wedge \mn{gate}_{g_2}(x,z)
		\end{split} \label{cq:or_true_despite_both_false}
	\end{align}

	Finally, we rule out the existence of monochromatic edges in countermodels via CQ (\ref{cq:monochromatic_edge}):
	\begin{align}
		\exists x,y,z \, \mn{hCol}(x,y) \wedge \mn{hCol}'(x,y) \wedge \mn{gate}_{\dot g}(x,z)  \wedge \mn{Val}_\mn{t}(z) \label{cq:monochromatic_edge}
	\end{align}
	where we here and in what follows denote by $\dot g$ the output gate of $C$.



	\begin{lemma}
		$G$ is 3-colorable iff $\Circ(\Kmc) \not \models q$.
	\end{lemma}
	``$\Rightarrow$''.
	Let $G$ be 3-colorable.
	Then there exists a 3-coloring $\pi : V \rightarrow \{\mn{r},\mn{g},\mn{b}\}$ such that $\pi(v) \neq \pi(u)$ for all $(v,u) \in V$.
	To show that $\Circ(\Kmc) \not \models q$, we construct a countermodel \Imc of $\Circ(\Kmc)$ as follows.
	Let $X \coloneqq \bigcup_{j = 1, \cdots, 2n} \{a_{j,1}, \cdots, a_{j,2^j}\}$; then set

	\begin{align*}
		\Delta^\Imc \coloneqq          & \Ind(\Amc) \cup X                                             \\
		A^\Imc \coloneqq          & \{ a \mid A(a) \in \Amc\}                                     \\
		\mn{neq}^\Imc \coloneqq        & \{ (a,b) \mid \mn{neq}(a,b) \in \Amc\}                        \\
		\mn{Index}^\Imc_i \coloneqq    & \{a_{i,1}, \cdots, a_{i,2^i}\}                                \\
		\mn{next}^\Imc_{i,\mn{t}} \coloneqq & \bigcup_{j = 1, \cdots, 2^{i-1}} \{(a_{i-1,j}, a_{i,2j})\}     \\
		\mn{next}^\Imc_{i,\mn{f}} \coloneqq & \bigcup_{j = 1, \cdots, 2^{i-1}} \{(a_{i-1,j}, a_{i,2j - 1})\} \\
		\begin{split}
			\mn{hBit}^\Imc_i \coloneqq &  \bigcup_{j = 1, \cdots, 2^{i-1}} \big( \{(a,\mn{val}_\mn{t}) \mid a \in  R(a_{i,2j})\} \\
			& \phantom{\bigcup_{1 \leq j \leq 2^{i-1}}\big(} \cup \{(a,\mn{val}_\mn{f}) \mid a \in  R(a_{i,2j-1})\} \big)
		\end{split}
	\end{align*}
	for all concept names $A \in \{\mn{Col}, \mn{Val}_\mn{t},\mn{Val}_\mn{f},\mn{Min},\mn{Index}_0\}$ and $i\in \{1, \cdots, 2n\}$, where we use
	\begin{itemize}
		\item $a_{0,1}$ to denote $a_\mn{tree}$, and
		\item $R(a_{k,l})$ to denote the set of instances that are reachable in $\Imc$ from $a_{k,l} \in X$ via the family of role names $\mn{next}_{i,\mn{t}}$ and $\mn{next}_{i,\mn{f}}$ (including $a_{k,l}$ itself).
	\end{itemize}
	We furthermore set  $\mn{hCol}^\Imc, {\mn{hCol}'}^\Imc$ as follows.
	Given some $a_{2n,i} \in X$ with $i \in \{1, \cdots 2^{2n}\}$, we denote by $\vartheta(a_{2n,i})$ the encoded pair $(v_k,u_l) \in V^2$, i.e., such that for all $j \in \{1, \cdots, 2n\}$, we have $(a_{2n,i}, \mn{val}_b) \in \mn{hBit}_j^\Imc$ iff the $j$-th bit in $\bar k \cdot \bar l$ is $b$ with $b \in \{\mn{t},\mn{f}\}$.
	\begin{align*}
		\mn{hCol}^\Imc \coloneqq    & \left\{ \big(a_{2n,i}, \pi(v)\big) \mid \vartheta(a_{2n,i}) = (v,u) \right\} \\
		{\mn{hCol}'}^\Imc \coloneqq & \left\{ \big(a_{2n,i}, \pi(u)\big) \mid \vartheta(a_{2n,i}) = (v,u) \right\}
	\end{align*}
	Finally, we set $\mn{gate}_g$ for all gates $g$ in $G_C$.
	We denote the sub-circuit of $G_C$ that contains exactly $g$ and all its ancestors with $G_C^g$.
	Note that the output gate of $G_C^g$ is $g$, i.e., $G_C^g(\bar k, \bar l)$ computes the value of $g$ in $G_C$ given the input $(\bar k, \bar l)$.
	Then, for all gates $g$, set:
	\begin{align*}
		\begin{split}
			\mn{gate}_g^\Imc \coloneqq & \{(a_{2n,i}, \mn{val}_b) \mid \vartheta(a_{2n,i}) = (v_k,u_l)  \\
			&\phantom{\{(a_{2n,i}, \mn{val}_\mn{t}) \mid} \text{ and } G_C^g(\bar k, \bar l) = b \} \\
		\end{split}
	\end{align*}

	It is obvious that \Imc is a model of \Kmc.
	As $a \in \mn{Min}^\Imc$ iff $\mn{Min}(a)$ for all $a \in \Delta^\Imc$, and $\mn{Min}$ is the only minimized concept, \Imc must furthermore be minimal.

	To see that $\Imc \not \models q$, it is readily checked that
	\begin{itemize}
		\item \Imc by definition has no answer to CQs (\ref{cq:hasBit_not_to_color})--(\ref{cq:v_u_color_consistent}),
		\item additionally, \Imc has no answer to CQs (\ref{cq:input_gates})--(\ref{cq:or_true_despite_both_false}), as otherwise, the computation of $G_C$ would be inconsistent, and
		\item finally, \Imc has no answer to CQ (\ref{cq:monochromatic_edge}), as otherwise, $\pi$ would not be a 3-coloring of $G$.
	\end{itemize}


	\smallskip

	``$\Leftarrow$''.
	Let $\Circ(\Kmc)\not \models q$.
	Then there is a model \Imc of $\Circ(\Kmc)$ such that $\Imc \not \models q$.
	We aim to extract a function $\pi : V \rightarrow \{\mn{r},\mn{g},\mn{b}\}$ and to show that $\pi$ is indeed a 3-coloring of $G$.
	This is rather laborious, as we have to prepare the final argument in five steps. In Claim (\ref{claim:relation_ranges})--(\ref{claim:output_gate_edge}), respectively, we
	\begin{enumerate}
		\item briefly argue that $\mn{hBit}$ and $\mn{gate}$ have the range $\{\mn{val}_\mn{t}, \mn{val}_\mn{f}\}$ and that $\mn{hCol}$ and $\mn{hCol}'$ have the range $\{\mn{r},\mn{g},\mn{b}\}$, which will be convenient later on,
		\item identify within \Imc the intended tree-structure,
		\item show that each leaf indeed encodes a pair $(v_k,v_l) \in V^2$, setting precisely the bits from $\bar k \cdot \bar l$,
		\item extract $\pi$ from \Imc and show that it is sound in the sense that any two leaves that encode the same vertex also assign it the same color, and
		\item proof that the leaves correctly compute all gates.
	\end{enumerate}
	We start with the ranges of $\mn{hBit}, \mn{gate}$ and $\mn{hCol}$:
	\begin{claim} \label{claim:relation_ranges}
		The following properties hold.
		\begin{itemize}
			\item both $(x,y) \in \mn{hBit}^\Imc_i$ and $(x,y) \in \mn{gate}^\Imc_g$ entail $y \in \{\mn{val}_\mn{t}, \mn{val}_\mn{f}\}$ for all $i \in \{1, \cdots, 2n\}$ and gates $g$ in $C$, and
			\item both $(x,y) \in \mn{hCol}^\Imc$ and $(x,y) \in {\mn{hCol}'}^\Imc$ entail $y \in \{\mn{r},\mn{g},\mn{b}\}$.
		\end{itemize}
	\end{claim}
	\noindent The proof is simple: As \Imc is a model of \Amc, axioms (\ref{ax:min_value_col}) imply $\{\mn{r},\mn{g},\mn{b},\mn{val}_\mn{t},\mn{val}_\mn{f}\} \subseteq \mn{Min}^\Imc$.
	Furthermore, it is obvious that there cannot be any $a \in \mn{Min}^\Imc$ with $\mn{Min}(a) \notin \Amc$, since \Imc would otherwise not be $<_\CP$ minimal (it would then be straightforward to construct a smaller model).
	Ergo, $\mn{Min}^\Imc = \{\mn{r},\mn{g},\mn{b},\mn{val}_\mn{t},\mn{val}_\mn{f}\}$.

	The rest of the claim follows from axioms (\ref{ax:colors_def}), (\ref{ax:values_def}), (\ref{ax:hBit_min}), (\ref{ax:hCol_min}), (\ref{ax:gate_min}), CQs (\ref{cq:hasBit_not_to_color})--(\ref{cq:hasCol2_not_to_val}), and $\Imc \not \models q$:
	For example, in the case of $\mn{hCol}$, consider some $(x,y) \in \mn{hCol}^\Imc$.
	By axiom (\ref{ax:hCol_min}), we have that $y \in\mn{Min}^\Imc$.
	CQ (\ref{cq:hasCol_not_to_val}) and the fact that $\Imc \not \models q$ together yield $y \in\mn{Min}^\Imc \setminus (\mn{Val}_\mn{t} \cup \mn{Val}_\mn{f})$.
	As $\{\mn{val}_\mn{t}, \mn{val}_\mn{f}\} \subseteq \mn{Val}_\mn{t} \cup \mn{Val}_\mn{f}$ by axioms (\ref{ax:values_def}), we thus have that $y \in \{\mn{r},\mn{g},\mn{b}\}$.
	The other cases are analogous.

	This finishes the proof of Claim (\ref{claim:relation_ranges}).

	\smallskip

	We go on by identifying within \Imc the intended tree structure:
	\begin{claim}
		For all $i \in \{1, \cdots, 2n\}$ there is a set $Y_i \subseteq \mn{Index}^\Imc_i$ that satisfies the following properties:

		\begin{enumerate} \label{claim:tree_induction}
			\item for all $y \in Y_i$, there is exactly one \emph{parent $p \in Y_{i-1}$ of $y$} with $(p,y) \in \mn{next}_{i-1,b}^\Imc$ for some $b \in \{\mn{t}, \mn{f}\}$ (we set $Y_0 = \{a_\mn{tree}\}$),
			\item for all $y \in Y_i$ and $j \in \{1, \cdots, i \}$, there is exactly one $b \in \{\mn{t}, \mn{f}\}$ such that $(y, \mn{val}_n) \in \mn{hBit}_j$,
			\item for all $x, y \in Y_i$ with $x \neq y$, there is some $j \in \{1, \cdots, i\}$ such that $(x, \mn{val}_\mn{t}) \in \mn{hBit}_j$ iff $(y, \mn{val}_\mn{f}) \in \mn{hBit}_j$, and
			\item $|Y_i| = 2^i$.
		\end{enumerate}
	\end{claim}
	\noindent The proof of Claim (\ref{claim:tree_induction}) goes via induction over $i$:
	\begin{description}
		\item[Base case ($i=1)$.]
			Axioms (\ref{ax:tree_root})--(\ref{ax:next_index}) require the existence of some $x,y \in \mn{Index}_1^\Imc$ such that $(a_\mn{tree},x)\in\mn{next}_{1,\mn{t}}^\Imc$ and $(a_\mn{tree},y)\in\mn{next}_{1,\mn{f}}^\Imc$.
			Set $Y_1 = \{x,y\}$.
			Property 1.\ of Claim (\ref{claim:tree_induction}) obviously holds as $a_\mn{tree}$ is, on the one hand, the only element in $Y_0$, and, on the other hand, the parent of both $x$ and $y$.
			Together with Claim (\ref{claim:relation_ranges}), axioms (\ref{ax:neq_values}), and (\ref{ax:index_hBit}), and CQs (\ref{cq:next_correct_additional_bit}) (and the fact that $\Imc \not \models q$), we have that $(x, \mn{val}_\mn{t}) \in \mn{hBit}_1^\Imc$ and $(y, \mn{val}_\mn{f}) \in \mn{hBit}_1^\Imc$, while $(x, \mn{val}_\mn{f}) \notin \mn{hBit}_1^\Imc$ and $(y, \mn{val}_\mn{t}) \notin \mn{hBit}_1^\Imc$, i.e., also properties 2.\ and 3.\ hold.
			The above furthermore entails $x \neq y$, which is why also property 4.\ must hold.
		\item[Induction step ($i > 1$).]
			By the induction hypothesis, there is a set $Y_{i-1}$ with the claimed properties.
			For all $y \in Y_{i-1}$, axioms (\ref{ax:tree_root})--(\ref{ax:next_index}) entail the existence of elements $a_y,a_y' \in \mn{Index}_i^\Imc$ such that $(y,a_y)\in\mn{next}_{i,\mn{t}}^\Imc$ and $(y,a'_y)\in\mn{next}_{i,\mn{f}}^\Imc$.
			Similar as in the base case, we deduct that $(a_y, \mn{val}_\mn{t}) \in \mn{hBit}_i^\Imc$ and $(a'_y, \mn{val}_\mn{f}) \in \mn{hBit}_i^\Imc$, while $(a_y, \mn{val}_\mn{f}) \notin \mn{hBit}_i^\Imc$ and $(a'_y, \mn{val}_\mn{t}) \notin \mn{hBit}_i^\Imc$ $(\dag)$ via Claim (\ref{claim:relation_ranges}), axioms (\ref{ax:neq_values}) and (\ref{ax:index_hBit}), and CQs (\ref{cq:next_correct_additional_bit}) (and the fact that $\Imc \not \models q$).
			Set $Y_i = \bigcup_{y \in Y_{i-1}} \{a_y,a_y'\}$.

			To see that $Y_i$ satisfies property 1., assume to the contrary that there is some $a \in Y_i$ with two parents from $Y_{i-1}$, i.e., that there are two elements $y,y' \in Y_{i-1}$ such that $(y,a) \in \mn{next}_{i-1,b}^\Imc$ and $(y',a) \in \mn{next}_{i-1,b'}^\Imc$ for some $b,b' \in \{\mn{t},\mn{f}\}$.
			By property 3.\ of the induction hypothesis, there must be some $j \in \{1,\cdots,i-1\}$ such that $(x, \mn{val}_\mn{t}) \in \mn{hBit}_j^\Imc$ iff $(x', \mn{val}_\mn{f}) \in \mn{hBit}_j^\Imc$.
			Furthermore observe that Claim (\ref{claim:relation_ranges}) and axioms (\ref{ax:index_hBit}) require $a$ to assign the $j$-th bit, i.e., we have $(a, \mn{val}_\mn{t}) \in \mn{hBit}_j^\Imc$ or $(a, \mn{val}_\mn{f}) \in \mn{hBit}_j^\Imc$.
			In either case, due to axioms (\ref{ax:neq_values}), we can find a homomorphism from CQ (\ref{cq:next_same_bits}) to \Imc, which contradicts the fact that \Imc is a countermodel to $q$.
			Ergo, $Y_i$ must satisfy property 1.

			Property 2.\ follows, in case of $j = i$, from $(\dag)$, and, if otherwise $j \in \{1,\cdots,i-1\}$, from $Y_{i-1}$ satisfying property 2.\ and the fact that Claim (\ref{claim:relation_ranges}), axioms (\ref{ax:neq_values}) and (\ref{ax:index_hBit}), and CQs (\ref{cq:next_same_bits}) require all $a \in Y_i$ to assign the bits $1, \cdots, i-1$ exactly as their parents.

			For property 3, consider some $a, a' \in Y_i$ with $a \neq a'$.
			If $a$ and $a'$ have the same parent $x$, property 3.\ follows from $(\dag)$ and property 1.
			Otherwise, if $a$ has a parent $x$ and $a'$ has a parent $y$ such that $x \neq y$, observe that property 3.\ of the induction hypothesis implies the existence of some $j \in \{1,\cdots,i-1\}$ such that $(x, \mn{val}_\mn{t}) \in \mn{hBit}_j$ iff $(y, \mn{val}_\mn{f}) \in \mn{hBit}_j$.
			Property 3.\ then follows from the fact that Claim (\ref{claim:relation_ranges}), axioms (\ref{ax:neq_values}) and (\ref{ax:index_hBit}), and CQs (\ref{cq:next_same_bits}) require $a$ and $a'$ to assign the bits $1, \cdots, i-1$ exactly as their parents.

			Regarding property 4., observe that by $(\dag)$, $Y_i$ contains two elements $a_y,a_y'$ with $a_y \neq a_y'$ for each $y \in Y_{i-1}$.
			Given any two elements $x,y \in Y_{i-1}$, it is furthermore easy to see that $\{a_x,a_x'\}$ and $\{a_y,a_y'\}$ are disjoint, as by property 2.\ and 3.\ of the induction hypothesis, there must be some bit $j \in \{1,\cdots,i-1\}$ that $x$ and $y$ assign differently, which subsequently must then also be assigned differently by $\{a_x,a_x'\}$ and $\{a_y,a_y'\}$ due to Claim (\ref{claim:relation_ranges}), axioms (\ref{ax:neq_values}) and (\ref{ax:index_hBit}), and CQs (\ref{cq:next_same_bits}).
			As, thus, $|Y_i| = 2 |Y_{i-1}|$, property 4.\ of the induction hypothesis implies $|Y_i| = 2^i$.
	\end{description}
	This finishes the proof of Claim (\ref{claim:tree_induction}).

	\smallskip

	We go on to identify the instances in \Imc that encode pairs of vertices:
	\begin{claim} \label{claim:identify_encoding_pairs}
		We can identify a set $Y \subseteq \mn{Index}_{2n}^\Imc$ and a bijective function $\vartheta : Y \rightarrow V^2$ such that $\vartheta(x) = (v_k,u_l)$ implies $(x, \mn{val}_b) \in \mn{hBit}_i^\Imc$ iff the $i$-th bit in $\bar k \cdot \bar l$ is $b$ for all $b \in \{\mn{t},\mn{f}\}$ and $i \in \{1, \cdots, 2n\}$.
	\end{claim}

	\noindent Claim (\ref{claim:identify_encoding_pairs}) builds upon Claim (\ref{claim:tree_induction}); set $Y = Y_{2n}$.
	By property 2.\ of Claim (\ref{claim:tree_induction}), each $y \in Y$ assigns a unique $b \in \{\mn{t},\mn{f}\}$ to every $i$-th bit with  $i \in \{1, \cdots, 2n\}$, encoded via $(y, \mn{val}_b) \in \mn{hBit}_i^\Imc$.
	Ergo, we can find a function $\vartheta : Y \rightarrow V^2$ such that $\vartheta(y) = (v_k,u_l)$ implies $(y, \mn{val}_b) \in \mn{hBit}_i^\Imc$ iff the $i$-th bit in $\bar k \cdot \bar l$ is $b$ for all $b \in \{\mn{t},\mn{f}\}$ and $i \in \{1, \cdots, 2n\}$.

	Note that $\vartheta$ must be surjective, as by $|Y| = 2^{2n}$ (property 4\ of Claim (\ref{claim:tree_induction})) and the fact that all $x,y \in Y$ pairwise assign at least one bit differently (property 3.\ of Claim (\ref{claim:tree_induction})).

	To see that $\vartheta$ is injective, consider some $x,y \in Y$ with $\vartheta(x) = (v_k,u_l) = \vartheta(y)$.
	By definition of $\vartheta$, we have that $(x, \mn{val}_b) \in \mn{hBit}_i^\Imc$ iff $(y, \mn{val}_b) \in \mn{hBit}_i^\Imc$ for all with $b \in \{\mn{t},\mn{f}\}$ and $i \in \{1, \cdots, 2n\}$, and, together with property 2 \ of Claim (\ref{claim:tree_induction}), even that $(x, \mn{val}_\mn{t}) \in \mn{hBit}_i^\Imc$ iff $(y, \mn{val}_\mn{f}) \notin \mn{hBit}_i^\Imc$.
	Property 3.\ of Claim (\ref{claim:tree_induction}) then entails $x = y$.

	This finishes the proof of Claim (\ref{claim:identify_encoding_pairs}).

	\smallskip

	We go on to extract $\pi$ from \Imc and show that \Imc is sound in the sense that any two elements $x,y \in Y$ that encode the same vertex also assign it the same color as defined by $\pi$.

	\begin{claim} \label{claim:sound_colors}
		The relation $\pi : V \rightarrow \{\mn{r},\mn{g},\mn{b}\}$ with
		\begin{align*}
			\begin{split}
				\pi(v) \coloneqq \{c \in \{\mn{r},\mn{g},\mn{b}\} \mid & \exists y \in Y \text{ such that } \vartheta(y) = (v,u) \\
				& \text{ and } (y, c) \in \mn{hCol}^\Imc \}.
			\end{split}
		\end{align*}
		is a function such that $(y,c_1) \in \mn{hCol}^\Imc$ and $(y,c_2) \in {\mn{hCol}'}^\Imc$ hold for all $y \in Y$ with $\vartheta(y) = (v,u)$, $\pi(v) = c_1$ and $\pi(u) = c_2$.
	\end{claim}
	\noindent Indeed, $\pi$ is a function:
	On one hand, Claim (\ref{claim:relation_ranges}) and axioms (\ref{ax:hCol_index}) ensure the existence of some $c,d \in \{\mn{r},\mn{g},\mn{b}\}$ such that $(y, c) \in \mn{hCol}^\Imc$ and $(y, d) \in \mn{hCol}'^\Imc$ for all $y \in Y$.
	On the other hand, $\pi(y)$ is unique for all $y \in Y$:
	Assume to the contrary that there are two elements $y,y' \in Y$ such that $\vartheta(y) = (v,u)$ with $(y, c) \in \mn{hCol}^\Imc$ and $\vartheta(y') = (v,u')$ with $(y', c') \in \mn{hCol}^\Imc$ and $c' \neq c$.
	By definition of $\vartheta$, $y$ and $y'$ assign the bits with $1,\cdots,n$ the same, i.e., $(y,z) \in \mn{hBit}_i^\Imc$ iff $(y',z) \in \mn{hBit}_i^\Imc$.
	Furthermore, $c' \neq c$ yields $(c',c) \in \mn{neq}^\Imc$ via axioms (\ref{ax:neq_colors}).
	Then, however, there must be a homomorphism from CQ (\ref{cq:v_color_consistent}) to \Imc, which contradicts the fact that $\Imc \not \models q$.
	
	Now, consider some $y \in Y$ with $\vartheta(y) = (v,u)$, $\pi(v) = c_1$ and $\pi(u) = c_2$.
	We first show that $(y,c_1) \in \mn{hCol}^\Imc$.
	The definition of $\pi$ requires the existence of some $y' \in Y$ such that $\vartheta(y') = (v,u')$ and $(y',c_1) \in \mn{hCol}^\Imc$, while there must also be some $c \in \{\mn{r},\mn{g},\mn{b}\}$ such that $(y,c) \in \mn{hCol}^\Imc$ via Claim (\ref{claim:relation_ranges}) and axioms (\ref{ax:hCol_index}).
	By $\vartheta(y) = (v,u)$ and $\vartheta(y') = (v,u')$, Claim (\ref{claim:identify_encoding_pairs}) implies that $(y, \mn{val}_b) \in \mn{hBit}_i^\Imc$ iff $(y', \mn{val}_b) \in \mn{hBit}_i^\Imc$ for all $i \in \{1, \cdots, n\}$ and $b \in \{\mn{t},\mn{f}\}$.
	Then, we have that $(c_1,c) \notin \mn{neq}^\Imc$ via CQs (\ref{cq:v_color_consistent}) and the fact that $\Imc \not \models q$.
	Ergo, by Claim (\ref{claim:relation_ranges}), and axioms (\ref{ax:neq_colors}) and (\ref{ax:hCol_index}), we have that $c = c_1$, i.e., $(y,c_1) \in \mn{hCol}^\Imc$.

	We go on to show that $(y,c_2) \in {\mn{hCol}'}^\Imc$.
	Again, the definition of $\pi$ requires the existence of some $y' \in Y$ such that $\vartheta(y') = (u,u')$ and $(y',c_2) \in \mn{hCol}^\Imc$, while there must also be some $c \in \{\mn{r},\mn{g},\mn{b}\}$ such that $(y,c) \in {\mn{hCol}'}^\Imc$ via Claim (\ref{claim:relation_ranges}) and axioms (\ref{ax:hCol_index}).
	By $\vartheta(y) = (v,u)$ and $\vartheta(y') = (v,u')$, Claim (\ref{claim:identify_encoding_pairs}) implies that $(y, \mn{val}_b) \in \mn{hBit}_{i+n}^\Imc$ iff $(y', \mn{val}_b) \in \mn{hBit}_i^\Imc$ for all $i \in \{1, \cdots, n\}$ and $b \in \{\mn{t},\mn{f}\}$.
	Then, we have that $(c_2,c) \notin \mn{neq}^\Imc$ via CQs (\ref{cq:v_u_color_consistent}) and the fact that $\Imc \not \models q$.
	Ergo, by Claim (\ref{claim:relation_ranges}), and axioms (\ref{ax:neq_colors}) and (\ref{ax:hCol_index}), we have that $c = c_1$, i.e., $(y,c_1) \in {\mn{hCol}'}^\Imc$.

	This finishes the proof of Claim (\ref{claim:sound_colors}).

	\smallskip

	As a last step before proving that $\pi$ indeed is a 3-coloring, we need to argue that the computation of the logical gates within \Imc is sound.

	\begin{claim} \label{claim:output_gate_edge}
		For all $y \in Y$, we have that $\vartheta(y) \in E$ implies $(y, \mn{var}_\mn{t}) \in \mn{gate}_{\dot g}^\Imc$, where $\dot g$ is the output gate of $C$.
	\end{claim}
	\noindent As $\vartheta(y) \in E$ implies $C^{\dot g}(\bar k,\bar l) =  \mn{t}$, it suffices to show that $C^g(\bar k,\bar l) =  b$ implies $(y, \mn{var}_b) \in \mn{gate}_g^\Imc$ for all $y \in Y$ with $\vartheta(y) = (u_k,v_l)$, gates $g$ in $C$ and $b \in \{\mn{t}, \mn{f}\}$.
	We do this per induction over the structure of $C^g$:
	\begin{description}
		\item[$g$ is an INPUT gate.]
			Let $C^g(\bar k,\bar l) =  b$.
			Then, as $g$ is an INPUT gate, the $i_g$-th bit in $\bar k \cdot \bar l$ must be $b$.
			Thus, $(y, \mn{var}_b) \in \mn{hBit}_{i_g}^\Imc$ by definition of $\vartheta$.
			Claim (\ref{claim:relation_ranges}) and axioms (\ref{ax:index_gate}) entail $(y, \mn{var}_{b'}) \in \mn{gate}_g^\Imc$ for some $b' \in \{\mn{t}, \mn{f}\}$.
			Finally, $(y, \mn{var}_b) \in \mn{hBit}_{i_g}^\Imc$, axioms (\ref{ax:neq_values}), CQs (\ref{cq:input_gates}), and the fact that $\Imc \not \models q$ entail $b' = b$, i.e., $(y, \mn{var}_b) \in \mn{gate}_g^\Imc$.
		\item[$g$ is a NOT gate with parent $g'$.]
			Let $C^g(\bar k,\bar l) =  b$.
			As $g$ is a NOT gate with parent $g'$, we have that $C^{g'}(\bar k,\bar l) =  b'$ for some $b' \in \{\mn{t}, \mn{f}\}$ such that $b' \neq b$.
			By the induction hypothesis, we then have that $(y,\mn{var}_{b'}) \in \mn{gate}_{g'}^\Imc$.
			Claim (\ref{claim:relation_ranges}) and axioms (\ref{ax:index_gate}) entail $(y, \mn{var}_{b''}) \in \mn{gate}_g^\Imc$ for some $b'' \in \{\mn{t}, \mn{f}\}$.
			Due to $b' \neq b$, CQs (\ref{cq:not_gate}) and the fact that $\Imc \not \models q$, it must be so that $b'' = b$, i.e., $(y, \mn{var}_b) \in \mn{gate}_g^\Imc$.
		\item[$g$ is a AND or OR gate.]
			We omit the details, as these cases are very similar to the NOT case, except that the respective semantics of the logical operators need to be applied.
	\end{description}

	We are now in the position to show that $\pi$ indeed is a 3-coloring.
	Assume to the contrary that there is an edge $(v,u) \in E$ with $\pi(v) = \pi(u) \coloneqq c$.
	Since $\theta$ is a bijection as of Claim (\ref{claim:identify_encoding_pairs}), there exists a unique $y \in Y$ such that $\vartheta(y) = (v,u)$.
	We have $(y,\pi(v)) \in \mn{hCol}^\Imc$ and $(y,\pi(u)) \in {\mn{hCol}'}^\Imc$ by Claim (\ref{claim:sound_colors}), i.e., $(y,c) \in \mn{hCol}^\Imc$ and $(y,c) \in {\mn{hCol}'}^\Imc$.
	Furthermore, by Claim (\ref{claim:output_gate_edge}), $(v,u) \in E$ entails $(y, \mn{var}_\mn{t}) \in \mn{gate}_{\dot g}^\Imc$.
	Together with axioms (\ref{ax:values_def}), there must be a homomorphism from CQ (\ref{cq:monochromatic_edge}) to $\Delta^\Imc$.
	However, this contradicts the fact that \Imc is a countermodel to $q$.
	Ergo, $\pi$ must indeed be a 3-coloring, i.e., $G$ is 3-colorable.
\end{proof}

\section{Proofs for Section~\ref{subsection-dllite-data}}

\lemmacriteriadata*
\begin{proof}
	Before proving the lemma, we recall the set $E$ gathers elements $e_{t_1, t_2}$ which will serve as references for others with
	 same types.
	 To lighten notation, we introduce the mapping $\refmap : e \mapsto e_{\atypeinof{\abox}{e}, \typeinof{\I}{e}}$ defined on $\domain{\I}$.
	 
	``$1 \Rightarrow 2$''. 
	Assume $\Imc \models \circkb$.
	Consider $\Pmc \subseteq \domain{\I}$ s.t.\ $\sizeof{\Pmc} \leq 2\sizeof{\tbox} +1$.
	We first verify that $\I_\Pmc$ models $\kb_\Pmc$:
	\begin{itemize}
		\item $\I_\Pmc$ models $\abox_\Pmc$ as it inherits interpretations of concepts on $\indsof{\abox'}$ from $\I$, which is a model of $\abox$.
		\item Axioms in $\tbox$ with shape $\axtop$, $\axand$, $\axnotright$ or $\axnotleft$ are satisfied since $\I_\Pmc$ inherits interpretations of concept names from $\I$, which is a model of $\tbox$.
		\item Axioms in $\tbox$ with shape $A \incl \exists r$ are witnessed with an $r$-edge pointing to $w_r$.
		\item Axioms in $\tbox$ with shape $\exists r \incl A$ are satisfied since every element in $\I_\Pmc$ receiving some $r$-edge already received some in $\I$ (and interpretations of concept names are preserved).
		\item Role inclusions $r \incl s$ are satisfied directly from the definition of $s^{\Imc_\Pmc}$.
	\end{itemize}
	It remains to prove that $\I_\Pmc$ also complies with $\CP$.
	Assume to the contrary that there exists a model $\J'$ of $\kb_\Pmc$ with $\J' <_\CP \I_\Pmc$.
	Notice that, for each role $r$ s.t.\ $r^{\J'} \neq \emptyset$, there must be an element $w_r' \in (\exists r^-)^{\J'}$.
	We can chose such an element $w_r'$ for each such role $r$, and build the following interpretation $\J$:
	\begin{align*}
		\domain{\J} = \; & \domain{\I}
		\\
		\cstyle{A}^{\J} = \; & \cstyle{A}^{\J'} \cup \{ e \in \domain{\I} \setminus \domain{\J'} \mid \refof{e} \in A^{\J'} \}
		\\
		\rstyle{r}^{\J} = \; & \{ (a, b) \mid \kb \models r(a, b) \} \\
		&
		\cup \{ (e, w_s') \mid \refof{e} \in (\exists s)^{\J'}, \tbox \models s \incl r \}
		\\
		& \cup \{ (w_s', e) \mid \refof{e} \in (\exists s)^{\J'}, \tbox \models s \incl r^- \}
	\end{align*}
	It is easily verified that $\J$ is a model of $\kb$, the most interesting cases being axioms in $\tbox$ with shape either $A \sqsubseteq \exists r$ or $\exists r \sqsubseteq A$.
	The former are satisfied using the witness $w_r$.
	The latter, when triggered on $a \in \domain{\I} \setminus \domain{\J'}$ due to a pair $(a, b)$ s.t.\ $\kb \models r(a, b)$, hold thanks to $\refof{a}$ having same ABox-type as $a$.
	We now prove $\J <_\CP \I$, which will contradict $\I \models \circkb$.
	It suffices to prove that for each concept name $A$ and $\odot \in \{\subseteq, \supseteq\}$, we have $A^{\J} \odot A^\Imc$ iff $A^{\Jmc'} \odot A^{\I_\Pmc}$.
	Let $A$ be a concept name.
	We prove the claim for $\odot =\; \subseteq$ only; for $\odot =\; \supseteq$, the arguments are similar.
	
	``$\Rightarrow$''.
	Let $A^{\J} \subseteq A^\I$ and $d \in A^{\J'}$.
	In particular, $d \in A^{\J}$ by definition of $A^{\J}$.
	Thus, by hypothesis, $d \in A^{\I}$, and since $d \in \domain{\J'} = \domain{\I_\Pmc}$ we obtain $d \in A^{\I_\Pmc}$ by definition of $A^{\I_\Pmc}$.
	
	``$\Leftarrow$''.
	Let $A^{\J'} \subseteq A^{\I_\Pmc}$ and $d \in A^{\J}$; we need to show $d \in A^{\I}$.
	It is straightforward if $d \in \domain{\Jmc'}$.
	Otherwise, we know that $\refof{d} \in A^{\J'}$ by definition of $A^{\J}$.
	Therefore $\refof{d} \in A^{\I_\Pmc}$ by hypothesis, that is $A \in \typeinof{\I}{\refof{d}}$.
	Since $d$ and $\refof{d}$ have same type, it yields $A \in \typeinof{\I}{d}$, that is $d \in A^{\I}$.

	``$2 \Rightarrow 1$''. 
	Assume $\I_\Pmc \models \Circ(\kb_\Pmc)$ for all $\Pmc \subseteq \domain{\I}$ s.t.\ $\sizeof{\Pmc} \leq 2\sizeof{\tbox} +1$.
	By hypothesis, we already have $\I \models \kb$.
	By contradiction, assume now there exists $\J < \I$.
	Based on $\J$, we build a subset $\Pmc \subseteq \domain{\J}$ containing:
	\begin{itemize}
		\item for each role $r$ s.t.\ $r^\J \neq \emptyset$, an element $w_r' \in (\exists r^-)^\J$;
		\item for each $A \in \Msf$ s.t.\ $A^\J \not\subseteq A^\I$, an element $e_A \in B^\J \setminus B^\I$ for some $B \prec A$ (Condition~3 in the definition of $\J <_\CP \I$ ensures existence of such $B$ and $e_A$);
		\item an element $e_\Msf \in A^\J \setminus A^\I$ for some $A \in \Msf$ s.t.\ $A^\J \subsetneq A^\I$ and for all $B \prec A$, $B^\J = B^\I$
		(Condition~4 in the definition of $<_\CP$ ensures the existence of such $A$ and $e_\Msf$).
	\end{itemize}
	Notice $\Pmc$ has size at most $2\sizeof{\tbox} +1$.
	We now build $\J'$ as:
	\begin{align*}
		\domain{\J'} = \; & \domain{\I_\Pmc}
		\\
		\cstyle{A}^{\J'} = \; & \cstyle{A}^{\J} \cap \domain{\J'}
		\\
		\rstyle{r}^{\J'} = \; & \{ (e, w_s') \mid e \in (\exists s)^\J  \cap \domain{\I_\Pmc}, \tbox \models s \incl r \}
		\\
		& \cup \{ (w_s', e) \mid e \in (\exists s)^\J  \cap \domain{\I_\Pmc}, \tbox \models s \incl r^- \}
	\end{align*}
	It is easily verified that $\J'$ is a model of $\kb_\Pmc$, and we now prove $\J' <_\CP \I_\Pmc$, which will contradict $\I_\Pmc \models \Circ(\kb_\Pmc)$.
	We now check that all four conditions from the definition of $<_\CP$ are satisfied:
	\begin{enumerate}
		\item 
		By definition, we have $\domain{\J'} = \domain{\I_\Pmc}$.
		\item 
		Let $A \in \Fsf$. Definitions of $\I_\Pmc$ and $\J'$ ensure $A^\I \cap {\domain{\I_\Pmc}} = A^{\I_\Pmc}$ and $A^\J \cap {\domain{\I_\Pmc}} = A^{\J'}$.
		From $\J <_\CP \I$, we get $A^{\I} = A^{\J}$, which thus yields $A^{\I_\Pmc} = A^{\J'}$.
		\item 
		Let $A \in \Msf$ such that $A^{\J'} \not\subseteq A^{\I_\Pmc}$.
		Therefore $A^{\J} \not\subseteq A^{\I}$, and recall we kept in $\Pmc$ an element $e_A \in B^\J \setminus B^\I$ for some $B \prec A$ to belong to $\Pmc$.
		Joint with $B^{\J'} \subseteq B^{\I_\Pmc}$ being trivial, $e_A$ additionally witnesses that $B^{\J'} \subsetneq B^{\I_\Pmc}$.
		\item
		Recall we kept in $\Pmc$ an element $e_\Msf \in A^\J \setminus A^\I$ for some $A \in \Msf$ s.t.\ $A^\J \subsetneq A^\I$ and for all $B \prec A$, $B^\J = B^\I$.
		It gives immediately that $A^{\J'} \subsetneq A^{\I_\Pmc}$ and for all $B \prec A$, $B^{\J'} = B^{\I_\Pmc}$.
	\end{enumerate}
	This proves $\J' <_\CP \I_\Pmc$, contradicting $\I_\Pmc \models \Circ(\kb_\Pmc)$ as desired.
      \end{proof}

\thmdataupperhornh*
\begin{proof}
	We exhibit an $\NPclass$ procedure to decide the complement of our problem, that is existence of a countermodel for UCQ $q$ over $\dlliteboolh$ cKB $\circkb$.
	Our procedure starts by guessing an interpretation $\I$ whose domain has size at most $|\Amc|+(2^{|\Tmc|+2}+1)^{3|q|}$.
	This can be done in linearly many non-deterministic steps (w.r.t.\ data complexity)
	It further checks whether $\I$ is a model of $\kb$ that does not entail $q$ and rejects otherwise.
	This can essentially be done naively in $\sizeof{\domain{\I}}^{\sizeof{\tbox}\sizeof{q}}$ deterministic steps, that is still polynomial
	w.r.t.\ data complexity.
	The procedure finally checks whether $\I$ complies with $\CP$ by
	iterating over each $\Pmc$ with $\sizeof{\Pmc} \leq 2 \sizeof{\tbox} + 1$ and checking whether there exists a model $\J_\Pmc$ of $\kb$ s.t.\ $\J_\Pmc < \I_\Pmc$.
	If, for a choice of $\Pmc$ and we can find such a $\J_{\Pmc}$, then our procedure rejects.
	Notice that iterating over such $\Pmc$ can be done in $\sizeof{\domain{\I}}^{2 \sizeof{\tbox} + 1}$ iterations since we require $\sizeof{\Pmc} \leq 2 \sizeof{\tbox} + 1$.
	Computing $\I_\Pmc$ follows directly from the choice of $\Pmc$.
	Since each $\I_\Pmc$ has constant size w.r.t.\ data complexity, iterating over each possible $\J_\Pmc$ can be done naively with constantly many steps.
	Overall, the procedure uses a linear number of non-deterministic steps at the beginning and further performs several checks using a polynomial number of 
	additional deterministic steps.
	
	It remains to prove that an accepting run exists if there is a countermodel for $q$ over $\circkb$.
	If there exists an accepting run, then the corresponding guessed interpretation $\I$ is a model of $\kb$ that does not entail $q$, and the ``$2 \Rightarrow 1$'' direction from Lemma~\ref{lemma-criteria-data} ensures it also comply+ies with $\CP$.
	Conversely, if there exists a countermodel, then Lemma~\ref{lem-quotient} ensures existence of a countermodel $\I$ whose domain has size at most $|\Amc|+(2^{|\Tmc|+2}+1)^{3|q|}$.
	The procedure can guess this $\I$, and as \Imc is indeed a model of $\kb$ not entailing $q$, it will not be rejected.
	The ``$1 \Rightarrow 2$'' direction from Lemma~\ref{lemma-criteria-combined} further ensures $\I$ also passes the remaining checks performed by the procedure.
      \end{proof}

      \thmdatalowerdllitepos*

\begin{proof}
	\newcommand{\redge}{\mathsf{edge}}
	\newcommand{\cvertex}{\mathsf{Vertex}}
	\newcommand{\rhascol}{\mathsf{hasCol}}
	\newcommand{\ccolor}{\mathsf{Color}}
	We reduce the complement of the graph 3-colorability problem (3Col) to evaluating the Boolean CQ:
	\[
	q = \exists y_1 \, \exists y_2 \, \exists y ~ \redge(y_1, y_2) \land \rhascol(y_1, y) \land \rhascol(y_2, y)
	\]
	over the \dllitepos TBox:
	\[
	\tbox = \{ \cvertex \incl \exists \rhascol, \exists \rhascol^- \incl \ccolor \},
	\]
	where $\ccolor$ is minimized while all other predicates vary.
	Let $\CP$ be the resulting circumscription pattern.
	Given an instance $\graph = (\vertices, \edges)$ of 3Col, we build an ABox $\abox$ containing the following assertions:
	\begin{align}
		\cvertex(v) & \text{ for all } v \in \vertices
		\\
		\redge(v_1, v_2) & \text{ for all } \{v_1, v_2\} \in \edges
		\\
		\ccolor(c) & \text{ for all } c \in \{ r, g, b \}
	\end{align}
	We now prove the following claim:
	\[
	\graph \notin \text{3Col} \iff \circkb \models q
	\]
	First notice that every model $\I$ of $\circkb$ verifies $\ccolor^\I = \{ r, g, b \}$.
	Indeed, by contradiction, if $\I$ contains $e \in \ccolor^\I \setminus \{ r, g, b \}$, then we can build $\Jmc$ a model of $\kb$ with $\Jmc <_\CP \I$ by modifying $\I$ as follows:
	\begin{itemize}
		\item remove $e$ from $\ccolor^\I$, that is $\ccolor^\Jmc = \ccolor^\Imc \setminus \{ e \}$;
		\item reroute $\rhascol$ as: $\rhascol^\Jmc = \{ (v, r) \mid v \in (\exists \rhascol)^\I \}$.
	\end{itemize}
	It is then straightforward that $\Jmc$ satisfies the desired properties, contradicting $\I$ being a model of $\circkb$.
		
	``$\Rightarrow$''. Assume $\graph \notin \text{3Col}$ and consider a model $\I$ of $\Jmc$.
	By the above remark, we have $\ccolor^\I = \{ r, g, b \}$ and from $\I$ being a model of $\kb$, we can find a mapping $\tau : \vertices \rightarrow \{ r, g, b\}$ of $\graph$ such that for all $v \in \vertices$, if $\tau(v) = c$ then $(v, c) \in \rhascol^\I$.
	Since, by assumption, $\graph \notin \text{3Col}$, there exists an edge $\{ v_1, v_2 \} \in \edges$ such that $\tau(v_1) = \tau(v_2) = c$ for some $c \in \{ r, g, b \}$, which ensures $y_1 \mapsto v_1, y_2 \mapsto v_2, y \mapsto c$ is an homomorphism of $q$ in $\I$.
	
	``$\Leftarrow$''. Assume $\graph \in \text{3Col}$ and consider a 3-coloring $\tau : \vertices \rightarrow \{ r, g, b\}$.
	We build a model $\I_\tau$ of $\circkb$ by interpreting all concepts and roles as in $\abox$ except for $\rhascol$, which is interpreted as:
	\[
	\rhascol^{\I_\tau} = \{ (v, c) \mid v \in \vertices, \tau(v) = c \}.
	\]
	It is straightforward that $\I$ models $\circkb$ and does not embed $q$ (since $\tau$ is a 3-coloring).
      \end{proof}

\section{Proofs for Section~\ref{subsection-instance-alchi}}

We briefly discuss the variation of Proposition~\ref{prop:nonom} mentioned in the main part of the paper. 
Assume given a circumscribed $\mathcal{ALCHIO}$ KB $\circkb$ with $\kb = (\tbox, \abox)$, an AQ $A_0(x)$ and an individual $a_0 \in \mathsf{ind}(\mathcal{A})$.
We again replace all occurrences of a nominal $a$ with fresh concept $A_a$, add $A_a(a), B_a(a)$ to $\abox$, minimize $B_a$ and set $B_a \prec A$ for all $A \in \Msf$. 
However, in contrast with Proposition~\ref{prop:nonom}, we also add axioms $A_a \sqcap \neg B_a \sqsubseteq \exists r.(X \sqcap A_0)$ and $X \sqsubseteq Y$ to $\tbox$ for some fresh concept names $X$ and $Y$ and fresh role name $r$, and $Y(a_0)$ to $\abox$.
Finally, we also minimize $Y$ with higher preference than concepts from $\Msf$: We set $Y \prec A$ for all $A \in \Msf$.
Denoting $\mn{Circ}_{\CP'}(\kb')$ the resulting circumscribed \ALCHI KB, it can then be verified that 
	$\Circ(\Kmc) \models A_0(a_0)$ iff 
	$\mn{Circ}_{\mn{CP}'}(\Kmc') \models A_0(a_0)$.

We now move to a proof of Theorem~\ref{thm-combined-lower-el-aq}.

\thmcombinedlowerelaq*
\begin{proof}
Let $\mn{Circ}_\CP(\Kmc)$ be a circumscribed \ALC KB, with $\Kmc=(\Tmc,\Amc)$, $A_0$ an AQ, and $a_0 \in \Ind(\Amc)$.  We construct a
        circumscribed $\EL$ KB $\mn{Circ}_\CP(\Kmc')$, with
        $\Kmc'=(\Tmc',\Amc')$, such that
        $\mn{Circ}_\CP(\Kmc) \models A_0(a_0)$ iff $\mn{Circ}_\CP(\Kmc') \models A_0(a_0)$.

        With $\ELU_\bot$, we mean the extension of \EL with disjunction and $\bot$.
        It is well-known that every \ALC TBox $\Tmc$ can be rewritten in polynomial time into an $\ELU_\bot$ TBox $\Tmc^\ast$ that is a conservative extension of $\Tmc$ in the sense that every
        model of $\Tmc^\ast$ is a model of $\Tmc$ and every model of $\Tmc$ can be extended to a model of $\Tmc^\ast$ by interpreting the fresh concept and role names in $\Tmc^\ast$ \cite{DBLP:conf/ijcai/BaaderBL05}. It follows that $\mn{Circ}_\CP(\Kmc)
        \models A_0(a_0)$ iff $\mn{Circ}_\CP(\Tmc^\ast,\Amc) \models A_0(a_0)$, assuming that all the fresh concept names in $\Tmc^\ast$ vary in \CP.
        In the following, we recall the rewriting, which proceeds in three steps:
        \begin{enumerate}

        \item \emph{Remove value restrictions} by replacing every subconcept
          $\forall r . C$ with $\neg \exists r . \neg C$.

        \item \emph{Remove negation of compound concepts} by replacing
          every subconcept $\neg C$, with $C$ compound, by $\neg X$
          where $X$ is a fresh concept name, and adding the CIs
          $$
             C \sqsubseteq X \qquad X \sqsubseteq C.
          $$
          
        \item \emph{Remove negation entirely} by replacing every subconcept
          $\neg A$ with the fresh concept name $\overline A$, and adding
          the CIs
          $$
             \top  \sqsubseteq A \sqcup \overline{A} \qquad
             A \sqcap \overline{A} \sqsubseteq \bot.
          $$

        \end{enumerate}
        We may thus assume w.l.o.g.\ that the TBox \Tmc in the given circumscribed KB
        \Kmc is
        an $\ELU_\bot$ TBox. We may further assume that disjunction occurs in \Tmc
        only in CIs of the form
        $$
           B \sqsubseteq B_1 \sqcup B_2
        $$
        where $B,B_1,B_2$ are concept names, by replacing every disjunction
        $C_1 \sqcup C_2$ with a fresh concept name $X$ and adding the CIs
        $$
        \begin{array}{r@{\;}c@{\;}lcr@{\;}c@{\;}lcr@{\;}c@{\;}l}
          X &\sqsubseteq& Y_1 \sqcup Y_2 && Y_1 &\sqsubseteq& X &&
                                                                   Y_2 &\sqsubseteq& X \\[1mm]
          C_i & \sqsubseteq & Y_i && Y_i &\sqsubseteq& C_i && \text{for } i \in \{1,2\}
        \end{array}
        $$
        where also $Y_1,Y_2$ are fresh concept names. Finally, we may assume
        that \Tmc contains only a single CI $B \sqsubseteq B_1 \sqcup B_2$
        by replacing every such CI with
        $$
        B \sqsubseteq \exists r_{B_1,B_2} . D \quad
        \exists r_{B_1,B_2} . D_1 \sqsubseteq B_1 \quad 
        \exists r_{B_1,B_2} . D_2 \sqsubseteq B_2 
        $$
        where $r_{B_1,B_2}$ is a fresh role name and $D,D_1,D_2$ are fresh concept names,
        and adding the CI
        $$
            D \sqsubseteq D_1 \sqcup D_2.
        $$
        We may clearly also assume that $\bot$ occurs in \Tmc only in the form
        $C \sqsubseteq \bot$ with $C$ an \EL-concept.
        Note that the TBoxes resulting from these transformations are conservative
        extensions of the original ones. In \CP, all of the freshly introduced
        concept names vary.

        \smallskip
        We are now ready for the actual reduction, which combines ideas from
        the proofs of Theorems~\ref{thm:combined-lower-el} and~\ref{thm-data-lower-el}. To construct the TBox $\Tmc'$,
        We start with relativized versions of the \EL-CIs in \Tmc. 
        For an \EL
        concept $C$,
        inductively define the concept $C_\X$ as follows:
        \begin{align*}
                A_\X    & \coloneqq \X \sqcap A & (\exists r.D)_\X & \coloneqq \X \sqcap \exists r.D_\X \\
                \top_\X & \coloneqq \X          & (D \sqcap D')_\X & \coloneqq D_\X \sqcap D'_\X
        \end{align*}
        We then include in $\Tmc'$ all CIs 
        \begin{align}
                C_\X \sqsubseteq D_\X 
        \end{align}
        such that $C \sqsubseteq D \in \Tmc$ is neither the unique disjunctive
        CI $B \sqsubseteq B_1 \sqcup B_2$
        nor of the form $C \sqsubseteq \bot$. To construct the ABox $\Amc'$,
        we start with the extension of \Amc by
        \begin{align}
          X(a) \quad \text{ for all } a \in \Ind(\Amc).
          \label{Xmarking}
        \end{align}

        Of course, we need to compensate for removing the mentioned CIs.
        To simulate the disjunctive CI $B \sqsubseteq B_1 \sqcup B_2$, we
        extend $\Tmc'$ with
        \begin{align}
        B &\sqsubseteq \exists r_D . (D \sqcap \mn{Pos}) \qquad
            \exists r_D . D_1 \sqsubseteq B_1
            \label{blablaone}
          \\
                                                             B &\sqsubseteq \exists r_D . (D \sqcap \mn{Neg}) 
                                                          \qquad
                                                                 \exists r_D . D_2 \sqsubseteq B_2
                                                                 \label{blablatwo}
        \end{align}
        %
        where $D,D_1,D_2,\mn{Pos}$, and $\mn{Neg}$ are fresh concept names and $r_D$ is a fresh role name.
        We update \CP so that $D$ is minimized. We further extend \Amc with the
        assertions
\begin{align}
	D(d_1^+) 	& & D(d_1^-) 	& &D(d_2^+) 	& & D(d_2^-) \\
	D_1(d_1^+) 	& & D_1(d_1^-) 	& &
                                                            D_2(d_2^+)
                                                                                & & D_2(d_2^-)  	 \\
	\mn{Pos}(d_1^+) 	& & \mn{Neg}(d_1^-) 	& & \mn{Pos}(d_2^+) 	& &\mn{Neg}(d_2^-).
\end{align}
where $d^+_1,d^-_1,d^+_2,d^-_2$ are fresh individuals. Note that these individuals are \emph{not} marked with $X$, exempting them from $\Tmc$ is why the CIs from \Tmc are relativized in $\Tmc'$. Note that the way we simulate disjunction is similar to how we choose truth values in the proof of Theorem~\ref{thm-data-lower-el}, and also the reason for using a positive and a negative witness is the same.
We need to make sure that the positive and negative witnesses are distinct, which is achieved by otherwise making the query true.
%
To this end, we introduce a fresh role name $u$ and fresh concept
names $G, A_0'$. We update \CP so that $G$ is minimized and extend $\Tmc'$ with:%
\begin{align}
	\mn{Pos} \sqcap \mn{Neg} & \sqsubseteq \exists u.(G \sqcap A'_0)
	\label{barlabel} \\
	A_0 &\sqsubseteq A'_0
	\label{barbarlabel}
\end{align}
We add
$
G(a_0)
$
to $\Amc'$ and the new query is $A'_0$.

To compensate for removing the CIs \mbox{$C \sqsubseteq \bot$}, we consider another fresh concept name $M$ and update \CP so that $M$ is minimized. We further extend $\Tmc'$ with
%
\begin{align}
	C \sqsubseteq M \sqcap \exists u . (G \sqcap A'_0) \quad \text{ for all } C \sqsubseteq \bot \in \Tmc.
	\label{lastone}
\end{align}
So if a concept $C$ with $C \sqsubseteq \bot \in \Tmc$ is satisfied, then we force the minimized concept name $M$ to be non-empty and make the (new) query true.

%
%
Let $\CP'$ be the extended circumscription pattern.
It remains to establish the correctness of the reduction.
        \\[2mm]
        {\bf Claim.} 
                $\Circ(\Kmc) \models A_0(a_0)$ iff $\Circ(\Kmc') \models A'_0(a_0)$.
                \\[2mm]
                To prepare for the proof of the claim, we first observe that every
                model \Imc of \Kmc gives rise to a model $\Imc'$ of $\Kmc'$
                in a natural way. More precisely, we set $\Delta^{\Imc'} = \Delta^\Imc \uplus \{ d^+_1,d^-_1,d^+_2,d^-_2\}$,
                define the extension of all concept and role names that occur in \Kmc
                exactly as in \Imc, and interpret the fresh concept and role names as
                follows:
                $$
                \begin{array}{rcl}
                  X^{\Imc'} &=& \Delta^\Imc \\[1mm]
                  D^{\Imc'} &=& \{ d^+_1,d^-_1,d^+_2,d^-_2\} \\[1mm]
                  D_i^{\Imc'}&=& \{d^+_i, d^-_i \} \text{ for } i \in \{1,2\} \\[1mm]
                  \mn{Pos}^{\Imc'} &=& \{ d^+_1,d^+_2 \}\\[1mm]
                  \mn{Neg}^{\Imc'} &=& \{ d^-_1,d^-_2 \}\\[1mm]
                  r_D^{\Imc'} &=& \{ (d,d^+_1),(d,d^-_1) \mid d \in (B \sqcap B_1)^\Imc \} \, \cup \\[1mm]
                         && \{ (d,d^+_2),(d,d^-_2) \mid d \in (B \sqcap B_2)^\Imc \} \\[0.5mm]
                  \mn{check}^{\Imc'} &=& \{ (a_0,a) \mid a \in \mn{ind}(\Amc) \\[1mm]
                  {A'_0}^{\Imc'} &=& A_0^\Imc \\[1mm]
                  u^{\Imc'} &=& \emptyset \\[1mm]
                  M^{\Imc'} &=& \emptyset \\[1mm]
                  G^{\Imc'} &=& \{ a_0 \}\\[1mm]
                \end{array}
                $$
                It is easily checked that ${\Imc'}$ is indeed a model of $\Kmc'$.
        Now for the actual proof of the claim
        \smallskip
                
        ``$\Leftarrow$''.  Assume that $\Circ(\Kmc) \not\models A_0(a_0)$. Then there is a model \Imc of \Kmc that is minimal w.r.t.\ \CP and satisfies $\Imc\not\models A_0(a_0)$. Let ${\Imc'}$ be the corresponding model of $\Kmc'$ defined above. We
                clearly have
            ${\Imc'} \not\models A'_0(a_0)$. To show that $\Circ(\Kmc') \not\models A'_0(a_0)$, it thus remains to show that $\Imc'$ is minimal w.r.t.\ $\CP'$.

            Assume to the contrary that there is a model $\Jmc' <_{\CP'} \Imc'$ of~$\Kmc'$.
            As $\Imc'$ interprets $D,M,$ and $G$ minimally among all models of $\Kmc'$,
            we must have $D^{\Jmc'} = D^{\Imc'} = \{ d^+_1,d^-_1,d^+_2,d^-_2\}$, $M^{\Jmc'} = \emptyset$, and
            $G^{\Jmc'} = G^{\Imc'} = \{ a_0 \}$. We show that the restriction $\Jmc$
            of $\Jmc'$ to domain $X^{\Jmc'}$ is a model
            of $\Kmc$. Moreover, $\Jmc' <_{\CP'} \Imc'$, $D^{\Jmc'} = D^{\Imc'}$, $M^{\Jmc'} = M^{\Imc'}$, and
            $G^{\Jmc'} = G^{\Imc'}$ imply $\Jmc <_\CP \Imc$, which then yields a contradiction against the minimality of \Imc.

            Based on the facts that $\Jmc'$ is a model of $\Amc'$, $\Ind(\Amc) \subseteq X^{\Jmc'}$ due to (\ref{Xmarking}), and $\Amc \subseteq \Amc'$, it is clear that $\Jmc$ is a model of \Amc. Moreover, since $\Jmc'$ is a model of $\Tmc'$, by
            construction of $\Jmc$ it is easy to see that $\Jmc$ satisfies all CIs in \Tmc
            with the possible
            exceptions of the disjunctive CI $B \sqsubseteq B_1 \sqcup B_2$ and the
            CIs of the form $C \sqsubseteq \bot$. The former, however, is satisfied due to
            CIs~(\ref{blablaone}) and~(\ref{blablatwo}) and since
            $D^{\Jmc'} = \{ d^+_1,d^-_1,d^+_2,d^-_2\}$. And the latter is satisfied
            due to the CIs~(\ref{lastone}) and since $M^{\Jmc'}=\emptyset$.
            
            %

\smallskip
            
``$\Rightarrow$''.  Assume that $\Circ(\Kmc') \not\models A'_0(a_0)$. Then there is a model $\Imc'$ of $\Kmc'$ that is minimal w.r.t.\ $\CP'$ and satisfies $\Imc'\not\models A'_0(a_0)$. 
We first argue that $G^{\Imc'}=\{ a_0 \}$ as, otherwise, we can find a model $\Jmc' <_{\CP'} \Imc'$ of~$\Kmc'$, contradicting the minimality of $\Imc'$.
We have $G^{\Imc'} \subseteq \{ a_0 \}$ since $\Imc'$ is a model of $\Amc'$. Assume to the contrary of what we want to show that the converse fails.  Then we construct $\Jmc'$ by modifying $\Imc'$ as follows:
\begin{itemize}
	
	\item set $G^{\Jmc'} = \{ a_0 \}$;
	
	\item add $a_0$ to ${A'_0}^{\Jmc'}$;
	
	\item reroute $u$ as $u^{\Jmc'} = \{ (e, a_0) \mid e \in (\exists u)^{\Imc'} \}$.
	
\end{itemize}
It is easy to see that, indeed, $\Jmc'$ is still a model of $\Kmc'$ and $\Jmc' <_{\CP'} \Imc'$.

Since $\I'$ satisfies CIs~(\ref{blablaone}) and~(\ref{blablatwo}), for every $e \in B^{\I'}$ we can choose two elements $e^+ \in (D \sqcap \mn{Pos})^{\Imc'}$ and $e^- \in (D \sqcap \mn{Neg})^{\Imc'}$ such that $(e,e^+), (e, e^-) \in r_D^{\I'}$. If $e^+ = e^-$ for some $e$, then by~(\ref{barlabel}) and $G^{\Imc'} = \{ a_0\}$, we have $a_0 \in A_0^{\Imc'}$,
which is impossible since $\Imc'\not\models A'_0(a_0)$. Thus  $e^+ \neq e^-$ for
all $e$.

We next argue that the following conditions are satisfied:
\begin{enumerate}

\item for all $e \in B^{\Imc'}$:
$\{ e^+, e^- \} \cap \{ d^+_1, d^-_1,d^+_2, d^-_2\}
\neq \emptyset$;


\item $M^{\Imc'}=\emptyset$.

\end{enumerate}
If, in fact, any of the above conditions are violated, then we can again find a model $\Jmc' <_{\CP'} \Imc'$ of~$\Kmc'$, contradicting the minimality of $\Imc'$.
We start with Condition~1.

Assume to the contrary that for some $e \in B^{\Imc'}$, $e^+, e^- \notin \{ d^+_1, d^-_1,d^+_2, d^-_2\}$.  Then we construct
$\Jmc'$ by modifying \Jmc as follows:
\begin{itemize}

\item remove $e^-$ from $D^{\Jmc'}$;

\item add $e^+$ to $\mn{Neg}^{\Jmc'}$;
  
\item add $a_0$ to ${A'_0}^{\Jmc'}$;

\item add $(e^+, a_0)$ to $u^{\Jmc'}$.

\end{itemize}
It is easy to see that, indeed, $\Jmc'$ is still a model of $\Kmc'$ and $\Jmc' <_{\CP'} \Imc'$.  

%
%
%

In the case of Condition~2, we can construct the model $\Jmc'$ simply as \Jmc modified by setting $M^{\Jmc'}=\emptyset$. The resulting $\Jmc'$ is a model of $\Kmc'$.  The CIs~(\ref{lastone}) are the only potentially problematic items. But since $\Imc'\not\models A'_0(a_0)$ and $G^{\Imc'}=\{a_0\}$, we can infer from~(\ref{lastone}) that $C^{\Imc'} = \emptyset$ for all $C \sqsubseteq \bot \in \Tmc$, and thus the same holds for $\Jmc'$. Consequently, the CIs~(\ref{lastone}) are never triggered in $\Jmc'$ and it is safe to set $M^{\Jmc'}=\emptyset$.

Let $\Imc$ be the restriction of $\Imc'$ to $X^{\Imc'}$. It can be verified that $\Imc'$ is a model of \Amc and that it satisfies all CIs in \Tmc
with the possible
            exceptions of the disjunctive CI $B \sqsubseteq B_1 \sqcup B_2$ and the
            CIs of the form $C \sqsubseteq \bot$. The former, however, is satisfied by
            Condition~1 and due to
            CIs~(\ref{blablaone}) and~(\ref{blablatwo}). And the latter is satisfied
            by Condition~2 and 
            due to CI~(\ref{lastone}).

            Since $\Imc' \not\models A'_0(a_0)$ and by (\ref{barbarlabel}), we have $\Imc \not\models A_0(a_0)$. To show that $\Kmc \not\models A_0(a_0)$, it thus remains
            to show that $\Imc$ is minimal w.r.t.\ $<_\CP$.
            Assume to the contrary that there is a model $\Jmc <_\CP \Imc$ of $\Kmc$
            and let $\Jmc'$ be the corresponding model of $\Kmc'$ constructed at
            the beginning of the correctness proof. Since $\Jmc'$ interprets
            the minimized concept names $D$, $M$, and $G$ in a minimal way among
            the models of $\Kmc'$, it follows from $\Jmc <_\CP \Imc$ that
            $\Jmc' <_{\CP'} \Imc'$, in contradiction to $\Imc'$ being minimal.
\end{proof}

\section{Proofs for Section~\ref{subsection-instance-dllite-combined}}

We now move to a proof of Theorem~\ref{thm-combined-lower-bool-instance}, which strongly combines \dllitebool expressiveness with fixed concept names.
Interestingly, fixed concept names are actually not needed.
Indeed, given a circumscribed $\dllitebool$ KB \circkb, each $A \in \Fsf$ can be replaced by the two minimized predicates $A$ and $\overline{A}$, where $\overline{A}$ is a fresh concept name set to be the complement of $A$ via the \dllitebool axioms $\overline{A} \incl \lnot A$ and $ \lnot A \incl \overline{A}$ that can be added in the TBox.
Concept names $A$ and $\overline{A}$ are then given higher preference than any other concept name from the original set of minimized concept names $\Msf$.
It is then straightforward to verify that the resulting circumscribed \dllitebool KB $\mn{Circ}_{\CP'}(\kb')$ has the same models as $\circkb$ (up to the interpretation of $\overline{A}$, which always equals the interpretation of $\lnot A$).

\thmcombinedlowerboolinstance*
\begin{proof}
	\newcommand{\ccenter}{\mathsf{Center}}
	\newcommand{\cbit}{\mathsf{Bit}}
	\newcommand{\cgoal}{\mathsf{Goal}}
	\newcommand{\callpairs}{\mathsf{AllPairs}}
	\newcommand{\cgoodcol}{\mathsf{GoodCol}}
	\newcommand{\cbadcol}{\mathsf{BadCol}}
	\newcommand{\cred}{\mathsf{R}}
	\newcommand{\cgreen}{\mathsf{G}}
	\newcommand{\cblue}{\mathsf{B}}
	\newcommand{\cgate}{G}
	\newcommand{\ctest}{\mathsf{T}}
	\newcommand{\istested}{\mathsf{Test}}

	The proof proceeds by reduction from the complement of the $\NExpTime$-complete \succinct\tcol\ problem. 
	An instance of \succinct\tcol\ consists of a Boolean circuit $\circuit$ with $2n$ input gates.
	The graph $\graph_\circuit$ encoded by $\circuit$ has $2^n$ vertices, identified by binary encodings on $n$ bits.
	Two vertices $u$ and $v$, with respective binary encodings $u_1 \dots u_n$ and $v_1 \dots v_n$, are adjacent in $\graph_\circuit$ iff $\circuit$ returns True when given as input $u_1 \dots u_n$ on its first $n$ gates and  $v_1 \dots v_n$ on the second half.
	The problem of deciding if $\graph_\circuit$ is 3-colorable has been proven to be $\NExpTime$-complete in \cite{Papadimitriou1986}.
	
	Let $\circuit$ be a Boolean circuit with $2n$ input gates.
	We build a \dllitebool cKB $\circkb$ with $\kb = (\tbox, \abox)$ s.t.\ the AQ $\cgoal(x)$ admits the individual $a$ as a certain answer iff $\circuit$ does not belong to \succinct\tcol.
	Let us already clarify that $\CP$ uses three minimized concept names, namely $\ccenter$, $\callpairs$ and $\cgoodcol$, with the following preference relation $\prec$:
	\begin{align}
		\ccenter \prec \callpairs \prec \cgoodcol
	\end{align}
	It also uses fixed concept names: each $\cbit_i$ for $1 \leq i \leq 2n$ and the six concept names $\cred_1$, $\cgreen_1$, $\cblue_1$, $\cred_2$, $\cgreen_2$ and $\cblue_2$, which will be useful to represent binary encodings and color assignments respectively.
	All other predicates vary.
	
	Models of $\kb$ contain a central element $a$, that we enforce via the ABox $\abox = \{ \ccenter(a) \}$ and which will be unique
	by virtue of the preference relation.
	This central element will further be used to detect undesired behaviors of the models via the AQ.
	In a model, each element represents a pair $(v_1, v_2)$ of vertices by its combination of fixed concepts $\cbit_i$ with $1 \leq i \leq 2n$ that corresponds to the binary encodings of $v_1$ and $v_2$.
	To detect models that do not represent all possible pairs, we require each element representing a pair $(v_1, v_2)$ to ``send'' at the center a binary sequence of $2n$ bits that must be different from the encoding of $(v_1, v_2)$.
	This is achieved by the following axioms in $\tbox$:
	\begin{align}
		\top \incl \; & \exists r_i \sqcup \exists \overline{r_i}
		\quad \text{ for each } 1 \leq i \leq 2n
		\\
		\exists r_i^- \incl \; & \ccenter
		\qquad \text{ for each } 1 \leq i \leq 2n
		\\
		\exists \overline{r_i}^- \incl \; & \ccenter
		\qquad \text{ for each } 1 \leq i \leq 2n
		\\
		\top \incl \; & \bigsqcup_{0 \leq i \leq 2n} (\exists r_i \sqcap \lnot \cbit_i) \sqcup (\exists \overline{r_i} \sqcap \cbit_i)
		\label{eq:send-bits}
	\end{align}
	We now require the number of sent sequences to be unique, if possible, via the minimized concept name $\callpairs$, triggered by the following axioms in $\tbox$:
	\begin{align}
		\exists r_i^- \sqcap \exists \overline{r_i}^- \incl \; & \callpairs \quad \text{ for each } 1 \leq i \leq 2n
	\end{align}
	Therefore, if (at least) a pair is not represented in a model, then \emph{all} elements can send the binary encoding of a same missing pair to avoid instantiating $\callpairs$.
	If, on the other hand, each pair is represented (at least once), then the number of sent sequences cannot be unique, and $\callpairs$ must be satisfied on the central element $a$.
	To detect this with the AQ, we add the following axiom to $\tbox$:
	\begin{align}
		 \lnot\callpairs \incl \; & \cgoal
	\end{align}
	Now that we can detect whether all pairs are represented, we move to the encoding of colors.
	The respective color of $v_1$ and $v_2$ is chosen locally via the following axioms in $\tbox$:
	\begin{align}
		\top \incl \; & \cred_1 \sqcup \cgreen_1 \sqcup \cblue_1
		\label{eq:exists-color-1}
		\\
		\top \incl \; & \cred_2 \sqcup \cgreen_2 \sqcup \cblue_2
		\label{eq:exists-color-2}
	\end{align}
	Recall concept names $\cred_1, \cgreen_1, \cblue_1$ and $\cred_2, \cgreen_2, \cblue_2$ are fixed.
	We additionally enforce that if an element represents a pair $(v_1, v_2)$ corresponding to an edge $\{ v_1, v_2 \}$ in $\graph_\circuit$, then it must assign different colors to $v_1$ and $v_2$.
	This is achieved by computing the output of $\circuit$ on input $v_1$ and $v_2$ via a fresh concept $\cgate$ for each gate $\cgate$ of $\circuit$ and adding the following axioms in $\tbox$:
	\begin{align}
		\cgate \equiv \; & \cbit_i \quad\qquad \text{ if } \cgate \text{ is the } i^{\text{th}} \text{ input gate of } \circuit
		\label{eq:consistent-circuits}
		\\
		\cgate \equiv \; & \cgate_1 \sqcap \cgate_2 \quad \text{ if } \cgate \text{ is a $\land$-gate with inputs } \cgate_1, \cgate_2
		\\
		\cgate \equiv \; & \cgate_1 \sqcup \cgate_2 \quad \text{ if } \cgate \text{ is a $\lor$-gate with inputs } \cgate_1, \cgate_2
		\\
		\cgate \equiv \; & \lnot \cgate' \quad\qquad \text{if } \cgate \text{ is a $\lnot$-gate with input } \cgate'
	\end{align}
	Assuming the output gate of $\circuit$ is $\cgate_{o}$, it now suffices to add the following axioms in $\tbox$:
	\begin{align}
		\cgate_{o} \sqcap \cred_1 \sqcap \cred_2 \incl \; & \bot
		\label{eq:consistent-edges-red}
		\\
		\cgate_{o} \sqcap \cgreen_1 \sqcap \cgreen_2 \incl \; & \bot
		\label{eq:consistent-edges-green}
		\\
		\cgate_{o} \sqcap \cblue_1 \sqcap \cblue_2 \incl \; & \bot
		\label{eq:consistent-edges-blue}
	\end{align}
	While the above allows ensuring that each element $e$ representing an edge $\{ v_1, v_2 \}$ assigns different colors to $v_1$ and $v_2$, there might still be inconsistent colorings in the sense that $v_1$ might have been assigned another color by some other element $e'$ representing an edge $\{ v_1, v_3 \}$.
	To detect such inconsistent colorings, the central element ``scans'' all colors assigned to a tested vertex, whose encoding corresponds to a binary sequence of $n$ bits represented by dedicated concepts $\ctest_i$.
	To do so, the interpretation of each $\ctest_i$ on the center element is first copied by all elements via the following axioms in $\tbox$, defined for each $1 \leq i \leq n$:
	\begin{align}
		\top \incl \; & \exists t_i \sqcup \exists \overline{t_i}
		\\
		\exists t_i^- \incl \; & \ccenter \sqcap \ctest_i
		&
		\exists \overline{t_i}^- \incl \; & \ccenter \sqcap \lnot \ctest_i
		\\
		\exists t_i \incl \; & \ctest_i
		&
		\exists \overline{t_i} \incl \; & \lnot\ctest_i
	\end{align}
	Each element representing a pair $(v_1, v_2)$ then compares whether the encoding of either $v_1$ or $v_2$ matches the sequence $\ctest_i$, via the following axioms:
	\begin{align}
		\bigsqcap_{1 \leq i \leq n} \left( (\cbit_i \sqcap \ctest_i) \sqcup (\lnot \cbit_i \sqcap \lnot \ctest_i) \right) \equiv \; & \istested_1
		\label{eq:test-1}
		\\
		\bigsqcap_{1 \leq i \leq n} \left( (\cbit_{n + i} \sqcap \ctest_i) \sqcup (\lnot \cbit_{n + i} \sqcap \lnot \ctest_i) \right) \equiv \; & \istested_2
	\end{align}
	It then sends the corresponding assigned color back to the central element via the following axioms in $\tbox$, defined for each $C \in \{ \cred, \cgreen, \cblue \}$:
	\begin{align}
	 	\istested_1 \sqcap C_1 \equiv \; & \exists s_C
	 	\\
	 	\istested_2 \sqcap C_2 \equiv \; & \exists s_C
	 	\\
	 	\exists s_C^- \incl \, & \ccenter
 	\end{align}
 	The center now detects whether it receives two different colors for the tested vertex:
 	\begin{align}
 		\bigsqcup_{\substack{ C, C' \in \{ \cred, \cgreen, \cblue \} \\ C \neq C'}} (\exists s_C^- \sqcap \exists s_{C'}^-) \equiv \; & \lnot \cgoodcol
	\label{eq:inconsistent-coloring}
 	\end{align}
	where $\cgoodcol$ is a minimized concept name, which forces the center element to detect an inconsistent coloring if it exists. 
	This is further translated in terms of the AQ with:
	 	\begin{align}
		\lnot \cgoodcol \incl \; & \cgoal
		\label{eq:bad-color-gives-goal}
	\end{align}
 	
 	We are now done with the construction of $\circkb$.
	It remains to prove that:
	\[
	\circkb \models \cgoal(a) \iff \circuit \notin \succinct\tcol.
	\]
	
	``$\Rightarrow$''.
	Assume $\circuit \in \succinct\tcol$ and let $\rho : \{ 0, \dots, 2^n - 1 \} \rightarrow \{ \cred, \cgreen, \cblue \}$ be a (legal) 3-coloring of the encoded graph $\graph_\circuit$.
	We build a model $\I$ of $\circkb$ based on $\rho$, whose domain $\domain{\I}$ contains $a$ and one fresh element $e_{v_1, v_2}$ per $(v_1, v_2) \in \{ 0, \dots, 2^n - 1\} \times \{ 0, \dots, 2^n -1 \}$.
	Model $\I$ further interprets concept and role names as follows:
	\begin{align*}
		{\ccenter}^\I & = \cgoodcol^\I = \callpairs^\I = \{ a \}
		\\
		\cbit_i^\I & = \{ e_{v_1, v_2} \mid i^\text{th} \text{ bit in } (v_1, v_2) \text{ encoding is } 1 \}
		\\
		C_i^\I & = \{ e_{v_1, v_2} \mid \rho(v_i) = C \}
		\\
		G & = \{ e_{v_1, v_2 } \mid G \in \circuit \text{ evaluates to } 1 \text{ on input } (v_1, v_2) \}
		\\
		\istested_i & = \{ e_{v_1, v_2} \mid v_i = 0 \}
		\\	
		\cgoal^\I & = \ctest_i^\I = \emptyset
		\\	
		r_i^\I & = \{ (e_{v_1, v_2}, a) \mid i^\text{th} \text{ bit in } (v_1, v_2) \text{ encoding is } 0 \}
		\\	
		\overline{r_i}^\I & = \{ (e_{v_1, v_2}, a) \mid i^\text{th} \text{ bit in } (v_1, v_2) \text{ encoding is } 1 \}
		\\
		t_i^\I & = \emptyset
		\\
		\overline{t_i}^\I & = \{ (e_{v_1, v_2}, a) \mid 0 \leq v_1, v_2 \leq 2^n -1 \}
		\\
		\begin{split}
		s_C & = \{ (e_{v_1, v_2}, a) \mid \rho(0) = C \\
		&\phantom{= \{ (e_{v_1, v_2}, a) \mid }\text{ and either } v_1 = 0 \text{ or } v_2 = 0 \}
		\end{split}
	\end{align*}
	Notice that we have arbitrarily chosen each element representing a pair to send back to the center it's own encoding in which all bits have been flipped (see interpretations of roles $r_i$ and $\overline{r_i}$), and the tested vertex to be $0$ (interpretations of each $\ctest_i$ being empty and $t_i$, $\overline{t_i}$ and $\ctest_i$ being set accordingly).
	With the two above remarks, it is easily verified $\I$ models $\kb$, and, by definition $a \notin \cgoal^{\I}$.
	It remains to verify $\I$ complies with $\CP$.
	By contradiction, assume there exists $\J <_\CP \I$.
	Notice that from the fixed predicates, $\J$ still encodes one instance of each possible pair and the very same coloring $\rho$.
	From $\J$ being a model of $\abox$ , we have $a \in \ccenter^\J$.
	From $\ccenter^\I = \{ a \}$ and $\ccenter$ being the most preferred minimized predicate, it follows that $\ccenter^\J = \{ a\}$ as otherwise we would have $\I <_\CP \J$.
	Therefore, since each possible pair is represented at least once, the $r_i$ and $\overline{r_i}$ mechanism ensures $a \in \callpairs^\J$.
	Together with $\callpairs^\I = \{ a\}$ and $\callpairs$ being second most preferred minimized predicates, we obtain $\callpairs^\J = \{ a \}$ (as otherwise, again, we would have $\I <_\CP \J$).
	Finally, since the coloring encoded in $\J$ must be exactly $\rho$, whatever the interpretations of concepts $\ctest_i$, the concept $\exists s_{C}^- \sqcap \exists s_{C'}^-$ will always be empty for all $C, C' \in \{ \cred, \cgreen, \cblue \}$ s.t.\ $C \neq C'$.
	Therefore, $a \in \cgoodcol^\J$.
	Now, from $\cgoodcol$ being also minimized and $\cgoodcol^\I = \{ a \}$, we obtain $\cgoodcol^\J = \{ a\}$.
	Overall $\J$ interprets the minimized predicates as $\I$ does, hence the desired contradiction of $\J <_\CP \I$.
		
	``$\Leftarrow$''.
	Assume $\circuit \notin \succinct\tcol$.
	Consider a model $\I$ of $\circkb$.
	First notice that $\ccenter^\I = \{ a \}$.
	Indeed, $\I$ models $\abox$ hence $a \in \ccenter^\I$, and, if there were $b \in \ccenter^\I$ with $b \neq a$, then we could find a model $\J$ of $\kb$ with $\J <_\CP \I$ by:
	\begin{enumerate}
		\item take $\domain{\J} = \domain{\I}$ and preserve all interpretations of fixed concept names;
		\item setting $\ccenter^\J = \{ a \}$;
		\item replace every $(e, e') \in r_i^\I$ by $(e, a) \in r_i^\J$ and same for $\overline{r_i}$;
		\item define accordingly $\callpairs^\J$;
		\item set for all $e \in \domain{\J}$: $e\in \ctest_i^\J$ iff $a \in \ctest_i^\I$;
		\item set $t_i^\J = \{ (e, a) \mid e \in \ctest_i^\J \}$ and $\overline{t_i}^\J = \{ (e, a) \mid e \notin \ctest_i^\J \}$;
		\item define accordingly interpretations of concepts evaluating the circuit and $\istested_1^\J$, $\istested_2^\J$;
		\item set $s_C^\J = \{ (e, a) \mid e \in ((\istested_1 \sqcap C_1) \sqcup (\istested_2 \sqcap C_2))^\J \}$;
		\item define accordingly $\cgoodcol^\J$ and $\cgoal^\J$.
	\end{enumerate}
	In particular, notice Step~1 removes $b$ from the interpretation of $\ccenter$.
	
	We now want to prove that if a pair $(v_1, v_2)$ is not represented in $\I$, that is:
	\[
	\bigcap_{i \in \text{Ones}(v_1, v_2)} \cbit_i^\J \cap \bigcap_{i \in \text{Zeros}(v_1, v_2)} (\lnot \cbit_i)^\J = \emptyset,
	\]
	where $\text{Ones}(v_1, v_2)$ is the set of $0 \leq i \leq 2n -1$ s.t.\ the $i^\text{th}$ bit in the binary encoding of $(v_1, v_2)$ is $1$ and $\text{Zeros}(v_1, v_2)$ is its complement, then the query must be satisfied.
	Indeed, if such a $(v_1, v_2)$ exists, then we must have $\callpairs^\J = \emptyset$, thus $a \in \cgoal^\J$, as otherwise we can obtain a model $\J$ of $\kb$ s.t.\ $\J <_\CP \I$ by:
		\begin{enumerate}
		\item take $\domain{\J} = \domain{\I}$ and preserve all interpretations of fixed concept names and $\ccenter$;
		\item set $r_i^\J = \domain{\I} \times \{ a \}$ if $i \in \text{Ones}(v_1, v_2)$ and $r_i^\J = \emptyset$ otherwise; set $\overline{r_i}^\J = \domain{\I} \times \{ a \}$ if $i \in \text{Zeros}(v_1, v_2)$ and $\overline{r_i}^\J = \emptyset$ otherwise;
		\item follow steps 3 to 8 from the previous construction of $\J$.
	\end{enumerate}
	In particular, Step~2 of the above complies with Equation~\ref{eq:send-bits} as $(v_1, v_2)$ is not represented in $\I$.
	It further implies that $\callpairs^\J = \emptyset$, which ensures the desired contradiction $\J <_\CP \I$.
	
	We are thus left with the case in which each possible pair $(v_1, v_2)$ is represented at least once in $\I$.
	From Equations~\ref{eq:exists-color-1} and \ref{eq:exists-color-2}, we know each element representing a pair assigns colors to its $v_1$ and $v_2$.
	From Equations~\ref{eq:consistent-circuits} to \ref{eq:consistent-edges-blue}, these colors cannot be the same if there is an edge $\{ v_1, v_2 \}$ in the graph $\graph_\circuit$.
	Thus, since we assume $\circuit \notin \succinct\tcol$, the overall choices of colors cannot be consistent: there must exist a vertex $v_0$ which is assigned a color $C_1$ by an element $e_1$ representing a pair $p_1$ and a different color $C_2$ by an element $e_2$ representing a pair $p_2$.
	Let us assume $p_1$ has shape $(v_0, v_1)$ and $p_2$ has shape $(v_0, v_2)$ (other cases works similarly).
	We can now prove that $\cgoodcol^\I = \emptyset$, thus $a \in \cgoal$, as otherwise we could find a model $\J$ of $\kb$ with $\J <_\CP \I$ by:
	\begin{enumerate}
		\item take $\domain{\J} = \domain{\I}$ and preserve all interpretations of fixed concept names, of $\ccenter$ and $\callpairs$, and of roles $r_i$ and $\overline{r_i}$;
		\item set for all $e \in \domain{\J}$: $e \in \ctest_i^\J$ iff $i \in \text{Ones}(v_0)$;
		\item set $t_i^\J = \{ (e, a) \mid e \in \ctest_i^\J \}$ and $\overline{t_i}^\J = \{ (e, a) \mid e \notin \ctest_i^\J \}$;
		\item define accordingly interpretations of concepts evaluating the circuit and $\istested_1^\J$, $\istested_2^\J$;
		\item set $s_C^\J = \{ (e, a) \mid e \in ((\istested_1 \sqcap C_1) \sqcup (\istested_2 \sqcap C_2))^\J \}$;
		\item define accordingly $\cgoodcol^\J$ and $\cgoal^\J$.
	\end{enumerate}
	In particular, Equation~\ref{eq:test-1} and Step~4 in the above ensure $e_1, e_2 \in \istested_1^\J$.
	Therefore Step~5 above yields $(e_1, a) \in s_{C_1}^\J$ and $(e_2, a) \in s_{C_2}^\J$, which triggers Equation~\ref{eq:inconsistent-coloring} in Step~6 and ensures $\cgoodcol^\J = \emptyset$.
	This yields the desired contradiction $\J <_\CP \I$.
	Thus, $\cgoodcol^\I = \emptyset$, in particular Equation~\ref{eq:bad-color-gives-goal} yields $a \in \cgoal^\I$.

      \end{proof}

We work towards a proof of Theorem~\ref{thm-combined-upper-hornh-instance}.
Assume given a model $\I$ of a $\dllitehornh$ cKB $\circkb$.
We ``forget'' about some parts of $\I$, only retaining which roles may be needed due to a combination of fixed predicates in the forgotten part.
To  this end, we define the fixed-type $\ftypeinof{\I}{e}$ of an element $e \in \domain{\I}$ as the set of \dllite concepts $C$ such that $\tbox \models ( \bigsqcap_{e \in F^\I, \ F \in \Fsf} F ) \incl C$.
Given a role $r \in \posroles$, we say that $r$ is \emph{forced} in $\I$ if there exists an element $f_r \in \domain{\I}$ such that $\exists r \in \ftypeinof{\I}{f_r}$.
For each forced role in $\I$, we assume chosen such an element $f_r$.
Similarly, for each role $r \in \posroles$, we assume chosen an element $w_r \in (\exists r^-)^\I$ if $(\exists r^-)^\I \neq \emptyset$.
We now construct an interpretation ${\I_0}$ whose domain $\domain{\I_0}$ consists in $\indsof{\abox}$, all chosen $f_r$ and all chosen $w_r$.
The interpretation $\I_0$ is now defined as:
\begin{align*}
	\cstyle{A}^{\I_0} = \; & \cstyle{A}^{\I} \cap \domain{\I_0}
	\\
	\rstyle{r}^{\I_0} = \; & \rstyle{r}^{\I} \cap (\indsof{\abox} \times \indsof{\abox}) 
	\\
	& \cup \{ (e, w_s) \mid e \in (\exists s)^\I \cap \domain{\I_0}, \tbox \models s \incl r \}
	\\
	& \cup \{ (w_s, e) \mid e \in (\exists s)^\I \cap \domain{\I_0}, \tbox \models s \incl r^- \}
\end{align*}

\begin{remark}
	Notice $\I_0$ is a special case of the models $\I_\Pmc$ considered in Section~\ref{subsection-dllite-combined}, and that $\sizeof{\I_0} \leq \sizeof{\abox} + 2\sizeof{\tbox}$.
\end{remark}

The key result is now the following:

\begin{lemma}
	\label{lemma-linear-countermodel}
	Let $q$ be an AQ.
	If $\I$ is a countermodel for $q$ over $\circkb$, then so is $\I_0$.
\end{lemma}

\begin{proof}
	Assume $\I$ is a countermodel for $q$ over $\circkb$.
	Setting $\Pmc = \{ f_r \mid r \text{ is forced in } \I \}$, we have $\I_\Pmc = \I_0$ (see Section~\ref{subsection-dllite-combined}) and thus Lemma~\ref{lemma-portion-are-models} ensures $\I_0 \models \kb$.
	Since concept interpretations in $\I_0$ are inherited from $\I$, it is clear $\I_0$ does not entail $q$.
	It remains to prove $\I_0$ also complies with $\CP$.
	By contradiction assume there exists $\J_0$ a model of $\kb$ with $\J_0 <_\CP \I_0$.
	Notice that, for each forced role $r$ in $\I$, there must exists an element $w_r' \in (\exists r^-)^{\J_0}$ as we kept element $f_r \in \domain{\I_0}$ and $r$ is a consequence of the fixed predicates on $f_r$ that must have been preserved in $\J_0$ (see Condition~2 from the definition of $<_\CP$).
	We assume chosen such an element $w_r'$ per forced role and build an interpretation $\Jmc$:
	\begin{align*}
			\domain{\J} = \; & \domain{\I}
			\\
			\cstyle{A}^{\J} = \; & \cstyle{A}^{\J_0} \cup \{ e \in \domain{\I} \setminus \domain{\J_0} \mid A \in \ftypeinof{\I}{e} \}
			\\
			\rstyle{r}^{\J} = \; & \rstyle{r}^{\J_0} 
			\cup \left\{ (e, w_s') \; \middle| \begin{array}{l} e \in \domain{\I} \setminus \domain{\J_0} \\ \exists s \in \ftypeinof{\I}{e}, \tbox \models s \incl r \end{array} \right\}
			\\
			\; & \phantom{\rstyle{r}^{\J_0} } \cup \left\{ (w_s', e) \; \middle| \begin{array}{l} e \in \domain{\I} \setminus \domain{\J_0} \\ \exists s \in \ftypeinof{\I}{e}, \tbox \models s \incl r^- \end{array} \right\}
		\end{align*}
	It is easily verified that $\J$ is a model of $\kb$, and we now prove $\J <_\CP \I$, which will contradict $\I \models \circkb$.
	We first notice a useful property: for all concept name $A$, we have $A^\J \cap (\domain{\I} \setminus \domain{\I_0}) \subseteq  A^\I \cap (\domain{\I} \setminus \domain{\I_0})$ ($\star$).
	Indeed, if $e \in A^\J \cap (\domain{\I} \setminus \domain{\I_0})$, then $A \in \ftypeinof{\I}{e}$, and therefore $e \in A^\I$.
	The converse does not hold in general.
	We now check that all four conditions from the definition of $<_\CP$ are satisfied:
	\begin{enumerate}
			\item 
			By definition, we have $\domain{\J} = \domain{\I}$.
			\item 
			Let $A \in \Fsf$. Definitions of $\I_0$ and $\J$ ensure $A^\I \cap {\domain{\I_0}} = A^{\I_0}$ and $A^\J \cap {\domain{\I_0}} = A^{\J_0}$.
			From $\J_0 <_\CP \I_0$, we get $A^{\I_0} = A^{\J_0}$, which yields $A^\I \cap {\domain{\I_0}} = A^\J \cap {\domain{\I_0}}$.
			For an element $e \in \domain{\I} \setminus \domain{\I_0}$, we remark: $e \in A^\I$ iff $A \in \ftypeinof{\I}{e}$ iff $e \in A^\J$.
			Altogether, we obtain $A^\I = A^\J$.
			\item 
			Let $A \in \Msf$ such that $A^\J \not\subseteq A^\I$.
			Let thus $e \in A^\J \setminus A^\I$.
			Notice $e$ must belong to $\domain{\I_0}$, as otherwise $\star$ yields $e \in A^\I$.
			It follows that $e \in A^{\J_0} \not\subseteq A^{\I_0}$.
			Since $\J_0 <_\CP \I_0$, there exists $B \prec A$ s.t.\ $B^{\J_0} \subsetneq B^{\I_0}$.
			By $\star$, this extends into $B^{\J} \subsetneq B^{\I}$ and we found $B$ as desired.
			\item
			From $\J_0 <_\CP \I_0$, there exists $A \in \Msf$ s.t.\ $A^{\J_0} \subsetneq A^{\I_0}$ and for all $B \prec A$, $B^{\J_0} = B^{\I_0}$.
			Joint with $\star$, it follows that $A^{\J} \subsetneq A^{\I}$ and for all $B \prec A$, $B^{\J} \subseteq B^{\I}$.
			If ever there exists such a $B \prec A$, such that additionally $B^{\J} \subsetneq B^{\I}$, then we select a minimal such $B$ w.r.t.\ $\prec$, which provides the desired minimized concept.
			Otherwise, $A$ fits.
		\end{enumerate}
	It $\J <_\CP \I$, contradicting $\I \models \circkb$ as desired.
\end{proof}

\thmcombinedupperhornhinstance*
\begin{proof}
	Guess a candidate interpretation $\I$ with size $\sizeof{\abox} + 2\sizeof{\tbox}$.
	Check whether it is a model of $\kb$ and if the AQ of interest is \emph{not} satisfied in $\I$. 
	Using a $\NPclass$ oracle, check whether $\I$ is a model of $\circkb$.
	If all the above tests succeed, then accept; otherwise, reject.
	From Lemma~\ref{lemma-linear-countermodel}, it is immediate that there exists an accepting run if there exists a countermodel for the AQ over $\circkb$.
	Conversely, if a run accepts, then the guessed model $\I$ is a countermodel.
\end{proof}


\lemmadataindep*

\begin{proof}
	\newcommand{\typeref}{\mathsf{ref}}
	``$\Rightarrow$''.
	Assume $\circkb \models A_0(a_0)$ and let $\I'$ be a model of $\Circ(\kb')$.
	
	We extend $\I'$ into an interpretation $\I$ whose domain is $\domain{\I} = \domain{\I'} \uplus (\indsof{\abox} \setminus \indsof{\abox'})$.
	To define a suitable interpretation of concepts on an individual $a$ from $\Delta = \indsof{\abox} \setminus \indsof{\abox'}$, we consider its ABox-type $t = \atypeinof{\abox}{a}$.
	Notice $m_{t} = 4^\sizeof{\tbox}$ as otherwise we would
        have $a \in  \indsof{\abox'}$.
	Consider now the types in $\I$ of elements $a_{t, 1}, \dots a_{t, 4^\sizeof{\tbox}}$.
	Since there are $4^\sizeof{\tbox}$ such elements for only $2^{\sizeof{\tbox}}$ possible types, there exists a type $t'$ realized in $\I$ at least $2^\sizeof{\tbox}$ times.
	We now chose an element $\typeref(t) \in \domain{\I'}$ s.t.\ $\atypeinof{\abox}{\typeref(t)} = t$ and $\typeinof{\I'}{\typeref(t)} = t'$.
	Such an element $\typeref(t)$ serves as a reference to interpret concepts and roles on individuals $a \in \Delta$ with ABox type $t$, as the following construction shows:
	\[
	\begin{array}{r@{\;}c@{\;}l}
		A^\I & = & A^{\I'} \; \cup \; \{ a \mid a \in \Delta, \typeref(\atypeinof{\abox}{a}) \in A^{\I'} \} 
		\smallskip \\
		p^\I & = & p^{\I'} \; \cup \; \{ (a, b) \mid \kb \models p(a, b) \}
		\smallskip \\
		& &  
		\cup \; \{ (a, e) \mid a \in \Delta, (\typeref(\atypeinof{\abox}{a}, e) \in p^{\I'} \}
		\smallskip \\
		& &
		\cup \; \{ (e, a) \mid a \in \Delta, (e, \typeref(\atypeinof{\abox}{a}) \in p^{\I'} \}.
	\end{array}
	\]
	It is easily verified that $\I$ is a model of $\kb$,
        particularly that it satisfies all role assertions from $\abox$. 
	We now prove $\I \models \circkb$, which, by hypothesis, gives $\I \models A_0(a_0)$, and thus $\I' \models A_0(a_0)$ by definition of $A_0^\I$.
	Assume by contradiction one can find a model $\J$ of $\kb$ s.t.\ $\J <_\CP \I$.
	We construct $\J' <_\CP \I'$ which will contradict $\I'$ being a model of $\Circ(\kb')$.
	Intuitively, we proceed as in Lemma~\ref{lem-lemma5} and simulate $\J$ in $\J'$.
	For each ABox type $t$ s.t.\ there exists $a \in \Delta$ with $\atypeinof{\abox}{a} = t$, we set:
	\begin{align*}
	D_t = & \; \{ d \in \domain{\I'} \mid \atypeinof{\abox}{d} = t, \atypeinof{\I'}{d} = \typeinof{\I'}{\typeref(t)} \}
	\smallskip \\
	S_t = & \; \{ \typeinof{\J}{d} \mid \atypeinof{\abox}{d} = t, \atypeinof{\I'}{d} = \typeinof{\I'}{\typeref(t)} \}
	\end{align*}
	From the definition of $\typeref(t)$, we have $\sizeof{D_t} \geq 2^\sizeof{\tbox}$, while, clearly, $\sizeof{S_t} \leq 2^\sizeof{\tbox}$.
	Since $S_t \neq \emptyset$ (it contains e.g.\ $\typeinof{\J}{\typeref(t)}$), we can find a surjective function $\pi_t : D_t \rightarrow S_t$.
	We consider $\pi$ the union of all such $\pi_t$, that is with $t$ s.t.\ there exists $a \in \Delta$ with $\atypeinof{\abox}{a} = t$.
	We further extend the domain of definition of $\pi$ to $\domain{\I'}$ by setting $\pi(d) = \typeinof{\J}{d}$ for the remaining elements of $\domain{\I'}$.
	We now define the interpretation $\J'$ as:
	\[
	\begin{array}{r@{\;}c@{\;}l}
		\domain{\J'} & = & \domain{\I'}
		\smallskip \\
		A^{\J'} & = & \{ d \in \domain{\I'} \mid A \in \pi(d) \} 
		\smallskip \\
		\rstyle{r}^{\J'}      & =	& \{ (d,e) \in \Delta^{\Imc'}\times \Delta^{\Imc'} \mid \pi(d) \rightsquigarrow_\rstyle{r} \pi(e) \}.
	\end{array}
	\]
	It can now be verified that $\J'$ is a model of $\kb'$ such that $\J' <_\CP \I'$.
	
	 ``$\Leftarrow$''.
	 Assume $\Circ(\kb') \models A_0(a_0)$ and let $\I$ be a model of $\circkb$.
	 Here again, we denote $\Delta = \indsof{\abox} \setminus \indsof{\abox'}$.
	 By definition of $W$, all elements from $\Delta$ have an ABox type $t$ that is already realized at least $4^\sizeof{\tbox}$ times in $W$.
	 Therefore, we can define a permutation $\sigma$ on $\domain{\I}$ such that for all $d \in \domain{\I}$, the following conditions are respected:
	 \begin{enumerate}
	 	\item 
	 	$\sigma(a_0) = a_0$;
	 	\item 
	 	$\atypeinof{\abox}{\sigma(d)} = \atypeinof{\abox}{d}$;
	 	\item 
	 	If $d \in \Delta$, then there exists an element $\typeref(d) \in W$ with $\atypeinof{\abox}{\typeref(d)} = \atypeinof{\abox}{d}$ and $\typeinof{\I}{\typeref(d)} = \typeinof{\I}{d}$.
	 \end{enumerate} 
	 We can then define an interpretation $\I_\sigma$ as follows:
	 \[
	 \begin{array}{r@{\;}c@{\;}l}
	 	\domain{\I_\sigma} & = & \domain{\I}
	 	\smallskip \\
	 	A^{\I_\sigma} & = & \sigma(A^\I)
	 	\smallskip \\
	 	\rstyle{r}^{\I_\sigma}     & =	& \{ (d,e) \mid \kb \models r(d, e) \} \; \cup \; (\sigma \times \sigma)(r^\I).
	 \end{array}
	 \]
	 It is straightforward that $\I_\sigma$ models $\kb$ (notably relying on Condition~2), that preserves interpretation of concepts on the individual $a_0$ of interest (from Condition~1).
	 Additionally, $\I_\sigma$ models $\circkb$ as otherwise a model $\J$ of $\kb$ with $\J <_\CP \I_\sigma$ immediately yields a model $\J_{\sigma^{-1}}$ of $\kb$ with $\J_{\sigma^{-1}} <_\CP \I$, contradicting $\I$ being a model of $\circkb$.
	 
	 We extend the mapping $\typeref$ from Condition~3 to the identity for elements from $\domain{\I} \setminus \Delta$, and define an interpretation $\I'$ as follows:
	\[
	\begin{array}{r@{\;}c@{\;}l}
		\domain{\I'} & = & \domain{\I} \setminus \Delta
		\smallskip \\
		A^{\I'} & = & A^{\I_\sigma} \cap \domain{\I'}
		\smallskip \\
		\rstyle{r}^{\I'}    & = & (\typeref \times \typeref)(r^{\I_\sigma} \cap (\domain{\I'} \times \domain{\I'})).
	\end{array}
	\]
	 It can be verified that $\I'$ is a model of $\kb'$.
	 We now prove it is also a model of $\Circ(\kb')$, which, by hypothesis, gives $\I' \models A_0(a_0)$, and thus $\I \models A_0(a_0)$ as unfolding definitions gives $a_0 \in A_0^\I$ iff $a_0 \in A_0^{\I_\sigma}$ iff $a_0 \in A_0^{\I'}$ (notably via Condition~1 on $\sigma$).
	 Assume by contradiction one can find a model $\J'$ of $\kb'$ with $\J' <_\CP \I'$.
	 We now construct $\J <_\CP \I$ which will contradict $\I$ being a model of $\circkb$.
	 Relying on the mapping $\typeref$, we set:
	 \[
	 \begin{array}{r@{\;}c@{\;}l}
	 	\domain{\J} & = & \domain{\I}
	 	\smallskip \\
	 	A^{\J} & = & \typeref^{-1}(A^{\J'})
	 	\smallskip \\
	 	\rstyle{r}^{\J}    & = & \{ (d,e) \mid \kb \models r(d, e) \} \; \cup \; (\typeref^{-1} \times \typeref^{-1})(r^{\J'}).
	 \end{array}
	 \]
	 It is easily checked that $\J$ models $\kb$.
	 From Condition~3 on $\sigma$, it then follows $\J <_\CP \I$, ie the desired contradiction.
\end{proof}

\section{Proofs for Section~\ref{subsection-negative-ri}}

\thmdatalowercoreh*
\begin{proof}
	
	\renewcommand{\cstyle}[1]{{\mathsf{#1}}}
	\renewcommand{\rstyle}[1]{{\mathsf{#1}}}
	\newcommand{\ixvar}{{x}} 
	\newcommand{\iyvar}{{y}} 
	\newcommand{\itrue}{t} 
	\newcommand{\itrueof}[1]{\itrue_{#1}} 
	\newcommand{\ifalse}{f}
	\newcommand{\ifalseof}[1]{\ifalse_{#1}} 
	\newcommand{\iclause}{c} 
	
	\newcommand{\cxvar}{\cstyle{XVar}} 
	\newcommand{\cyvar}{\cstyle{YVar}} 
	\newcommand{\cclause}{\cstyle{Clause}} 
	\newcommand{\cformula}{\cstyle{Formula}} 
	
	\newcommand{\rxeval}{\rstyle{eval}_{\cstyle{X}}}
	\newcommand{\ryeval}{\rstyle{eval}_{\cstyle{Y}}}
	\newcommand{\rceval}{\rstyle{eval}_{\cstyle{C}}}
	\newcommand{\reval}{\rstyle{eval}}
	\newcommand{\rnegeval}{\overline{\reval}}
	
	\newcommand{\cxval}{\cstyle{XVal}}
	\newcommand{\cyval}{\cstyle{YVal}}
	\newcommand{\ccval}{\cstyle{CVal}}
	
	\newcommand{\cvval}{\cstyle{VVal}}
	\newcommand{\cfval}{\cstyle{FVal}}
	
	\newcommand{\iformula}{f}
	\newcommand{\rfeval}{\rstyle{eval}_{\cstyle{F}}}
	
	\newcommand{\cgoal}{\cstyle{Goal}}

	We give a polynomial time reduction from
        $\forall\exists\mn{3SAT}$, c.f.\ the proof
        of Theorem~\ref{thm-data-lower-el}. Let a
        $\forall\exists$-3CNF sentence
        $\forall \bar x \exists \bar y \, \vp$ be given where
        $\bar x=x_1 \cdots x_m$, $\bar y = y_1 \cdots y_n$, and
        $\varphi=\bigwedge^\ell_{i=1} \bigvee_{j=1}^3 L_{ij}$ with
        $L_{ij} = v$ or $L_{ij} = \lnot{v}$ for some
        $v \in \{ x_1, \dots x_m, y_1, \dots, y_n \}$.  We construct a
        circumscribed \dllitecoreh KB $\circkb$ and an
        atomic query $\query$ such that $\circkb \models \query$ iff
        $\forall \bar x \exists \bar y \, \varphi$ is true.

        The circumscription pattern \CP involves four minimized
        concept names with the preference
        $$\cvval \prec \cxval \prec \ccval \prec \cfval,$$ and no
        fixed concept names. 
        We now
        describe the construction of the KB $\kb=(\Tmc,\Amc)$, not
        strictly
        separating \Tmc from \Amc. We first introduce an ABox
        individual for each variable, marking the existential
        variables
        with the concept name $\cxvar$ and the universal ones with
        $\cyvar$:
	\begin{align}
		\cxvar(\ixvar)
		& \quad \text{for all } x \in \bar x
		\label{eq:xvars} \\
		\cyvar(\iyvar)
		& \quad \text{for all } y \in \bar y.
		\label{eq:yvars}
	\end{align}
        To choose truth values for variables, we use the minimized concept name $\cvval$.
		We introduce two instances of $\cvval$ for each variable, one representing true and the other false:
        	\begin{align}
		\cvval(\itrueof{v})
		& \quad \text{for all } v \in \bar x \cup \bar y
		\label{eq:vvalstrue} \\
		\cvval(\ifalseof{v})
		& \quad \text{for all } v \in \bar x \cup \bar y.
		\label{eq:vvalsfalse}
	\end{align}
        Each variable must choose an instance of $\cvval$ via the role
        names $\rxeval$ an $\ryeval$, depending on whether it is
        existential
        or universal:
	\begin{align}
		\cxvar & \sqsubseteq \, \exists \rxeval
		\label{eq:xvariables_have_value} \\
		\cyvar & \sqsubseteq \, \exists \ryeval
		\label{eq:yvariables_have_value}\\
		\exists {\rxeval}^- & \sqsubseteq \, \cvval
		\label{eq:xvalues_are_vvalues} \\
		\exists {\ryeval}^- & \sqsubseteq \, \cvval.
		\label{eq:yvalues_are_vvalues}
	\end{align}
        There is no guarantee yet, however, that a variable $v$ chooses one of the instances $\itrueof{v}$ and $\ifalseof{v}$ of $\cvval$ reserved for it.
		To ensure this, we first install
        the following role inclusions:
		\begin{align}
		\rxeval  & \sqsubseteq \, \reval
		\label{eq:xeval_is_eval} \\
		\ryeval  & \sqsubseteq \, \reval
		\label{eq:yeval_is_eval} \\
		\reval & \sqsubseteq \, \lnot \rnegeval.
		\label{eq:eval_is_not_negeval}
	\end{align}
	The negative role inclusion allows us to flexibly control the targets for the existential restrictions in CIs~(\ref{eq:xvariables_have_value}) and~(\ref{eq:yvariables_have_value}), as follows:
	\begin{align}
		\rnegeval(v, \itrueof{v'})
		& \quad \text{for all } v, v' \in \bar x \cup \bar y
           \text{ with }  v \neq v'
		\label{eq:varblock} \\
		\rnegeval(v, \ifalseof{v'})
		& \quad \text{for all } v, v' \in \bar x \cup \bar y \text{ with }  v \neq v'.
		\label{eq:varblocktwo}
	\end{align}
	To force the truth values of the variables in $\bar x$ to be
        identical in all models that are smaller w.r.t.\
        `$<_\CP$,' we mark those truth values with the minimized
        concept name
        $\cxval$:
	\begin{align}
          \exists {\rxeval}^- & \sqsubseteq \, \cxval.
                                \label{eq:implxval}
	\end{align}
	We next introduce an individual for each clause:
	\begin{align}
		\cclause(\iclause_i)
		& \quad \text{for } 1 \leq i \leq \ell.
		\label{eq:clauses}
	\end{align}
	We also assign a truth value to every clause via the role name $\rceval$, another subrole of $\reval$.
	Every element with an incoming $\rceval$-edge must make the concept name $\ccval$ true, which is minimized, but with lower priority than $\cvval$ and $\cxval$.
	This reflects the fact that the truth values of variables determine the truth values of clauses. We put:
	\begin{align}
		\cclause & \sqsubseteq \, \exists \rceval
		\label{eq:clauses_have_value'} \\
		\rceval  & \sqsubseteq \, \reval
		\label{eq:ceval_is_eval} \\
		\exists {\rceval}^- & \sqsubseteq \, \ccval
		\label{eq:rceval_is_ccval}
	\end{align}
	The truth values of clauses are represented by the same
        individuals
        as the truth values of variables:
	\begin{align}
	\ccval(\itrueof{v})
	& \quad \text{for all } v \in \bar x \cup \bar y
	\label{eq:cvalstrue} \\
	\ccval(\ifalseof{v})
	& \quad \text{for all } v \in \bar x \cup \bar y
	\label{eq:vvalsfalse'} 
	\end{align}
	We again use the role $\rnegeval$ to control which instance of
        $\cclause$ can be used as truth values for which clause. This
        is based on the literals that occur in the clause.  If a
        variable $v$ occurs positively in the $i^{th}$ clause, then
        $\iclause_i$ has access to $\ifalseof{v}$, while if $v$ occurs
        negatively, then $\iclause_i$ has access to $\itrueof{v}$. And
        those are the only instances of $\cvval$ to which $\iclause_i$
        has access: for $1 \leq i \leq \ell$ and all $v \in \bar x
        \cup \bar y$,
        put
	\begin{align}
		\rnegeval(\iclause_i, \itrueof{v})
		& \quad \text{if } \lnot v \notin \{ L_{i, 1}, L_{i, 2}, L_{i, 3} \}
		\label{eq:clause_block_but_positive} \\
		\rnegeval(\iclause_i, \ifalseof{v})
		& \quad \text{if }  v \notin \{ L_{i, 1}, L_{i, 2}, L_{i, 3} \}.
		\label{eq:clause_block_but_negative}
	\end{align}
        Note that clause individual 
        $\iclause_i$ has access via $\rceval$ to $\itrueof{v}$ if
        making $v$ \emph{false} leads to satisfaction of the $i^{th}$
        clause, and likewise for $\ifalseof{v}$ and making $v$
        \emph{true}. Moreover, $v$ being made true is represented by
        an $\rxeval$- or $\ryeval$-edge from $v$ to $\itrueof{v}$ and
        likewise for $v$ being made false and $\ifalseof{v}$. Also
        recall that $\itrueof{v}$ and $\ifalseof{v}$ cannot have such
        incoming edges from elsewhere. Consequently, the $i^{th}$
        clause evaluates to true if we can find a target for the
        $\rceval$-edge that does \emph{not} have an incoming
        $\rxeval$- or $\ryeval$-edge:
	\begin{align}
		\exists {\rxeval}^- & \sqsubseteq \, \lnot \exists {\rceval}^-
		\label{eq:disjone} \\	
		\exists {\ryeval}^- & \sqsubseteq \, \lnot \exists {\rceval}^-.
		\label{eq:disjtwo}
	\end{align}
	In the case that a clause evaluates to false, it cannot reuse
        one of the admitted instances of $\ccval$. We introduce an
        extra instance $f$ of $\ccval$ that can be used instead:
		\begin{align}
		\ccval(\iformula).
		\label{eq:iformula_is_cval}
	\end{align}
	To ensure that $f$ is indeed used as an $\rceval$-target only if at least one clause is falsified, we use the fourth minimized concept name $\cfval$.
	Note that it is minimized with the least priority.
	We make sure that $f$ must have an
        outgoing $\rfeval$-edge leading to an instance of $\cfval$,
        that $f$ itself is an instance of $\cfval$, and that having an
        incoming $\rfeval$-edge precludes having an incoming
        $\rceval$-edge. If all clauses evaluate to true, then there is
        no need to use $f$ as a target for $\rceval$ and $f$ can use
        itself as the $\rfeval$-target. But if a clause evaluates to
        false, then we must use $f$ as a last resort $\rceval$-target
        and cannot use it as an $\rfeval$-target. Since $\cfval$ is
        minimized with least priority, the latter will simply lead to a fresh instance of $\rfeval$
       to be created. In summary, all clauses evaluate to true
        if and only of $f$ has an incoming $\rfeval$-edge.  We put:
	\begin{align}
		\cformula(\iformula)
		& \qquad \cfval(\iformula)
		\label{eq:iformula_is_fval} \\
		\cformula & \sqsubseteq \, \exists \rfeval
		\\
		\exists {\rfeval}^- & \sqsubseteq \, \cfval
		\\
		\exists {\rceval}^- & \sqsubseteq \, \lnot \exists {\rfeval}^-.
		\label{eq:rcevalfcevaldisj}
	\end{align}
        By what was said above, to finish the reduction, it suffices to add
        \begin{align}
          \exists {\rfeval}^- & \sqsubseteq \, \cgoal.
	\end{align}
        to choose as the query $q=\cgoal(x)$, and to ask whether $f$ is an answer.

        To prove correctness, we thus have to show the following.
        \\[2mm]
        {\bf Claim.} $\circkb \models q(f)$ iff
        $\forall \bar x \exists \bar y \, \varphi$ is true.
        \\[2mm]
        To prepare for the proof of the claim, we first observe that
        every valuation $V$ for $\bar x \cup \bar y$ gives rise
        to a corresponding model $\Imc_V$ of \Kmc. We use domain 
        $$
          \Delta^{\Imc_V} = \Ind(\Amc) 
        $$
        and set
        $$
        A^{\Imc_V} = \{ a \mid A(a) \in \Amc \}
        $$
        for all concept names $$A \in \{ \cxvar, \cyvar, \cvval,
        \ccval, \cclause, \cformula \}.$$
        We interpret $\rxeval$ and $\ryeval$ according to $V$,
        that is
        $$
        \begin{array}{rcl}
          \rxeval^{\Imc_V} &=& \{ (x,\itrueof{v}) \mid x \in \bar x \text{
                           and }V(x)=1\} \,\cup\\[1mm]
                            && \{ (x,\ifalseof{v}) \mid x \in \bar x \text{
                           and }V(x)=0\} 
        \end{array}
        $$
        and analogously for $\ryeval$. Next we put
        $$
        \cxval^{\Imc_V} = \{ e \mid (e,d) \in \rxeval^{\Imc_V} \}.
        $$
        If $V$ makes the $i^{th}$ clause true, then there is a literal
        in it that is true. Choose such a literal $L$ and set
        $w_i = \itrueof{v}$ if $L=\neg v$ an $w_i = \ifalseof{v}$ if
        $L=v$. If $V$ makes the $i^{th}$ clause false, then set
        $w_i=f$. We proceed with the definition of $\Imc_V$:
        $$
        \begin{array}{rcl}
          \rceval^{\Imc_V} &=& \{ (c_i,w_i) \mid 1 \leq i \leq \ell \}
          \\[1mm]
          \reval^{\Imc_V} &=& \rxeval^{\Imc_V} \cup \ryeval^{\Imc_V} \cup
                          \rceval^{\Imc_V} \\[1mm]
          \rnegeval^{\Imc_V} &=& \{(a,b) \mid \rnegeval(a,b) \in \Amc \}
          \\[1mm]
          \cfval^{\Imc_V} &=& \{ f,c_1 \}  \\[1mm]
          \rfeval^{\Imc_V} &=& \{ (f,c_10) \} \\[1mm]
          \cgoal^{\Imc_V} &=& \{ c_1 \}.
        \end{array}
        $$
        The choice of $c_1$ is somewhat arbitrary in the last three
        statements, any other individual without an incoming
        $\rceval$-edge would also do.
        It is straightforward to verify that $\Imc_V$ is a model of
        \Kmc with $\Imc_V \not\models q(f)$.

        We can define a variation $\Imc'_V$ of $\Imc_V$ that is still
        a model of~\Kmc provided that all clauses are satisfied by $V$
        (which guarantees that $f$ has no incoming $\rceval$-edge), and
        that satisfies $\Imc'_V \models q(f)$. To achieve this, we
        replace the last three lines from the definition of $\Imc_V$
        with
        $$
        \begin{array}{rcl}
          \cfval^{\Imc_V} &=& \{ f \}  \\[1mm]
          \rfeval^{\Imc_V} &=& \{ (f,f) \} \\[1mm]
          \cgoal^{\Imc_V} &=& \{ f \}.
        \end{array}
        $$

        \medskip
        Now for the actual proof of the claim

        \smallskip
        
	``$\Rightarrow$''.  Assume that
        $\forall \bar x \exists \bar y \, \varphi$ is false, and thus
        $\exists \bar x \forall \bar y \, \lnot \varphi$ is true and
        there is a valuation $V_{\bar x}: \bar x \rightarrow \{0,1\}$
        such that $\forall \bar y \, \lnot \varphi'$ holds where
        $\varphi'$ is obtained from $\varphi$ by replacing every
        variable $x \in \bar x$ with the truth constant
        $V_{\bar x}(x)$. Consider any extension $V$ of $V_{\bar x}$ to
        the variables in $\bar y$ and the interpretation
        $\Imc_V$. Since $\Imc_V$ is a model of \Kmc with
        $\Imc_V \not\models q(f)$, to show that
        \mbox{$\circkb \not\models q(f)$} it remains to prove
        that there is no model $\Jmc <_\CP \I$ of \Kmc.

        Assume to the contrary that there is such a $\Jmc$. We first
        observe that $A^\Jmc=A^{\Imc_V}$ for all
        $A \in \{ \cvval, \cxval, \ccval \}$. This holds for $\cvval$
        since $\cvval^\Jmc \subsetneq \cvval^{\Imc_V}$ means that \Jmc
        does not satisfy \Amc, and likewise for $\ccval$. It also
        holds for $\cxval$ since every model of \Kmc must make
        $\cxval$ true at at least one of $\itrueof{x}$ and
        $\ifalseof{x}$ for all $x \in \bar x$. This is because of
        CIs~(\ref{eq:xvariables_have_value}),
        (\ref{eq:xvalues_are_vvalues}), (\ref{eq:xeval_is_eval}), and
        (\ref{eq:implxval}), the role
        disjointness~(\ref{eq:eval_is_not_negeval}), the
        assertions~(\ref{eq:varblock}) and~(\ref{eq:varblocktwo}), and
        the fact that minimization of $\cvval$ is preferred over
        minimization of $\cxval$.  But no subset of $\cxval^{\Imc_V}$
        satisfies this condition, so we must have  $\cxval^{\Imc_V}
        \subseteq \cxval^{\Jmc}$,
        consequently $\cxval^{\Jmc}=\cxval^{\Imc_V}$.

The above and $\Jmc <_\CP \I$ means that $\cfval^\Jmc \subsetneq
\cfval^{\Imc_V}$. Since \Jmc is a model of \Kmc, this implies $\cfval^\Jmc
= \{ f \}$.

Since $\cxval^\Jmc = \cxval^{\Imc_V}$ and \Jmc is a model of \Kmc,
$V(x)=1$ implies $(x,\itrueof{x}) \in \reval^\Jmc$ and $V(x)=0$
implies $(x,\ifalseof{x}) \in \reval^\Jmc$. Moreover, since
$\cvval^\Jmc = \cvval^{\Imc_V}$ and \Jmc is a model of \Kmc, for every
$y \in \bar y$ we have $(y,\itrueof{y}) \in \reval^\Jmc$ or
$(y,\ifalseof{y}) \in \reval^\Jmc$.  We thus find an extension $V'$ of
$V_{\bar x}$ to $\bar x \cup \bar y$ that is \emph{compatible with}
\Jmc in the sense that for all $v \in \bar x \cup \bar y$,
\begin{itemize}
\item 
$V'(v)=1$ implies $(v,\itrueof{v}) \in \reval^\Jmc$ and 
\item $V'(v)=0$ implies $(v,\ifalseof{v}) \in \reval^\Jmc$.
\end{itemize}
Choose some such $V'$. Since $\forall \bar y \, \lnot \varphi'$ holds,
at least one of the clauses must be made false by $V'$. As explained
alongside the construction of \Kmc, that clause can only have an
$\rceval$-edge to $f$. By~(\ref{eq:rcevalfcevaldisj}), the outgoing
$\rfeval$-edge from $f$ cannot end at $f$. But $\cfval^\Jmc = \{ f \}$
and thus~(\ref{eq:clauses_have_value'}) implies that $f$ is the only
point where that edge may end.  We have arrived at a contradiction.

\smallskip ``$\Leftarrow$''.  Assume that
$\forall \bar x \exists \bar y \, \varphi(\bar x, \bar y)$ is true and
let \Imc be a model of $\circkb$. We have to show that
$\Imc \models q(f)$.
To this end, we first observe that
\begin{equation*}
\cvval^\Imc = \{ a \mid \cvval(a) \in \Amc \}.
\tag{$*$}
\end{equation*}
If this is not the case, in fact, then we find a model
$\Jmc <_\CP \Imc$ of \Kmc, contradicting the minimality of \Jmc.
This model \Jmc is essentially $\Imc_V$, except that we might have to extend the domain with additional elements that do not occur in the extension of any concept or role name to make sure that the domains of \Imc and $\Imc_V$ are identical.

It follows from ($*$) and the fact that \Imc is a model of \Kmc that
for every $x \in \bar x$, we have $(x,\itrueof{x}) \in \reval^\Jmc$ or
$(x,\ifalseof{x}) \in \reval^\Jmc$.  Consequently, we find a valuation
$V_{\bar x}$ for $\bar x$ that is compatible with \Imc as defined in
the ``$\Rightarrow$'' direction of the proof. Since
$\forall \bar x \exists \bar y \, \varphi(\bar x, \bar y)$ is true, we
can extend $V_{\bar x}$ to a valuation $V$ for $\bar x \cup \bar y$
that satisfies $\varphi$. Consider the model $\Imc'_V$, extended to
domain $\Delta^\Imc$. By construction of $\Imc'_V$ and ($*$), we have
$\cvval^\Imc = \cvval^{\Imc'_V}$.
Using the facts that \Imc is a minimal model of \Kmc and the construction of $\Imc'_V$, it can now be observed that $\cxval^\Imc=\cxval^{\Imc'_V}$.
Moreover, by construction, $\Imc'_V$ interprets $\ccval$ and $\cfval$ in the minimal possible way among all models of \Kmc.
Since \Imc is a minimal
model, this implies $\ccval^\Imc=\ccval^{\Imc'_V}$ and
$\cfval^\Imc=\cfval^{\Imc'_V}$. In particular,
$\cfval^{\Imc}= \{ f \}$. But this implies $(f,f) \in \rfeval^\Imc$ as
all objects with an incoming $\rfeval$-edge must satisfy $\cfval$.  It
follows that $f \in \cgoal^\Imc$ and thus $\Imc \models q(f)$, as
desired.
\end{proof}

\end{document}